\documentclass[12pt]{article}

\usepackage[bigfiles]{pdfbase}
\ExplSyntaxOn
\NewDocumentCommand\embedvideo{smm}{
  \group_begin:
  \leavevmode
  \tl_if_exist:cTF{file_\file_mdfive_hash:n{#3}}{
    \tl_set_eq:Nc\video{file_\file_mdfive_hash:n{#3}}
  }{
    \IfFileExists{#3}{}{\GenericError{}{File~`#3'~not~found}{}{}}
    \pbs_pdfobj:nnn{}{fstream}{{}{#3}}
    \pbs_pdfobj:nnn{}{dict}{
      /Type/Filespec/F~(#3)/UF~(#3)
      /EF~<</F~\pbs_pdflastobj:>>
    }
    \tl_set:Nx\video{\pbs_pdflastobj:}
    \tl_gset_eq:cN{file_\file_mdfive_hash:n{#3}}\video
  }
  \pbs_pdfobj:nnn{}{dict}{
    /Type/RichMediaInstance/Subtype/Video
    /Asset~\video
    /Params~<</FlashVars (
      source=#3&
      skin=SkinOverAllNoFullNoCaption.swf&
      skinAutoHide=true&
      skinBackgroundColor=0x5F5F5F&
      skinBackgroundAlpha=0.75
    )>>
  }
  \pbs_pdfobj:nnn{}{dict}{
    /Type/RichMediaConfiguration/Subtype/Video
    /Instances~[\pbs_pdflastobj:]
  }
  \pbs_pdfobj:nnn{}{dict}{
    /Type/RichMediaContent
    /Assets~<<
      /Names~[(#3)~\video]
    >>
    /Configurations~[\pbs_pdflastobj:]
  }
  \tl_set:Nx\rmcontent{\pbs_pdflastobj:}
  \pbs_pdfobj:nnn{}{dict}{
    /Activation~<<
      /Condition/\IfBooleanTF{#1}{PV}{XA}
      /Presentation~<</Style/Embedded>>
    >>
    /Deactivation~<</Condition/PI>>
  }
  \hbox_set:Nn\l_tmpa_box{#2}
  \tl_set:Nx\l_box_wd_tl{\dim_use:N\box_wd:N\l_tmpa_box}
  \tl_set:Nx\l_box_ht_tl{\dim_use:N\box_ht:N\l_tmpa_box}
  \tl_set:Nx\l_box_dp_tl{\dim_use:N\box_dp:N\l_tmpa_box}
  \pbs_pdfxform:nnnnn{1}{1}{}{}{\l_tmpa_box}
  \pbs_pdfannot:nnnn{\l_box_wd_tl}{\l_box_ht_tl}{\l_box_dp_tl}{
    /Subtype/RichMedia
    /BS~<</W~0/S/S>>
    /Contents~(embedded~video~file:#3)
    /NM~(rma:#3)
    /AP~<</N~\pbs_pdflastxform:>>
    /RichMediaSettings~\pbs_pdflastobj:
    /RichMediaContent~\rmcontent
  }
  \phantom{#2}
  \group_end:
}
\ExplSyntaxOff

\usepackage{geometry}
\usepackage[utf8]{inputenc}
\usepackage{subfigure}
\usepackage{epsfig}
\usepackage{wrapfig}
\usepackage{tensor}
\usepackage{amssymb}
\usepackage{mathtools} 
\usepackage{blkarray, bigstrut}
\usepackage{xparse}
\usepackage{comment}
\usepackage{rotating}
\usepackage{bm}
\usepackage{fancyhdr}
\usepackage{algorithm}
\usepackage{algpseudocode}
\usepackage{times}
\usepackage{graphicx}
\usepackage{hyperref}
\usepackage{amsmath}
\usepackage{xcolor}
\usepackage[font=footnotesize,labelfont=bf]{caption}
\usepackage{floatpag}
\floatpagestyle{empty}

\newcommand{\mynameis}[1]{%
  \phantomsection#1%
  \renewcommand{\@currentlabel}{#1}%
  \renewcommand{\@currentlabelname}{#1}}

\makeatletter
\newcommand{\manuallabel}[2]{\def\@currentlabel{#2}\label{#1}}
\newcommand{\customlabel}[2]{%
   \protected@write \@auxout {}{\string \newlabel {#1}{{#2}{\thepage}{#2}{#1}{}} }%
   \hypertarget{#1}{}
}
\makeatother

\graphicspath{{Fig/},{Fig/S1},{Fig/S2},{Fig/S3},{Fig/S4},{Fig/S5},{Fig/S6},{Fig/S7},{Fig/S8}}

\title{Self-organizing Nervous Systems \\ for Robot Swarms} 

\author
{W. Zhu$^{1,\ast}$, S. Oguz$^{1,3,\ast}$, M.K. Heinrich$^{1,\dag,\ast}$, M. Allwright$^{1}$, \\
M. Wahby$^{1}$, A. Lyhne Christensen$^{2}$, E. Garone$^{3}$, M. Dorigo$^{1,\dag}$\\
\\
\normalsize{$^{1}$IRIDIA, Universit\'{e} Libre de Bruxelles, Brussels, Belgium}\\
\normalsize{$^{2}$SDU UAS Center, MMMI, University of Southern Denmark, Odense, Denmark}\\
\normalsize{$^{3}$SAAS, Universit\'{e} Libre de Bruxelles, Brussels, Belgium}\\
\\
\normalsize{$^\dag$To whom correspondence should be addressed; E-mails:}\\ \normalsize{mary.katherine.heinrich@ulb.be, mdorigo@ulb.ac.be.}\\
\normalsize{$^\ast$These authors contributed equally to this work and share first authorship.}
}

\date{}

\begin{document} 
\baselineskip16pt
\maketitle 

\begin{abstract}
The system architecture controlling a group of robots is generally set before deployment and can be either centralized or decentralized.
This dichotomy is highly constraining, because decentralized systems 
are typically fully self-organized and therefore difficult to design analytically, whereas centralized systems
have single points of failure and limited scalability.
To address this dichotomy, we present the Self-organizing Nervous System (SoNS), a novel robot swarm architecture based on self-organized hierarchy.
The SoNS approach enables robots to autonomously establish, maintain, and reconfigure dynamic
multi-level system architectures. 
For example, a robot swarm consisting of $n$ independent robots could transform into a single $n$-robot SoNS and then into several independent smaller SoNSs, where each SoNS uses a temporary and dynamic hierarchy.
Leveraging the SoNS approach, we show that sensing, actuation, and decision-making can be coordinated in a locally centralized way, without sacrificing the benefits of scalability, flexibility, and fault tolerance, for which swarm robotics is usually studied.
In several proof-of-concept robot missions---including binary decision-making and search-and-rescue---we demonstrate that the capabilities of the SoNS approach greatly advance the state of the art in swarm robotics.
The missions are conducted with a real heterogeneous aerial-ground robot swarm, using a custom-developed quadrotor platform.
We also demonstrate the scalability of the SoNS approach in swarms of up to 250 robots in a physics-based simulator, and demonstrate several types of system fault tolerance in simulation and reality.
\end{abstract}

\section{Introduction}

In the last two decades, swarm robotics research has demonstrated that it is possible to coordinate a large group of autonomous robots without any central coordinating entity. 
Elegant collective solutions have been developed for a broad scope of tasks, such as decision making~\cite{valentini2016collective}, navigation and transport~\cite{nouyan2009teamwork,dorigo2013swarmanoid}, construction~\cite{werfel2014designing}, and bio-hybrid interaction~\cite{wahby2018autonomously,halloy2007social}. 
By using strictly self-organized control within flat (single-level and fully decentralized) system architectures, swarm robotics behaviors have leveraged redundancy and parallelism to consistently achieve the hallmark advantages of a robot swarm---scalability, flexibility, the absence of single points of failure, and some degree of inherent fault tolerance. 
These characteristics are prohibitively difficult to obtain in fully centralized systems.

Despite the significant progress and the advantages of self-organized flat systems, the swarm behaviors being developed in abstract laboratory experiments are persistently slow to be adopted in real applications~\cite{dorigo2020reflections,dorigo2021swarm}. 
This slow adoption rate can be attributed to the fact that, although there are significant advantages to self-organized control using flat system architectures, there are also inherent limitations.
One crucial limitation is due to swarm behaviors occurring at the macroscale, but arising from self-organization among robots programmed at the microscale. Swarm behaviors are therefore difficult or impossible to design analytically, and fully self-organized swarms can take an undesirably long time to complete a task or converge on a decision.
Even experienced researchers in the field often conduct a long trial-and-error design process in order to develop an incrementally novel behavior, and after several distinct behaviors have been developed, it is also not trivial to combine them. 
Furthermore, if an environment occupied by a swarm changes, especially to conditions that were not explicitly forecast, it can be difficult to predict how this will influence the swarm's collective behavior.
As a consequence of design difficulty and slow development processes, there are a great many tasks that we currently do not know how to perform using strictly self-organized control in flat systems.

The key limitations seen in flat self-organization are not prevalent in centralized systems: it is much more straightforward to design and combine centralized behaviors than self-organized ones. Using centralized systems, we already know how to execute many sophisticated behaviors that we currently do not know how to accomplish with many robots in a strictly self-organized way, such as SLAM~\cite{howard2006multi} or optimization of online task scheduling or vehicle routing~\cite{psaraftis2016dynamic}.
However, bottlenecks and single points of failure are unavoidable in strictly centralized systems, bringing inherent scalability and fault tolerance limitations that are typically absent in self-organized behaviors. 
In this paper, we propose that key impediments to rapid progress in swarm robotics can be overcome by partially integrating centralized control into an otherwise self-organized system through the introduction of a dynamic multi-level architecture---in other words, a self-organized hierarchy.

Self-organized hierarchy has been identified as a key unsolved challenge for the future of swarm robotics~\cite{dorigo2020reflections,dorigo2021swarm}. 
Hierarchy can offer swarm robotics easier and faster behavior design and management as well as more flexibility when combining behaviors. However, not just any hierarchy is suitable. To still behave like a swarm and get the oft-cited benefits of scalability, flexibility, and fault tolerance for which swarm robotics is generally studied, the hierarchy cannot be imposed from the outside and must also be controllable. The members must autonomously establish an ad-hoc dynamic hierarchy among themselves and be able to comprehensively manage it in a self-organized way.

To undertake this challenge, we present Self-organizing Nervous Systems (SoNS) for robot swarms. SoNS is a novel robot swarm architecture that uses self-organized hierarchy to allow dynamically determined ``brain" robots to coordinate sensing, actuation, and decision making in temporarily centralized sub-swarms, without sacrificing the scalability, flexibility, and fault tolerance of self-organization.
The SoNS approach allows a completely decentralized system in which subsets of robots can self-organize into temporary locally centralized dynamic control structures when needed. 

\subsection{Self-organizing Nervous System (SoNS) concept} 

In the SoNS concept, robots autonomously organize themselves into dynamic multi-level system architectures using ad-hoc remote bidirectional connections. 
The result is a swarm composed of a number of reconfigurable Self-organizing Nervous Systems---SoNSs, see Fig.~\ref{fig:concept} and Movie~\ref{Movie1}. In each SoNS instance, each robot has chosen to temporarily grant explicit supervisory powers to a robot in the level above it, culminating in a single ``brain" robot that acts impermanently as a coordinating entity. 

The process by which robots establish and maintain explicit bidirectional connections in SoNSs is entirely self-organized, using exclusively local communication. 
Therefore, the basic structure of a SoNS includes not only the topology of the remote connections, but also the relative robot positions required to maintain those connections using local communication and sensing.

\begin{figure}[thbp]
    \centering
    \includegraphics[width=0.99\textwidth]{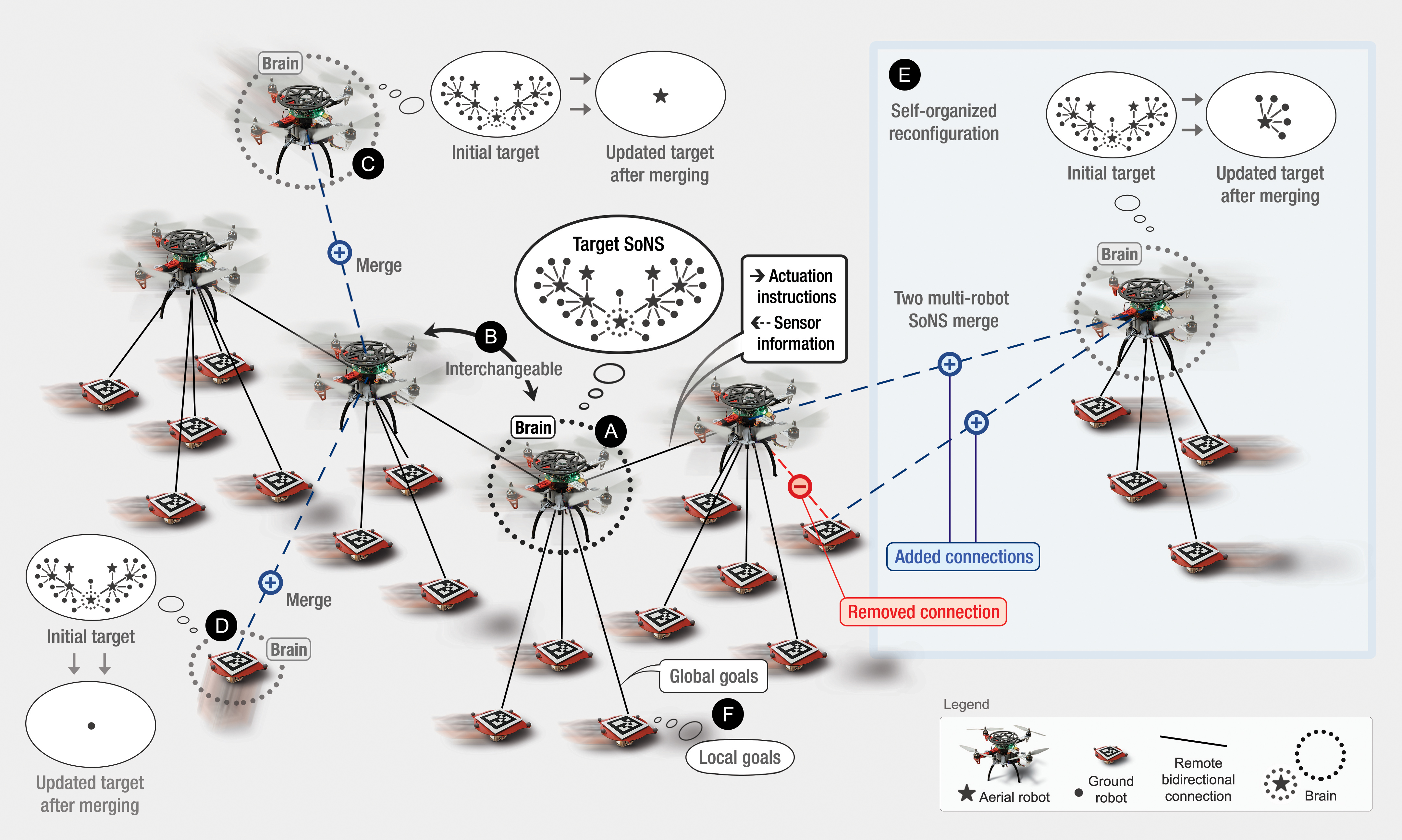}
    \caption{\textbf{The Self-organizing Nervous System (SoNS) concept: robots self-organize dynamic multi-level system architectures using exclusively local communication.} \textbf{(A)} In a SoNS, each robot chooses to temporarily grant explicit supervisory powers to a robot in the level above it, culminating in a single ``brain" robot that acts impermanently as a SoNS-wide coordinating entity. The brain robot does not communicate with all members directly---rather, at each bidirectional connection, the robot at a lower level in the architecture sends sensor information upstream, and the robot at a higher level sends actuation instructions downstream. The brain tries to establish its target SoNS and manages the global goals of the SoNS. \textbf{(B)} Any robot at any level of hierarchy can be interchanged with another robot---even the brain. \textbf{(C,D)} Each robot initializes as the brain of its own single-robot SoNS, with a map of its target SoNS and, potentially, mission-dependent global goals. If it encounters another SoNS, it can choose to accept recruitment and merge with it, thereby abdicating its ``brain" status. \textbf{(E)} The process by which robots establish and maintain connections is entirely self-organized. Therefore, the topology can be reconfigured on demand by removing and adding connections, and multi-robot SoNS can split and merge as needed. Here, we see a five-robot SoNS (right) that merges with a 20-robot SoNS (left). \textbf{(F)} The topology of bidirectional connections is used to grant supervisory powers and send global goals downstream, but robots can negotiate the inter-level control distributions on the fly, adapting the degree of centralization or decentralization in the decision-making processes of the SoNS.}
    \label{fig:concept}
\end{figure}

\vspace{3mm}
\noindent
\embedvideo*{\includegraphics[page=1, width=\linewidth]{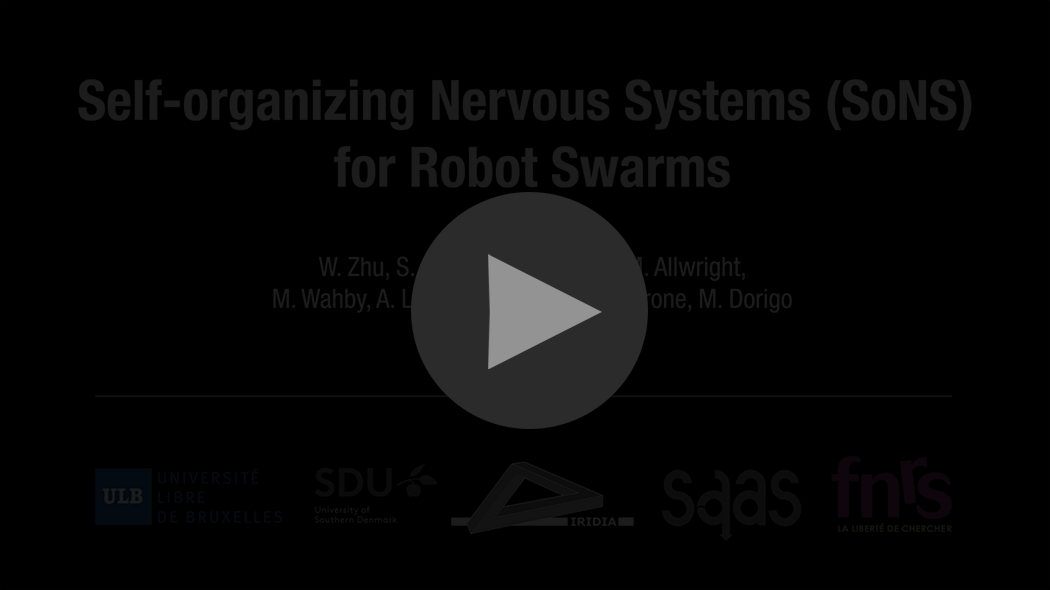}}{Movies/Movie_1.mp4}
\customlabel{Movie1}{1}
\begin{center}
    \vspace{-6mm}
    {\small \bf Movie \ref*{Movie1} : Explanation video of the Self-organizing Nervous System (SoNS) concept.}\\ \vspace{2mm}
    {\small \it (Note: To play embedded video, open the PDF in Adobe Acrobat.)}
\end{center}
\vspace{3mm}

Each connection, including its relative robot positions, is managed independently by the two robots sharing the connection. 
Therefore, the self-organized aspects of the SoNS architecture are scalable. For instance, a new connection does not become any more difficult to establish if a SoNS recruits more members.
Also, every robot is replaceable---even the brain. 
If any robot fails or is lost, it can be replaced automatically by another robot in its SoNS or by a new robot recruited from the surroundings. If no extra robots are available, a SoNS can reconfigure utilizing the robots it already has.

In SoNSs, unlike in centralized systems with root nodes, a brain robot does not communicate with all its members directly. 
Rather, at each bidirectional connection, the robot at a lower level in the architecture sends sensor information upstream, and the robot at a higher level sends actuation instructions downstream.
Members of one SoNS can thereby act seamlessly as a single virtual body, despite only communicating with their direct neighbors.

In a SoNS, the inter-level control distribution of the system---that is, the degree of centralization or decentralization of decision-making---is dynamic and thus not statically determined by the topology of the bidirectional connections. The topology is indeed used to grant supervisory powers and send global goals downstream. However, the robots can continuously negotiate the inter-level control distributions on the fly, adapting the degree of centralization or decentralization of decision-making to the requirements of the task being performed. For instance, to balance global and local goals, a robot might receive downstream instructions to move in a certain direction, but might temporarily negotiate a different inter-level control distribution while reacting locally to a small obstacle.

The system architecture and behavior of SoNSs are also dynamic and can be fully reorganized by the brain on demand (see Fig.~\ref{fig:concept} and Movie~\ref{Movie1}). This can include changes to the topology of connections, the relative positions used to maintain those connections, the global actuation goals such as a target motion trajectory, and even the number of separate SoNSs acting autonomously in the same environment. Indeed, the self-organization process by which SoNSs are built and managed also allows them to split and merge as needed. For instance, according to the number and sizes of SoNSs suitable for a certain task, the robots in a shared environment could start as many single-robot SoNSs that interact in a fully decentralized way, then merge into one shared multi-level SoNS to interact hierarchically, then split into several multi-robot SoNSs that can interact with each other as independent systems. In short, the SoNS approach is highly flexible and allows robot systems to fully self-reconfigure their architectures on demand.

\subsection{Related work}

In existing approaches, the overall architecture of a multi-robot system is normally set before deployment. In other words, the communication structure (the organization of communication links and the system levels they can span), control distribution (such as fully decentralized or fully centralized), and behavior structure of the system (for example, which sources of sensor information can influence which actions) are predefined, and the robots coordinate within this static architecture.
Traditionally, swarm robotics has used self-organized control with strict heterarchy (a network or system of unranked elements) which is seen, for example, in aerial robot flocking~\cite{vasarhelyi2018optimized}. This trend is not surprising, as the swarm robotics field was originally biologically inspired~\cite{dorigo2020reflections} and many biological systems with unranked members, seen in social organisms~\cite{buhl2006disorder,detrain2008collective,theraulaz1998origin}, have inspired artificial swarm intelligence~\cite{bonabeau1999swarm}. 
However, despite the development of many advanced behaviors (for example,~\cite{vasarhelyi2018optimized}), heterarchical swarm robotics approaches have rarely been used in real-world applications. To enable wider real-world use, the field is expected to develop more elaborate behavior design and coordination approaches~\cite{dorigo2020reflections,dorigo2021swarm}, such as automatic design~\cite{francesca2016automatic}, behavioral heterogeneity~\cite{Kengyel-etal2015}, and self-organizing hierarchy.

Many existing swarm robotics approaches have incorporated mechanisms relevant to hierarchy or leadership. For instance, behavioral heterogeneity can result in the implicit leadership of some individuals that are more informed than most of their peers~\cite{ferrante2012self,firat2020self}, send more messages when having higher information quality~\cite{valentini2015efficient,valentini2016collective}, have greater behavioral persistence~\cite{balazs2020adaptive}, or are partially remotely controlled by human operators~\cite{walker2014human,walker2014control}. However, this behavioral heterogeneity in swarms has normally been implemented with unranked members, i.e., in a decentralized and single-level system. By contrast, double-level systems have been used when robot swarms have incorporated an explicit leader, e.g., for disseminating information to other robots via one-hop broadcast~\cite{kaiser2022innate} or multi-hop broadcast over an ad-hoc structure~\cite{shan2020collective}. However, the explicit leadership allocation in these examples is static after deployment and is usually defined manually. 

Robot systems with explicit leadership and multiple leader-follower pairs have also been explored, especially for the task of flocking. However, these systems either use predefined and static leaders~\cite{gu2009leader,amraii2014explicit,zheng2020adversarial} or have unstructured relationships between pairs. In the latter type, interactions are one-way and therefore can be interpreted as a hierarchy of many temporary leader-follower pairs, for instance in flocking~\cite{dalmao2011cucker,pignotti2018flocking,jia2019modelling} or self-assembly~\cite{soorati2019photomorphogenesis}, but any ranking among these pairs is emergent, not explicitly controlled.
Therefore, these systems have the advantage of intrinsic flexibility---robots can be displaced unexpectedly and can immediately form new connections and be accommodated in their new positions---but have no mechanism for configuring or reconfiguring the organization of the hierarchy. 
Similarly, many sophisticated multi-drone navigation approaches that are partly centralized and partly decentralized have been developed, notably the recent trajectory planner for drone swarms by Zhou et al.~\cite{zhou2022swarm}. These approaches provide high-performing solutions for collective navigation. However, they use fixed collaboration structures, because configuring and reconfiguring swarm architectures is not the focus of these studies.

In short, no existing approach has provided a comprehensive way to self-organize dynamic and highly reconfigurable multi-level system architectures in a robot swarm. The SoNS approach is designed to address this gap. 

For our development of SoNS, we have taken inspiration from our `mergeable nervous systems' (MNS)~\cite{mathews2017mergeable}, an approach for physically-connected robots. We have incorporated some of its high-level ideas into our SoNS concept and robot swarm architecture and reference this inspiration by retaining the term `nervous systems.' 
We have also conducted some preliminary simulation-only studies on extending MNS ideas to other contexts~\cite{zhu2020formation,zhang2023self,jamshidpey2020multi,JamWahHeiAllZhuDor2021:techreport-008,jamshidpey2022reducing}. In this paper, we present for the first time our novel SoNS concept and robot swarm architecture and provide a thorough SoNS proof-of-concept using real robots.

\subsection{Novel features of SoNS}

The primary contribution of our SoNS robot swarm architecture is that a robot system can integrate the manageability advantages of centralized systems without sacrificing the scalability, flexibility, and fault-tolerance advantages of self-organized systems.
This contribution is founded on four novel features for robot swarms (see Fig.~\ref{fig:concept}): self-organized controllable hierarchy, interchangeable leadership (i.e, interchangeability of the brain), explicit inter-system reconfiguration, and reconfigurable swarm behavior structures. 
Together, these features enable robot swarms to self-organize dynamic multi-level system architectures, including their communication structures, control distributions, and system behaviors.

{\bf Self-organized controllable hierarchy.}
The SoNS approach allows a robot swarm to self-organize a dynamic hierarchical communication structure that (i) is built and maintained using exclusively local communication, (ii) is not imposed from the outside, and (iii) is comprehensively controllable (that is, the SoNS-wide multi-level structure can move from any initial state to any desired state in its configuration space of directed acyclic graphs). In other words, a SoNS hierarchy can be explicitly defined and redefined by the brain and the desired changes occur through robots configuring and reconfiguring their local connections in a self-organized way. Self-organized controllable hierarchy has been shown in physically-connected robots but is novel in robot swarms, which so far have shown only emergent hierarchy (for example,~\cite{ferrante2012self,firat2020self}), not controllable hierarchy. The remote connections of robot swarms bring significantly different requirements than physical connections. The physical locations and topology constrain each other less under remote connections than under physical connections, which provides much more flexibility in how a SoNS can be organized, but also adds the challenge that physical location and topology must both be actively (and sometimes separately) maintained.

{\bf Interchangeable leadership.}
In a SoNS, all robots occupy an explicitly defined position in a hierarchy, but any robot, at any level of hierarchy, can be interchanged autonomously and on demand. This interchange is self-organized by the robots using only local communication. This means that, if the brain fails, the SoNS self-organizes to automatically and immediately substitute it with the nearest robot, which continues to specify the same SoNS structure and mission goals as the previous leader. This is a novel feature for robot swarms with explicit leaders, which so far have used static (and sometimes manually defined) leadership. It is also a departure from many other types of multi-robot systems, which often use indiscriminate followers or groups of followers, but have not yet developed approaches in which the position of every robot in an explicit control hierarchy is defined by a self-organized process using strictly local communication. 

{\bf Explicit inter-system reconfiguration.}
The SoNS approach allows reconfiguration between multiple SoNS---that is, several SoNSs can split and merge themselves in a self-organized way that is coordinated by the brains of the SoNSs and uses exclusively local communication, without losing the existing sub-structures that could be retained. For example, several independent SoNSs could agree to merge simultaneously, and the robots would reorganize themselves around the new shared brain, retaining sub-sections of the old structures when possible. Robot systems in the literature have shown splitting and merging of sub-swarms or sub-teams, but with the non-leader members being unranked (for example,~\cite{ducatelle2009new,ducatelle2011self}), so no reconfiguration of explicit sub-system architectures was demonstrated.

{\bf Reconfigurable swarm behavior structures.}
The inter-level control distribution and system behaviors within a SoNS (for example, the global structure defining which information sources influence which actions) can be negotiated and explicitly reconfigured without breaking the system architecture.
Reconfiguration can occur (i) locally and temporarily to balance conflicting global and local goals; (ii) SoNS-wide for the purpose of global sensing, actuation, and decision-making goals set by the brain; but also (iii) locally for internal re-organization of a SoNS (for example, robots automatically redistributing themselves to compensate for a failed robot).
Based on this capacity for internal re-organization, if the needed changes to behavior are too substantial to be managed by inter-level negotiation, an entirely new SoNS architecture can be initialized by the brain and self-organized by the robots.
No existing work has presented a robot swarm architecture with these explicit reconfiguration capabilities.

\section{Results}

We demonstrate the capabilities of the SoNS approach (see Movie~\ref{Movie2}) in experiments using real heterogeneous aerial-ground swarms consisting of standard differential drive e-puck robots~\cite{mondada2009puck,millard2017pi} and our custom-developed S-drone quadrotors~\cite{OguHeiAllZhuWahGarDor2022:techreport-010}. 
To also demonstrate the capabilities in swarms larger than our real robot arena allows, we run experiments in a simulator that is cross-verified against the behavior of the real robots.
Within each experiment, all robots run an identical SoNS program and operate fully autonomously---without any global positioning system, remote control station, or off-board sensing. 
The robots use vision-based relative positioning and are allowed to communicate wirelessly only if one robot is in the other's field of view.
Actuation in the experiments is confined to motion. 
The open-source software used in all experiments (both with real robots and in simulation) and all experiment data are available in online repositories.
In addition to the experiments, we also provide theoretical analyses to show that the convergence and stability of the SoNS architecture are guaranteed.

\subsection{Robot missions }

We conduct five proof-of-concept robot missions that together demonstrate the key capabilities and novel features of the SoNS approach. 

The first novel feature, self-organized controllable hierarchy, is shown in all missions: the first mission (Sec.~\ref{sec:establishment-mission}) shows the process of establishing the self-organized hierarchy, the second mission (Sec.~\ref{sec:obstacles-mission}) shows that the self-organized hierarchy can be maintained despite external disturbances, and the reconfigurations in the rest of the missions show that the self-organized hierarchy is comprehensively controllable. 

The novel features of interchangeable leadership and explicit inter-system reconfiguration are both demonstrated in the mission on splitting and merging of systems (Sec.~\ref{sec:splitmerge-mission}), as robots reconfigure into different sets of SoNSs and also reconfigure their leadership allocations during the splits and merges. (Note that interchangeable leadership is also demonstrated in the later fault-tolerance experiments, see Sec.~\ref{sec:faulttolerance}). 

The last novel feature of reconfigurable swarm behavior structures is shown in all four missions after the first (Secs.~\ref{sec:obstacles-mission}-\ref{sec:splitmerge-mission}): reconfiguration is shown (i) locally and temporarily in the mission on balancing global and local goals (Sec.~\ref{sec:obstacles-mission}); (ii) SoNS-wide in the missions on global sensing and actuation and binary decision-making (Secs.~\ref{sec:sensing-mission} and \ref{sec:decision-mission}); and (iii) locally for internal re-organization during the last three missions (Secs.~\ref{sec:sensing-mission}-\ref{sec:splitmerge-mission}).

Together, the missions demonstrate the ability to self-organize dynamic multi-level system architectures, including their communication structures, inter-level control distributions, and system behaviors.

\vspace{3mm}
\noindent
\embedvideo*{\includegraphics[page=1, width=\linewidth]{Movies/poster-play.png}}{Movies/Movie_2.mp4}
\begin{center}
    \customlabel{Movie2}{2}
    {\small \bf Movie \ref*{Movie2} : Summary video of key results.}\\ 
    \vspace{2mm}
    {\small \it (Note: To play embedded video, open the PDF in Adobe Acrobat.)}
\end{center}
\vspace{3mm}

For each proof-of-concept mission, we report at least five trials with real robots (up to 12 robots) and 50 trials in simulation (with up to 75 robots), with a maximum run time of 15~minutes (constrained by the battery capacity of the quadrotor platform). 
The goals and scope of possible behaviors for each mission are designed offline.
We give the mission schematics (see Figs.~\ref{fig:establishment-mission}-\ref{fig:splitmerge-mission}), show that the qualitative goals of the mission were achieved, and assess the results in terms of actuation error, where the actuation error is the position tracking error $E$ (Eq.~\ref{eq:error}, Sec.~\ref{sec:methods:analysis-metrics}, Materials and Methods) with respect to the lower bound $B$ (Eq.~\ref{eq:lower_bound}, Sec.~\ref{sec:methods:analysis-metrics}, Materials and Methods). The actuation error is used here as a comprehensive metric because it encapsulates the other types of error that can occur (except for incomplete mission goals). 
Any errors in sensing and decision-making will cause large fluctuations in the actuation instructions sent in the system, and therefore cause significant increases in the actuation error.

\setcounter{figure}{1} 
\begin{figure}
    \centering
    \includegraphics[width=0.99\textwidth]{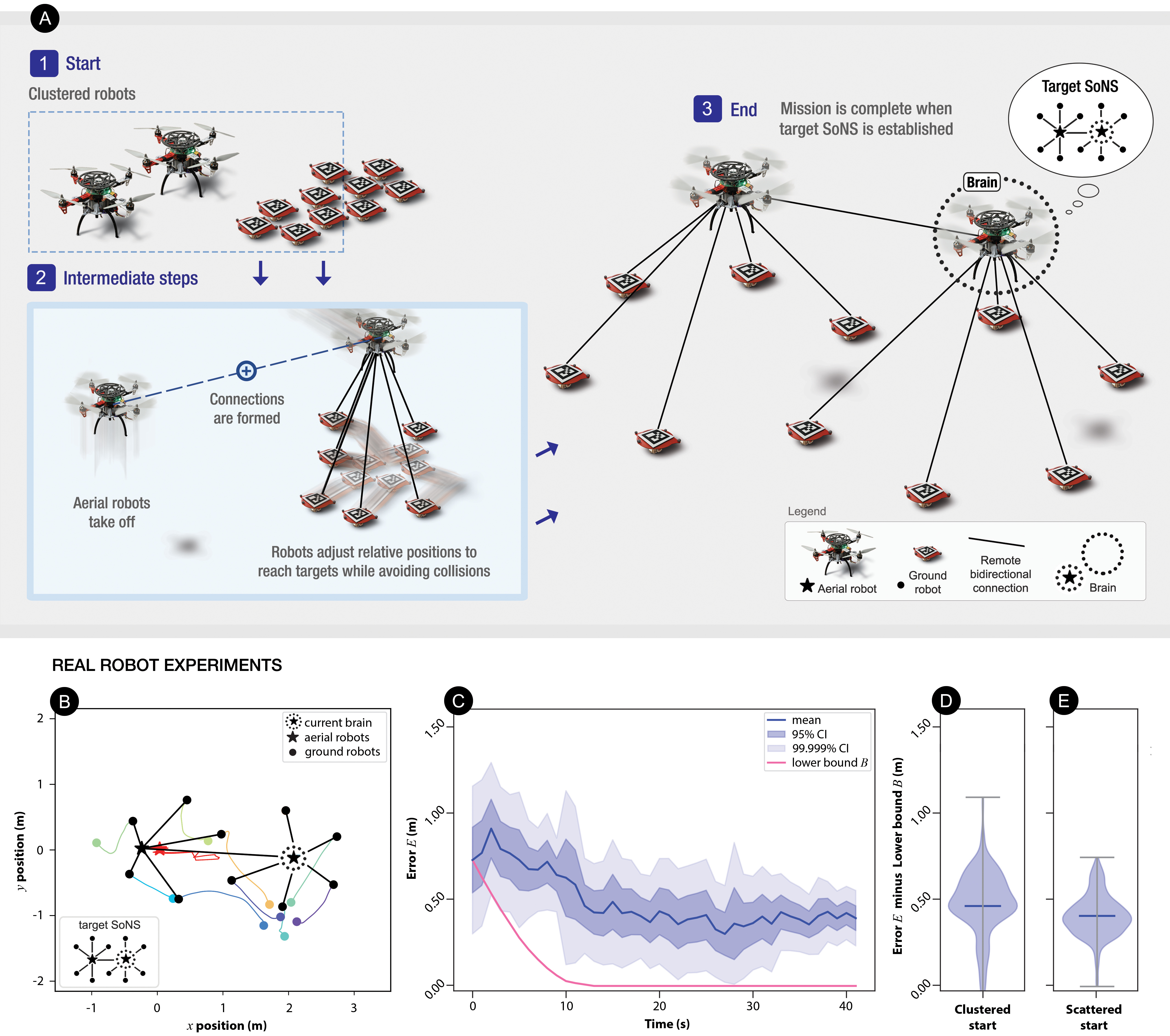}
    \caption{\textbf{Establishing self-organized hierarchy.} In this experiment, independent robots in arbitrary starting positions need to self-organize into a single hierarchical system with the target multi-level communication structure and target relative positions---i.e., into the target SoNS. \textbf{(A)} Mission schematic: (1) Starting on the ground, some robots are positioned in tight clusters, which increases the possibility for interference between robots. (In a second variant, not pictured here, all robots start in scattered positions throughout the arena.) (2) To self-organize into the target SoNS, first the aerial robots take off and all robots begin searching for peers. Then, robots start forming connections, merging their respective SoNSs, and reallocating themselves into positions that match the target SoNS. During this process, robots continually adjust their relative positions while coordinating locally to avoid collisions, until (3) the target SoNS is complete. \textbf{(B-E)} Results of real robot experiments. \textbf{(B)} Trajectories of robots over time, with the initial and final positions indicated in color and in black, respectively, and \textbf{(C)} 
    actuation error $E$ (mean and confidence interval per robot over time, see Eq.~\ref{eq:error}, with lower bound $B$, see Eq.~\ref{eq:lower_bound}, plotted for reference) in a real example trial (shown in Movie~S1 in the supplementary materials). \textbf{(D,E)} 
    Violin plots of the actuation error minus the lower bound $E - B$ 
    (mean and confidence interval per robot per second) in all real experiment trials, for both mission variants: \textbf{(D)} clustered start, six trials, and \textbf{(E)} scattered start, five trials. In this and following figures, the violin plots show that the given example trial is not an outlier and is indicative of the overall results of the respective mission. (For more detailed experiment results, including simulation experiments, see Sec.~\ref{SM:results} of the supplementary materials.)}
    \label{fig:establishment-mission}
\end{figure}

\subsubsection{Establishing self-organized hierarchy}
\label{sec:establishment-mission}

The first proof-of-concept mission demonstrates the establishment of SoNSs (see Fig.~\ref{fig:establishment-mission}): independent robots in arbitrary start positions need to self-organize into a single hierarchical system with the target multi-level communication structure, including the correct topology and correct relative positions. In the two mission variants, robots start on the ground either (i) with some of the robots clustered tightly (and therefore resulting in robot--robot interference) in varying arbitrary locations, or (ii) with all robots scattered in varying arbitrary positions throughout the arena. All robots run identical programs and begin as independent single-robot SoNSs, of which they are the brain by default. SoNSs try to recruit each other and change their relative positions until the mission goal is complete. In order to complete these missions in the tightly constrained arena, the SoNSs must perform many inter-system merge operations simultaneously, while also continuously re-organizing their internal structures by redistributing themselves (see example flowcharts in Fig.~\ref{fig:establishment-mission}A).

In all experiments, the robots complete the mission: they converge on one SoNS and establish the correct topology and relative positions for the target communication topology (see example trial in Fig.~\ref{fig:establishment-mission}B and Movie~S1 in the supplementary materials). The progression of a typical experiment can be seen in the example trial (see Fig.~\ref{fig:establishment-mission}C), in which the robots converge to the correct topology and reach a low steady-state actuation error.
All trials reach a low steady-state error, with the clustered start variant displaying slightly higher average error than the scattered start variant (see Figs.~\ref{fig:establishment-mission}D,E). 

\subsubsection{Balancing global and local goals}
\label{sec:obstacles-mission}
 
\begin{figure}
    \centering
    \vspace{-30mm}
    \includegraphics[width=0.99\textwidth]{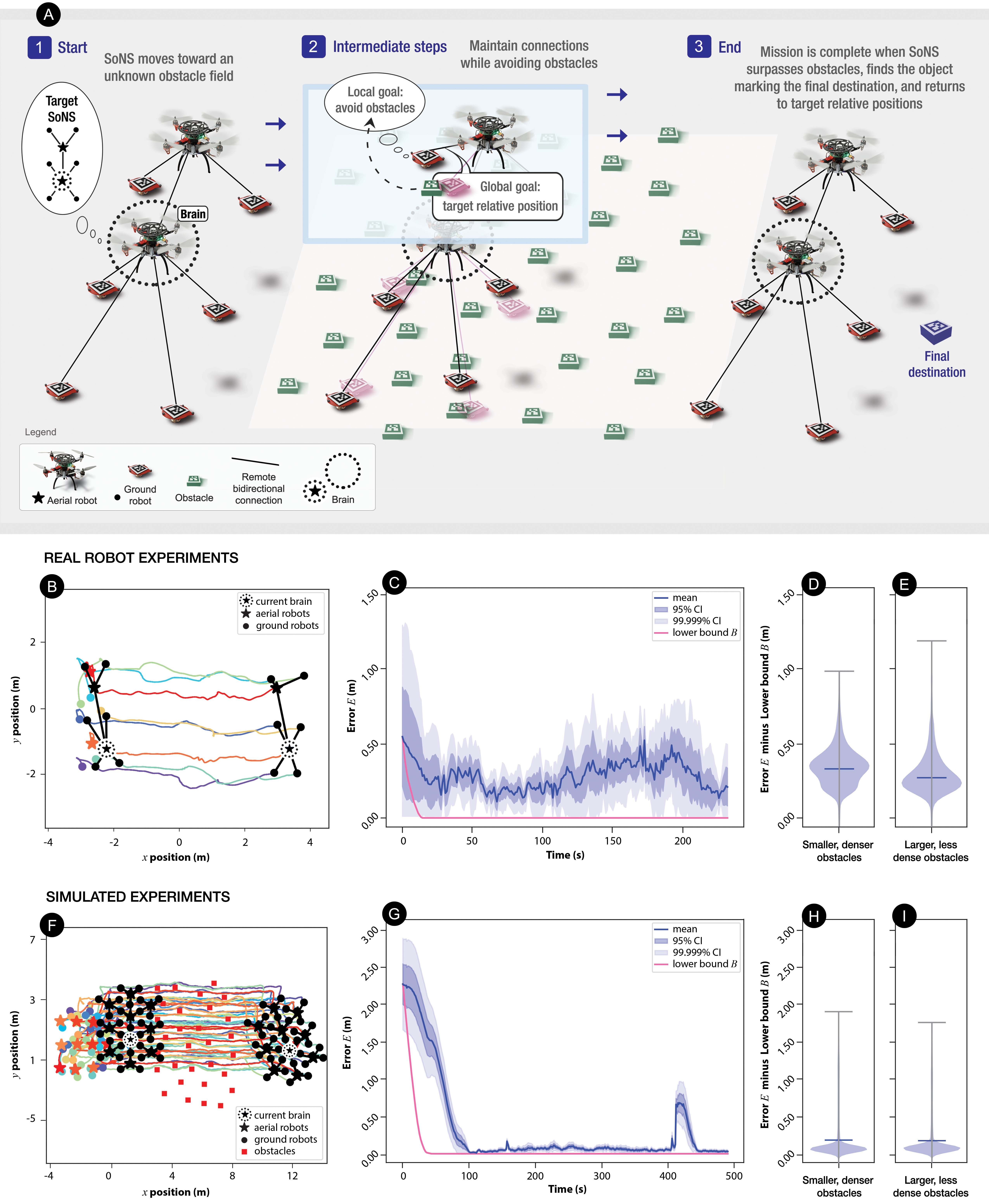}
    \caption{\textbf{Balancing global and local goals.} In this experiment, robots need to navigate a field of obstacles scattered in one portion of the arena, while maintaining the bidirectional connections of the target SoNS. \textbf{(A)} Mission schematic: (1) Robots begin as members of a single SoNS and begin moving across an environment with an unknown field of small, dense obstacles, searching for an object that marks the final destination. (In a second variant, not pictured here, the obstacles are larger and less dense.) (2) Robots move through the obstacle field, collaboratively balancing global and local goals at each bidirectional link, to avoid obstacles while still keeping the SoNS together, until (3) the SoNS surpasses the obstacle field and senses the final destination object.
    \textbf{(B-E)} Results of the real experiments. \textbf{(B)} Trajectories of robots over time, with the initial and final SoNS indicated in black (on the left and on the right, respectively) and \textbf{(C)} actuation error $E$ (mean and confidence interval per robot over time, see Eq.~\ref{eq:error}, with lower bound $B$, see Eq.~\ref{eq:lower_bound}, plotted for reference) in a real example trial (shown in Movie~S2 in the supplementary materials). \textbf{(D,E)} 
    Violin plots of the actuation error minus the lower bound $E - B$ (mean and confidence interval per robot per second) in all real experiment trials, for both mission variants: \textbf{(D)} smaller and denser obstacles, five trials, and \textbf{(E)} larger and less dense obstacles, five trials. \textbf{(F-I)} Results of the simulated experiments, given in the same format: \textbf{(F)} trajectories and \textbf{(G)} actuation error $E$ in a simulated example trial, with \textbf{(H,I)} violin plots of the actuation error minus the lower bound $E - B$ in all simulated trials, 50 trials per variant. (For more detailed experiment results, see Sec.~\ref{SM:results} of the supplementary materials.)}
    \label{fig:obstacles-mission}
\end{figure}

The next proof-of-concept demonstration is an obstacle-avoidance mission (see Fig.~\ref{fig:obstacles-mission}), in which robots in a SoNS must negotiate the inter-level control distribution on the fly (i.e., adapting the degree of centralization or decentralization in the decision-making of the SoNS). The SoNS needs to balance the global goal of an overall motion trajectory chosen by the brain with the local goals of ground robots circumventing small obstacles. Robots begin as members of a single SoNS and the brain begins with a straight trajectory in a given direction. The robots navigate through an unknown field of obstacles while the brain follows its straight trajectory, until reaching an object that marks the final destination (at an unknown position). The goal of these missions is for the robots firstly to maintain the target topology of the communication structure without any connection breaks, despite the disturbances from the environment caused by the presence of obstacles, and secondly to recover the target relative positions once the obstacle field has been surpassed. In this mission, the SoNS maintains a consistent communication structure and overall system behavior throughout, and allows the inter-level control distribution along each bidirectional connection to be adjusted as needed.

The obstacles are scattered in one portion of the arena and their positions and types are not known by any of the robots beforehand. In the two mission variants, either (i) obstacles are larger than the ground robots, so that circumventing the obstacles is challenging, or (ii) obstacles are roughly the same size as the ground robots and are positioned more densely, so that navigating through the gaps is challenging (this variant is shown in Fig.~\ref{fig:obstacles-mission}A).

In all trials, the robots successfully complete the mission goals and the actuation error returns to its pre-obstacles level (see final SoNS positions in Figs.~\ref{fig:obstacles-mission}B,F; see final error levels after 230\,s in trials with real robots in Fig.~\ref{fig:obstacles-mission}C and after 450\,s in simulation in Fig.~\ref{fig:obstacles-mission}G). 
The progression of a typical real experiment can be seen in the example trial in Fig.~\ref{fig:obstacles-mission}C: the actuation error drops as the robots reach their initial target relative positions (approx.\,0 to 30\,s), then rises when the SoNS begins to encounter obstacles (at approx.\,30\,s), remains unsteady and continues rising as the SoNS passes through the obstacle field (approx.\,30 to 190\,s), and starts to decline when the ground robots start to surpass the last obstacles (at approx.\,190\,s); then, once all robots have exited the obstacle field, the SoNS re-converges to its target relative positions.
The progression of a typical simulated experiment can be seen in the example trial in Fig.~\ref{fig:obstacles-mission}G: the actuation error drops to almost zero as the robots reach their initial target relative positions (at approx.\,100\,s), then rises slightly and becomes slightly unsteady as the SoNS passes through the obstacle field (approx.\,150 to 400\,s); after the SoNS exits the obstacle field, the robots at the front start to sense the final destination object and the SoNS adjusts its target trajectory to move towards it, causing the error to increase (approx.\,410 to 440\,s) until the adjustment is complete and the SoNS returns to negligible steady-state error (at approx.\,440\,s).
All trials reach a low steady-state error, in both reality and simulation, with the smaller, denser obstacles variant displaying slightly higher average error than the larger, less dense variant (see Figs.~\ref{fig:establishment-mission}D,E,H,I). 

\subsubsection{Collective sensing and actuation}
\label{sec:sensing-mission}
\begin{figure}
    \centering
    \vspace{-25mm}
    \includegraphics[width=0.99\textwidth]{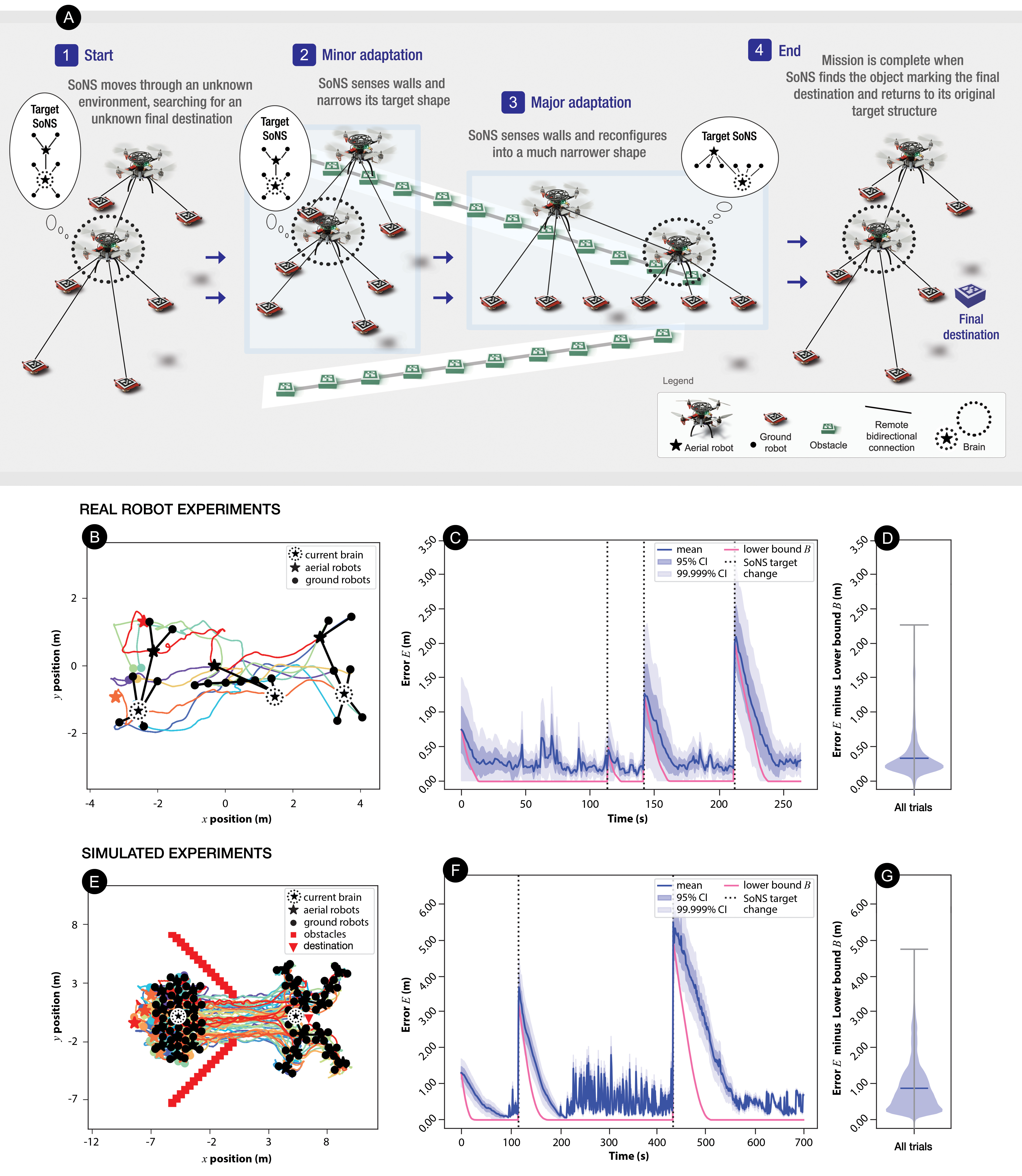}
    \caption{\textbf{Collective sensing and actuation.} In this experiment, robots need to sweep the environment while collectively reacting to the width of the passage. \textbf{(A)} Mission schematic: (1) Robots that have established their target SoNS begin moving across an environment, searching for an object that marks the final destination. (2) Robots sense obstacles that form the walls of a passage and start to adapt by narrowing the shape of the target SoNS. (3) Robots sense the walls narrowing and then adapt further by reconfiguring into a different target SoNS that has an even narrower shape. (4) Robots sense that there are no longer walls constraining them, return to their original target SoNS, and stop when they sense the final destination object, completing the mission. 
    \textbf{(B-D)} Results of the real experiments. \textbf{(B)} Trajectories of robots over time, with the initial, an example intermediate, and the final SoNS indicated in black (from left to right, respectively) and \textbf{(C)} actuation error $E$ (mean and confidence interval per robot over time, see Eq.~\ref{eq:error}, with lower bound $B$, see Eq.~\ref{eq:lower_bound}, plotted for reference) in a real example trial (shown in Movie~S3 in the supplementary materials). \textbf{(D)} 
    Violin plot of the actuation error minus the lower bound $E - B$ (mean and confidence interval per robot per second) in all five real experiment trials. \textbf{(E-G)} Results of the simulated experiments, given in the same format: \textbf{(E)} trajectories and \textbf{(F)} actuation error $E$ in a simulated example trial, with \textbf{(G)} a violin plot of the actuation error minus the lower bound $E - B$ in all 50 simulated trials.
    (For more detailed experiment results, see Sec.~\ref{SM:results} of the supplementary materials.)}
    \label{fig:sensing-mission}
\end{figure}

The next proof-of-concept setup is a sweeping mission (see Fig.~\ref{fig:sensing-mission}), in which a SoNS needs to navigate an environment while collectively sensing and reacting to unknown conditions. Like in the previous mission, the robots begin as members of one SoNS and the brain begins with a straight trajectory in a given direction. The goal of the mission is to move straight through a passage while keeping the SoNS shape as wide as can be fit, between walls composed of obstacles that enclose a passage of unknown width, until finding an object that marks the final destination. To accomplish this, the brain needs to fuse collective sensor information, determine the current width of the passage, and update the SoNS's target communication topology and relative positions as needed. When these updates occur, the robots in the SoNS need to collectively reorganize, maintaining portions of the old structure when possible. In this mission, the SoNS reconfigures its communication topology as needed, while maintaining a consistent overall system behavior and consistent inter-level control distributions at each connection.

The environment is set up so that the walls get narrower as the SoNS progresses (see Fig.~\ref{fig:sensing-mission}A). At first, the SoNS might be able to fit through the passage by simply narrowing the shape of its robot formation (i.e., the topology remains the same but the relative positions change). Later, the passage becomes so narrow that a full re-organization into a linear formation is required. Once the SoNS exits the narrow passage, it has the space to re-organize into a wider formation near the final destination object.

In all trials, the robots complete the mission goals. The actuation error starts low, spiking each time the brain initializes a new target structure, then declining gradually until the next re-organization (see Fig.~\ref{fig:sensing-mission}C,F). Between re-organizations, many small spikes occur, as robots at the edge of the formation collide with the walls but determine the disruption can be managed locally and adjust their positions accordingly. The largest spikes occur when both the target relative positions and the target topology change. This can be seen in the example trial with real robots (see Fig.~\ref{fig:sensing-mission}C): compare the first change (at approx.\,115\,s), which involves only positions, to the second (at approx.\,140\,s) and third (at approx.\,215\,s) target changes, which involve both topology and position. In the simulated example trial, the same progression can be observed, but with only two changes to the target SoNS (see Fig.~\ref{fig:sensing-mission}F). Finally, after the SoNS re-organizes and reaches the final destination object, it returns to low steady-state error.

\subsubsection{Binary decision making }
\label{sec:decision-mission}

\begin{figure}
    \centering
    \vspace{-30mm}
    \includegraphics[width=0.99\textwidth]{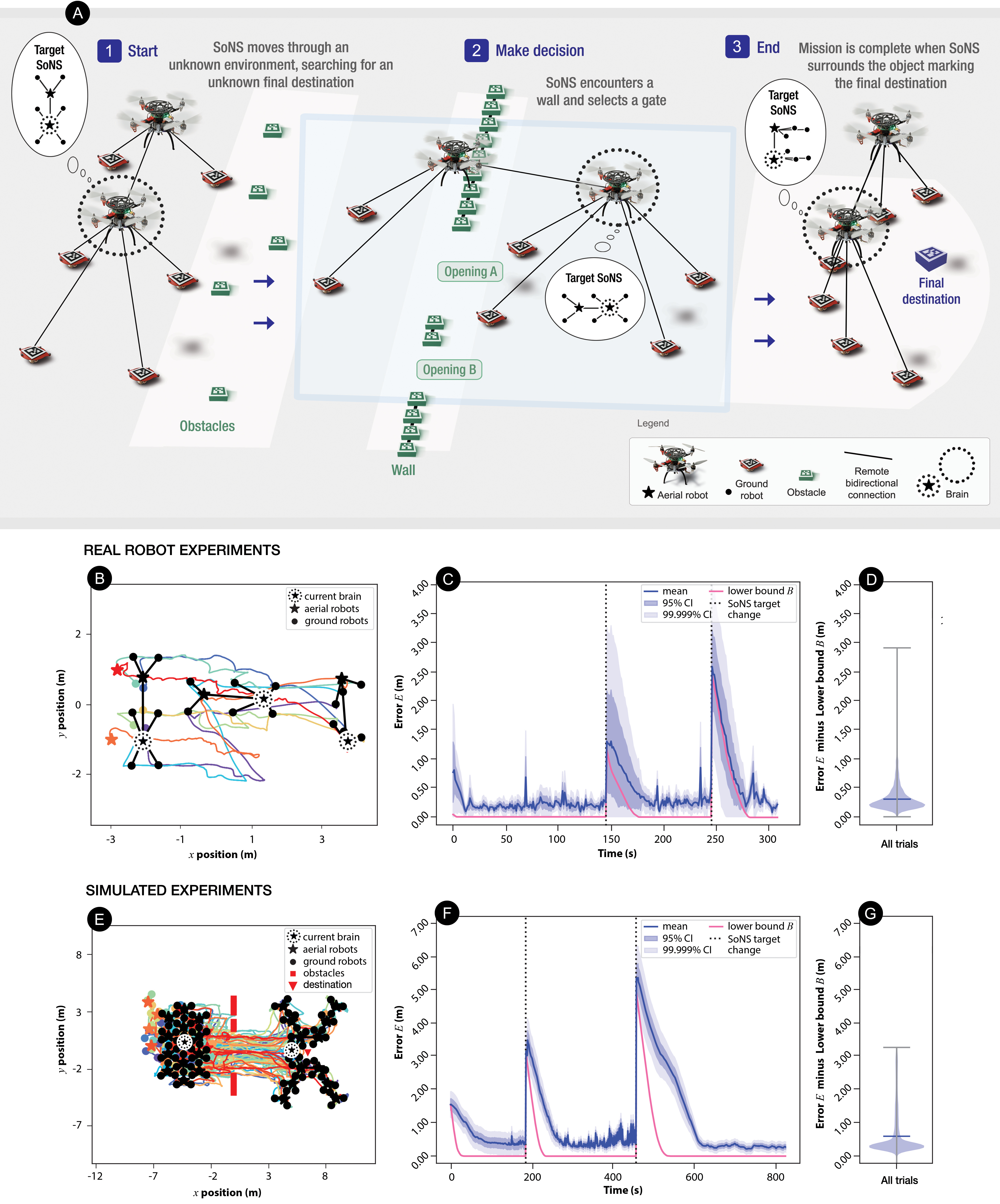}
    \caption{\textbf{Binary decision making.} In this experiment, robots need to sweep the environment to find an object that marks the final destination, making a binary choice between two possible paths. \textbf{(A)} Mission schematic: (1) Robots that have established their target SoNS begin moving across an environment, searching for the final destination object. (2) After surpassing an obstacle field, robots sense a wall, collaboratively choose the largest opening, and adjust the path of the SoNS to pass through the selected opening. (3) Robots sense that there are no longer walls constraining them, then sense the final destination object and change their target SoNS to surround it, completing the mission. 
    \textbf{(B-D)} Results of the real experiments. \textbf{(B)} Trajectories of robots over time, with the initial, an example intermediate, and the final SoNS indicated in black (from left to right, respectively) and \textbf{(C)} actuation error $E$ (mean and confidence interval per robot over time, see Eq.~\ref{eq:error}, with lower bound $B$, see Eq.~\ref{eq:lower_bound}, plotted for reference) in a real example trial (shown in Movie~S4 in the supplementary materials). \textbf{(D)} 
    Violin plot of the actuation error minus the lower bound $E - B$ (mean and confidence interval per robot per second) in all five real experiment trials. \textbf{(E-G)} Results of the simulated experiments, given in the same format: \textbf{(E)} trajectories and \textbf{(F)} actuation error $E$ in a simulated example trial, with \textbf{(G)} a violin plot of the actuation error minus the lower bound $E - B$ in all 50 simulated trials.
    (For more detailed experiment results, see Sec.~\ref{SM:results} of the supplementary materials.)}
    \label{fig:decision-mission}
\end{figure}

The next proof-of-concept setup is a reactive path planning mission (see Fig.~\ref{fig:decision-mission}), in which a SoNS must make a binary choice between possible paths and update its trajectory accordingly, while continuously reconfiguring itself to navigate an unknown environment and search for an object that marks the final destination (see Fig.~\ref{fig:decision-mission}A). 
The robots begin the mission as members of one SoNS and the brain begins with a straight trajectory in a given direction. To start, the SoNS begins sweeping the environment in as wide a formation as can fit in the environment, because the position of the final destination object is unknown. When it encounters a field of small randomly scattered obstacles, the SoNS needs to negotiate its inter-level control distribution while continuing to sweep. When it encounters a wall, the SoNS must collectively detect two openings located at varying positions in the wall and choose the one that is larger, then update its path and reorganize itself to fit through the larger opening. The decision-making process is self-organized: robots that sense openings vote on them according to size and reach a consensus, with the first robot to propose the selected opening then becoming the new brain. Once the SoNS passes through the opening, it must find the final destination object and reorganize to encircle it.
The mission goal is to reach the final destination object by choosing the correct path between binary options with different qualities, without suffering any connection breaks from environmental disturbances in the meantime.
In various portions of this mission, the SoNS adjusts its communication structure, inter-level control distributions, and overall system behavior as needed.

In all trials, the robots successfully complete the mission goals: the SoNS senses both openings and chooses the larger one, then reaches the final destination object, without any breaks in the system architecture (see Fig.~\ref{fig:decision-mission}B,E). The progression of a typical experiment can be seen in the example trials (see Fig.~\ref{fig:decision-mission}C,F): due to the complexity of the mission, the actuation error experiences many spikes. However, after these spikes, the robots converge to a topology encircling the final destination object, with low steady-state error. All trials reach a low steady-state error, with the average error with real robots being slightly lower than in simulation (see Figs.~\ref{fig:decision-mission}D,G).
 
\subsubsection{Splitting and merging systems}
\label{sec:splitmerge-mission}

\begin{figure}
    \centering
    \vspace{-20mm}
    \includegraphics[width=0.99\textwidth]{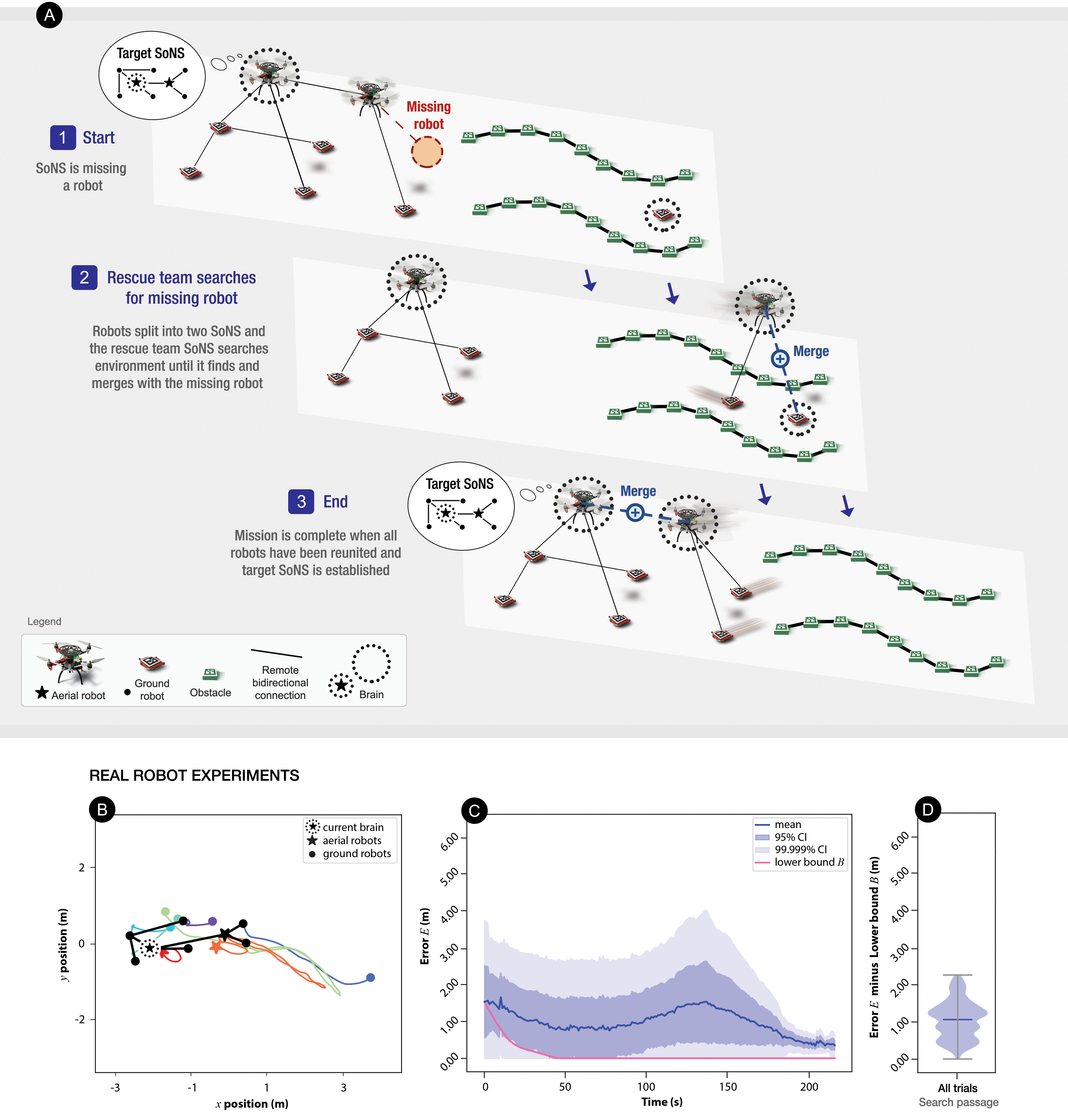}
    \caption{\textbf{Splitting and merging systems.} In this experiment, robots need to conduct multi-SoNS search-and-rescue missions to find missing robot(s) in a passage between two walls composed of obstacles and reunite into a single SoNS. (In a second variant, not pictured here, but included in Sec.~\ref{SM:results} of the supplementary materials, the robots need to collaborate to push an obstruction out of their way.) \textbf{(A)} Mission schematic: (1) When robots start, they are in a SoNS that is missing a member. (2) The brain instructs one of the robots to split from it and temporarily form its own multi-robot SoNS as a rescue team. The rescue team SoNS searches the environment until it finds the missing robot in its own single-robot SoNS and merges with it. (3) The merged rescue team SoNS then returns to the location where it initially split off and re-merges with the remaining SoNS, so that all robots are reunited.
    \textbf{(B-D)} Results of real robot experiments. \textbf{(B)} Trajectories of robots over time, with the initial and final positions indicated in color and in black, respectively, and \textbf{(C)} actuation error $E$ (mean and confidence interval per robot over time, see Eq.~\ref{eq:error}, with lower bound $B$, see Eq.~\ref{eq:lower_bound}, plotted for reference) in a real example trial (shown in Movie~S5 in the supplementary materials). \textbf{(D)} Violin plot of the actuation error minus the lower bound $E - B$ (mean and confidence interval per robot per second) in all five trials. (For more detailed experiment results, including the second variant and the simulation experiments, see Sec.~\ref{SM:results} of the supplementary materials.)}
    \label{fig:splitmerge-mission}
\end{figure}

The last proof-of-concept setup with real robots is a search-and-rescue mission (see Fig.~\ref{fig:splitmerge-mission}), in which SoNSs split or merge in order to reunite with the missing robots. At initiation, one or more single-robot SoNSs are isolated somewhere in the environment and wait there to be found by a rescue team. There is also a primary multi-robot SoNS somewhere in the environment. This SoNS notices that it is missing a robot and starts a rescue mission to find one (see Fig.~\ref{fig:splitmerge-mission}A).  
The primary SoNS does not know the direction of the missing robot(s), so the brain issues instructions to split into two independent SoNS with new mission goals. The SoNS associated with the original brain (`home' SoNS) will stay in place and wait, while the newly split `rescuer' SoNS will explore. The rescuer SoNS follows found landmarks along a non-convex barrier until it finds and merges with the isolated robot(s). It then guides each found robot out of the barriers, and returns to and merges with the home SoNS by backtracking along the landmarks. The goal of these missions is for all robots in the environment to eventually merge into a single shared SoNS. In this mission, adjusting the inter-level control distribution at each connection would not be sufficient for the task requirements, so the SoNSs completely reconfigure their communication structures and system behaviors to complete the mission.
In an alternative mission variant (see Sec.~\ref{SM:results} of the supplementary materials), the primary SoNS knows the direction of its missing robot, but needs to physically push an obstruction out of the way, then merge with the missing robot and guide it out of a convex barrier. (For results of simulation experiments, see Sec.~\ref{SM:results} of the supplementary materials.)

In all trials, the robots successfully complete the mission by splitting and then merging (see Fig.~\ref{fig:splitmerge-mission}B,D): the multi-robot SoNS(s) are successfully reunited with the isolated single-robot SoNS(s). The actuation error rises during the period of re-organizations, but all robots re-merge into one SoNS and the system returns to a low steady-state error (see Fig.~\ref{fig:splitmerge-mission}C,D).

\subsection{Scalability}
\label{sec:scalability}

\begin{figure}
    \centering
    \vspace{-15mm}
    \includegraphics[width=0.99\textwidth]{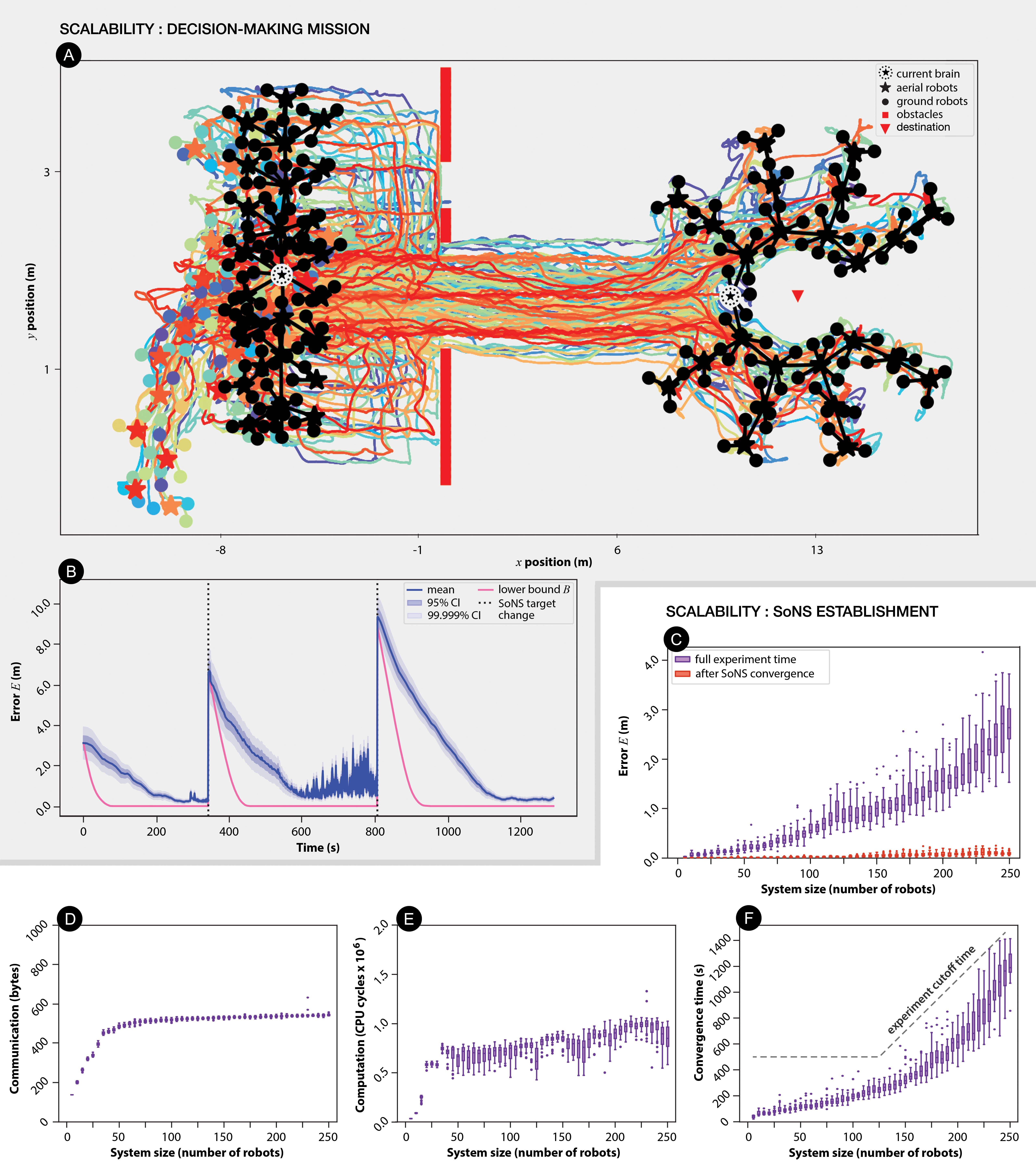}
    \caption{\textbf{Scalability study}. Results in simulation with up to 250 robots (20\% drones and 80\% ground robots). \textbf{(A-B)} Demonstration of a 125-robot SoNS completing the binary decision-making mission shown in Fig.~\ref{fig:decision-mission}. \textbf{(A)} Trajectories of robots over time (with the initial and final SoNS indicated in black, on the left and on the right, respectively) in an example trial with 125 robots (shown in Movie~S6 in the supplementary materials). \textbf{(B)} Actuation error $E$ (mean and confidence interval per robot over time, see Eq.~\ref{eq:error}, with lower bound $B$, see Eq.~\ref{eq:lower_bound}, plotted for reference) in the example trial shown in \textbf{(A)}. 
    \textbf{(C-F)} Scalability measured by different performance metrics when the number of simulated robots increases from 5 to 250 in steps of 5, with 30 trials per system size (several example trials shown in Movie~S7 in the supplementary materials). In this experiment setup, robots simply establish a SoNS, as in the mission shown in Fig.~\ref{fig:establishment-mission}. All performance metrics are calculated per robot per step. \textbf{(C)} Actuation error $E$ throughout the experiment (purple bars) and after the SoNS converges, thus reaching a steady state (red bars); \textbf{(D)} communication, measured in bytes of messages passed; \textbf{(E)} computation, measured in the maximum CPU clock cycles for any robot in the SoNS; and \textbf{(F)} convergence time. (For more detailed experiment results, see Sec.~\ref{SM:results} of the supplementary materials.)}
    \label{fig:scalability}
\end{figure}
 
We demonstrate the scalability of the SoNS architecture in swarms of up to 250 robots (200 ground robots, 50 aerial robots).
We use the binary decision-making mission setup (see Sec.~\ref{sec:decision-mission}) with four different system sizes, up to 125 robots (50 trials per system size), and the establishment mission setup (see Sec.~\ref{sec:establishment-mission}) with 50 different system sizes, up to 250 robots (30 trials per system size).
Given that the real arena size is limited, we run the scalability experiments only in simulation (for cross-verification between the simulator and reality, see Sec.~\ref{SM:cross-verify} in the supplementary materials). We also disregard the battery capacity of the quadrotor platform in these experiments.

The results of the binary decision-making experiments show that the robots complete all parts of the mission successfully (see example trial with 125 robots in Fig.~\ref{fig:scalability}A,B and Movie~S6 in the supplementary materials; for more detailed experiment results, see Sec.~\ref{SM:results} of the supplementary materials). These results provide a proof of concept that the demonstrated capabilities of the SoNS approach do not break down in larger swarms (up to 125 robots): the SoNS can balance global and local goals, collectively sense and react to an environment, make a collective binary decision, and reconfigure when needed, all without breaking the system architecture.

In the establishment mission experiments, we aim to test the scalability limits of the SoNS architecture under the current software implementation. 
We run experiments with system size $n= \{5, 10, 15, \dots, 250\}$ robots (30 trials for each system size) and a maximum experiment time of $t = 500$\,s for $n \leq 125$ robots and $t = 4n + 4(n-125)$\,s for $125 < n \leq 250$ robots (see dashed line in Fig.~\ref{fig:scalability}F).
In system sizes of $n \leq 125$ robots, all trials converge before the maximum experiment time. In systems of $125 < n < 220$ robots, one or two trials per system size do not converge before the maximum time (approx.\,5\% of trials, on average). In systems of $220 \leq n \leq 250$ robots, approx.\,20\% of trials do not converge before the maximum time. In trials that converge (shown in Fig.~\ref{fig:scalability}F), the mean convergence time rises superlinearly, with the rate of change increasing most noticeably after system sizes of 150 robots.
We therefore consider the performance of the current SoNS implementation to be fully reliable in system sizes up to 125 robots, to be somewhat stable until 220 robots, and to degrade substantially in larger systems. See Sec.~\ref{sec:limitations} for a discussion of actuation error and convergence times in swarms of 220 robots and larger.

The results of the establishment mission trials that converged before the maximum time are shown in Fig.~\ref{fig:scalability}C-F. (Several example trials are shown in Movie~S7 in the supplementary materials.) Fig.~\ref{fig:scalability}C (see purple bars) shows the mean and variance of the actuation error per robot. The error increases with the number of robots: slowly at first, and more substantially in systems larger than 100 robots. This increase in error according to system size can be mostly attributed to the rising Euclidean distance between the starting positions and the target positions in the eventual SoNS---i.e., the greater the number of robots, the higher the lower bound $B$ of the error (see Eq.~\ref{eq:lower_bound}). Indeed, when we measure only the actuation error after the SoNS has finished converging, and therefore the starting positions are no longer relevant (see red bars in Fig.~\ref{fig:scalability}C), we see that the steady-state error increases only slightly with system size and always remains low (less than $E = 0.5$\,m).

We measure the communication load in terms of mean bytes per robot (inclusive of inbound and outbound messages) and the computational work in terms of maximum CPU clock cycles for any robot in the SoNS, which would be the same on the real robots as in the simulation. (Our simulations run the same software modules as those run on the real robots, see Sec.~\ref{SM:cross-verify} of the supplementary materials.) 
The results show that, in system sizes of 50 robots or more, the communication load nearly plateaus (see Fig.~\ref{fig:scalability}D) and the mean CPU cycles rise only slightly, with a moderate increase in variability (see Fig.~\ref{fig:scalability}E). 
The rise in communication load that occurs in small swarm sizes (see Fig.~\ref{fig:scalability}D) is in part a function of the maximum communication range and maximum robot density (due to minimum safety distances). For small swarm sizes, the average number of robots each robot has in its field of view (including both connected robots and unconnected robots that might be candidates for recruitment operations) grows with swarm size until the robot density reaches its upper limit, after which, the average number of robots in each robot's field of view plateaus (at approx.\,50 robots). 
The similar rise in computational work in small swarm sizes (see Fig.~\ref{fig:scalability}E) is in part a function of the maximum number of connections per robot in the target SoNS's communication topology. For small swarm sizes, the average realized connections per robot rises until the swarm is large enough that most robots reach the maximum number of connections per robot allowed in the current topology design (max.\,8 per robot), after which, the average realized connections per robot plateaus (at approx.\,35 robots).
We conclude that the fusion and compression techniques used for sensing information and actuation instructions in this implementation of the SoNS concept are sufficient for the robots and missions at hand, and that the SoNS architecture can be considered scalable in terms of overall communication and calculation loads.

\subsection{Fault tolerance }
\label{sec:faulttolerance}

\begin{figure}
    \centering
    \vspace{-30mm}
    \includegraphics[width=0.99\textwidth]{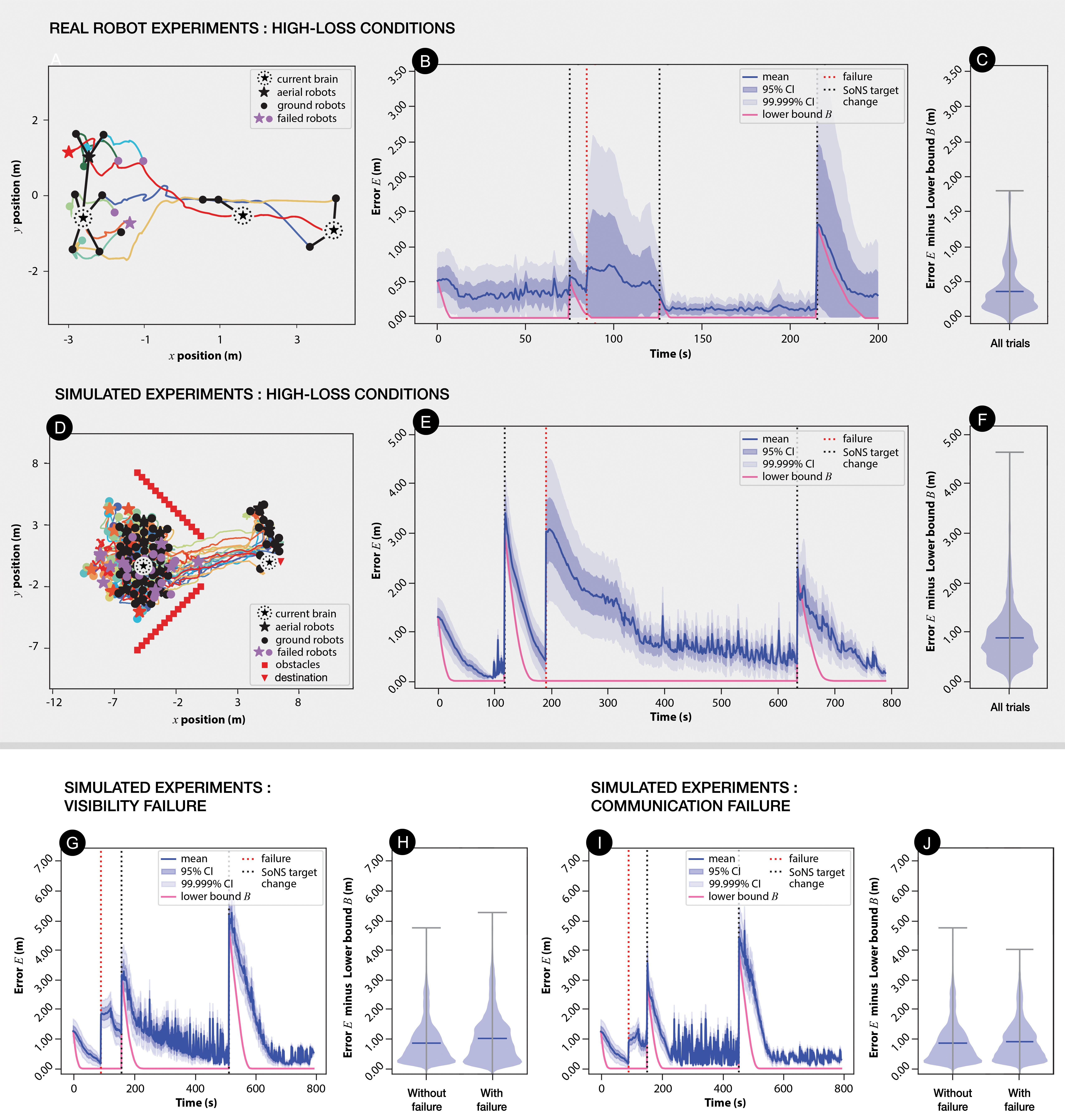}
    \caption{\textbf{Fault tolerance study.} Results with real robots and in simulation, testing both permanent robot failures in part of the SoNS (up to two-thirds of the robots) and temporary system-wide failures of communication or vision. 
    \textbf{(A-C)} High-loss conditions with real robots: arbitrary permanent failures of multiple robots. \textbf{(A)} Trajectories of all robots over time, with the initial, an example intermediate, and the final SoNS indicated in black (from left to right, respectively), and the failed robots indicated in purple at their shutdown positions, and \textbf{(B)} actuation error $E$ (mean and confidence interval per robot over time, see Eq.~\ref{eq:error}, with lower bound $B$, see Eq.~\ref{eq:lower_bound}, plotted for reference) in an example trial (shown in Movie~S9 in the supplementary materials); \textbf{(C)} violin plot of the actuation error minus the lower bound $E - B$ (mean and confidence interval per robot per second) in all five real robot experiment trials.
    \textbf{(D-F)} High-loss conditions in simulation, with either $33.\overline{3}$\% or $66.\overline{6}$\% probability for each robot to fail. \textbf{(D)} Trajectories of robots over time, with the initial and final SoNS indicated in black (on the left and the right, respectively), and the failed robots indicated in purple at their shutdown positions, and \textbf{(E)} actuation error $E$ in an example trial with $66.\overline{6}$\% probability for each robot to fail (shown in Movie~S10 in the supplementary materials); \textbf{(F)} violin plot of the actuation error minus the lower bound $E - B$ in all high-loss experiment trials (including both $33.\overline{3}$\% and $66.\overline{6}$\% failure probabilities, 50 trials per probability). 
    \textbf{(G-H)} System-wide failure of vision in simulation: \textbf{(G)} actuation error $E$ in an example trial with 30\,s failure (shown in Movie~S11 in the supplementary materials) and \textbf{(H)} violin plots comparing the actuation error minus the lower bound $E - B$ in all 50 trials with 30\,s vision failure to the 50 trials without failure. 
    \textbf{(I-J)} System-wide failure of communication in simulation: \textbf{(I)} actuation error $E$ in an example trial with 30\,s failure (shown in Movie~S12 in the supplementary materials) and \textbf{(J)} violin plots comparing the actuation error minus the lower bound $E - B$ in all 50 trials with 30\,s communication failure to the 50 trials without failure. (For more detailed experiment results, see Sec.~\ref{SM:results} of the supplementary materials.)}
    \label{fig:faulttolerance}
\end{figure}

We demonstrate several aspects of fault tolerance in SoNS, in real and simulated swarms. First, using real robots, we show replacement of a single robot that has failed permanently, including a failed brain (see Movie~S8 in the supplementary materials). In these demonstrations, one robot is remotely triggered to fail (for the aerial robots, this includes immediately landing in place). Then, a new robot of the same hardware type is manually placed in the arena and switched on, after which it is recruited by the SoNS. 
When a brain robot or a robot at an inner hierarchy level fails, it is immediately and automatically replaced by another robot already in the SoNS, and the SoNS reorganizes around the change. Then, when a new robot is recruited, it fills the leftover vacancy in the re-organized SoNS.

Using real robots, we also demonstrate SoNSs reorganizing in high-loss conditions, i.e., after arbitrary permanent failure(s) when the failed robot(s) cannot be replaced.
The setup used is that of the collective sensing and actuation mission (see Sec.~\ref{sec:sensing-mission}), with five trials run.
The robots start the mission as one SoNS and, after failure occurs, the SoNS continues the mission with the robots available. 
The full set of results (see Sec.~\ref{SM:results} in the supplementary materials) shows that when ground robots fail, the rest of the robots are able to stay connected in one SoNS and complete the mission, whereas when an aerial robot fails, some of the ground robots downstream from it will be disconnected from the primary SoNS, but the remaining connected robots are able to continue the mission. For example, in the trial shown in Figs.~\ref{fig:faulttolerance}A,B, one of the aerial robots fails (see purple star in A) at approx.\,85\,s (see red dotted line in B). The only aerial robot that remains functional recruits the ground robots that remain functional and reachable, and continues with the mission, eventually reaching the object marking the final destination.
In all trials in which at least one robot of each type remains functional, the SoNS is able to re-organize itself with the remaining robots and continue with the mission, resulting in a relatively low overall error rate for all trials (see Fig.~\ref{fig:faulttolerance}C).

In simulation, we run the same setup with a larger swarm of 65 robots to test task performance under high-loss conditions. In this setup, each robot has probability $p$ to fail regardless of hardware type. We test two variants by setting $p = 0.\overline{3}$ or $p=0.\overline{6}$, with 50 trials per variant.
In these simulated experiments, the brain can be one of the robots failing.
For example, in the trial in Fig.~\ref{fig:faulttolerance}D-F, after two-thirds of all robots fail (including the brain), the SoNS continues with the mission, eventually reaching the final destination object. In all trials, the SoNS is able to re-organize itself with available robots (i.e., robots that have not failed and are not stuck in place), returning to a low steady-state error (see Fig.~\ref{fig:faulttolerance}E).

Also in simulation in a swarm of 65 robots, we test two types of temporary system-wide failures that would be likely to occur in practice: vision failure (for example, because of an obstruction in the environment) and wireless communication failure. 
We test visibility and wireless communication failures with durations of 0.5\,s, 1.0\,s, and 30\,s, with 50 trials per duration for each failure type.
The results (see Fig.~\ref{fig:faulttolerance}G-J) show that, in all cases, the SoNS is able to re-establish itself after the system-wide failure, converging on the correct communication topology and relative positions. The results also show that in all trials the SoNS is able to continue with the mission, leaving behind less than 5\% of robots (i.e., three robots or less) in any trial, with the actuation error increasing temporarily after failure occurs (increasing by less than $E = 0.7$\,m after the red dotted line in Fig.~\ref{fig:faulttolerance}G,I) and then returning to a lower steady state.

\subsection{Convergence and stability analysis}

The convergence of the position errors and closed-loop stability of position tracking, with respect to the control inputs, are guaranteed in SoNS for both moving and motionless cases under ideal relative distance sensing. For all simulated or real setups studied in this paper, a SoNS of $n$ robots using a baseline control law will stably track the target relative positions. The large-scale swarms in our scalability experiments are shown theoretically to converge and be stable. A full description of the analysis results is provided in Sec.~\ref{SM:theory} of the supplementary materials.

\section{Discussion }
\label{sec:Discussion}

The presented results demonstrate that the SoNS approach greatly expands the state of the art in swarm robotics.
Its four novel features allow robot swarms, for the first time, to complete centrally-defined mission goals using a self-reconfigurable system architecture.

Firstly, the self-organized controllable hierarchy of the SoNS approach, which is novel for robot swarms, is demonstrated in all missions through the maintenance and reconfiguration of remote communication structures. The topology and relative positions are both controllable (sometimes separately), without any external intervention. 
For example, in the sweeping mission showing collective sensing and actuation (Sec.~\ref{sec:sensing-mission}), the SoNS adapts to moderate changes in the environment by reconfiguring relative positions separately from topology, then adapts to greater changes by reconfiguring them both.
The second novel feature of the SoNS approach---its self-organization of explicit and interchangeable leadership roles---enables robot systems to establish and reorganize themselves flexibly, without having to reinitialize the whole architecture if, for example, a brain loses connection or some of the robots at an inner hierarchy level experience environmental disturbances. This feature can be seen in the fault tolerance experiments (Sec.~\ref{sec:faulttolerance}), in which SoNSs are able to automatically reorganize after a brain is lost or a majority of robots fail, continuing their mission without having to backtrack on past progress or synchronize with an external reference.
Thirdly, inter-system reconfiguration is demonstrated during all SoNS establishment operations and splitting and merging operations, such as in the search-and-rescue mission (Sec.~\ref{sec:splitmerge-mission}). As another example, during establishment operations in the scalability experiments (Sec.~\ref{sec:scalability}), many SoNSs of a few robots are first formed, then these SoNSs begin to merge with each other simultaneously, redistributing themselves retaining existing sub-structures when possible.
The fourth novel feature of the SoNS approach is the ability to reconfigure the system's internal behavior structures, for instance by renegotiating the inter-level control distribution, as seen in the balancing of global and local goals during the obstacle-avoidance mission (Sec.~\ref{sec:obstacles-mission}). As another example, internal management of the SoNS-wide behaviors is demonstrated in the path-planning mission with binary decision-making (Sec.~\ref{sec:decision-mission}), during which a SoNS must multi-task, combining behaviors in an ad-hoc way. This on-demand multi-tasking might be straightforward in single robots or centralized systems, but is a novel advancement for robot swarms.

In short, we have shown the SoNS system architecture to provide a reconfigurable explicit control hierarchy for robot swarms, which is in itself novel. Beyond that, because the SoNS approach allows the system architecture to be multi-level and dynamic while still being self-organized, it enables qualitatively new behaviors and behavior combinations in robot swarms. The collective sensing and actuation (Sec.~\ref{sec:sensing-mission}), binary decision-making (Sec.~\ref{sec:decision-mission}), and splitting and merging missions (Sec.~\ref{sec:splitmerge-mission}) reflect the relative ease and quickness of behavior design and management as well as flexibility in behavior combinations. Using the SoNS approach, we have shown that robot swarms can handle increased mission complexity, self-reconfiguring their communication, control, and behavior structures on demand.

We have also shown that the SoNS approach retains the fault tolerance and scalability advantages for which robot swarms are often studied. 
Any robot, including the brain, is shown to be immediately and automatically replaceable in both reality and simulation, and SoNSs are able to recover and continue missions in high-loss conditions, as well as recover from temporary system-wide failures of detection or communication. In these fault tolerance experiments (Sec.~\ref{sec:faulttolerance}), we specifically tested a SoNS's ability to recover without adapting its overall mission, to assess baseline performance. If robots had instead been allowed to search for each other after experiencing disconnections (as shown in the search-and-rescue mission), then more robots could have been retained.
Although the robot group sizes in the real missions are constrained by arena size, all missions were also successfully completed in swarms of up to 250 robots in simulation, using a simulator whose reliability has been verified against the behavior of the real robots (see Sec.~\ref{SM:cross-verify} of the supplementary materials). Furthermore, in swarms of up to 250 robots, the simulation studies (Sec.~\ref{sec:scalability}) confirmed that the computation and communication metrics scale linearly or plateau, demonstrating the absence of problematic bottlenecks.
Overall, the results show that, using the SoNS, sensing, actuation, and decision-making can be coordinated SoNS-wide, without sacrificing scalability, flexibility, and fault tolerance.

\subsection{Limitations and future work }
\label{sec:limitations}

Although the actuation error in a SoNS remains relatively small in systems of less than 150 robots, we can see that it becomes much larger as the SoNS scales, especially in SoNSs of 220 robots and larger. This error is not only caused by the SoNS architecture: some of the error is unavoidably caused by using a purely reactive control law to manage relative positions when following a leader. In the current setup, robots calculate their new velocities after the velocities of the upstream robots have already been updated. Using this strategy, error can be guaranteed to be bounded (see Sec.~\ref{SM:theory} of the supplementary materials) but not guaranteed to be zero. The error could be reduced by incorporating feedforward terms into the motion control of a SoNS. Delays could be handled by, for example, sufficient preview of a reference signal: if the brain were to perform short-term online path planning, it could communicate a near-future reference trajectory downstream, instead of simply its current velocity. Similarly, in the current setup, robots calculate velocities based on distance measurements, which requires high sensing precision. The error could be reduced by instead using bearing (i.e., angle of arrival) measurements of two robots, which requires less sensing precision (for details, see our work on bearing-based frameworks~\cite{zhang2023self}). 

Closely related to this issue is the reaction time, which in SoNSs is based on communication speed and localization speed. In SoNSs with vision-based positioning, overall reaction time and stability will be increased as advances are made in visual tracking (for example, using ultraviolet LED markers~\cite{walter2018fast} or faster fiducial marker tracking~\cite{ulrich2022towards}). Advances made in other types of relative positioning that are less influenced by disturbances---such as vibration in powerful aerial robots or signal attenuation problems in underwater robots---will also be important for SoNS implementations with other types of robot platforms, for instance that can move at higher velocities than the laboratory ground robots used here. Certain approaches to relative positioning will require SoNSs with multi-layer networks (for instance, one layer for bearing-based relative positing and another layer for hierarchy of supervision~\cite{zhang2023self}). If it is found useful in the future for some layers of the network to be cyclic, then consensus mechanisms will need to be studied for the management of information fusion at the occurrence of cycles.

Currently, our open-source SoNS software supports experiments with real robots that use vision-based relative positioning, with aerial robots that are able to detect the relative positions and relative heading orientations of their ground robot neighbors. 
Note that the restrictions we imposed on the real experiments due to practical constraints are not limitations of the SoNS approach in general---for instance, SoNS could be applied to other types of robot platforms.
Additional versioning of our current software repository could also support other robot platforms and relative positioning techniques. 
Regardless of the approach used for positioning in a SoNS, a bidirectional connection can only be formed between two robots if at least one of them can detect the position of the other.

One of the biggest potentials moving forward is the possibility for more advanced SoNS brains and more advanced hierarchical computation, for instance by developing SoNSs with greater situational awareness (i.e., understanding of a situation and environment, especially for detecting risks), online learning, or autonomous mission planning capabilities. Automatic design approaches such as neuroevolution might help SoNS brains handle ever increasing mission complexity, or artificial neural networks such as autoencoders might expand the ability of a SoNS to internally manage and react to large amounts of sensor data. 

\section{Materials and Methods}

\subsection{SoNS control}

Full descriptions of the SoNS algorithms are provided in Sec.~\ref{SM:algorithm} of the supplementary materials and the SoNS software is available in an open-source repository. The primary operations are establishment, splitting, merging, node allocation, collective actuation via motion, and collective sensing and reaction, as described below. 

\textbf{Establishing a SoNS.}
The process by which a SoNS establishes, maintains, and reconfigures its dynamic hierarchical network is fully self-organized. At the start, the robots are all running identical SoNS software. Each robot starts as an independent single-robot SoNS, of which it is the brain by default. Each brain has a map of the communication structure it would like to build: the target topology represented by graph $G$ and target relative positions and other conditions represented by attributes $A$ associated to the nodes and links of $G$. The map can be, for example, calculated by a robot based on environmental features, predefined manually, or defined using a lookup table. Each brain then searches for robots to recruit in order to populate its map. When two SoNS that are searching for robots meet, each tries to recruit the other and, if they reach an agreement, the two merge under a single brain to become one SoNS. When a recruitment operation is completed and a new link $\{x_n, x_{n+1}\}$ is established, the robots have already reached an agreement about which robot will be at the higher level in the system architecture (i.e., $x_n$, the ``parent'' node for that connection) and which will be at the lower level (i.e., $x_{n+1}$, the ``child'' node). Once a brain $x_1$ with map $G$ has successfully recruited some children $\{x_2, x_3, \ldots x_n \}$, it sends each child a map of the respective structure that should be built downstream from it: the subgraphs $\{G'_1, G'_2, \ldots G'_n \}$ and the associated subsets $\{A_1, A_2, \ldots A_n \}$. Each child then takes full responsibility for the sub-structure directly downstream from it. It tries to recruit robots to populate the partial map it has received, becoming the parent of those new robots if an agreement is reached. It then repeats the process of map subdivision, and a new level of children begins to recruit robots and become parents. Meanwhile, as a robot recruits children successfully, if it receives messages from its parent that the new children are needed more urgently elsewhere in the SoNS, it might choose to hand over some of its children to its parent, to be redistributed in the SoNS in a self-organized way. These operations continue until the SoNS is complete, according to the brain's dynamic map $G$, or until no new robots can be found for recruitment.

\textbf{Splitting and merging.}
The simplest splitting operation is the departure of a single-robot SoNS, triggered by the former parent of the robot that departed. If, for instance, robot $x_n$ has completed its map $G'_n$, but the parent of robot $x_n$ then updates $G'_n$ from $E = \{\{x_n, x_{n+1}\},$ $\{x_n, x_{n+2}\}\}$ to $E = \{x_n, x_{n+1}\}$, robot $x_n$ might choose to expel one of its children by breaking the link to it. The expelled child will immediately and automatically return to being the brain of its own single-robot SoNS. A splitting operation that results in two multi-level SoNS is similar, but will be triggered at a higher level in the hierarchy. For example, a brain might update its own map $G$ from having two children $\{x_2, x_3\}$ with their subgraphs $G'_1$ and $G'_2$ to having only one child $x_2$. In that case, it will choose to expel $x_3$, which after splitting will maintain all its downstream links according to its map $G'_2$. The expelled robot $x_3$ will automatically become the brain of its own SoNS, and might choose to update its map, for instance from $G'_2 = (V = \{x_4, x_5, x_6, x_7\}, E = \{\{x_3, x_4\}, \{x_3, x_5\}, \{x_4, x_6\}, \{x_5, x_7\}\})$ to $G$ with the same $V$ set but $E = \{\{x_3, x_4\}, \{x_4, x_5\}, \{x_4, x_6\}, \{x_4, x_7\}\}$. In this case, $x_3$ will not reinitialize the whole structure, but will instead, for example, maintain links $\{x_3, x_4\}$ and $\{x_4, x_6\}$ and reorganize the other two links to match its new $G$. Unlike in splitting operations, the negotiations of merging operations always include the two brains of the respective SoNS. The simplest merging case is when one brain tries directly to recruit another brain. In this case the two brains will compare their internal assessments of their quality (for example, how many robots are in their SoNS), and the lower quality brain will agree to become the child of the higher quality brain. If the two brains find that they have equal quality, the allocation is chosen randomly. If a merging operation instead initiates downstream from one or both brains, the same quality comparison and agreement occurs, but some information propagation and node reallocation occurs to enable the comparison and the eventual merge. 
 
\textbf{Node allocation. }
All reconfiguration in a SoNS is done in a self-organized way using strictly local communication, and therefore multiple operations will often occur simultaneously. This makes it likely that a robot in the inner hierarchy levels will have to allocate more than one new child at the same time, due to recruitment or redistribution through handovers. For example, in the case of a local node allocation (i.e., a parent needs to allocate multiple candidates to its own children roles), the parent will compare the current relative position and current downstream order (i.e., number of vertices in the sub-graph) of each candidate with the target relative position and the target downstream order for its child nodes, as defined in its map $G_n$ and $A_n$. Similar operations assessing robots according to position and downstream order are used for all node allocation problems, including a candidate matching multiple local roles, multiple candidates matching a single local role, and a candidate matching no local role and therefore being reallocated upstream. Importantly, a parent can also use these operations to reallocate children it already has. A parent can replace an existing child with a candidate that is a better match for the role, then demote the former child to candidate status and enter it into the next round of local node allocation. Because of this replacement possibility, the robots in a SoNS continuously redistribute themselves within the SoNS. For instance, if a SoNS has one unoccupied node on its eastern-most side and recruits a new robot on its western-most side, the new robot will not be inefficiently handed over link-by-link until it reaches the unoccupied node. Rather, the robots in the SoNS will shift themselves in a self-organized way, so that the new robot is allocated to a nearby node, replacing a previous child, and a robot that was already near to the unoccupied node ends up being reallocated to it.

\textbf{Collective actuation via motion.}
A connection can only be established and maintained in a SoNS if one robot is in the other's field of view, so it is essential that the robots of a SoNS can move together collectively. In the simplest case, the brain defines and follows its motion trajectory independently of the robots downstream from it, and the rest of the robots act as followers of their respective parents, using strictly reactive control. 
In this case, each non-brain robot receives two controller inputs from its parent, which the parent calculates from the attribute set $A_n$ of its map $G_n$: a target displacement vector $\boldsymbol{d}$ and target orientation in unit quaternion $\boldsymbol{q}$. The parent will update the controller inputs $\boldsymbol{d}$ and $\boldsymbol{q}$ for its child when there is a change in its $A_n$. Using the inputs $\boldsymbol{d}$ and $\boldsymbol{q}$ and a mass-spring-damper model, each child outputs its own target linear velocity vector $\boldsymbol{v_t}$ and target angular velocity vector $\boldsymbol{\omega_t}$ and uses them to calculate its motor inputs for time $t$. In cases where the inter-level control distribution of the SoNS is incorporating some decentralized motion control (such as in the obstacle avoidance mission), a child might add another control layer to its mass-spring-damper model, or its parent might override its default behavior to temporarily diverge the calculation of $\boldsymbol{d}$ and $\boldsymbol{q}$ from the goals specified in $A_n$.

\textbf{Collective sensing and reaction.}
When a non-brain robot senses an external signal in the environment, it chooses both whether to respond to the signal directly and whether to send some compressed representation of this sensor information (for example, symbolic representations of identified objects). If a robot chooses to send sensor information to neighbors, the information is sent upstream. In exceptional cases, it can also choose to send the information downstream. Its choices might be based on a behavior model received from its parent or one defined by itself, depending on the inter-level control distribution of the SoNS. If it sends the sensor information upstream or downstream, the recipient robot will in turn make its own choices about response and propagation, based on elaborations of the compressed information. If it chooses to send the information onwards, it might first fuse it with other information from another child or its own onboard sensing, and send some compressed representation of the fused information. If robots continue to choose to send information upstream, the information will culminate at the brain. At any point along the hierarchy, a robot that chooses to respond to information it received might not only update its own behavior but might also update behavior models or actuation instructions sent to its child(ren). If the brain chooses to respond in this way, it triggers a SoNS-wide response in the SoNS. It is important to note that SoNS-wide responses coming from the brain are not necessarily triggered by sensor information fused SoNS-wide. For instance, in the case of a very high-priority signal observed by only one robot, the signal would get sent to the brain and all other robots (at one time step per hop of the shortest path) and trigger an immediate SoNS-wide response. 

\subsection{Analysis metrics}
\label{sec:methods:analysis-metrics}

To analyze the empirical results according to actuation error, we use Euclidean distance to calculate the position tracking error $E$ at each timestep, as follows:
\begin{equation}
\label{eq:error}
    E = \frac{1}{n}\sum_{i = 1}^{n} E_i,~~~~~~E_i = |~ d(\mathbf{p}_i - \mathbf{p}_1) - d(\mathbf{f}_i - \mathbf{f}_1) ~|,
\end{equation}
where $n$ is the total number of robots, $\mathbf{p}_i$ is the current position of robot $r_i$, $\mathbf{f}_i$ is the target position of robot $r_i$, and $i=1$ is the brain. The brain's error $E_1$ is always zero, because the brain's position is always the same as its target position.

The lower bound $B$ of position tracking error $E$ indicates the minimum error that would be present over time if the robots always moved directly to their target positions at maximum speed on the shortest Euclidean path, as if no obstacles, self-organized re-configuration, nor inter-robot collisions were present. It represents the total Euclidean distance between all robot target positions and their respective start positions when that target was set, such that
\begin{equation}
\label{eq:lower_bound}
B = 
\left\{ 
\begin{aligned}
& \frac{1}{n}\sum_{i = 1}^{n} B_i, &~~~\text{if}~ B_i > 0\\
& 0 &~~~\text{otherwise}
\end{aligned} \right.~,
~~~~~~~~~ B_i = | d(\mathbf{p_\epsilon}_i - \mathbf{f_\epsilon}_i) | - \kappa_i(t - t_\epsilon),
\end{equation}
where
$t$ is the current time, 
$t_\epsilon$ is the start time of the current target communication structure,
$\mathbf{p_\epsilon}_i$ is the position of robot $r_i$ at $t_\epsilon$, 
$\mathbf{f_\epsilon}_i$ is the target position of $r_i$ at $t_\epsilon$,
$\kappa_i$ is a constant describing the maximum speed of $r_i$ according to its type (ground robot or aerial robot).

\subsection{Theoretical guarantees} 

In the experiments conducted in this paper, to verify our SoNS architecture, we consider a simple distance-based positioning approach and reactive control law to establish and maintain the target relative positions using onboard measurements. This simple approach can be understood as a performance baseline for the tracking of target positions within a SoNS. To provide theoretical analysis, we represent an $n$-robot SoNS as a system of $n - 1$ robot pairs, use a proportional control law that is strictly reactive, and derive the leader--follower tracking kinematics assuming the robots cannot access any global information or external reference frame, instead using only local sensing of relative information that is available in our real setup. 
We provide theoretical analyses and mathematical proofs regarding the convergence of the position errors and closed-loop stability of position tracking in a SoNS with respect to the control inputs. A full description of the analysis method is provided in Sec.~\ref{SM:theory} of the supplementary materials.

\subsection{Experiment setup }

To be able to conduct real hardware experiments, we developed an open-source quadrotor platform with sensing and computational capabilities suited for swarm robotics experiments and an accompanying simulator model (see technical report~\cite{OguHeiAllZhuWahGarDor2022:techreport-010}, which includes links to open-source repositories, and Secs.~\ref{SM:simulator},~\ref{SM:aerial} in the supplementary materials). We use this quadrotor with standard e-puck ground robots~\cite{mondada2009puck,millard2017pi} mounted with fiducial tags (specifically AprilTags, see Sec.~\ref{SM:ground} of the supplementary materials).
For the purpose of data logging and drone safety, we also built a drone arena equipped with an off-the-shelf motion capture system (see Sec.~\ref{SM:real-arena} in the supplementary materials) and an open-source software package for experiment management (see technical report~\cite{All2022:techreport-004} and open-source repository\footnote{\url{https://github.com/iridia-ulb/supervisor}}).

In the experiments presented here, the only actuation utilized is motion. The SoNS software produces kinematic control outputs for all robots regardless of type, and a second control layer is used to calculate motor inputs for the differential drive ground robots and the quadrotors. These two layers are used in both reality and simulation, as the simulator is equipped with models of the internal dynamics of the robots where needed (for details about the robot models and the control layers, see Secs.~\ref{SM:algorithm}, \ref{SM:aerial}, and~\ref{SM:ground} of the supplementary materials). 
The only sensors utilized in the SoNS experiments (i.e., excluding an optical flow camera and single-point LiDAR the quadrotor uses for flight stabilization) are the downward-facing visual cameras onboard the quadrotors. The ground robots do not use any onboard sensing and must rely on virtual sensor information they receive from the quadrotors through local communication over the SoNS architecture (in other words, a ground robot in a single-robot SoNS is blind). Due to the downward-facing field of view of the quadrotors' cameras, the quadrotors cannot sense each other directly. Therefore, for relative positioning between two quadrotors, the quadrotors also must rely on information they receive through local communication over the SoNS architecture. Communication in the SoNS occurs over a wireless network, and two robots are only allowed to communicate with each other if one of them is in the other's field of view (using onboard sensing). These sensing and communication constraints are maintained in the simulator, matching reality.

The indoor arena for the real experiments is tightly constrained, so we have supplemented the real demonstrations with simulated experiments conducted in the multi-robot simulator ARGoS~\cite{pinciroli2012argos,allwright2018argos,allwright2018simulating}. 
We have conducted empirical cross-verification of the simulator and the real setup, to ensure that the results of the simulated experiments are a reliable approximation of the behavior of SoNS on the real robots (see Sec.~\ref{SM:cross-verify} of the supplementary materials).

\subsubsection*{Acknowledgments}

W.Z., S.O., and M.K.H. contributed equally to this work and share co-first authorship.
All authors made substantial contributions to the conception of the work.
The experiment design was led by W.Z., M.K.H., and M.D., and contributed to by all authors.
The experiments were conducted and the data was collected by W.Z. and S.O., supervised by M.K.H. and M.D.
The main algorithms used in the study were developed by W.Z. and supervised by M.K.H., M.A., and M.D., with contributions from S.O.
The hardware, software, and infrastructure used to support the experiments were developed by W.Z., S.O., M.K.H., M.A. and M.W., led by M.A.
The analysis of experimental results was conducted by W.Z. and M.K.H., supervised by M.D.
The theoretical analysis was developed by S.O., supervised by M.K.H., E.G., and M.D.
The presentation of the results, including figures and movies, was executed by W.Z., S.O., M.K.H., and M.W., supervised by M.D.
The supplementary materials were prepared by W.Z., S.O., and M.K.H., supervised by M.D., with contributions from all authors.
The writing of the manuscript was led by M.K.H., supervised by M.D., with contributions from all authors.
The original idea was provided by M.D. and A.L.C.
All authors read and approved the content of this manuscript.

\subsubsection*{Funding}
This work was partially supported by the Program of Concerted Research Actions (ARC) of the Universit{\'e} libre de Bruxelles, by the Belgian F.R.S.-FNRS under Grant J.0064.20, by the Office of Naval Research Global (Award N62909-19-1-2024), by the European Union's Horizon 2020 research and innovation programme under the Marie Sk\l{}odowska-Curie grant agreement No 846009, by the Independent Research Fund Denmark under grant 0136-00251B, and by the China Scholarship Council Award No 201706270186.
Mary Katherine Heinrich and Marco Dorigo acknowledge support from the Belgian F.R.S.-FNRS, of which they are a Postdoctoral Researcher and a Research Director respectively.

\section*{Figures and Tables}

\noindent
{\bf Fig.1.} The Self-organizing Nervous System (SoNS) concept: robots self-organize dynamic multi-level system architectures using exclusively local communication. \\
{\bf Fig.2.} Establishing self-organized hierarchy. \\
{\bf Fig.3.} Balancing global and local goals. \\
{\bf Fig.4.} Collective sensing and actuation. \\
{\bf Fig.5.} Binary decision making. \\
{\bf Fig.6.} Splitting and merging systems. \\
{\bf Fig.7.} Scalability study. \\
{\bf Fig.8.} Fault tolerance study. \\

\section*{Supplementary Materials}

\customlabel{SM:results}{S1}
\customlabel{SM:cross-verify}{S2}
\customlabel{SM:theory}{S3}
\customlabel{SM:algorithm}{S4}
\customlabel{SM:simulator}{S5}
\customlabel{SM:aerial}{S6}
\customlabel{SM:ground}{S7}
\customlabel{SM:real-arena}{S8}

{\bf Section \ref*{SM:results}.} Full set of experiment results \\
{\bf Section \ref*{SM:cross-verify}.} Cross-verification of the SoNS in simulation and on real hardware \\
{\bf Section \ref*{SM:theory}.} Theoretical analysis and mathematical proofs \\
{\bf Section \ref*{SM:algorithm}.} SoNS control algorithm details \\
{\bf Section \ref*{SM:simulator}.} Simulator setup \\
{\bf Section \ref*{SM:aerial}.} Aerial robot setup \\
{\bf Section \ref*{SM:ground}.} Ground robot setup \\
{\bf Section \ref*{SM:real-arena}.} Real indoor arena setup \\
{\it Figures S1-S62 are included in Secs.~S1-S8.}\\
{\it Movies S1-S12 each show an example trial, except for S8, which shows a demonstration.}\\
{\bf Movie S1.} Establishing self-organized hierarchy with real robots \\
{\bf Movie S2.} Balancing global and local goals with real robots \\
{\bf Movie S3.} Collective sensing and actuation with real robots \\
{\bf Movie S4.} Binary decision making with real robots \\
{\bf Movie S5.} Splitting and merging systems with real robots \\
{\bf Movie S6.} Scalability in the binary decision-making mission, 125 robots in simulation \\
{\bf Movie S7.} Scalability in the establishing self-organized hierarchy mission, several example system sizes in simulation \\
{\bf Movie S8.} Fault tolerance demonstration showing interchangeability of a failed brain robot \\
{\bf Movie S9.} Fault tolerance under multiple permanent failures, with real robots \\
{\bf Movie S10.} Fault tolerance under high-loss conditions in simulation, $66.\overline{6}$\% probability to fail \\
{\bf Movie S11.} Fault tolerance under 30\,s system-wide vision failure in simulation \\
{\bf Movie S12.} Fault tolerance under 30\,s system-wide communication failure in simulation

\clearpage

\newtheorem{problem}{Problem}
\newtheorem{theorem}{Theorem}
\newtheorem{remark}{Remark}
\newtheorem{assumption}{Assumption}
\newtheorem{proof}{Proof}
\newtheorem{definition}{Definition}
\newcommand{\cmnt}[1]{\textcolor{red}{\bf[#1]}}
\newcommand\topstrut[1][1.2ex]{\setlength\bigstrutjot{#1}{\bigstrut[t]}}
\newcommand\botstrut[1][0.9ex]{\setlength\bigstrutjot{#1}{\bigstrut[b]}}
\newcommand{\RN}[1]{%
  \textup{\uppercase\expandafter{\romannumeral#1}}%
}
\newcommand{\R}{\mathbb{R}}

\makeatletter
\renewcommand{\thefigure}{S\@arabic\c@figure}
\renewcommand{\theequation}{S\@arabic\c@equation}
\renewcommand{\theproblem}{S\@arabic\c@problem}
\renewcommand{\thetheorem}{S\@arabic\c@theorem}
\renewcommand{\theproof}{S\@arabic\c@proof}
\renewcommand{\theremark}{S\@arabic\c@remark}
\makeatother

\setcounter{MaxMatrixCols}{15}
\setcounter{figure}{0}
\setcounter{equation}{0}

\fontsize{11}{12}\selectfont
\newgeometry{headheight=100pt,headsep=25pt,margin=3cm}
\baselineskip14pt

{\centering
~\vspace{5mm}\\
{\Large Supplementary Materials for\\
\textsc{\Large Self-organizing Nervous Systems\\ for Robot Swarms}\\}}
\date{}
\vspace{15mm}

\section*{Section \ref*{SM:results}. Full set of experiment results}
\lhead{Section \ref*{SM:results}. Full results}
In this section, for each of the five robot missions (see Sec.~2.1 in the main paper), we provide key frames of a simulation to illustrate the general setup. We also provide the results of all experiment trials with real robots as well as several example trials for each type of simulation experiment with a larger system size. 

For each trial included in this section, we provide the following results: (on the left) trajectories of the robots over time, with the initial, final, and sometimes intermediary SoNS indicated in black; and (on the right) the mean and confidence interval per robot of the actuation error $E$ over time (see Eq.~1 in Sec.~4.2 in the main paper), with the lower bound $B$ (see Eq.~2 in Sec.~4.2 in the main paper) plotted for reference.

Note that, for each type of experiment (see the five missions, Sec.~2.1, scalability setups, Sec.~2.2, and the fault tolerance setups, Sec.~2.3, in the main paper), a video of an example trial is included within the supplementary materials (see Movies S1-S12). 
The code for each type of experiment is available in the online code repository. The experiment data for all trials, videos of all trials with real robots, and videos of example simulation trials are all available in the online data repository.

\vspace{7mm}
\noindent
{\it (Section continued on next page.)}
\clearpage
\rhead{Mission: Establishing self-organized hierarchy}

\subsection*{Mission: Establishing self-organized hierarchy (see Sec.~2.1.1 in the main paper)}

This mission includes two different variants, both run in experiments with real robots and in simulation.

\begin{figure}[h!]
\centering
\subfigure[]{
\includegraphics[trim=90 60 90 75,clip,width=0.4\textwidth]{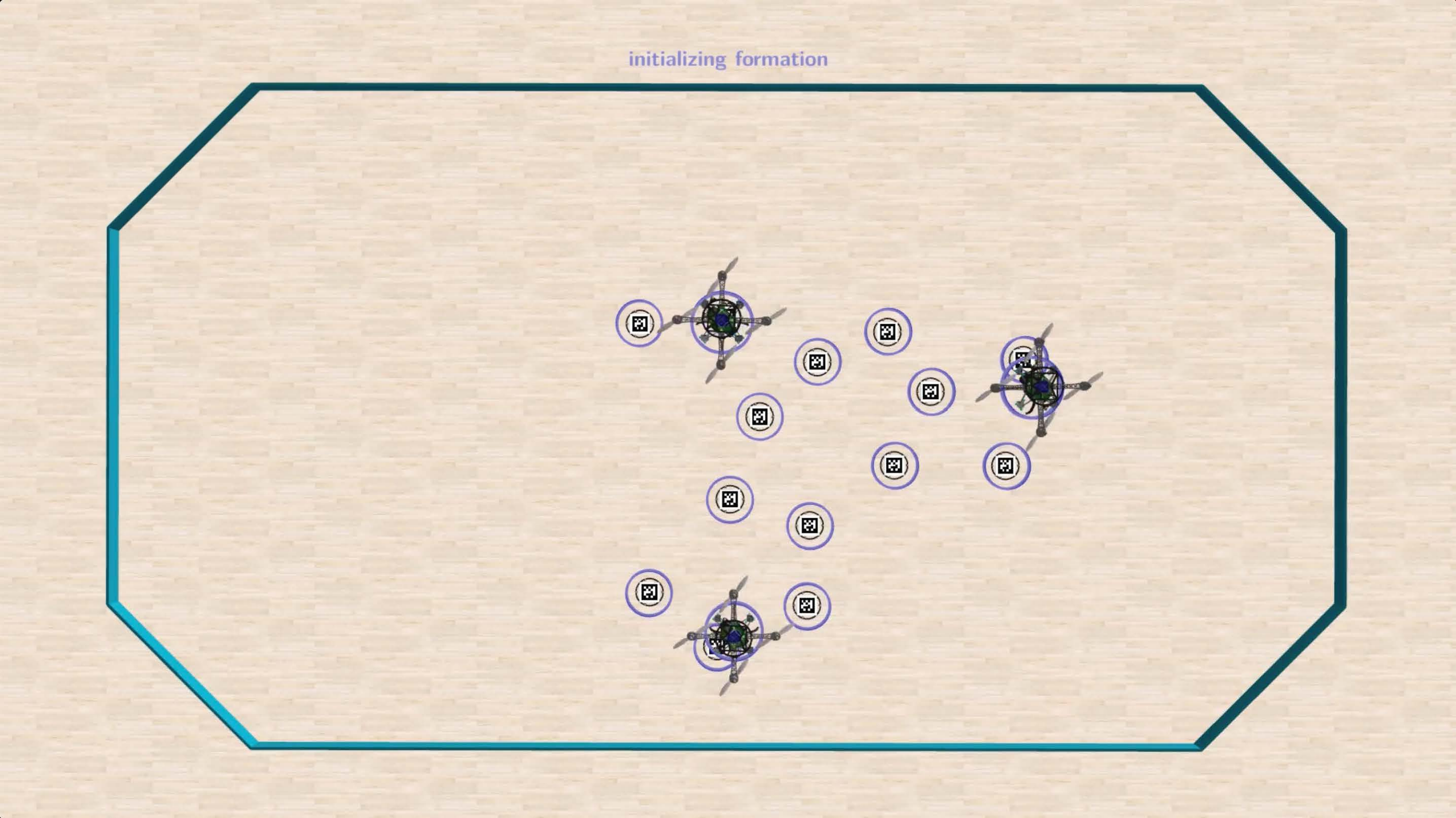}}
\subfigure[]{
\includegraphics[trim=90 60 90 75,clip,width=0.4\textwidth]{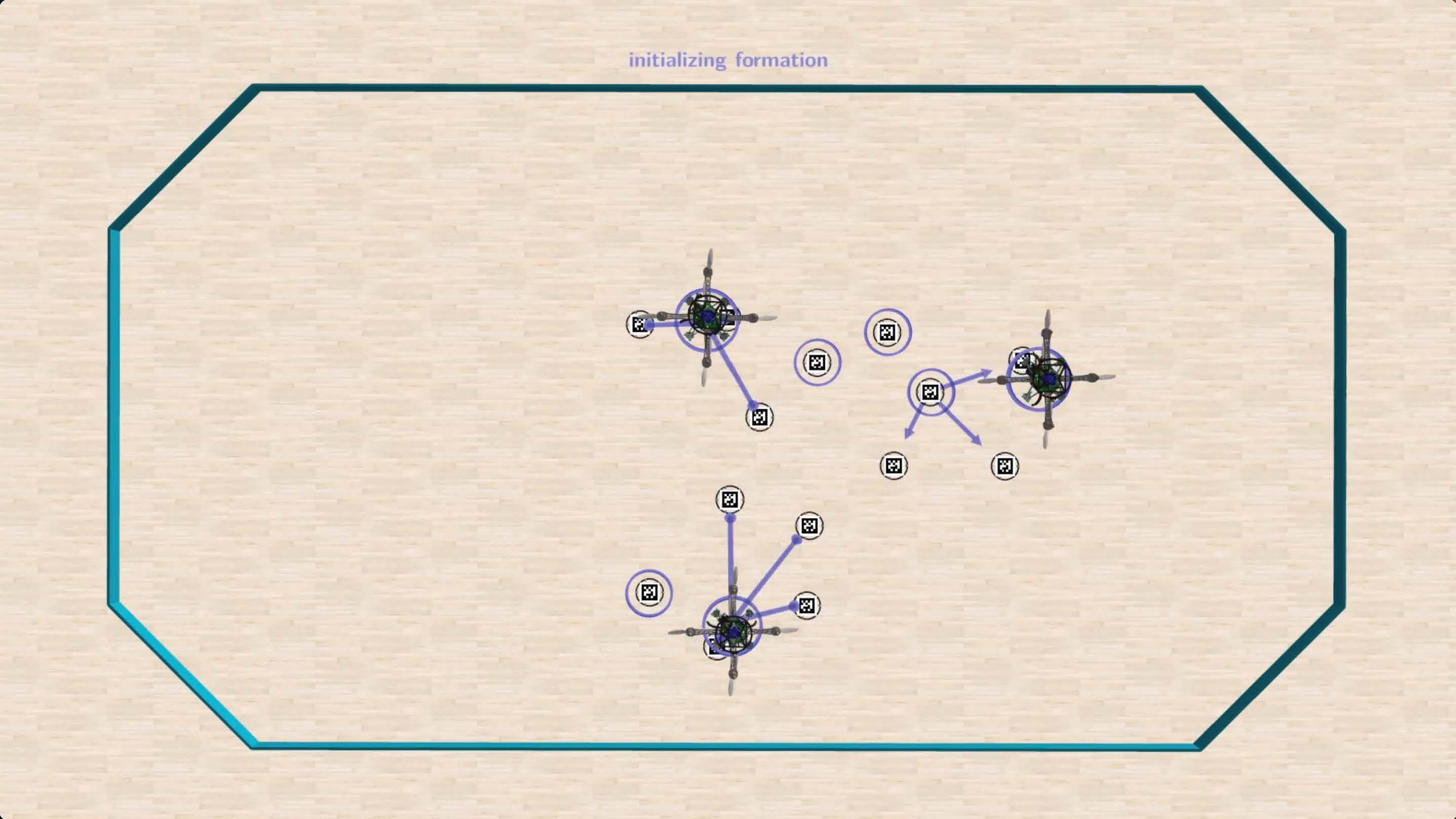}}\\
\vspace{-2mm}
\subfigure[]{
\includegraphics[trim=90 60 90 75,clip,width=0.4\textwidth]{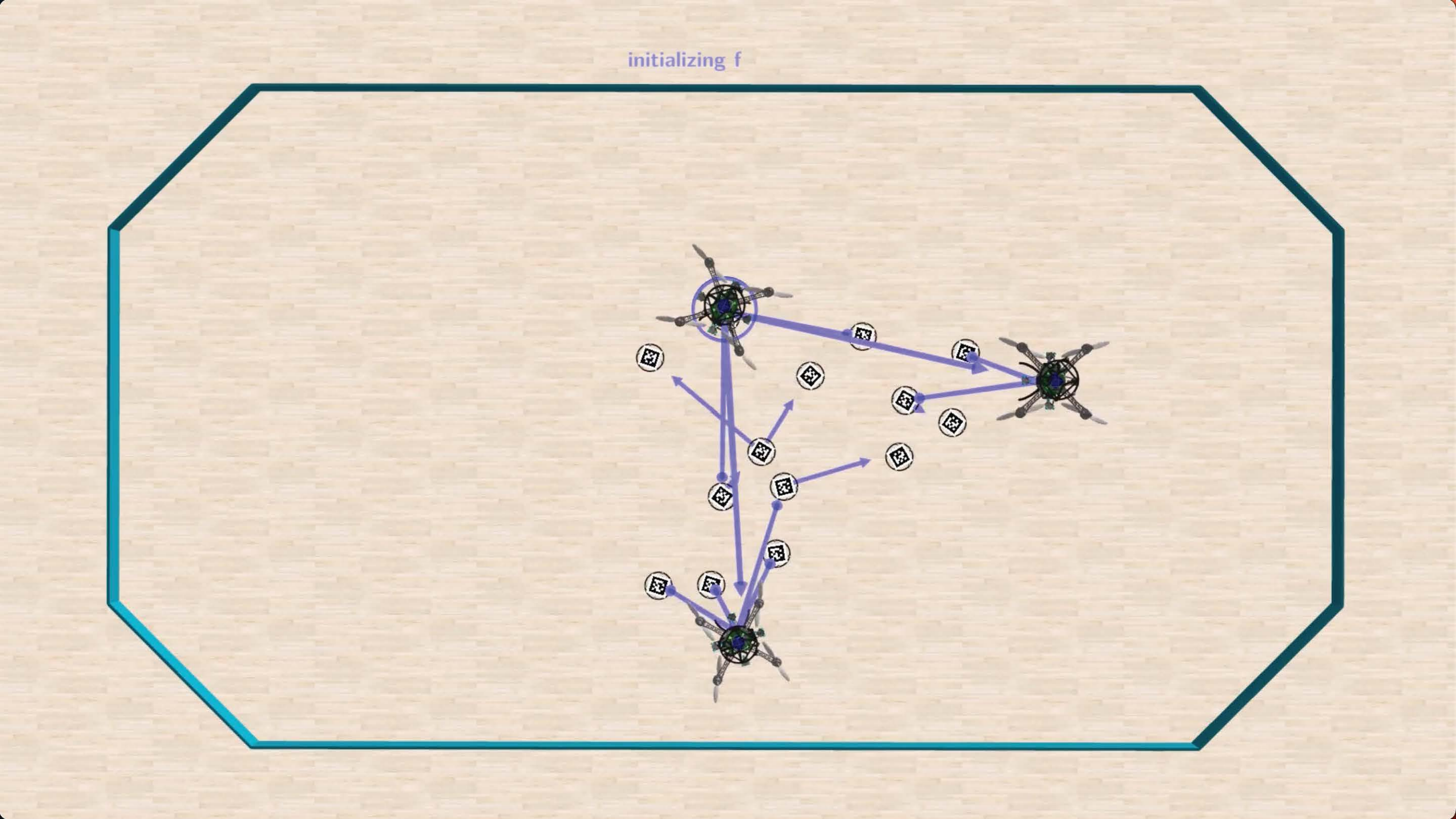}}
\subfigure[]{
\includegraphics[trim=90 60 90 75,clip,width=0.4\textwidth]{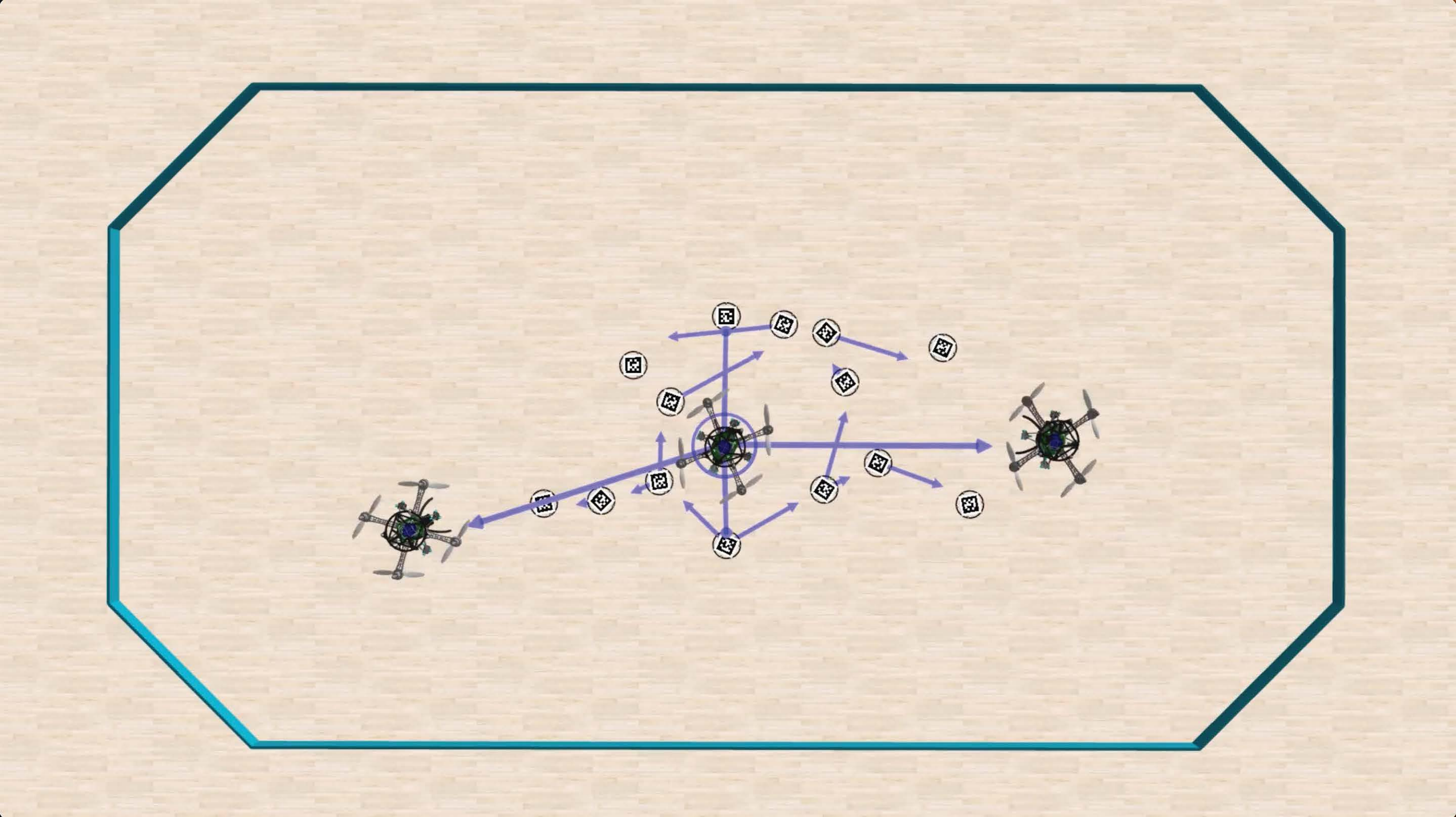}}\\
\vspace{-2mm}
\subfigure[]{
\includegraphics[trim=90 60 90 75,clip,width=0.4\textwidth]{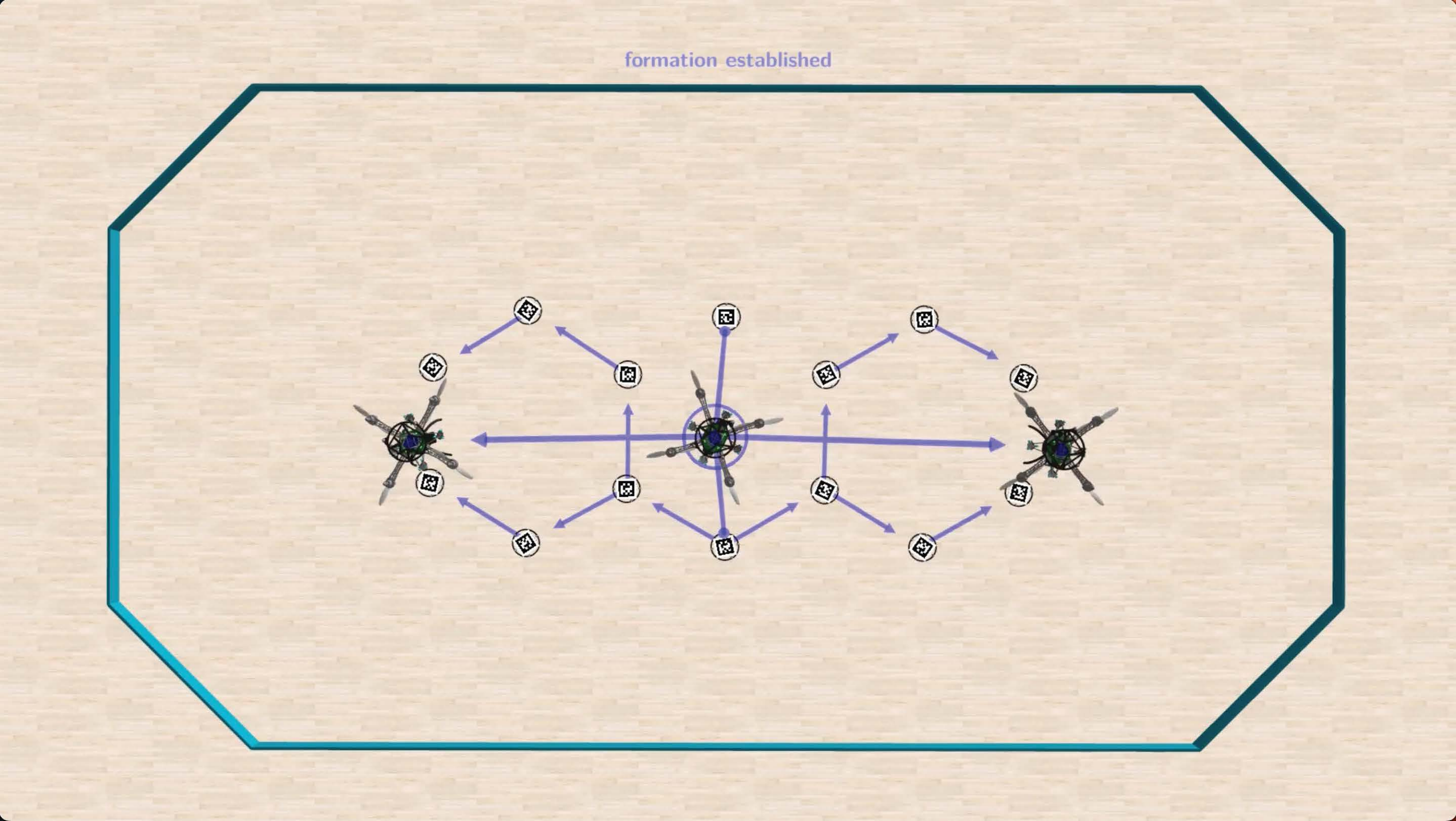}}\\
\vspace{-2mm}
\caption{{\bf Establishing self-organized hierarchy: Key frames.} (a) Robots start in scattered or clustered positions on the ground and the aerial robots take off. (b) All robots begin searching for peers and some robots start forming connections. (c,d) Robots merge their respective SoNSs and reallocate themselves into positions that match the target SoNS, continually adjusting their relative positions while coordinating locally to avoid collisions. (e) The target SoNS is complete.}
\label{fig:mission1-keyframes}
\end{figure}

\vspace{7mm}
\noindent
{\it (Section continued on next page.)}
\clearpage

\subsubsection*{Variant: Clustered start}
\begin{figure}[h!]
\centering
\includegraphics[trim=120 60 122 80,clip,width=0.38\textwidth]{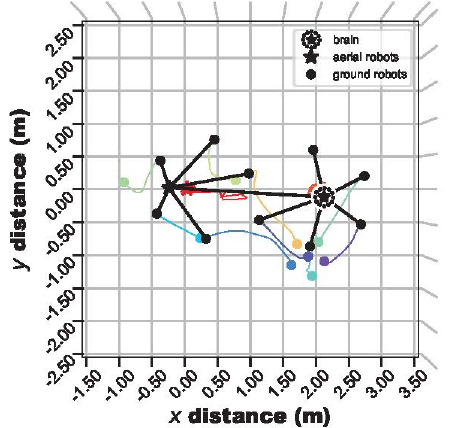}
\includegraphics[trim=16 0 40 20,clip,width=0.59\textwidth]{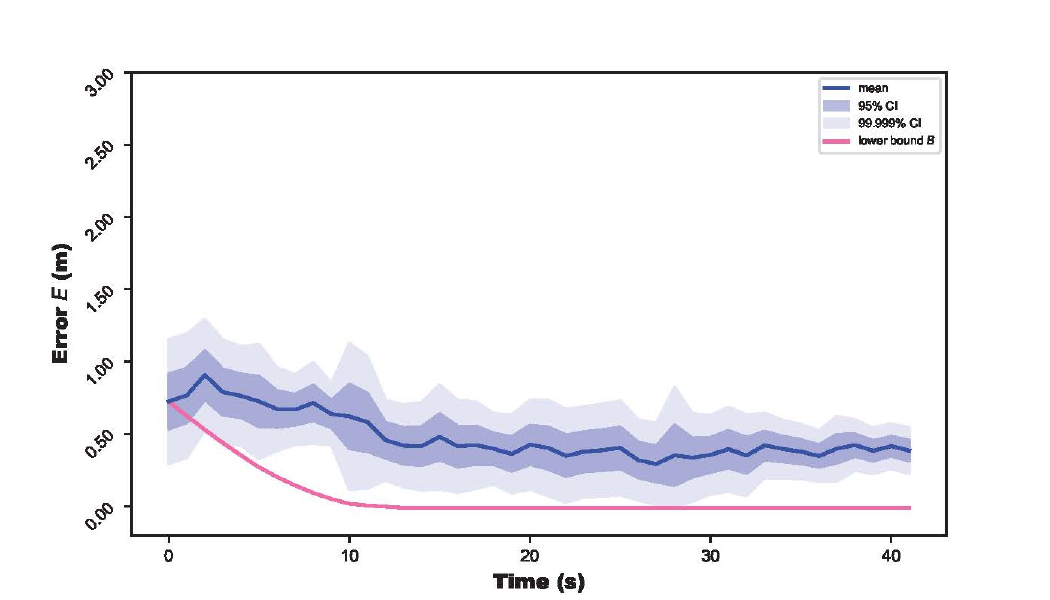}\\
\includegraphics[trim=120 60 120 80,clip,width=0.38\textwidth]{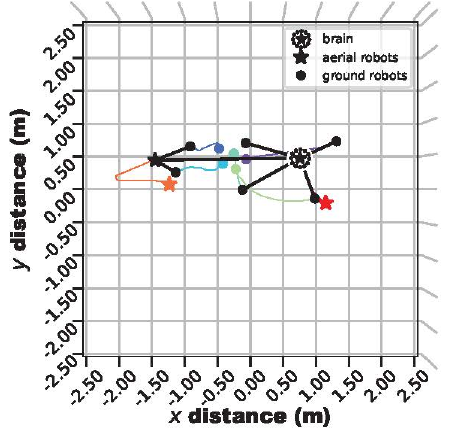}
\includegraphics[trim=20 0 40 20,clip,width=0.59\textwidth]{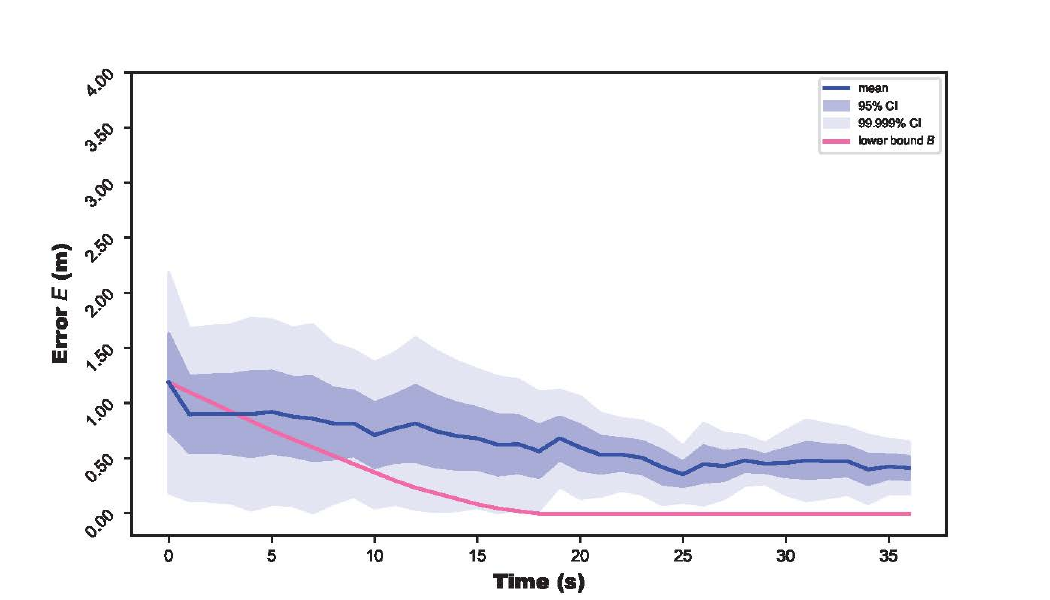}\\
\includegraphics[trim=120 60 120 80,clip,width=0.38\textwidth]{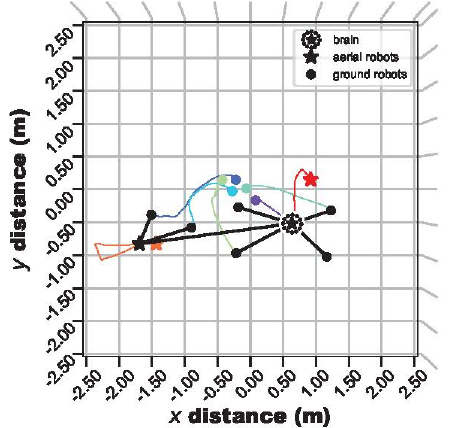}
\includegraphics[trim=20 0 40 20,clip,width=0.59\textwidth]{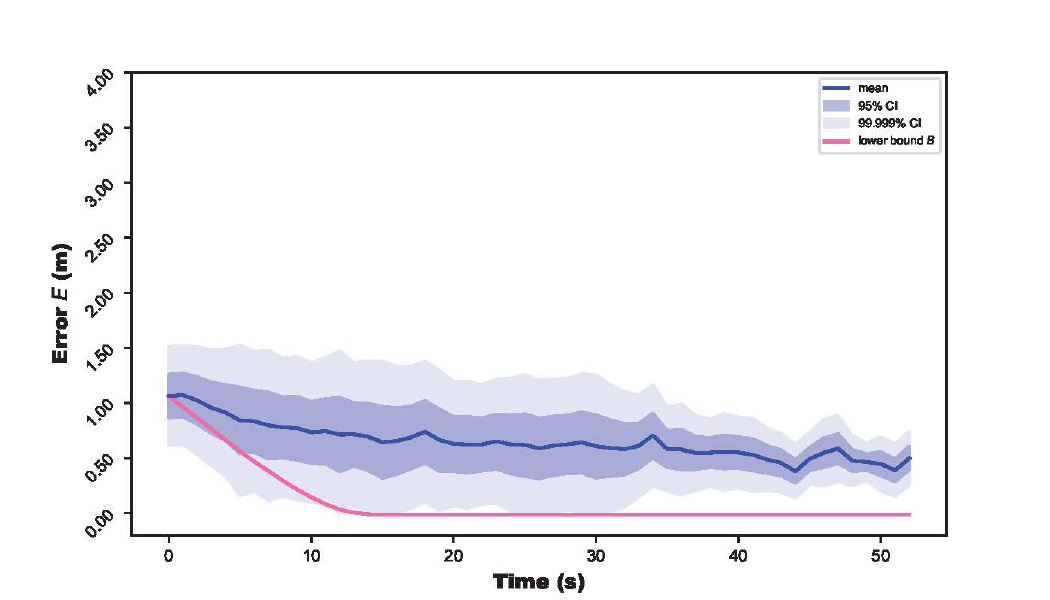}\\
\caption{{\bf Establishing self-organized hierarchy, clustered start: Real robot trials.} Note that, 12 robots were used in the first trial shown, but this pushed the safety limits of the indoor arena for this experiment type, so eight robots were used in the remaining five trials. In total, six trials with real robots were conducted {\it (figure continued on next page)}.}
\label{fig:mission1-variant1-hardware}
\end{figure}

\begin{figure}[h!]
\ContinuedFloat
\centering
\includegraphics[trim=120 60 120 80,clip,width=0.38\textwidth]{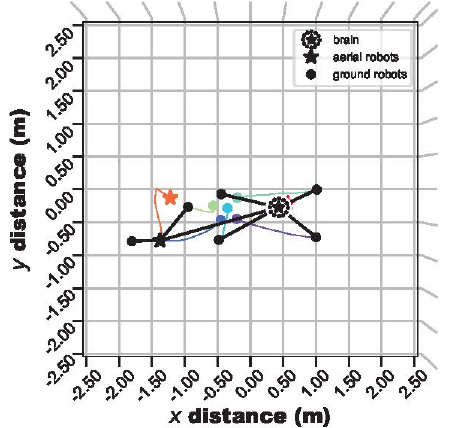}
\includegraphics[trim=20 0 40 20,clip,width=0.59\textwidth]{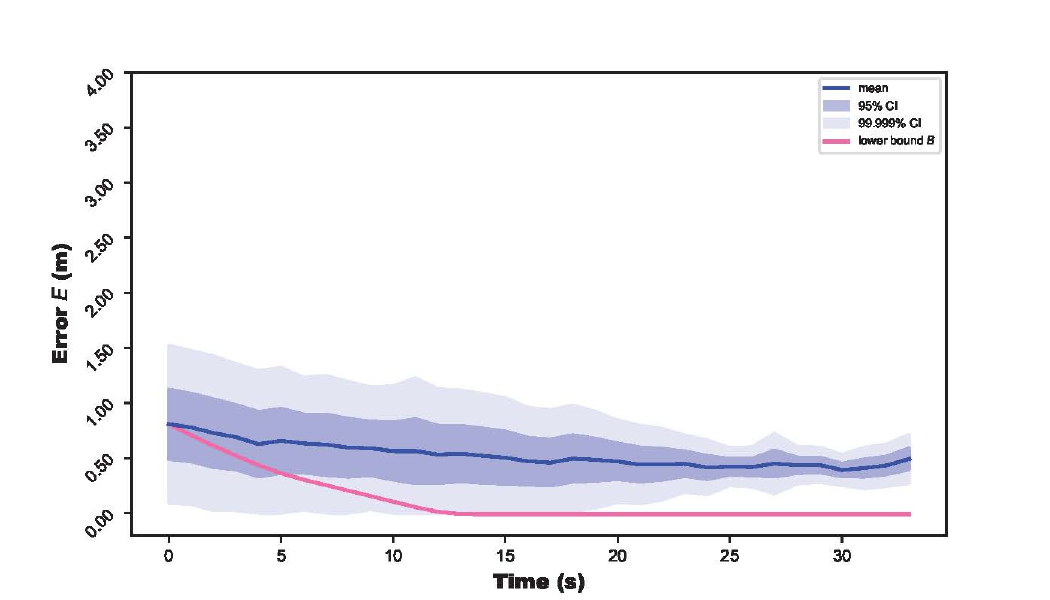}\\
\vspace{-1mm}
\includegraphics[trim=120 60 120 80,clip,width=0.38\textwidth]{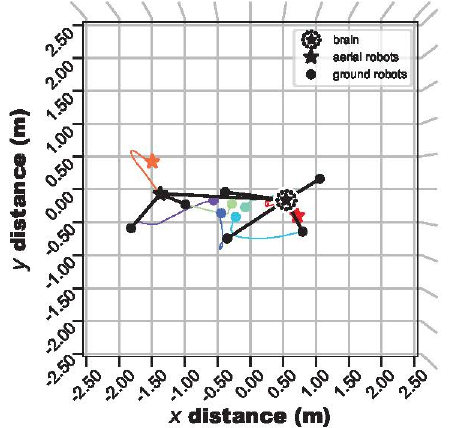}
\includegraphics[trim=20 0 40 20,clip,width=0.59\textwidth]{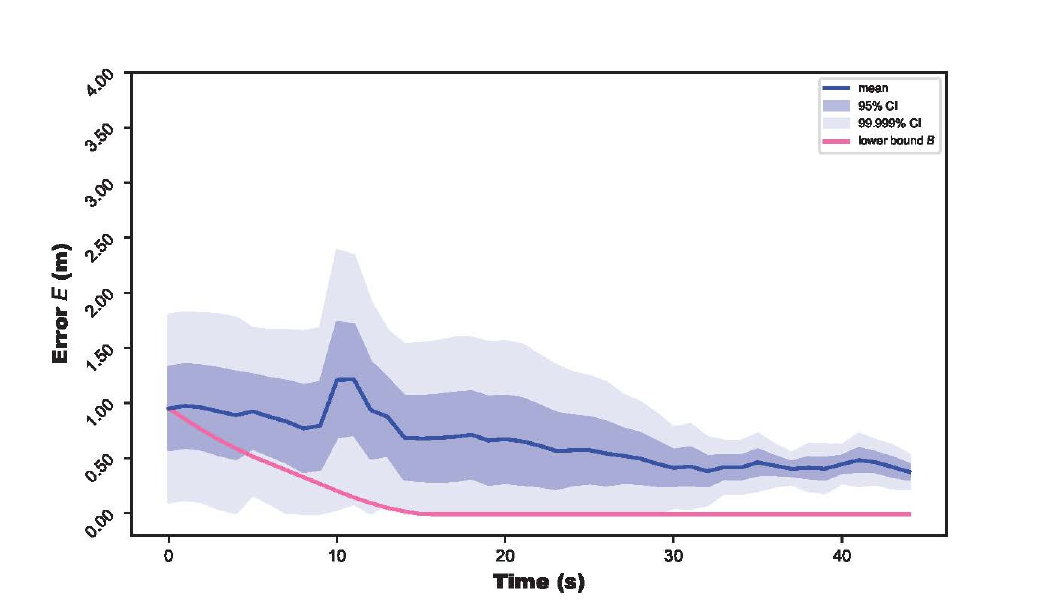}\\
\vspace{-1mm}
\includegraphics[trim=120 60 120 80,clip,width=0.38\textwidth]{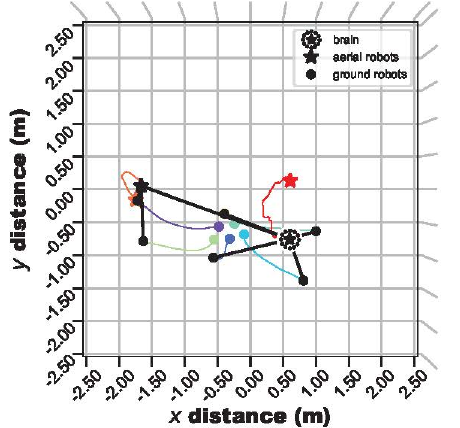}
\includegraphics[trim=20 0 40 20,clip,width=0.59\textwidth]{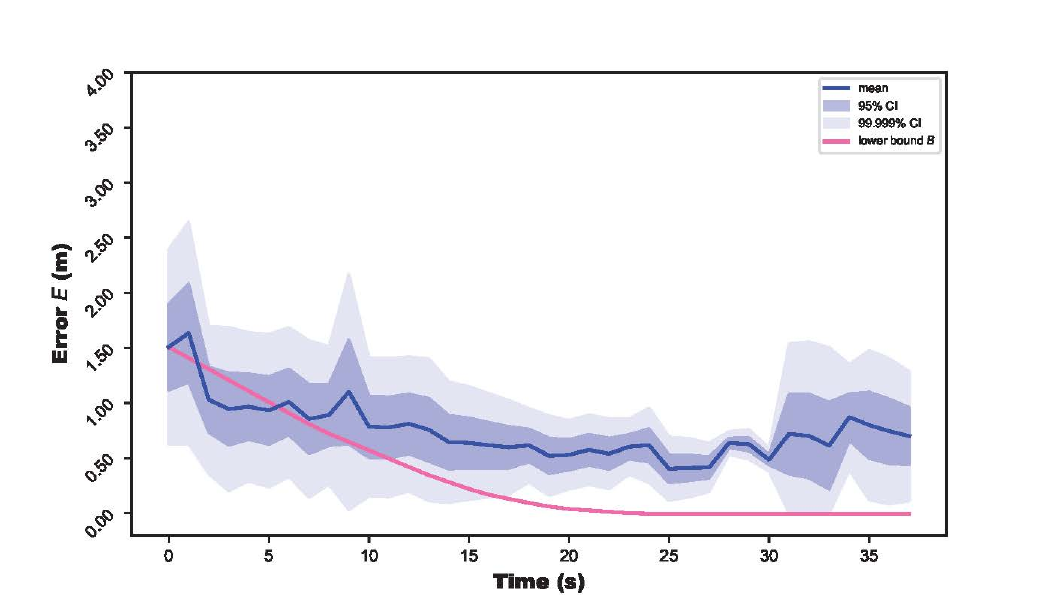}\\
\vspace{-1mm}
\caption{{\it (cont'd)} {\bf Establishing self-organized hierarchy, clustered start: Real robot trials.} Note that, 12 robots were used in the first trial shown, but this pushed the safety limits of the indoor arena for this experiment type, so eight robots were used in the remaining five trials. In total, six trials with real robots were conducted.}
\label{fig:mission1-variant1-hardware}
\end{figure}

\clearpage
\subsubsection*{Variant: Clustered start}
\begin{figure}[h!]
\centering
\includegraphics[trim=120 60 120 80,clip,width=0.38\textwidth]{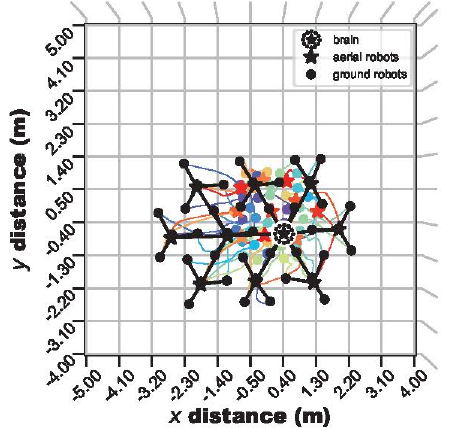}
\includegraphics[trim=20 0 40 20,clip,width=0.59\textwidth]{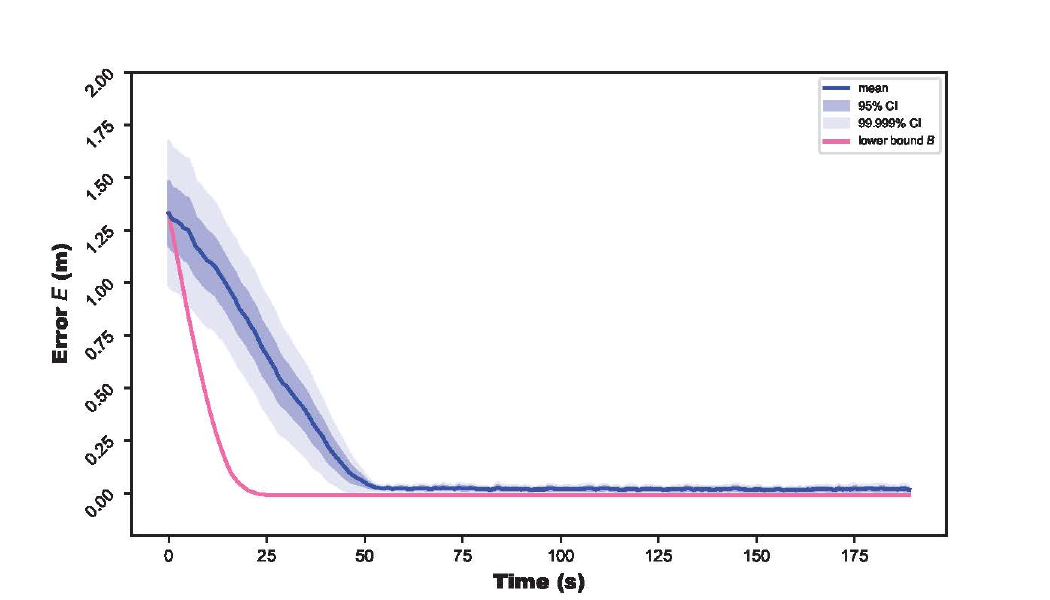}\\
\includegraphics[trim=120 60 120 80,clip,width=0.38\textwidth]{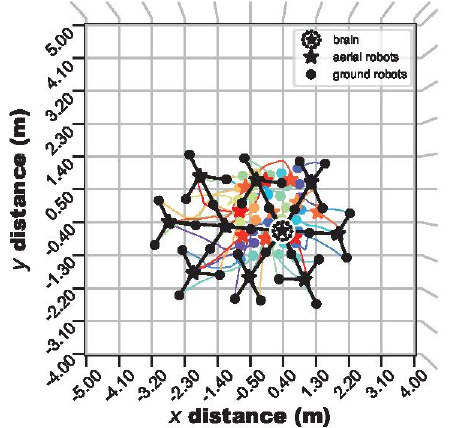}
\includegraphics[trim=20 0 40 20,clip,width=0.59\textwidth]{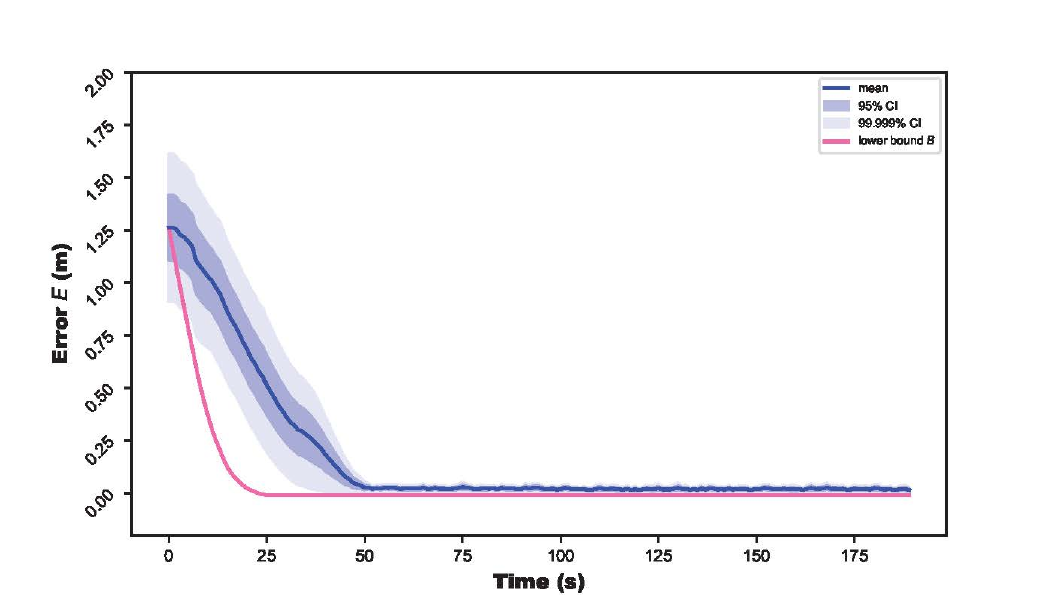}\\
\includegraphics[trim=120 60 120 80,clip,width=0.38\textwidth]{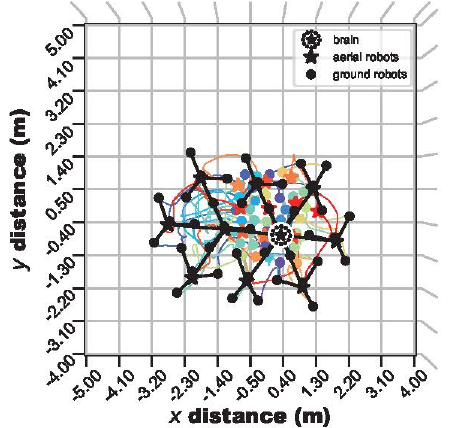}
\includegraphics[trim=20 0 40 20,clip,width=0.59\textwidth]{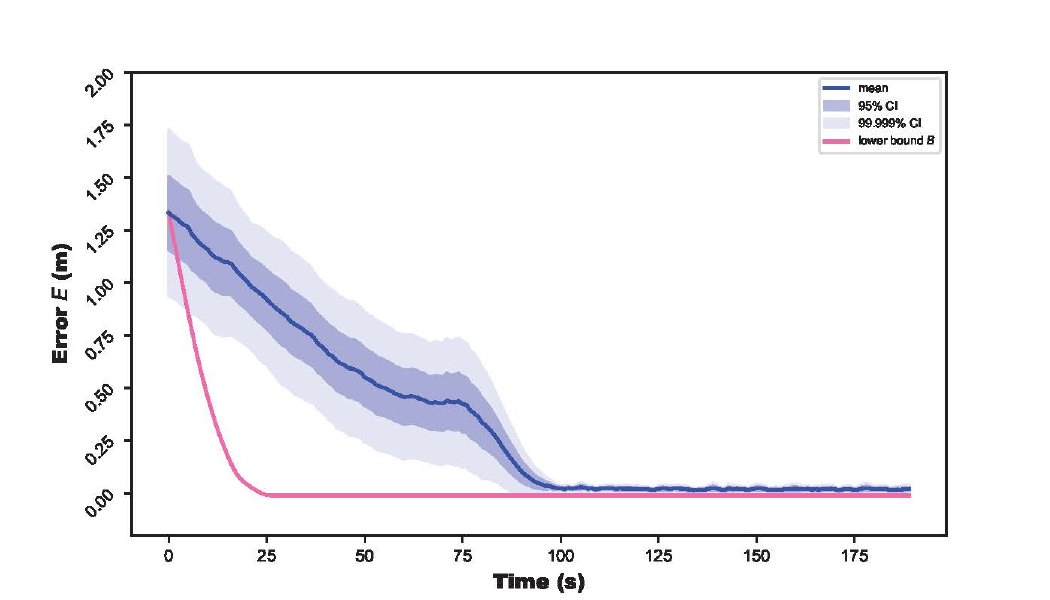}\\
\caption{{\bf Establishing self-organized hierarchy, clustered start: Example simulation trials.} 50 trials were conducted in simulation, each with 50 robots.}
\label{fig:mission1-variant1-simulation}
\end{figure}

\clearpage
\subsubsection*{Variant: Scattered start}
\begin{figure}[h!]
\centering
\includegraphics[trim=120 60 120 80,clip,width=0.38\textwidth]{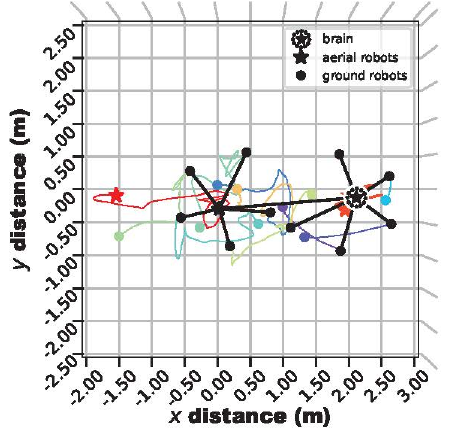}
\includegraphics[trim=20 0 40 20,clip,width=0.59\textwidth]{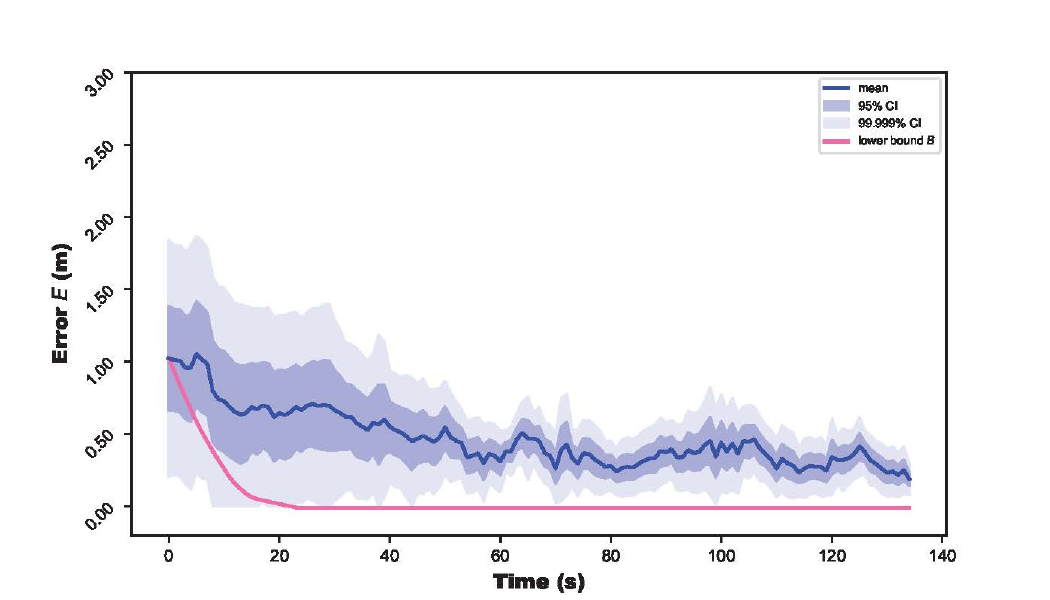}\\
\includegraphics[trim=120 60 120 80,clip,width=0.38\textwidth]{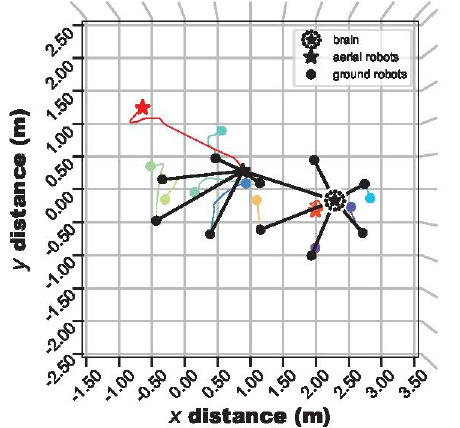}
\includegraphics[trim=20 0 40 20,clip,width=0.59\textwidth]{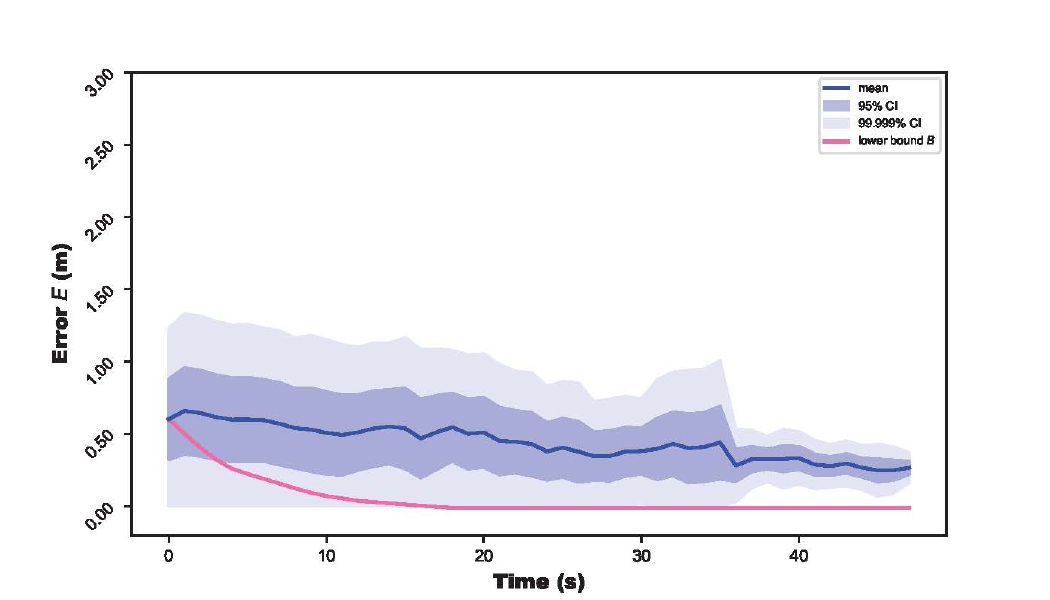}\\
\includegraphics[trim=120 60 120 80,clip,width=0.38\textwidth]{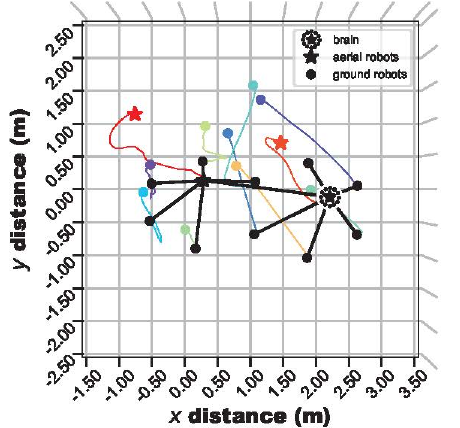}
\includegraphics[trim=20 0 40 20,clip,width=0.59\textwidth]{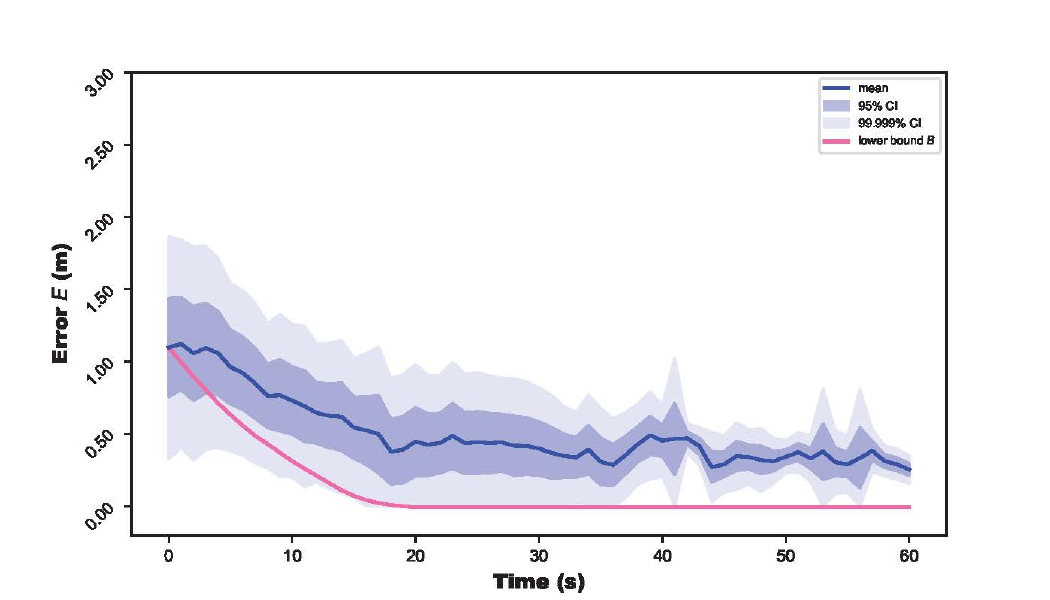}\\
\caption{{\bf Establishing self-organized hierarchy, scattered start: Real robot trials.} Five trials with real robots were conducted, each with 12 robots {\it (figure continued on next page)}.}
\label{fig:mission1-variant2-hardware}
\end{figure}

\clearpage

\begin{figure}[h!]
\ContinuedFloat
\centering
\includegraphics[trim=120 60 120 80,clip,width=0.38\textwidth]{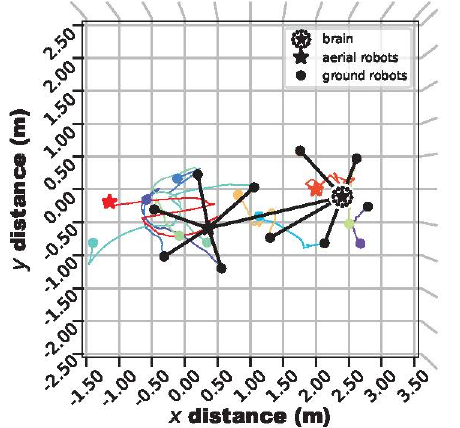}
\includegraphics[trim=20 0 40 20,clip,width=0.59\textwidth]{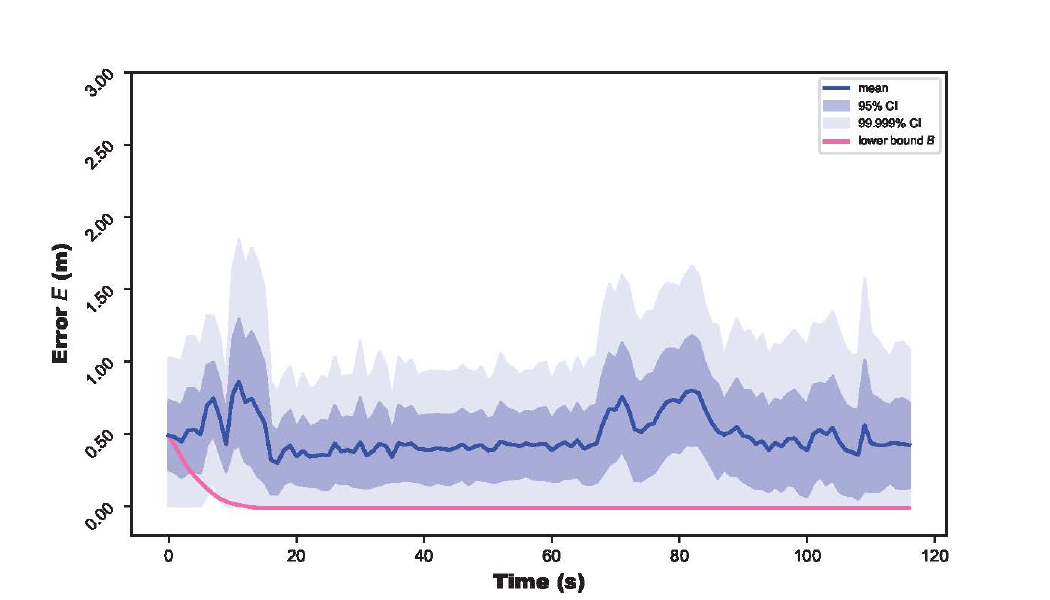}\\
\vspace{-1mm}
\includegraphics[trim=120 60 120 80,clip,width=0.38\textwidth]{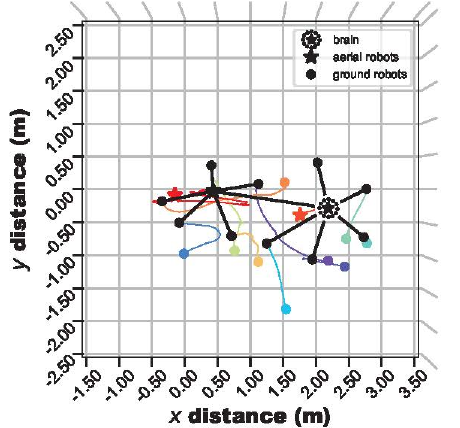}
\includegraphics[trim=20 0 40 20,clip,width=0.59\textwidth]{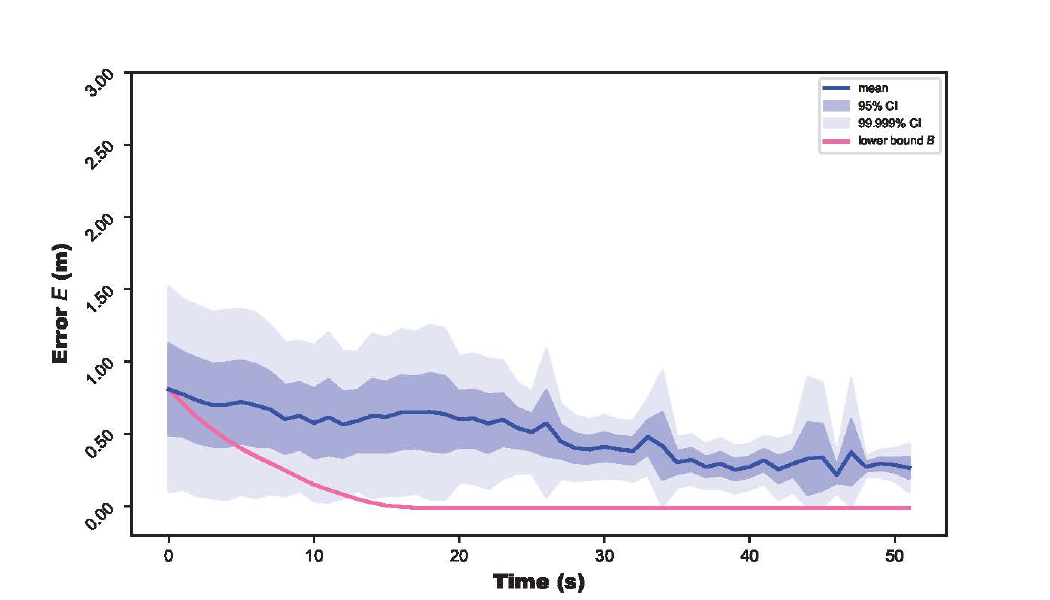}\\
\vspace{-1mm}
\caption{{\it (cont'd)} {\bf Establishing self-organized hierarchy, scattered start: Real robot trials.} Five trials with real robots were conducted, each with 12 robots.}
\label{fig:mission1-variant2-hardware}
\end{figure}

\vspace{7mm}
\noindent
{\it (Section continued on next page.)}

\clearpage
\subsubsection*{Variant: Scattered start}
\begin{figure}[h!]
\centering
\includegraphics[trim=120 60 120 80,clip,width=0.38\textwidth]{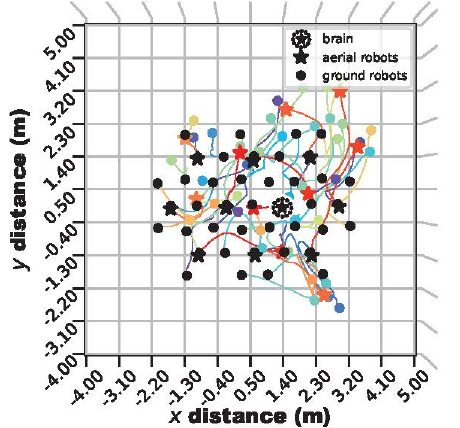}
\includegraphics[trim=20 0 40 20,clip,width=0.59\textwidth]{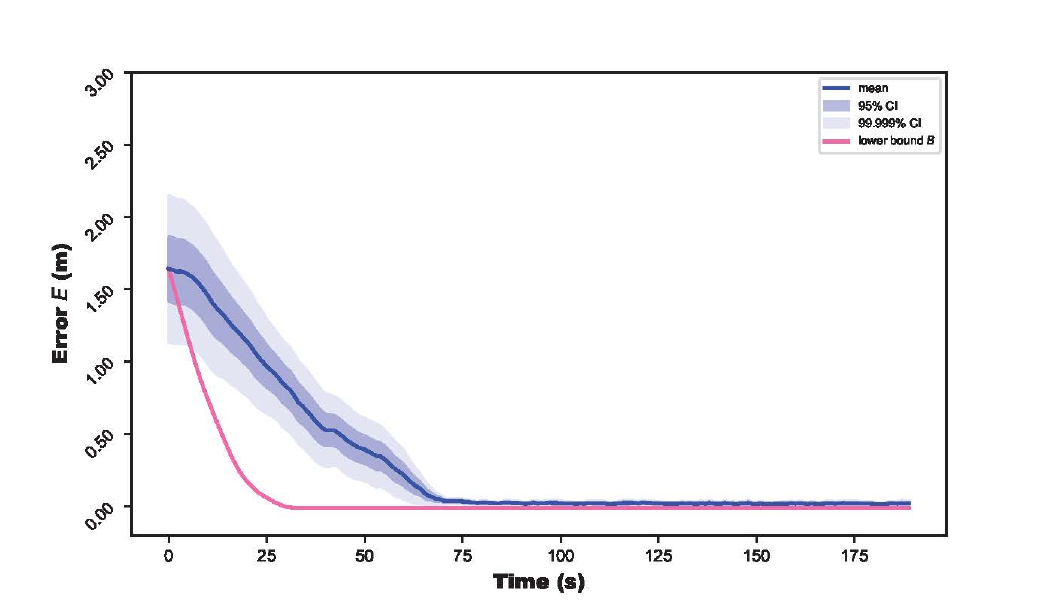}\\
\includegraphics[trim=120 60 120 80,clip,width=0.38\textwidth]{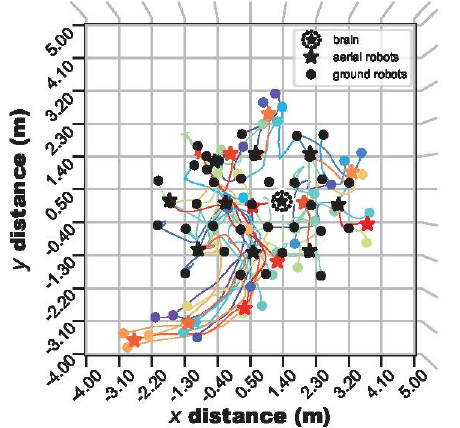}
\includegraphics[trim=20 0 40 20,clip,width=0.59\textwidth]{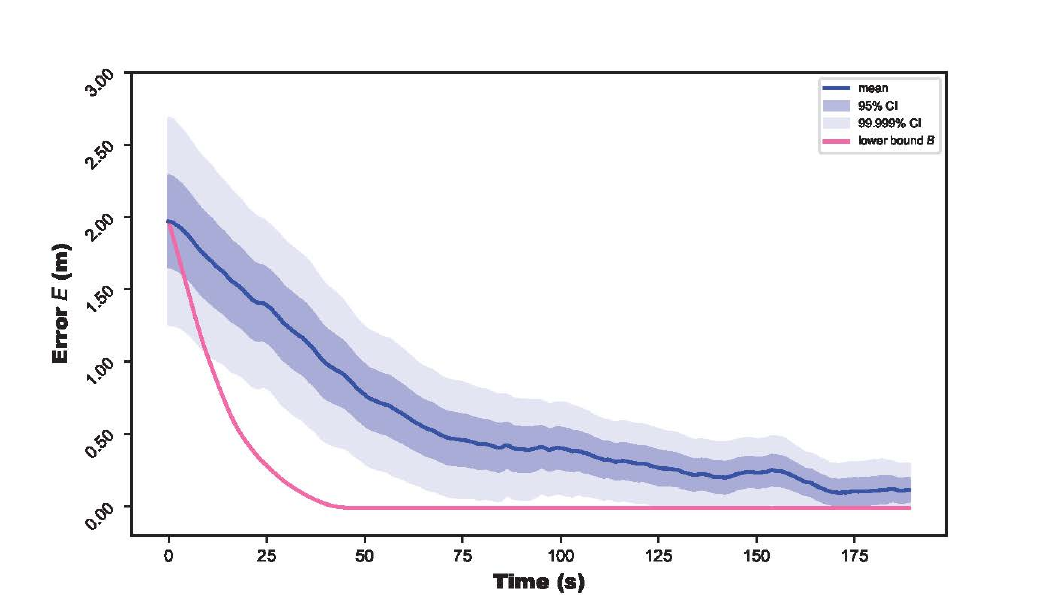}\\
\includegraphics[trim=120 60 120 80,clip,width=0.38\textwidth]{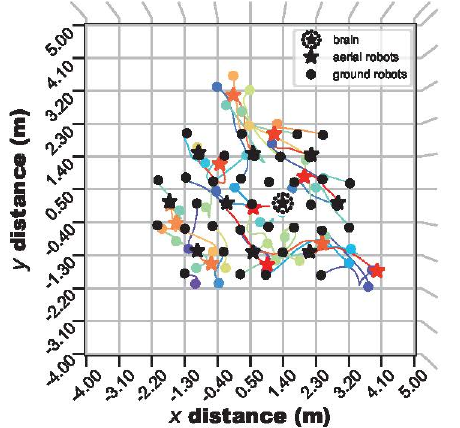}
\includegraphics[trim=20 0 40 20,clip,width=0.59\textwidth]{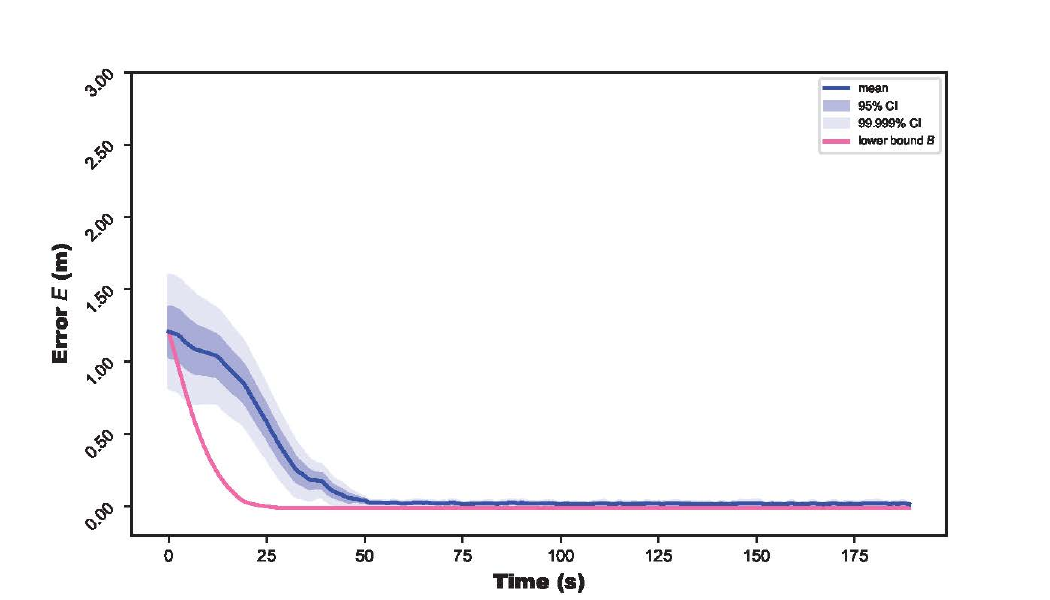}\\
\caption{{\bf Establishing self-organized hierarchy, scattered start: Example simulation trials.} 50 trials were conducted in simulation, each with 50 robots.}
\label{fig:mission1-variant2-simulation}
\end{figure}

\clearpage
\subsection*{Mission: Balancing global and local goals \textbf{(see Sec.~2.1.2 in the main paper)}}
\rhead{Mission: Balancing global and local goals}

This mission includes two different variants, both run in experiments with real robots and in simulation.

\begin{figure}[h!]
\centering
\subfigure[]{
\includegraphics[trim=90 60 90 75,clip,width=0.4\textwidth]{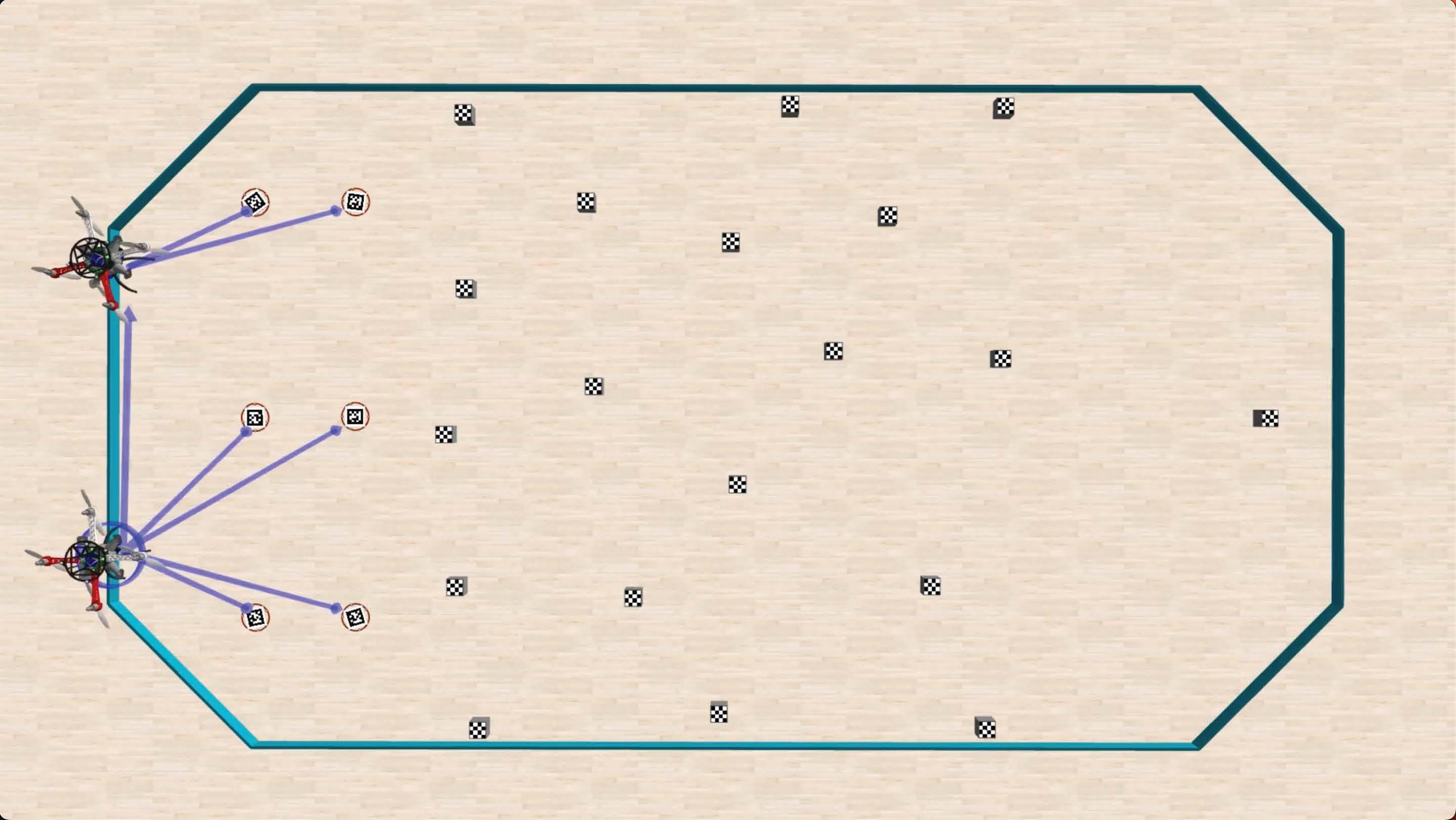}}
\subfigure[]{
\includegraphics[trim=90 60 90 75,clip,width=0.4\textwidth]{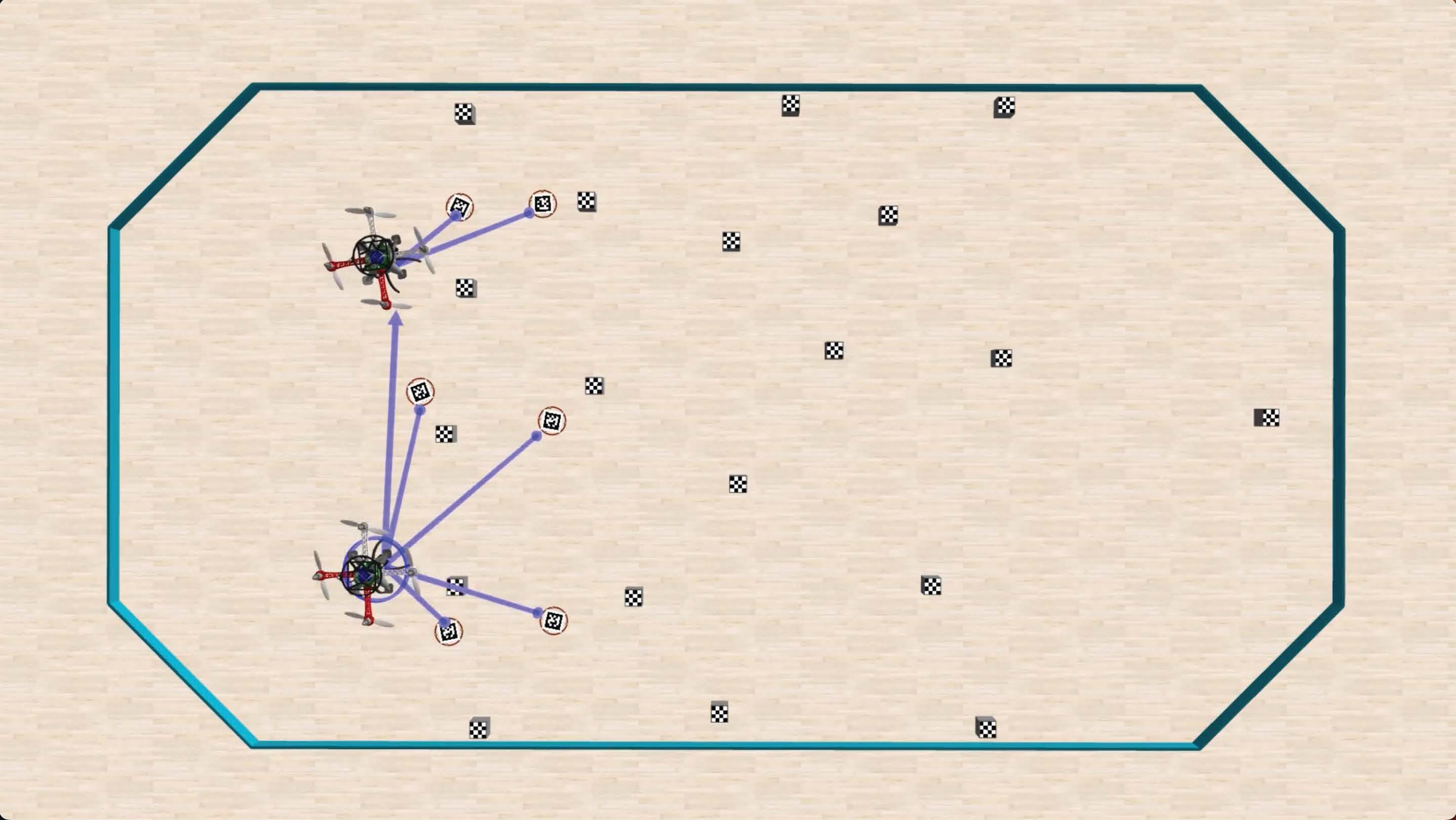}}\\
\vspace{-2mm}
\subfigure[]{
\includegraphics[trim=90 60 90 75,clip,width=0.4\textwidth]{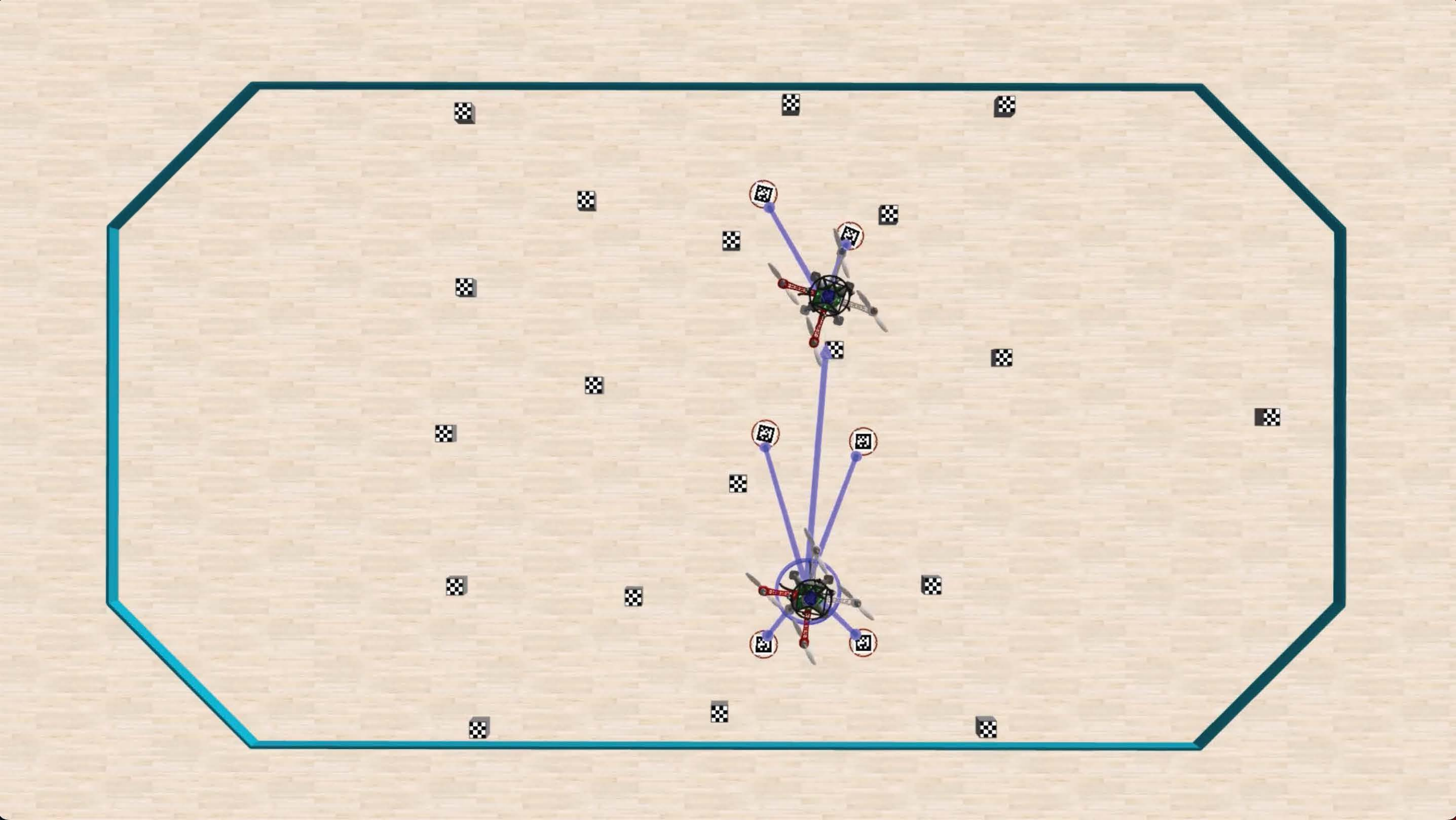}}
\subfigure[]{
\includegraphics[trim=90 60 90 75,clip,width=0.4\textwidth]{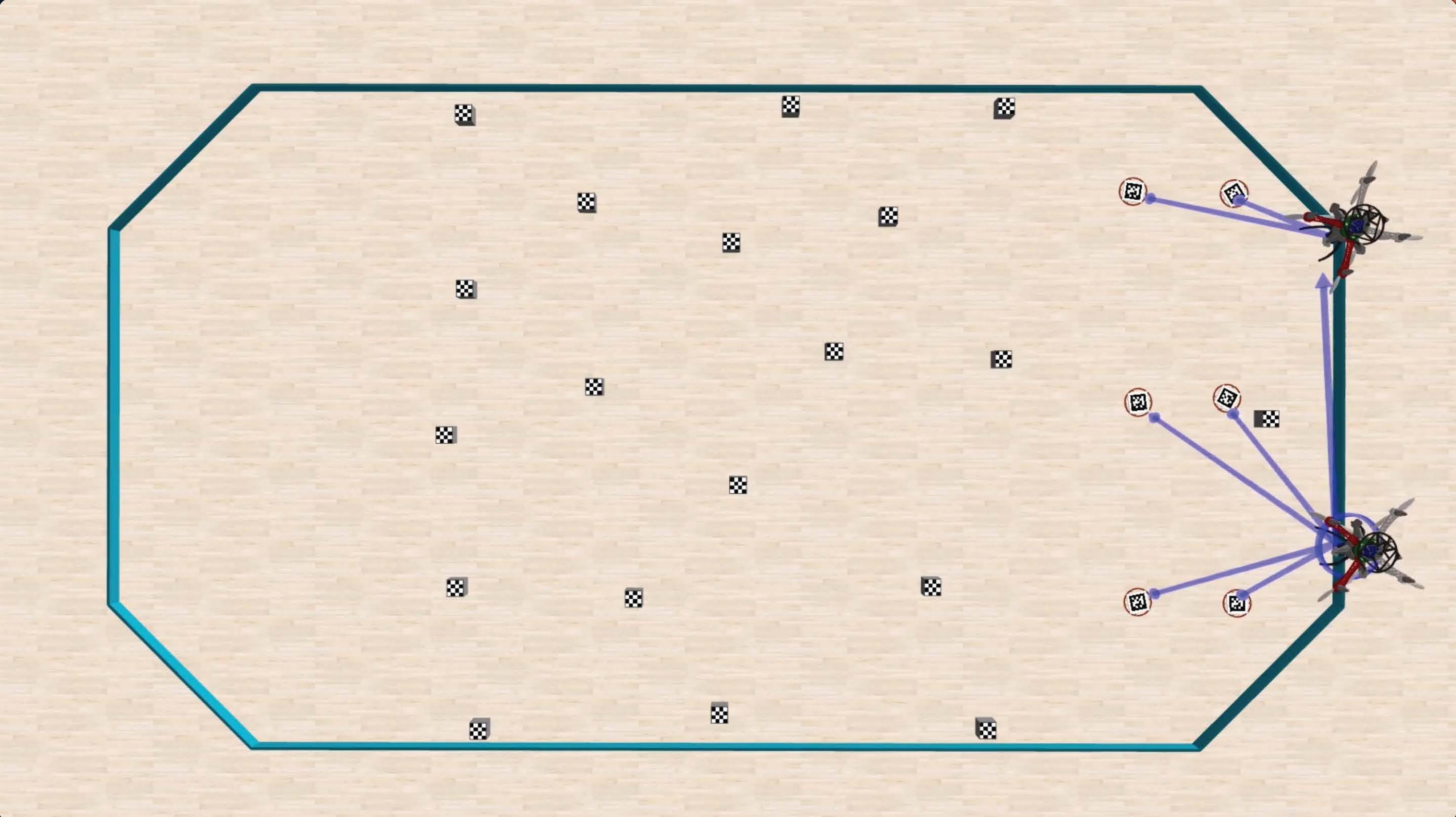}}\\
\vspace{-2mm}
\caption{{\bf Balancing global and local goals: Key frames.} (a,b) Robots begin as members of a single SoNS and begin moving across an environment with an unknown field of small, dense obstacles, searching for an object that marks the final destination. (c) As the robots move through the obstacle field, they collaboratively balance global and local goals at each bidirectional link, to avoid obstacles while still keeping the SoNS together. (d) The SoNS surpasses the obstacle field and senses the final destination object, and the mission is complete.}
\label{fig:mission2-keyframes}
\end{figure}

\vspace{7mm}
\noindent
{\it (Section continued on next page.)}

\clearpage
\subsubsection*{Variant: Smaller, denser obstacles}
\begin{figure}[h!]
\centering
\includegraphics[trim=120 60 120 80,clip,width=0.38\textwidth]{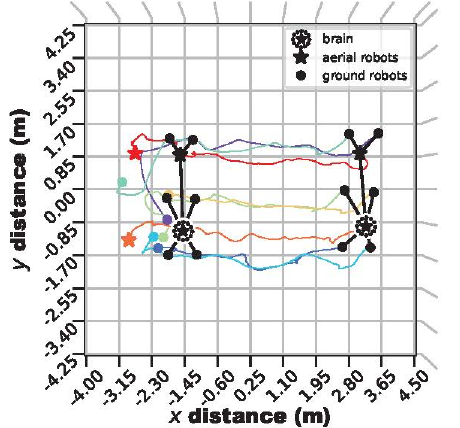}
\includegraphics[trim=20 0 40 20,clip,width=0.59\textwidth]{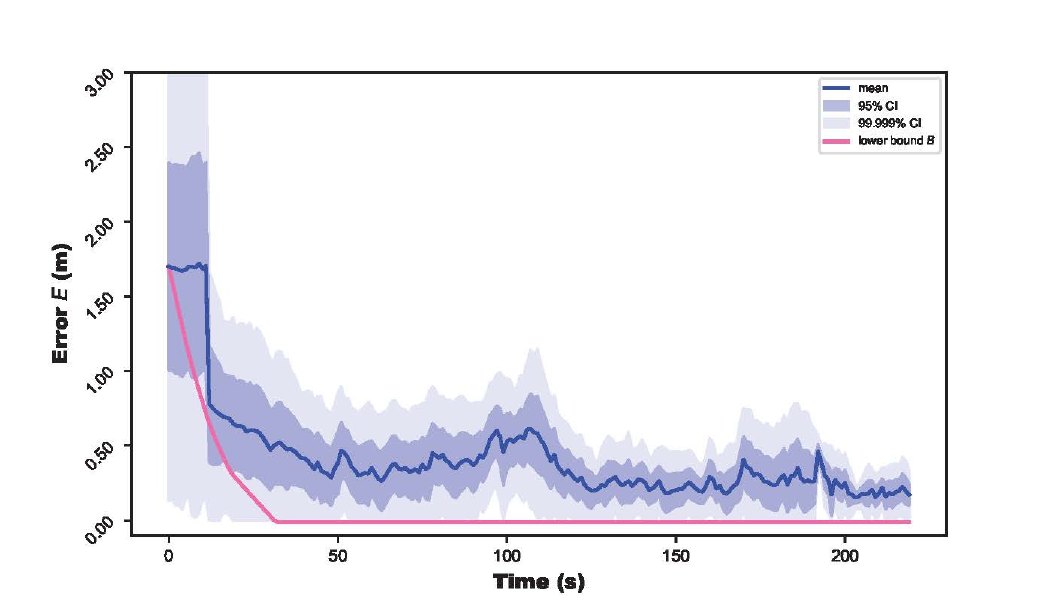}\\
\includegraphics[trim=120 60 120 80,clip,width=0.38\textwidth]{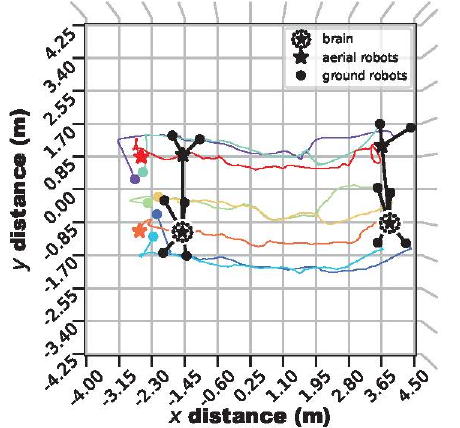}
\includegraphics[trim=20 0 40 20,clip,width=0.59\textwidth]{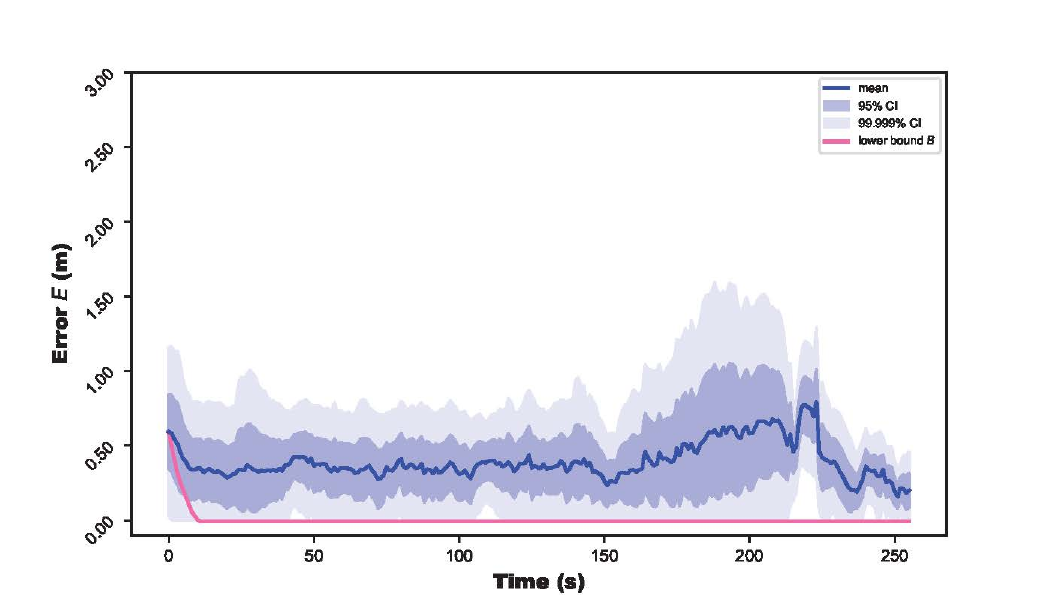}\\
\includegraphics[trim=120 60 120 80,clip,width=0.38\textwidth]{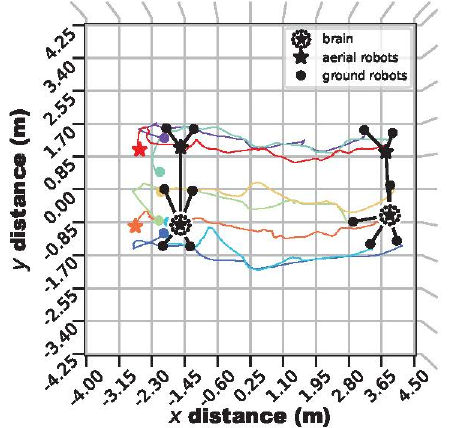}
\includegraphics[trim=20 0 40 20,clip,width=0.59\textwidth]{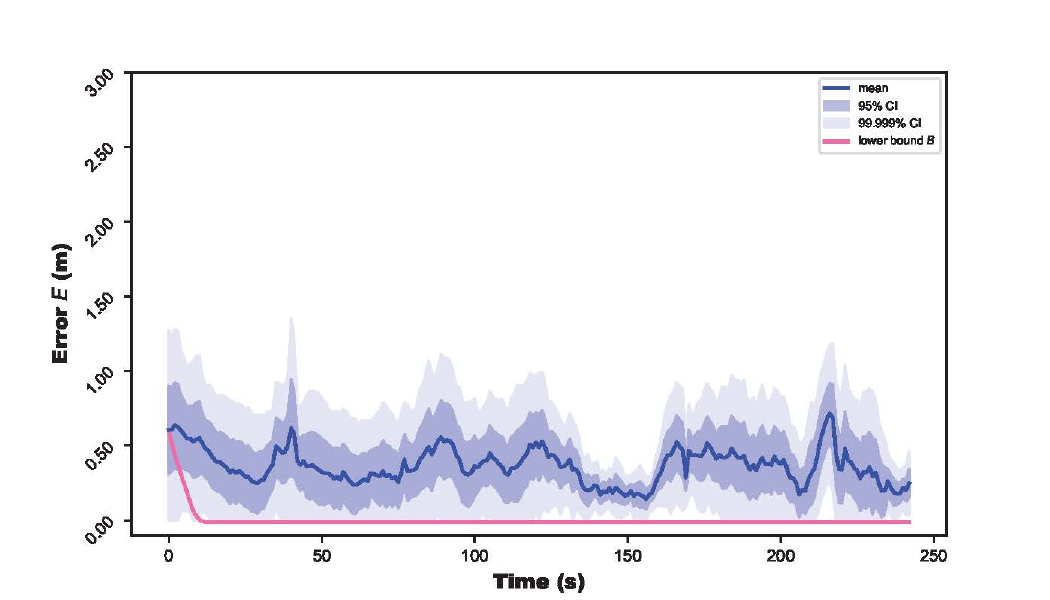}\\
\caption{{\bf Balancing global and local goals, with smaller, denser obstacles: Real robot trials.} Five trials with real robots were conducted, each with eight robots {\it (figure continued on next page)}.}
\label{fig:mission2-variant1-hardware}
\end{figure}

\clearpage
\begin{figure}[h!]
\ContinuedFloat
\centering
\includegraphics[trim=120 60 120 80,clip,width=0.38\textwidth]{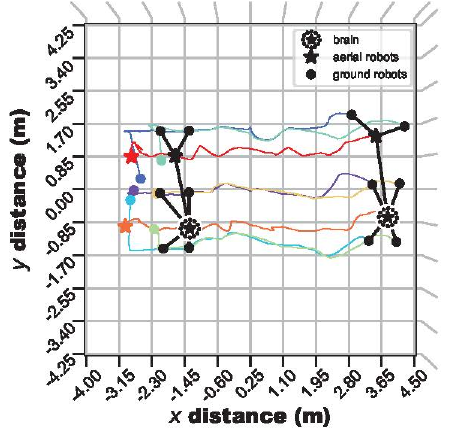}
\includegraphics[trim=20 0 40 20,clip,width=0.59\textwidth]{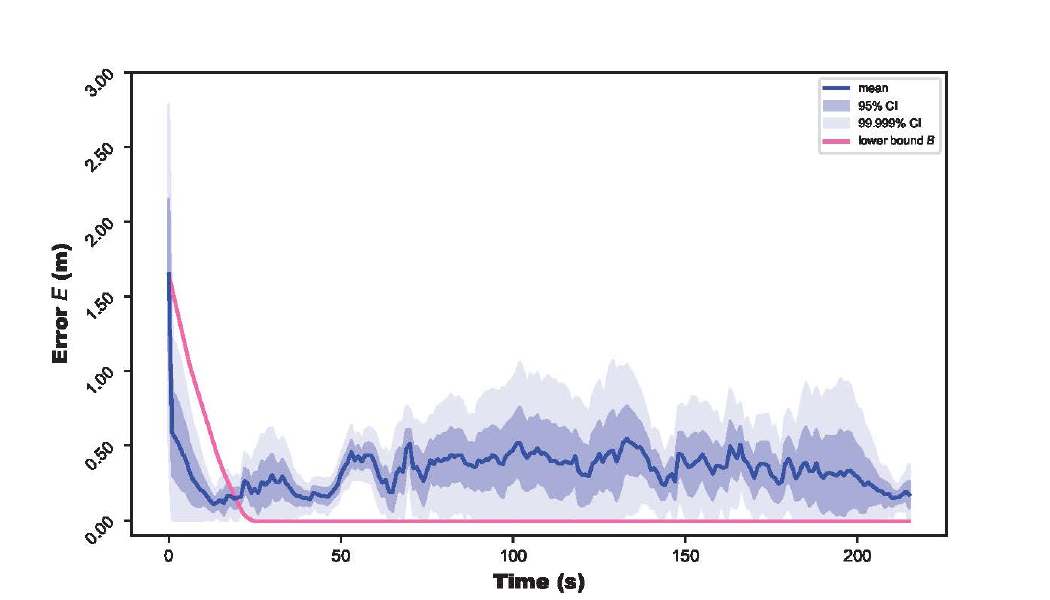}\\
\includegraphics[trim=120 60 120 80,clip,width=0.38\textwidth]{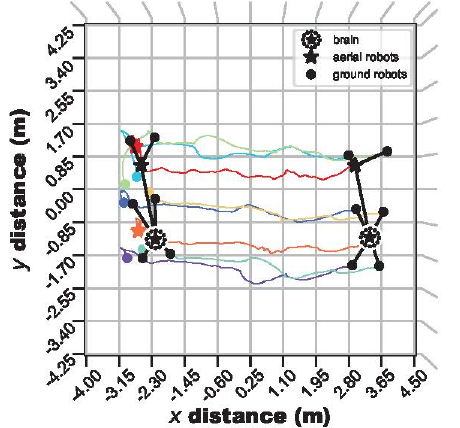}
\includegraphics[trim=20 0 40 20,clip,width=0.59\textwidth]{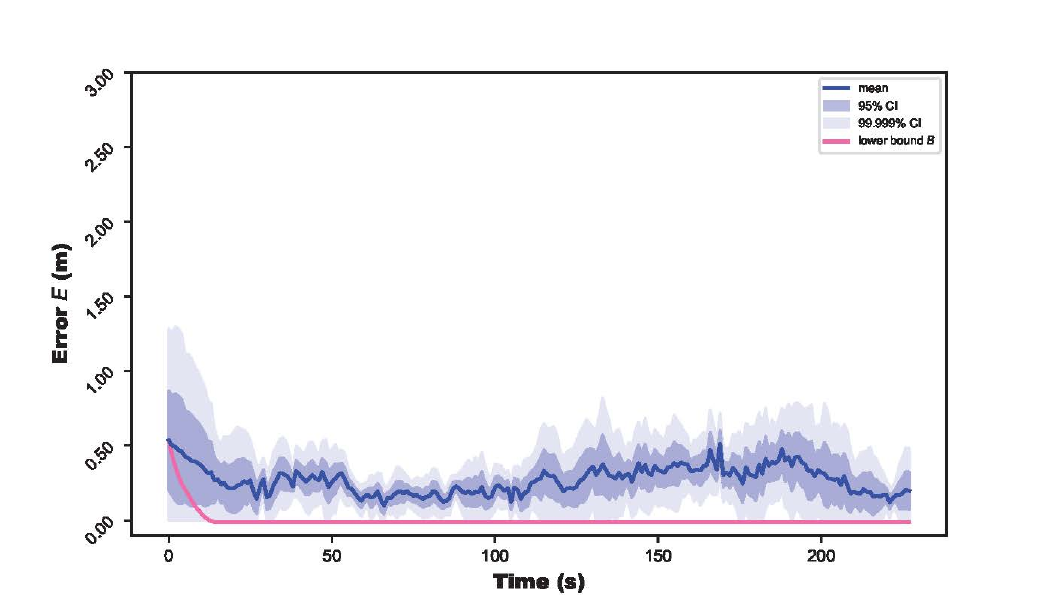}\\
\caption{{\it (cont'd)} {\bf Balancing global and local goals, with smaller, denser obstacles: Real robot trials.} Five trials with real robots were conducted, each with eight robots.}
\label{fig:mission2-variant1-hardware}
\end{figure}

\vspace{7mm}
\noindent
{\it (Section continued on next page.)}

\clearpage
\subsubsection*{Variant: Smaller, denser obstacles}
\begin{figure}[h!]
\centering
\includegraphics[trim=120 60 120 80,clip,width=0.38\textwidth]{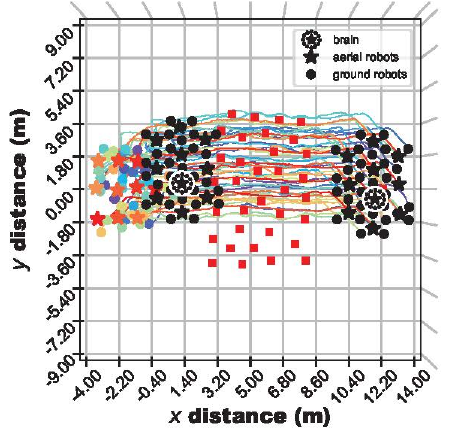}
\includegraphics[trim=20 0 40 20,clip,width=0.59\textwidth]{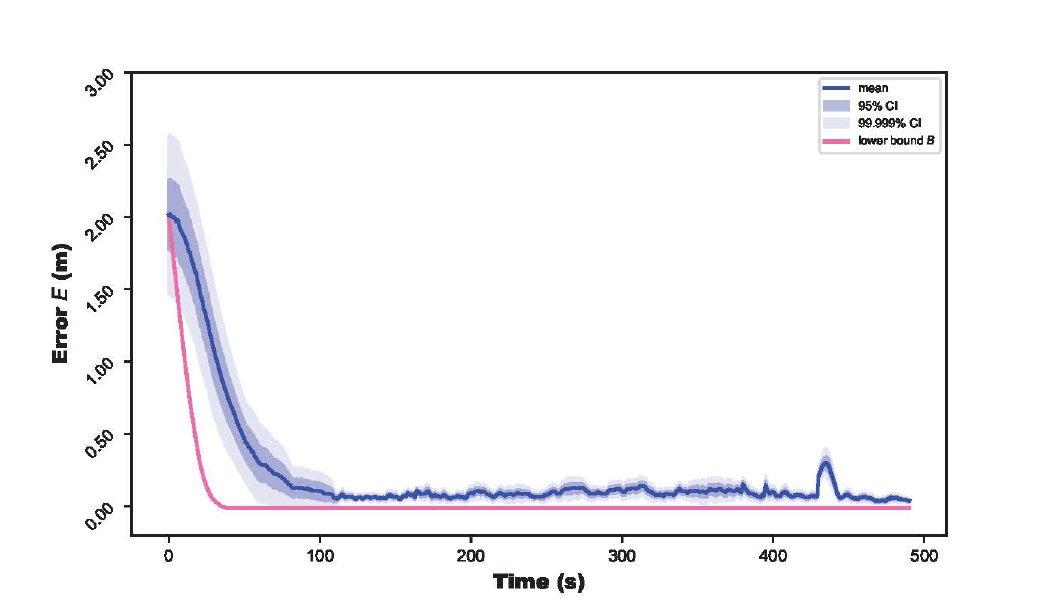}\\
\includegraphics[trim=120 60 120 80,clip,width=0.38\textwidth]{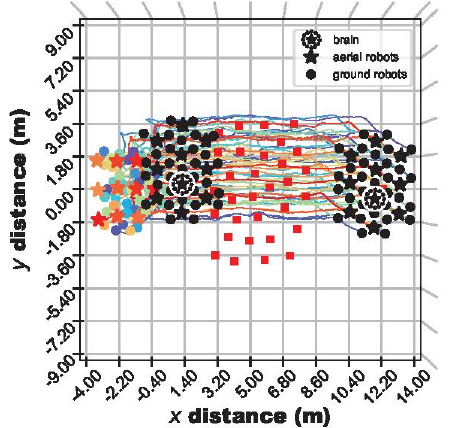}
\includegraphics[trim=20 0 40 20,clip,width=0.59\textwidth]{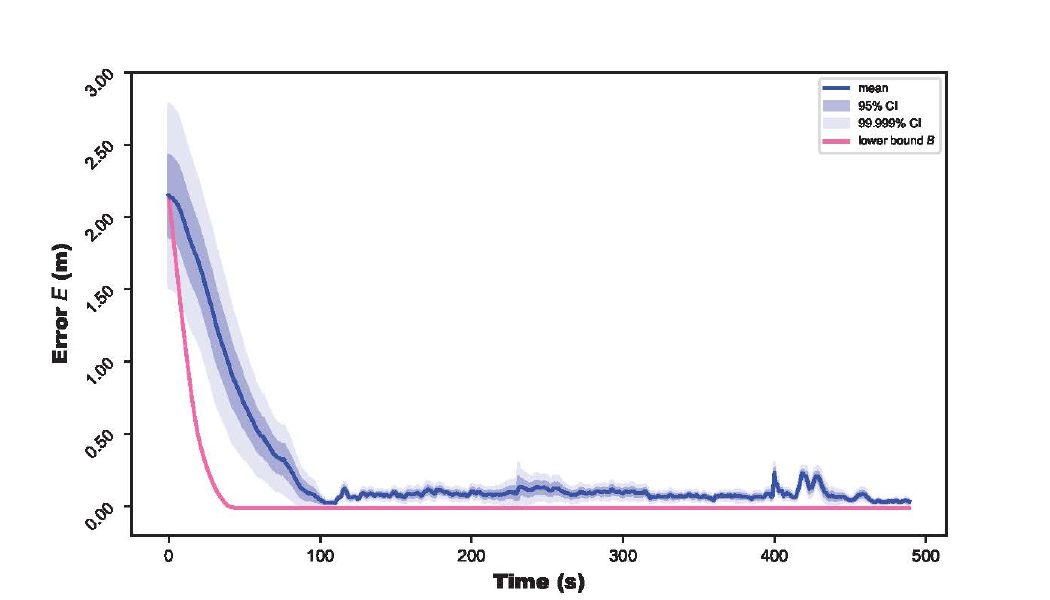}\\
\includegraphics[trim=120 60 120 80,clip,width=0.38\textwidth]{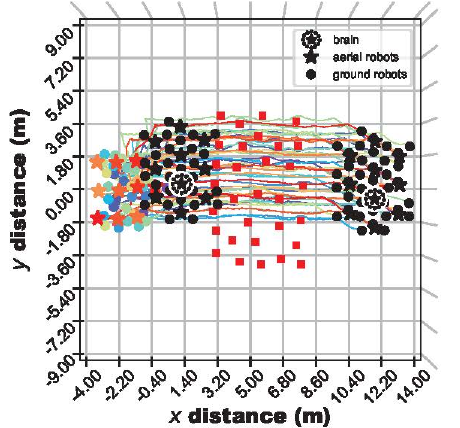}
\includegraphics[trim=20 0 40 20,clip,width=0.59\textwidth]{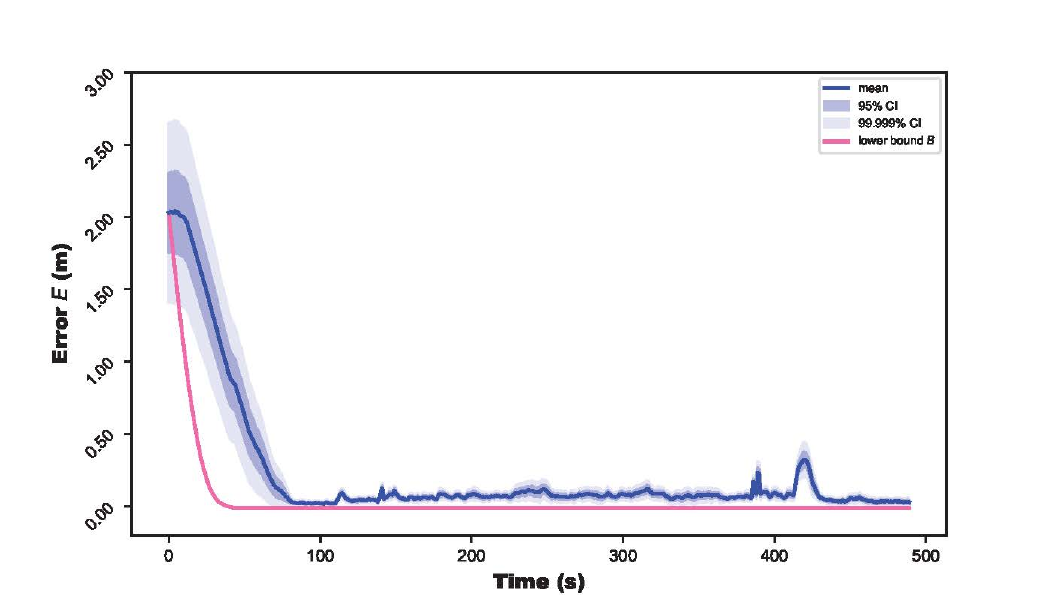}\\
\caption{{\bf Balancing global and local goals, with smaller, denser obstacles: Example simulation trials.} 50 trials were conducted in simulation, each with 50 robots.}
\label{fig:mission2-variant1-simulation}
\end{figure}

\clearpage
\subsubsection*{Variant: Larger, less dense obstacles}
\begin{figure}[h!]
\centering
\includegraphics[trim=120 60 120 80,clip,width=0.38\textwidth]{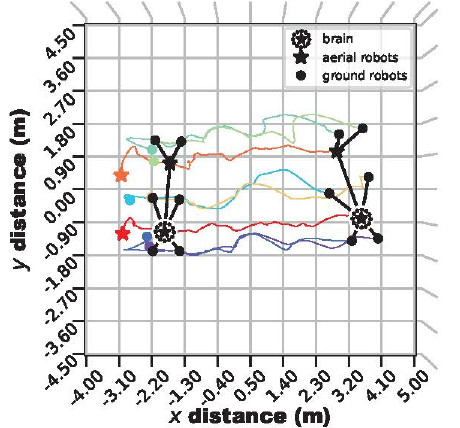}
\includegraphics[trim=20 0 40 20,clip,width=0.59\textwidth]{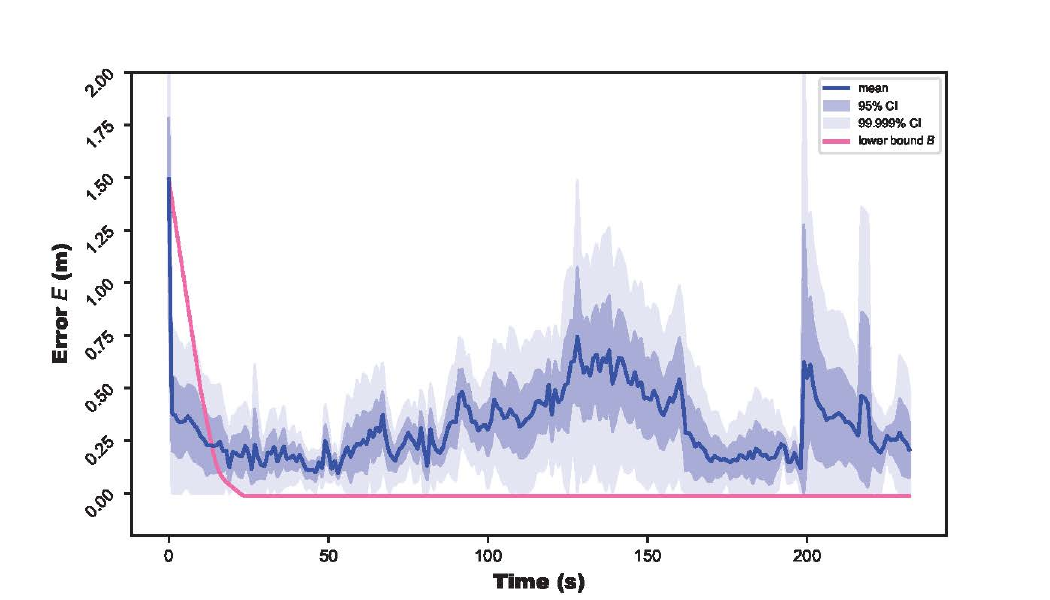}\\
\includegraphics[trim=120 60 120 80,clip,width=0.38\textwidth]{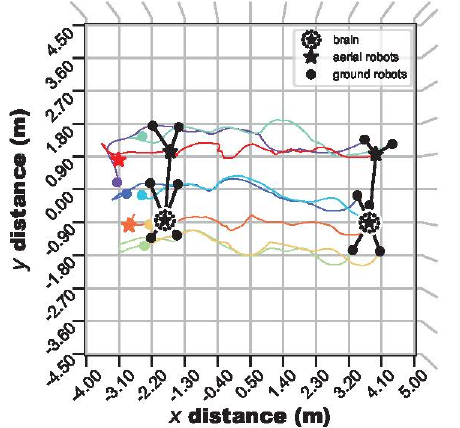}
\includegraphics[trim=20 0 40 20,clip,width=0.59\textwidth]{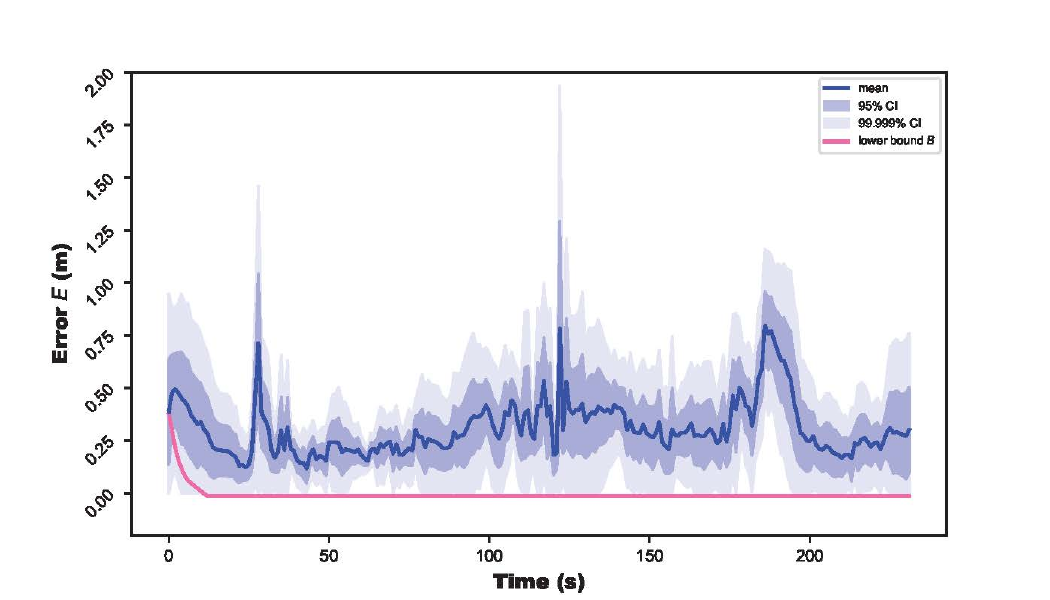}\\
\includegraphics[trim=120 60 120 80,clip,width=0.38\textwidth]{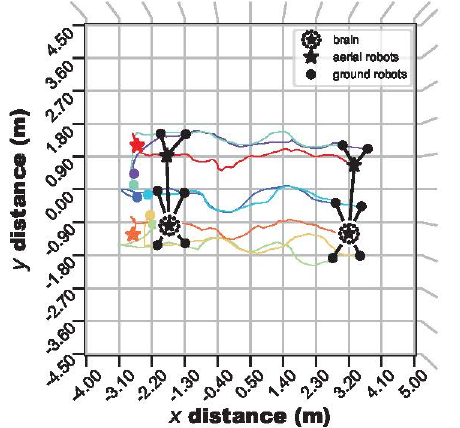}
\includegraphics[trim=20 0 40 20,clip,width=0.59\textwidth]{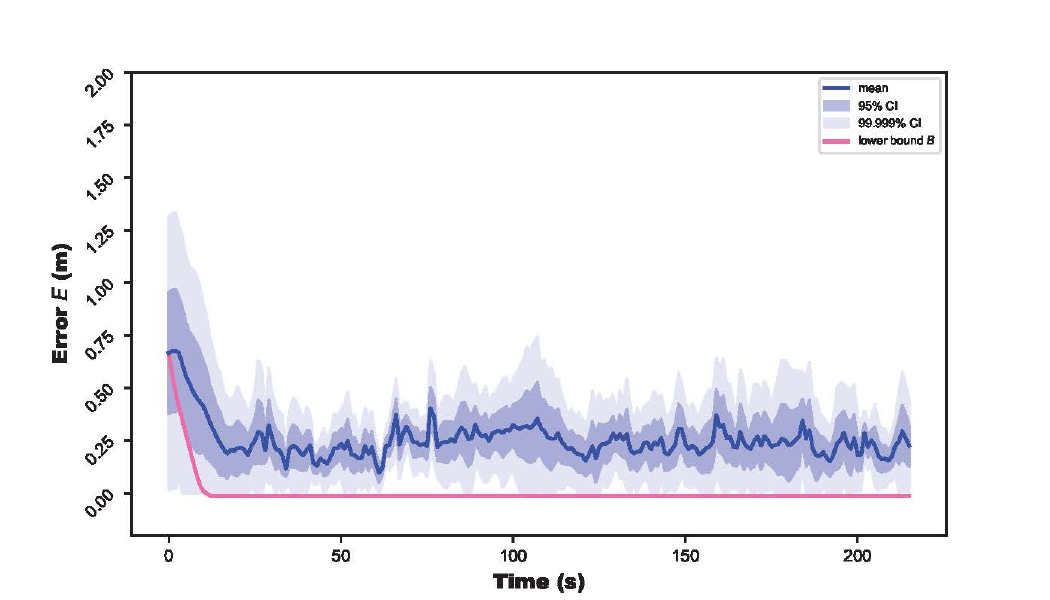}\\
\caption{{\bf Balancing global and local goals, with larger, less dense obstacles: Real robot trials.} Five trials with real robots were conducted, each with eight robots {\it (figure continued on next page)}.}
\label{fig:mission2-variant2-hardware}
\end{figure}

\clearpage

\begin{figure}[h!]
\ContinuedFloat
\centering
\includegraphics[trim=120 60 120 80,clip,width=0.38\textwidth]{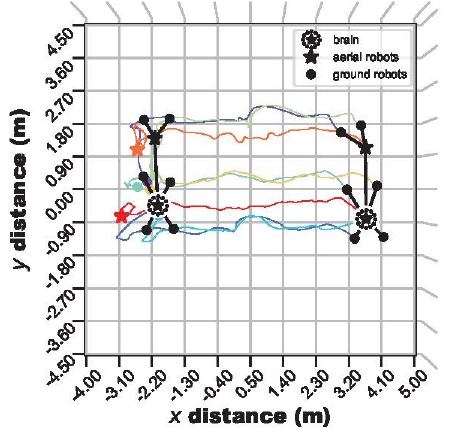}
\includegraphics[trim=20 0 40 20,clip,width=0.59\textwidth]{Mission2_Obstacle_avoidance_Variant1_Smaller_denser_obstacles_Real_robot_Hardware_run4_ErrorLog.pdf}\\
\includegraphics[trim=120 60 120 80,clip,width=0.38\textwidth]{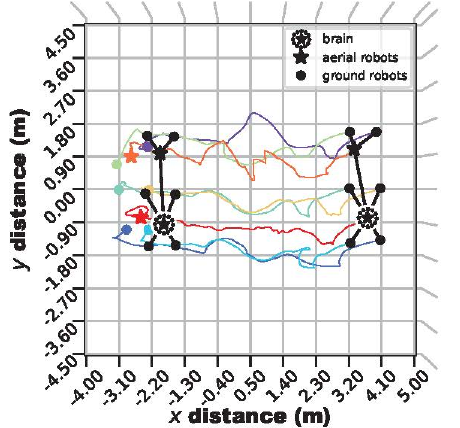}
\includegraphics[trim=20 0 40 20,clip,width=0.59\textwidth]{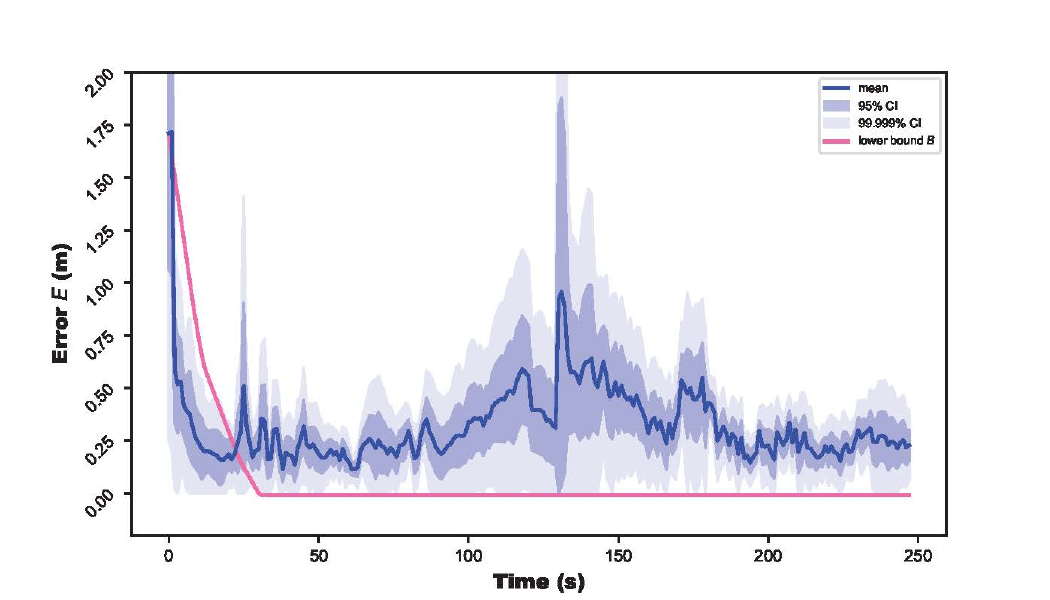}\\
\caption{{\it (cont'd)} {\bf Balancing global and local goals, with larger, less dense obstacles: Real robot trials.} Five trials with real robots were conducted, each with eight robots.}
\label{fig:mission2-variant2-hardware}
\end{figure}

\vspace{7mm}
\noindent
{\it (Section continued on next page.)}

\clearpage
\subsubsection*{Variant: Larger, less dense obstacles}
\begin{figure}[h!]
\centering
\includegraphics[trim=120 60 120 80,clip,width=0.38\textwidth]{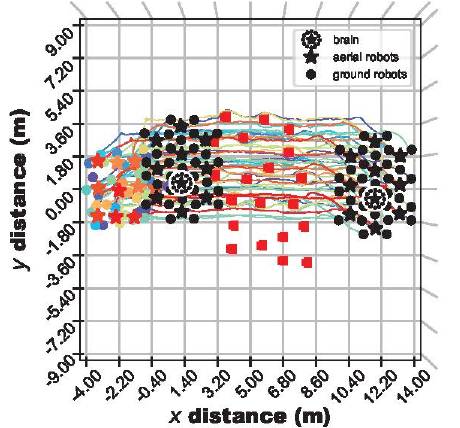}
\includegraphics[trim=20 0 40 20,clip,width=0.59\textwidth]{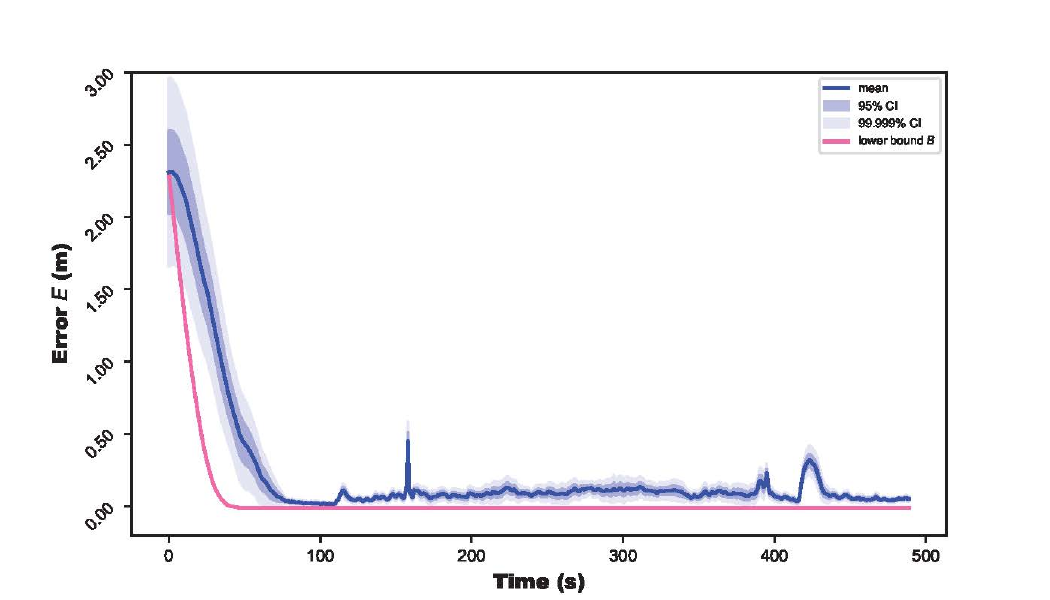}\\
\includegraphics[trim=120 60 120 80,clip,width=0.38\textwidth]{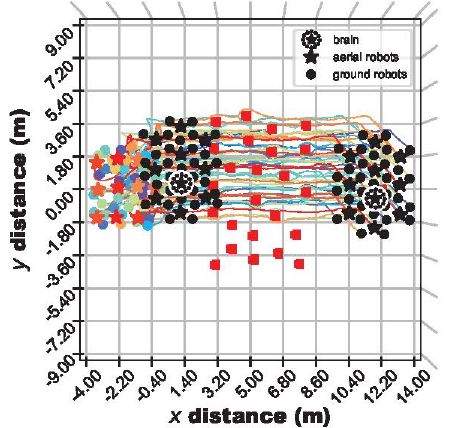}
\includegraphics[trim=20 0 40 20,clip,width=0.59\textwidth]{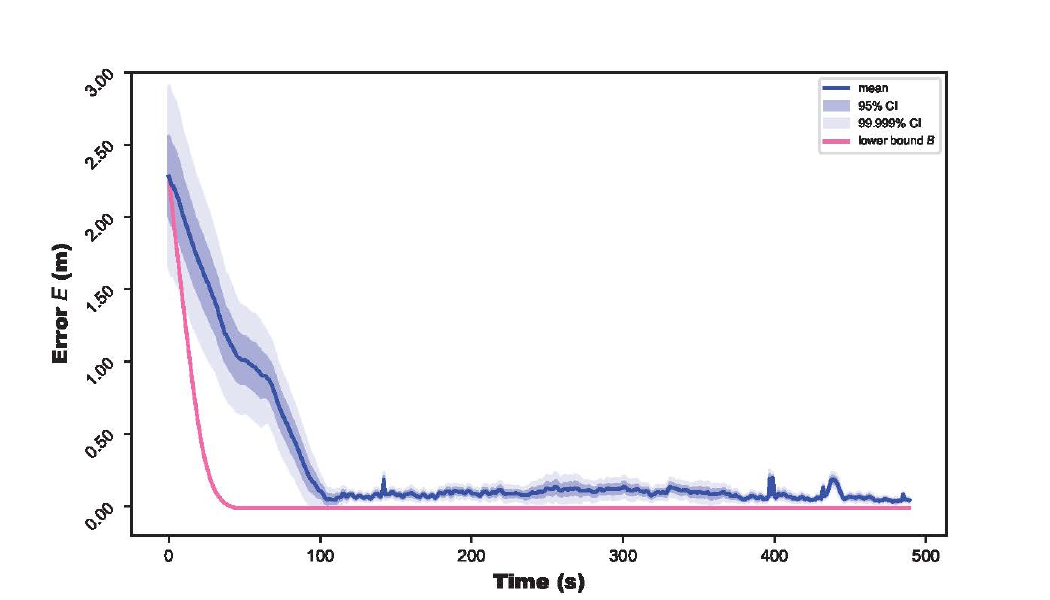}\\
\includegraphics[trim=120 60 120 80,clip,width=0.38\textwidth]{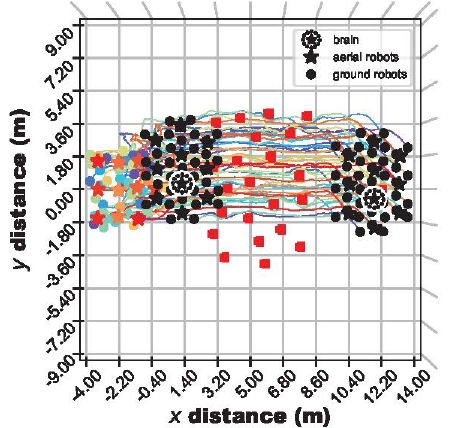}
\includegraphics[trim=20 0 40 20,clip,width=0.59\textwidth]{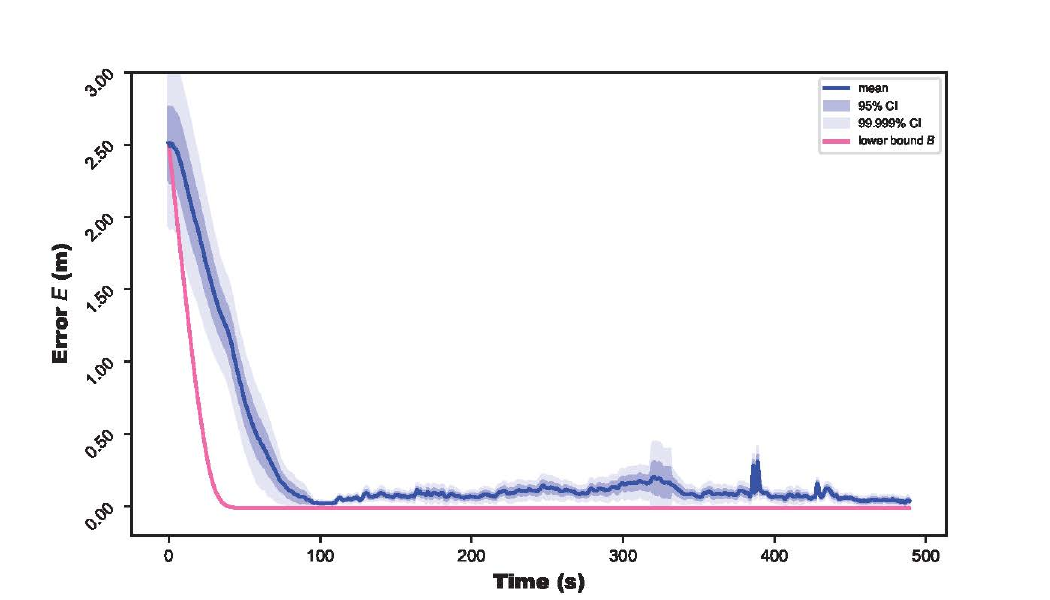}\\
\caption{{\bf Balancing global and local goals, with larger, less dense obstacles: Example simulation trials.} 50 trials were conducted in simulation, each with 50 robots.}
\label{fig:mission2-variant2-simulation}
\end{figure}

\clearpage
\subsection*{Mission: Collective sensing and actuation (see Sec.~2.1.3 in the main paper)}
\rhead{Mission: Collective sensing and actuation}

This mission includes only one variant, run in experiments with real robots and in simulation.

\begin{figure}[h!]
\centering
\subfigure[]{
\includegraphics[trim=90 60 90 75,clip,width=0.39\textwidth]{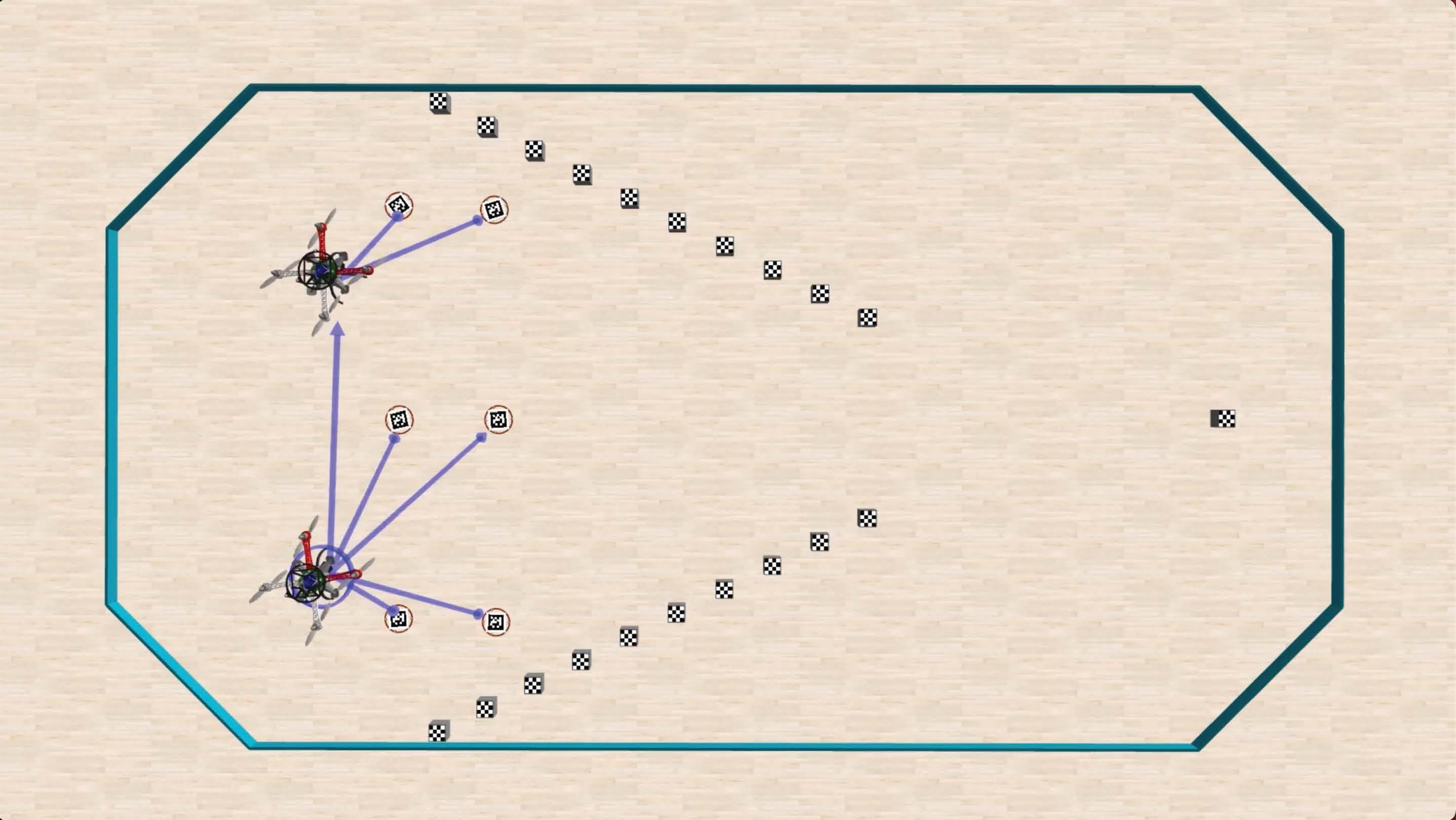}}
\subfigure[]{
\includegraphics[trim=90 60 90 75,clip,width=0.39\textwidth]{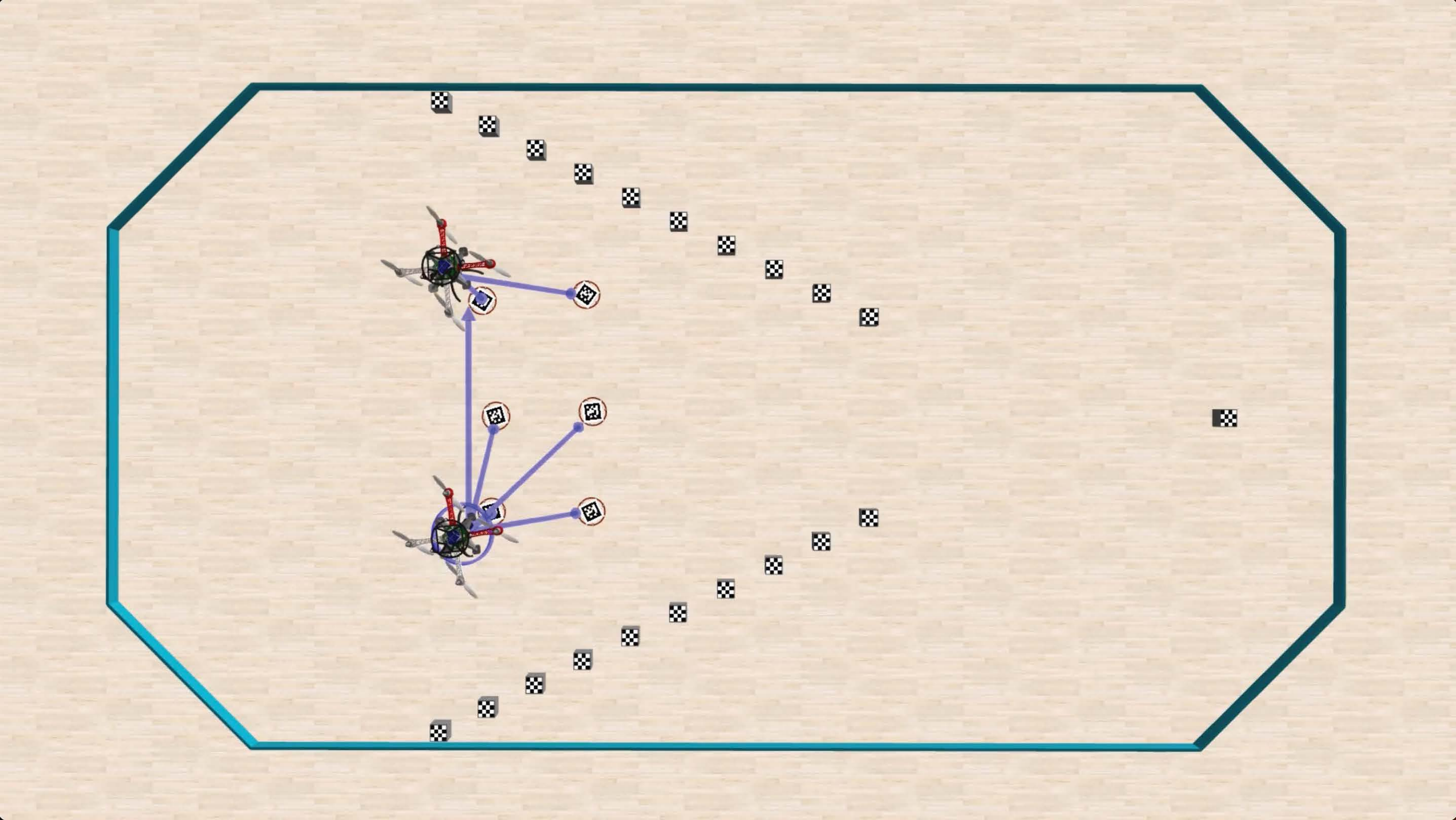}}\\
\vspace{-2mm}
\subfigure[]{
\includegraphics[trim=90 60 90 75,clip,width=0.39\textwidth]{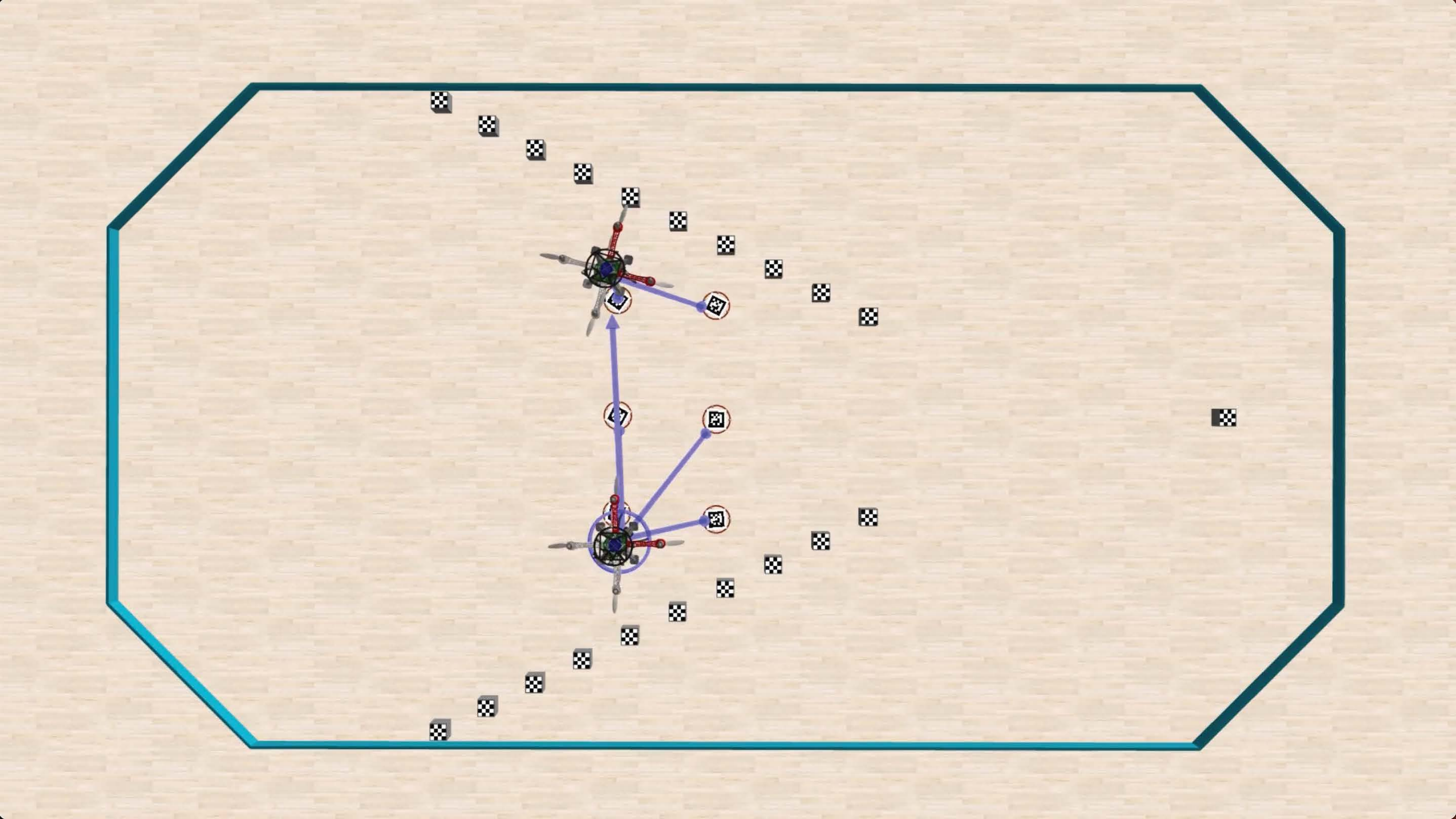}}
\subfigure[]{
\includegraphics[trim=90 60 90 75,clip,width=0.39\textwidth]{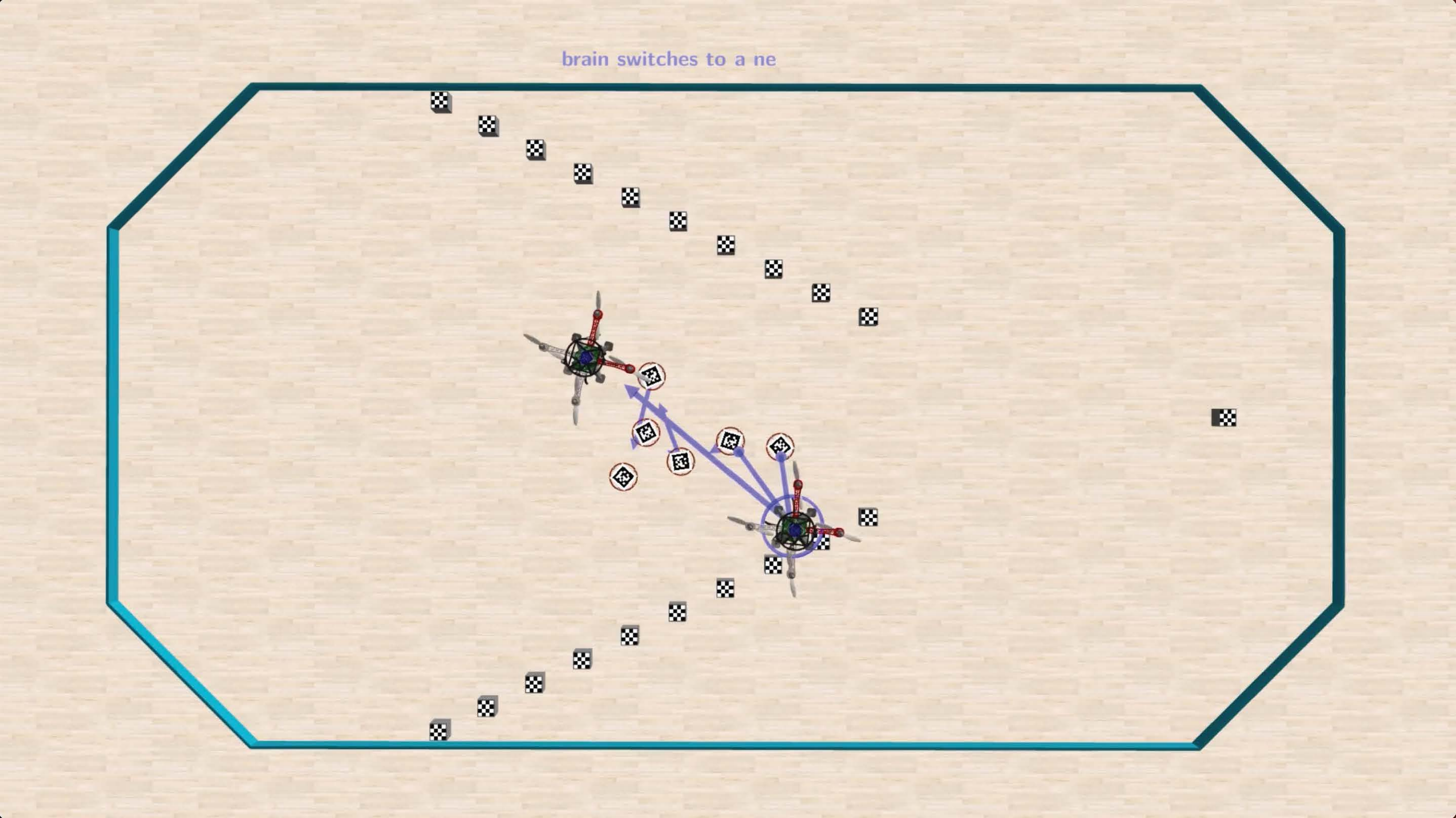}}\\
\vspace{-2mm}
\subfigure[]{
\includegraphics[trim=90 60 90 75,clip,width=0.39\textwidth]{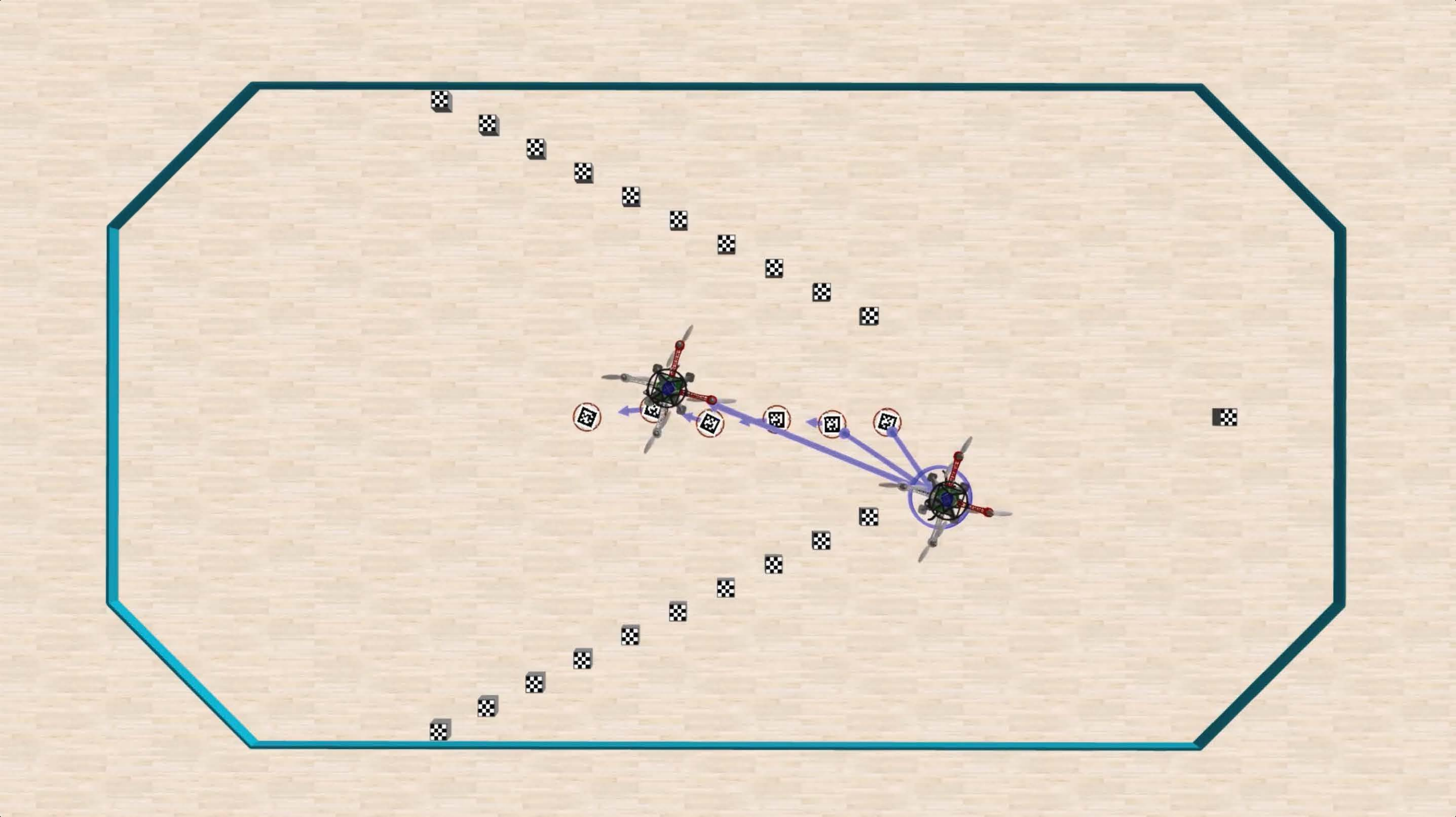}}
\subfigure[]{
\includegraphics[trim=90 60 90 75,clip,width=0.39\textwidth]{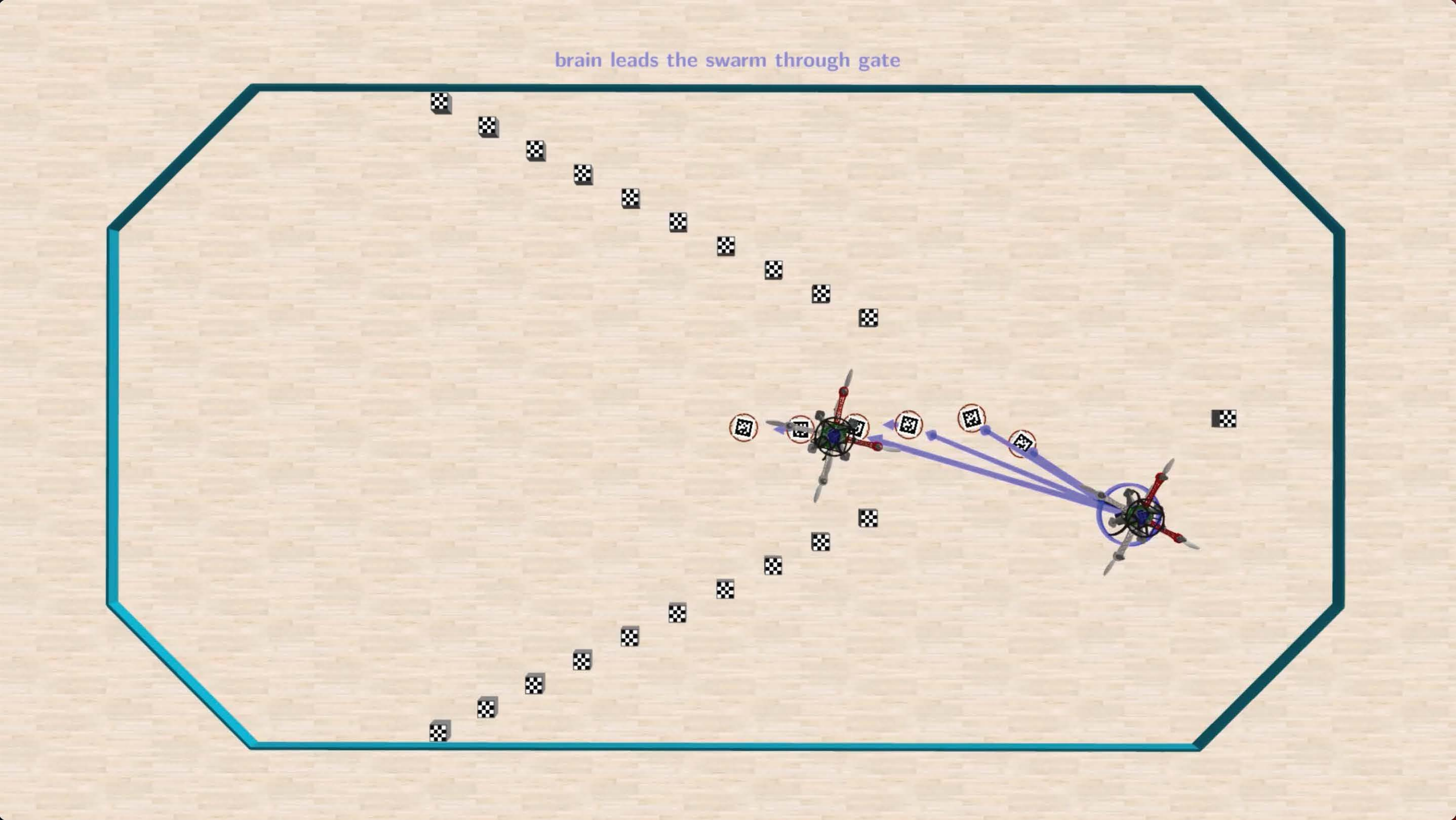}}\\
\vspace{-2mm}
\subfigure[]{
\includegraphics[trim=90 60 90 75,clip,width=0.39\textwidth]{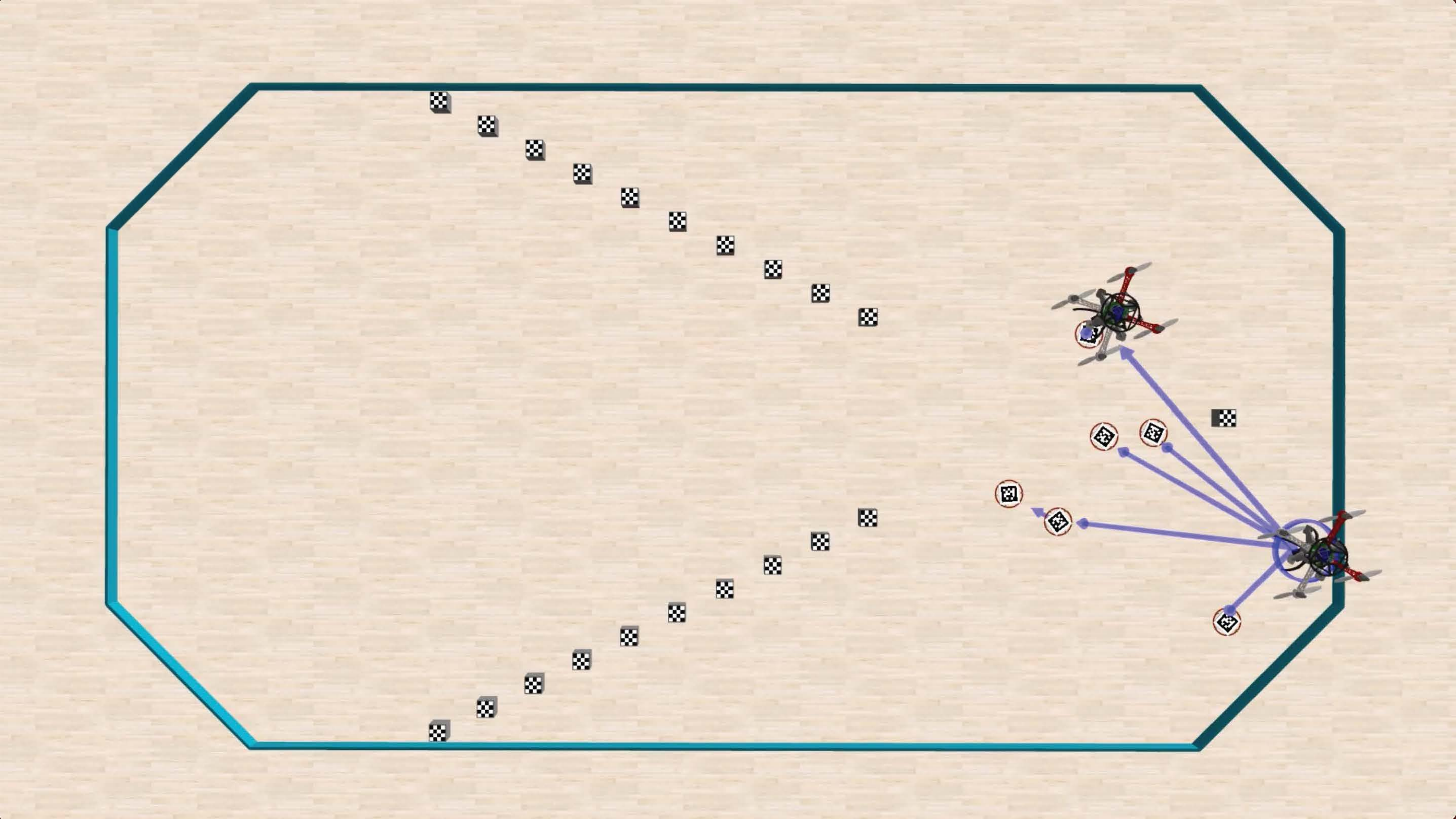}}
\subfigure[]{
\includegraphics[trim=90 60 90 75,clip,width=0.39\textwidth]{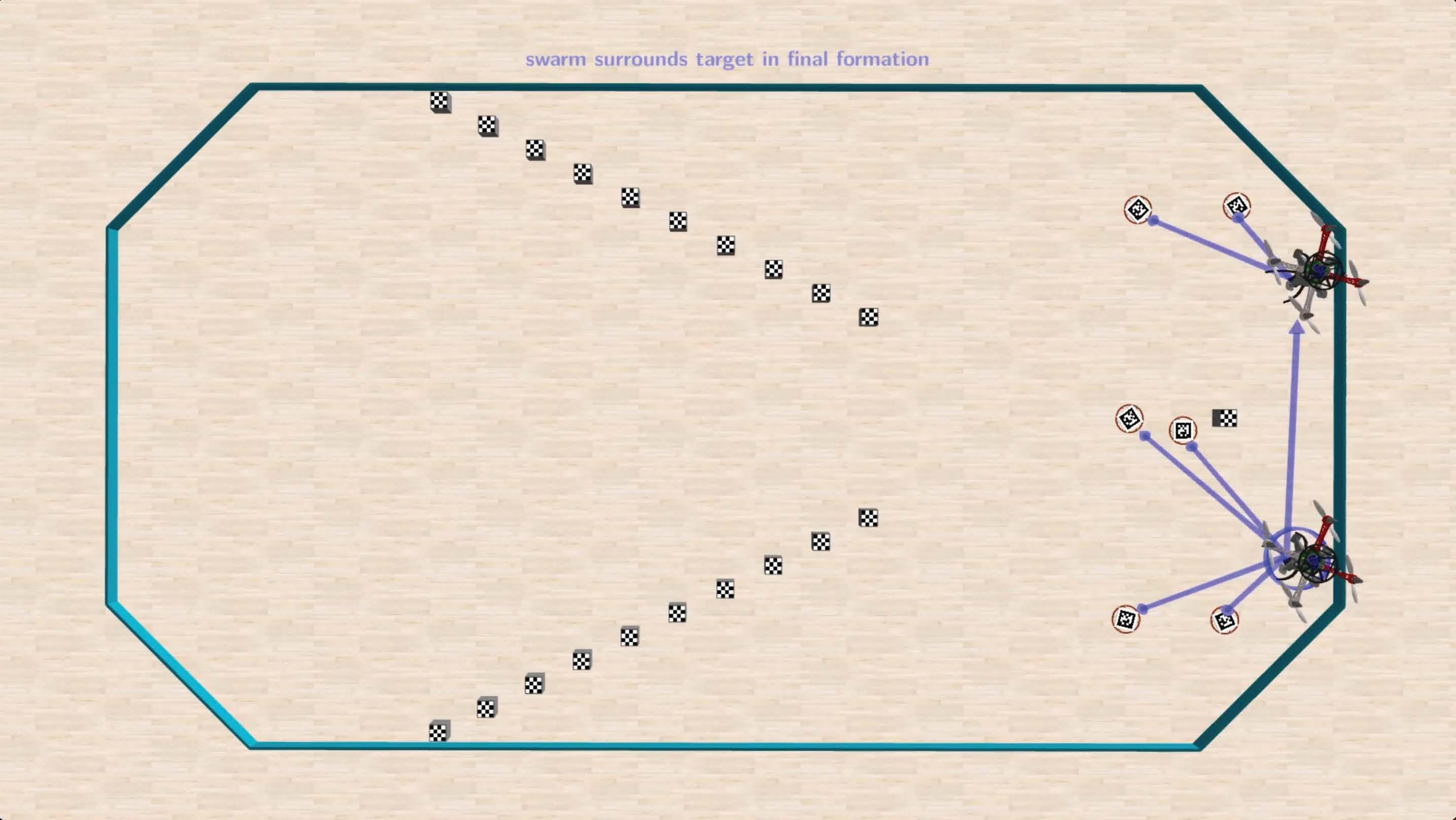}}\\
\vspace{-2mm}
\caption{{\bf Collective sensing and actuation: Key frames.} (a) Robots that have established their target SoNS begin moving across an environment, searching for an object that marks the final destination. (b,c) Robots sense obstacles that form the walls of a passage and start to adapt by narrowing the shape of the target SoNS. (d,e) Robots sense the walls narrowing and then adapt further by reconfiguring into a different target SoNS that has an even narrower shape. (f,g) Robots sense that there are no longer walls constraining them and start to return to their original target SoNS. (h) Robots sense the final destination object and return to their original target SoNS, and the mission is complete.}
\label{fig:mission3-keyframes}
\end{figure}

\clearpage
\subsubsection*{Mission: Collective sensing and actuation}

\begin{figure}[h!]
\centering
\includegraphics[trim=120 60 120 80,clip,width=0.38\textwidth]{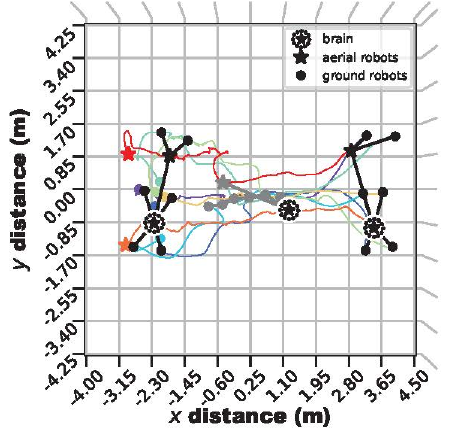}
\includegraphics[trim=20 0 40 20,clip,width=0.59\textwidth]{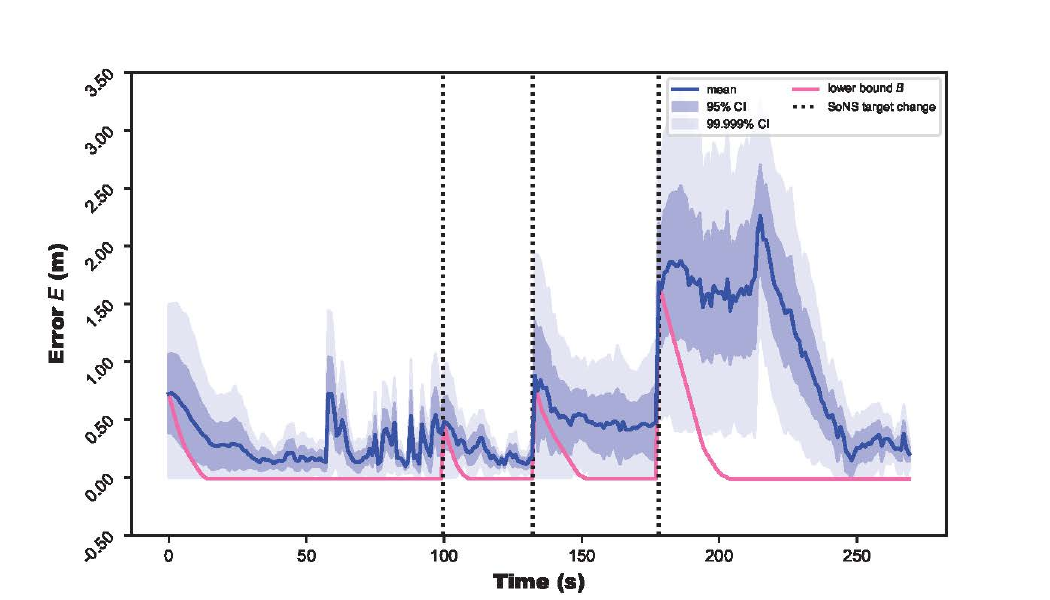}\\
\includegraphics[trim=120 60 120 80,clip,width=0.38\textwidth]{Mission3_Collective_sensing_actuation_Real_robot_Hardware_run1_TrackLog.pdf}
\includegraphics[trim=20 0 40 20,clip,width=0.59\textwidth]{Mission3_Collective_sensing_actuation_Real_robot_Hardware_run1_ErrorLog.pdf}\\
\includegraphics[trim=120 60 120 80,clip,width=0.38\textwidth]{Mission3_Collective_sensing_actuation_Real_robot_Hardware_run1_TrackLog.pdf}
\includegraphics[trim=20 0 40 20,clip,width=0.59\textwidth]{Mission3_Collective_sensing_actuation_Real_robot_Hardware_run1_ErrorLog.pdf}\\
\caption{{\bf Collective sensing and actuation: Real robot trials.} Five trials with real robots were conducted, each with eight robots {\it (figure continued on next page)}.}
\label{fig:mission3-hardware}
\end{figure}

\clearpage

\begin{figure}[h!]
\ContinuedFloat
\centering
\includegraphics[trim=120 60 120 80,clip,width=0.38\textwidth]{Mission3_Collective_sensing_actuation_Real_robot_Hardware_run1_TrackLog.pdf}
\includegraphics[trim=20 0 40 20,clip,width=0.59\textwidth]{Mission3_Collective_sensing_actuation_Real_robot_Hardware_run1_ErrorLog.pdf}\\
\includegraphics[trim=120 60 120 80,clip,width=0.38\textwidth]{Mission3_Collective_sensing_actuation_Real_robot_Hardware_run1_TrackLog.pdf}
\includegraphics[trim=20 0 40 20,clip,width=0.59\textwidth]{Mission3_Collective_sensing_actuation_Real_robot_Hardware_run1_ErrorLog.pdf}\\
\caption{{\it (cont'd)} {\bf Collective sensing and actuation: Real robot trials.} Five trials with real robots were conducted, each with eight robots.}
\label{fig:mission3-hardware}
\end{figure}

\vspace{7mm}
\noindent
{\it (Section continued on next page.)}

\clearpage

\subsubsection*{Mission: Collective sensing and actuation}

\begin{figure}[h!]
\centering
\includegraphics[trim=120 60 120 80,clip,width=0.38\textwidth]{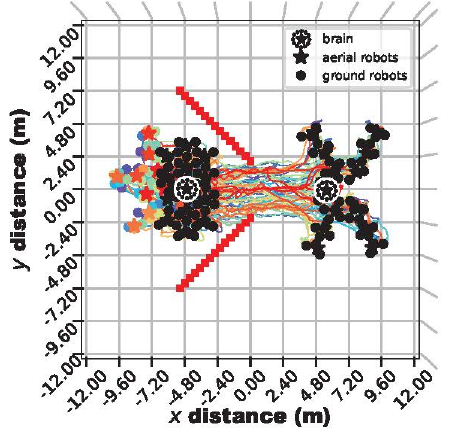}
\includegraphics[trim=20 0 40 20,clip,width=0.59\textwidth]{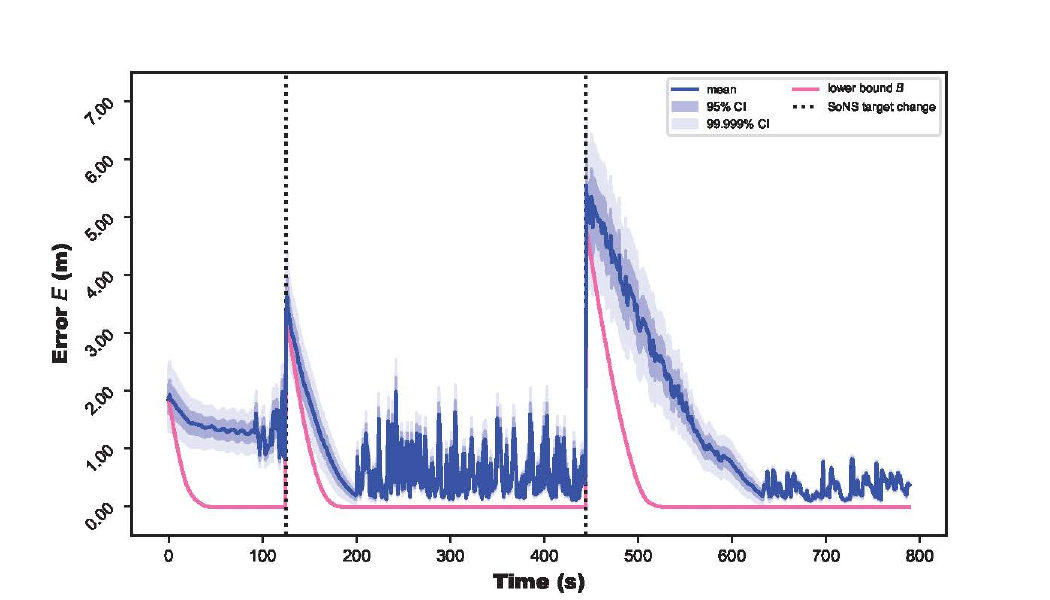}\\
\includegraphics[trim=120 60 120 80,clip,width=0.38\textwidth]{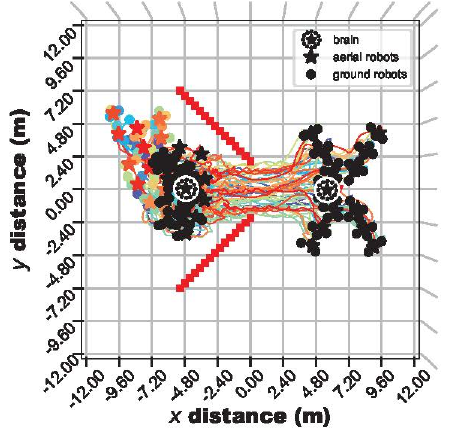}
\includegraphics[trim=20 0 40 20,clip,width=0.59\textwidth]{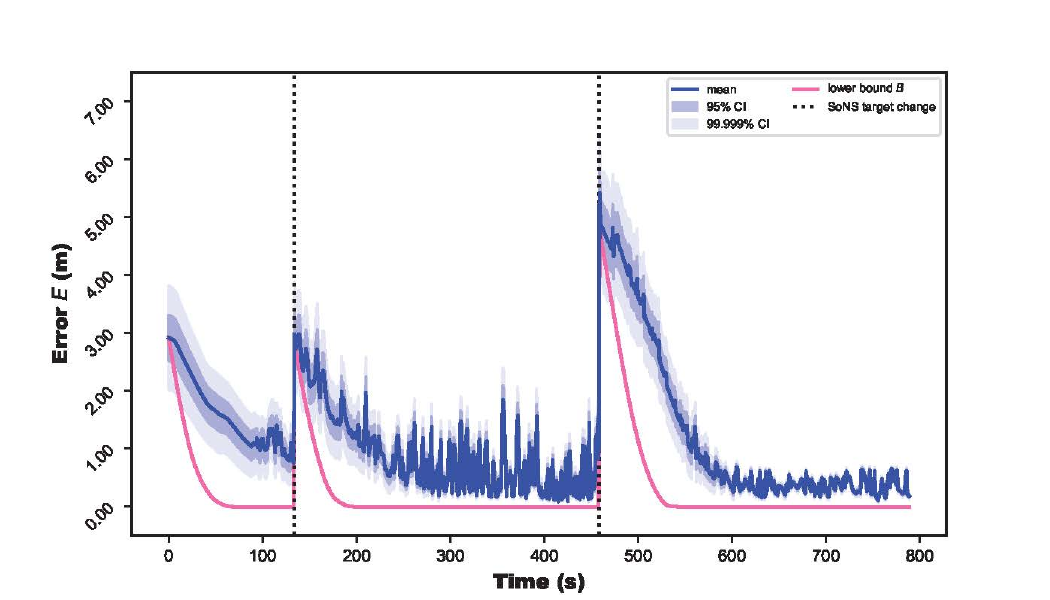}\\
\includegraphics[trim=120 60 120 80,clip,width=0.38\textwidth]{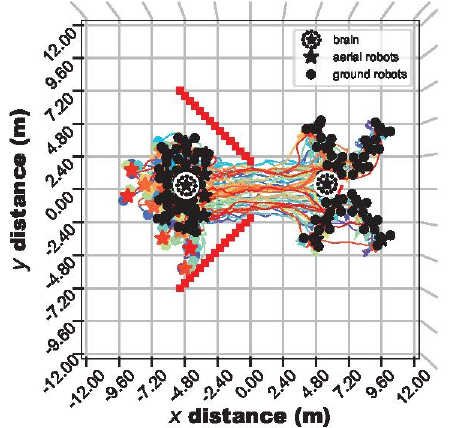}
\includegraphics[trim=20 0 40 20,clip,width=0.59\textwidth]{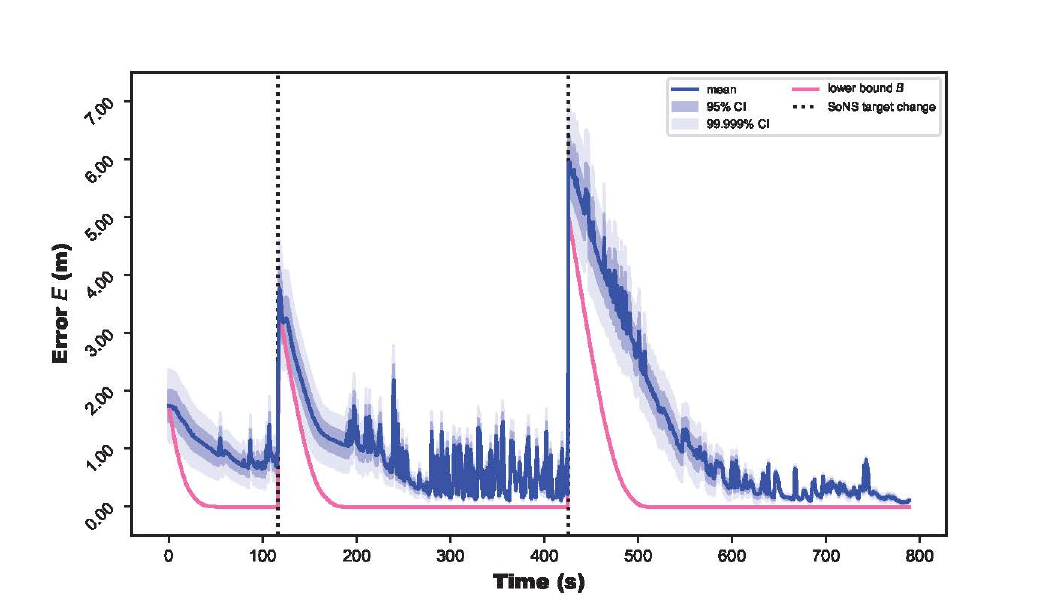}\\
\caption{{\bf Collective sensing and actuation: Example simulation trials.} 50 trials were conducted in simulation, each with 50 robots.}
\label{fig:mission3-simulation}
\end{figure}

\clearpage
\subsection*{Mission: Binary decision making (see Sec.~2.1.4 in the main paper)}
\rhead{Mission: Binary decision making}

This mission includes only one variant, run in experiments with real robots and in simulation.
Note that simulation trials for this mission are provided in the ``Scalability" subsection, below.

\begin{figure}[h!]
\centering
\subfigure[]{
\includegraphics[trim=90 60 90 75,clip,width=0.4\textwidth]{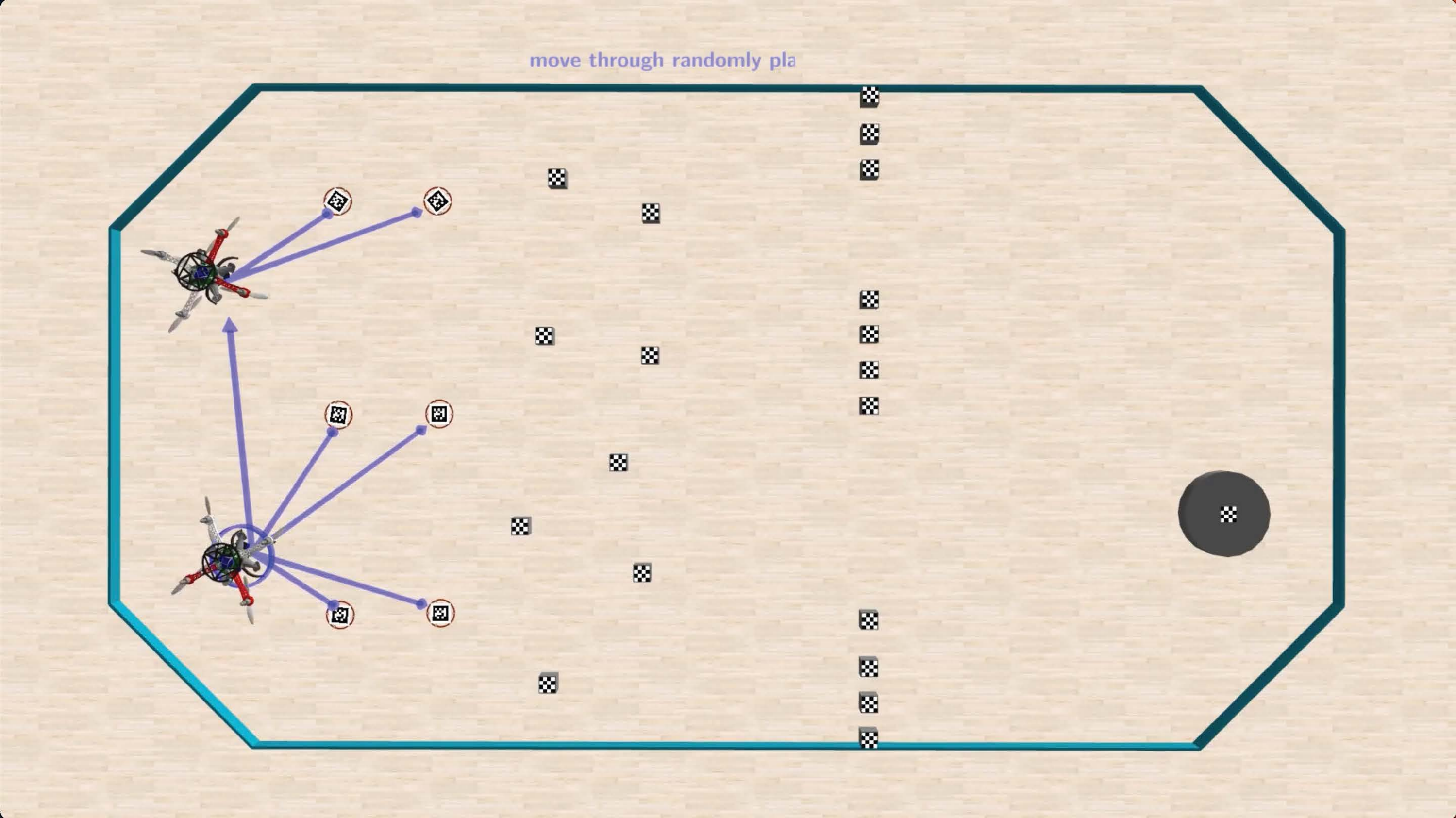}}
\subfigure[]{
\includegraphics[trim=90 60 90 75,clip,width=0.4\textwidth]{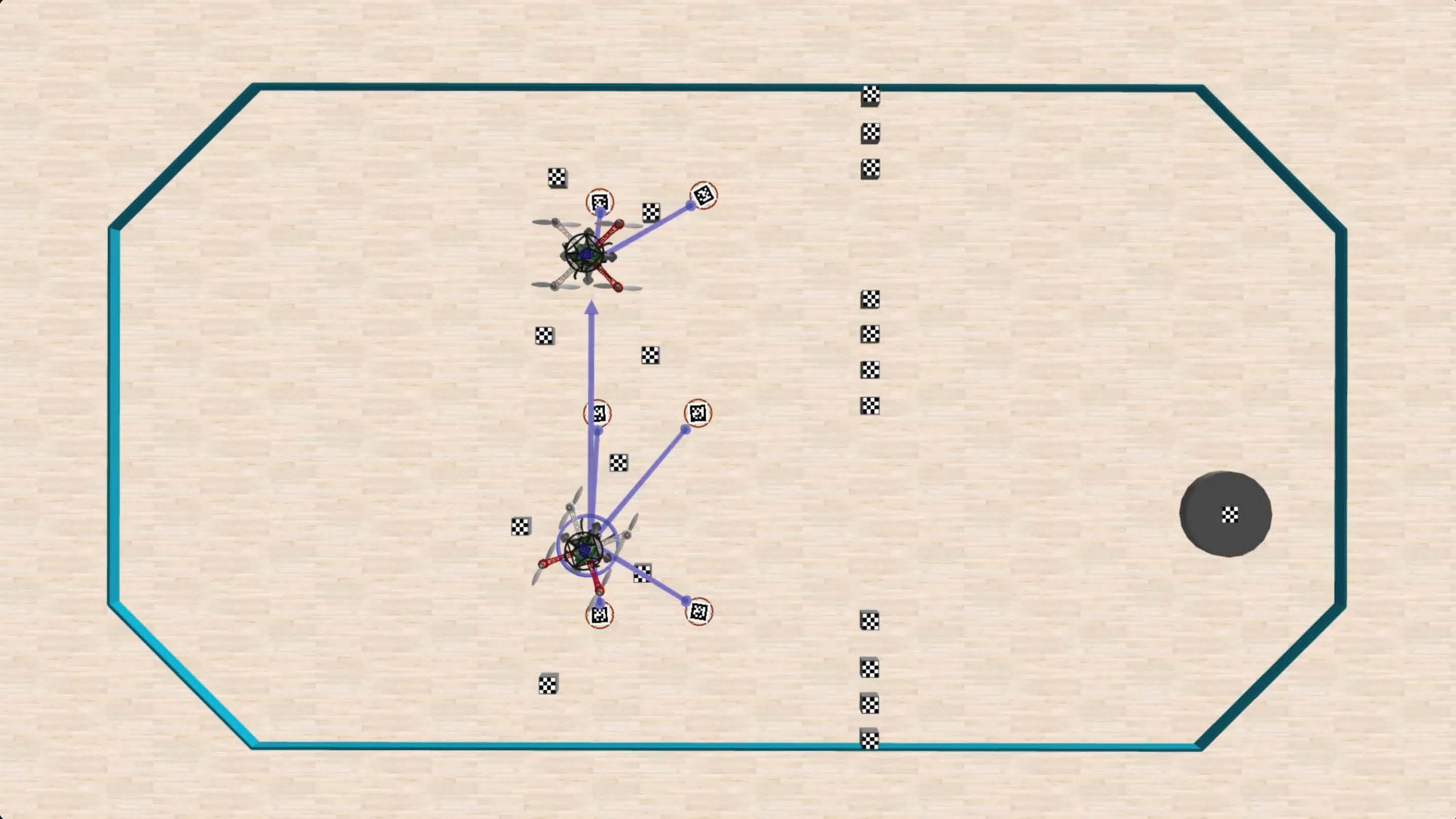}}\\
\vspace{-2mm}
\subfigure[]{
\includegraphics[trim=90 60 90 75,clip,width=0.4\textwidth]{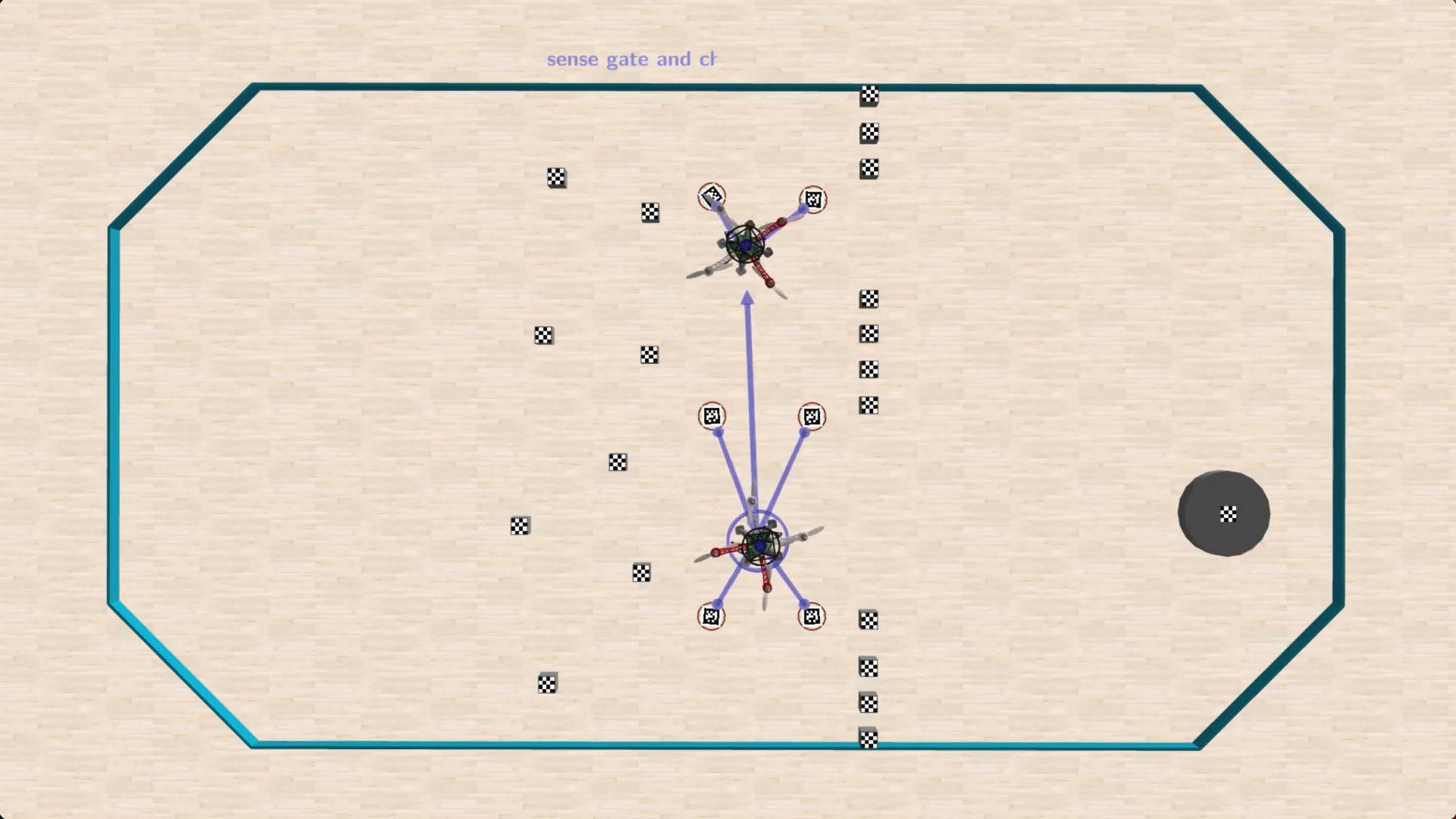}}
\subfigure[]{
\includegraphics[trim=90 60 90 75,clip,width=0.4\textwidth]{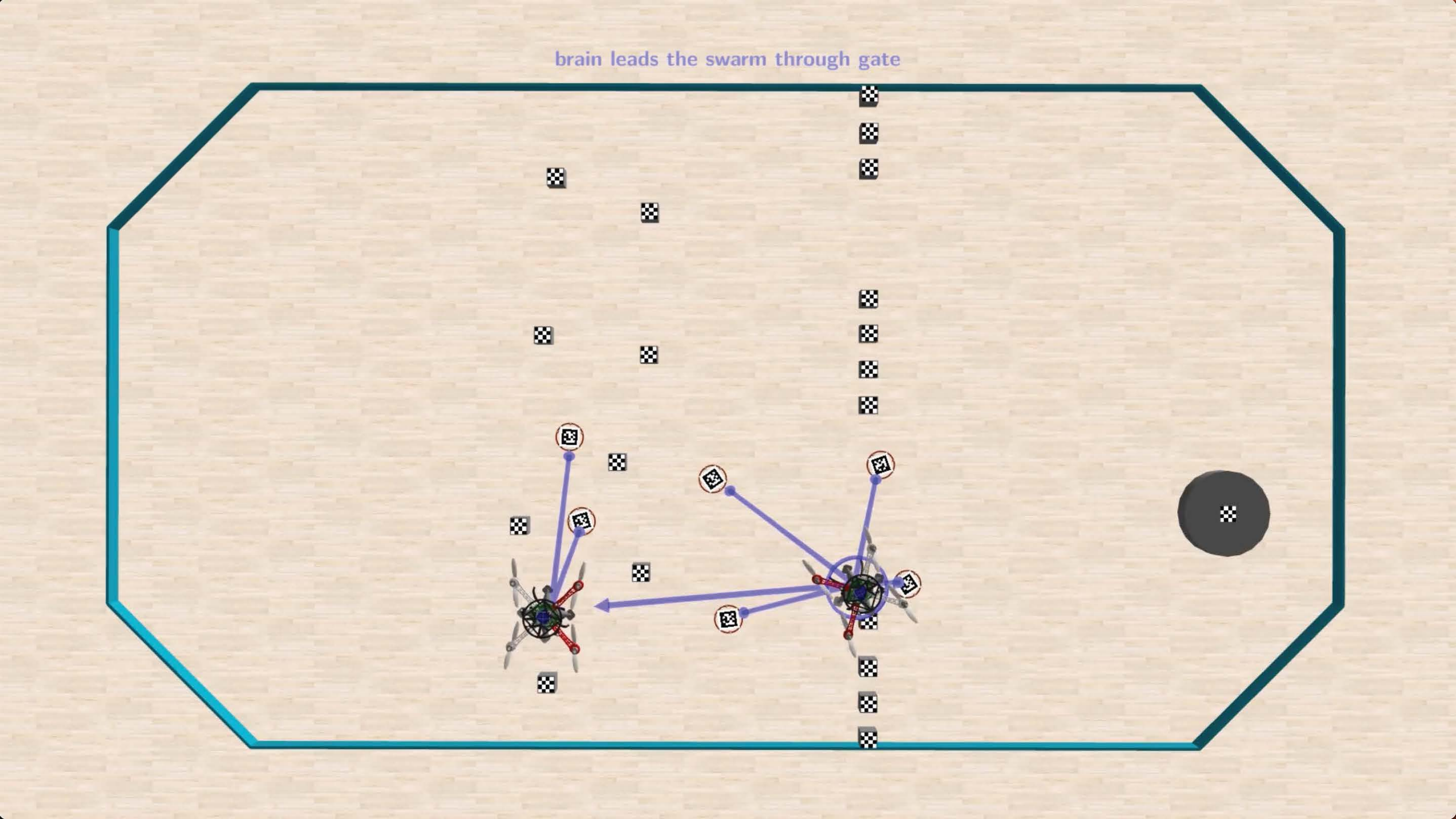}}\\
\vspace{-2mm}
\subfigure[]{
\includegraphics[trim=90 60 90 75,clip,width=0.4\textwidth]{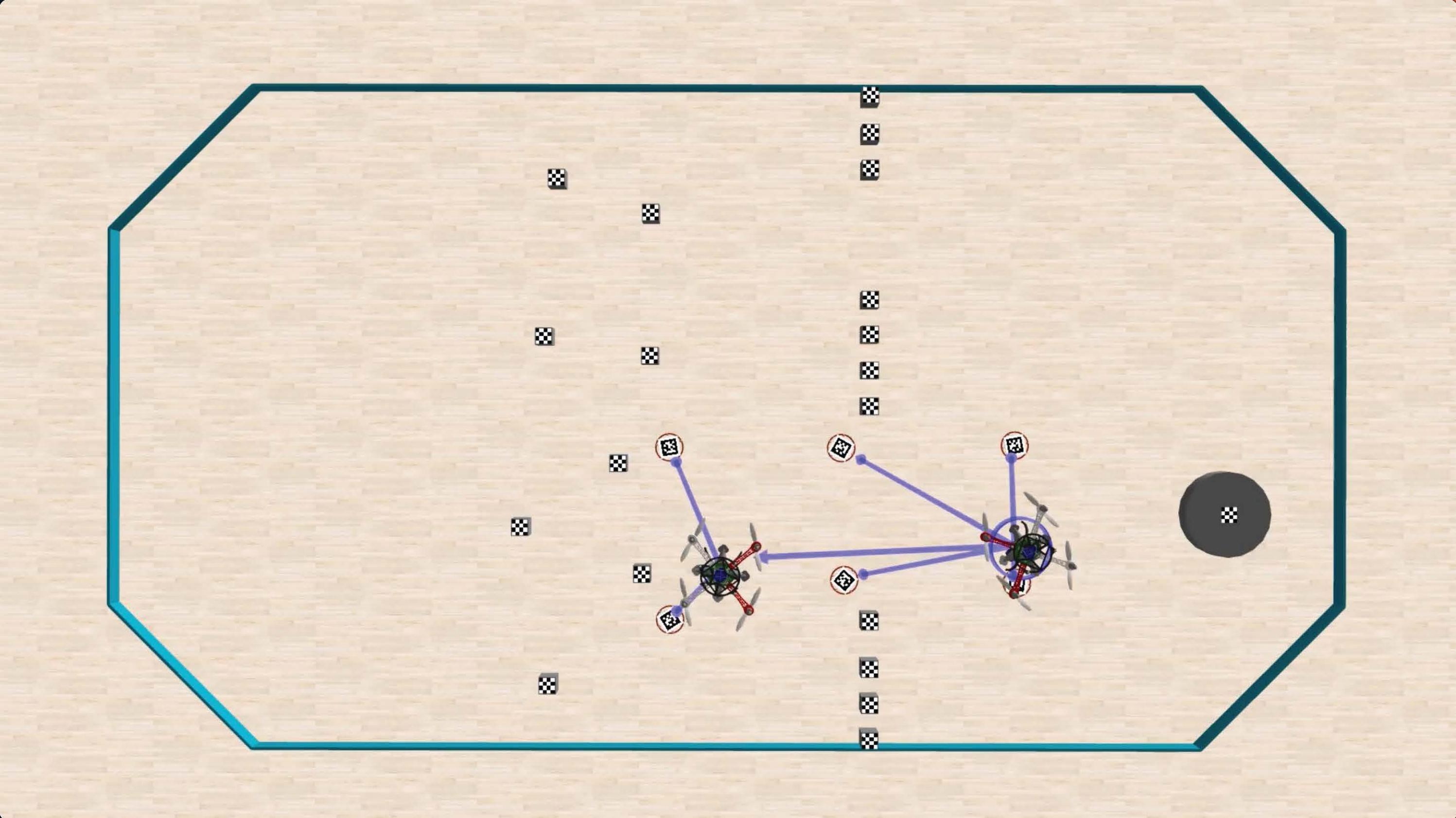}}
\subfigure[]{
\includegraphics[trim=90 60 90 75,clip,width=0.4\textwidth]{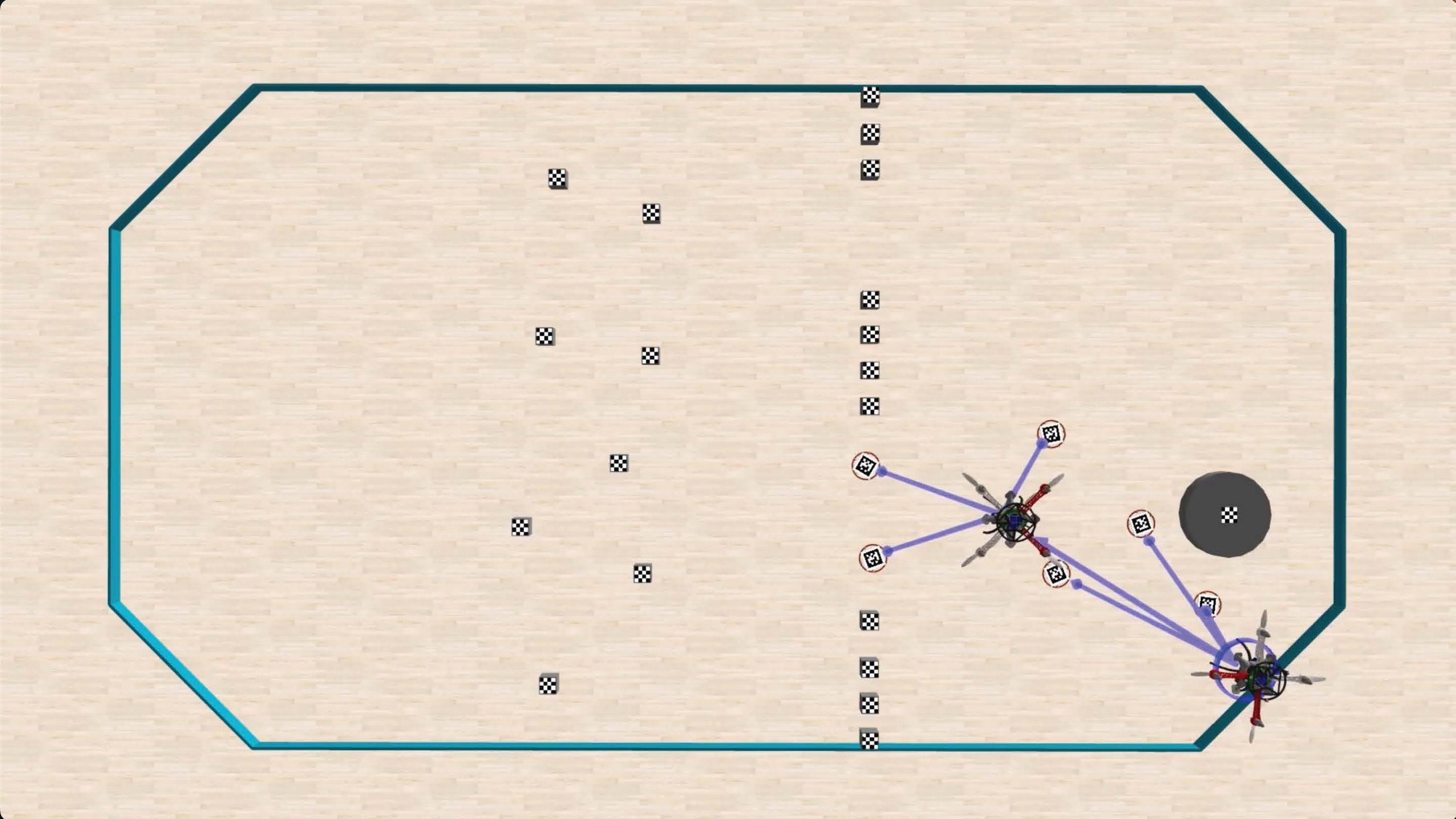}}\\
\vspace{-2mm}
\subfigure[]{
\includegraphics[trim=90 60 90 75,clip,width=0.4\textwidth]{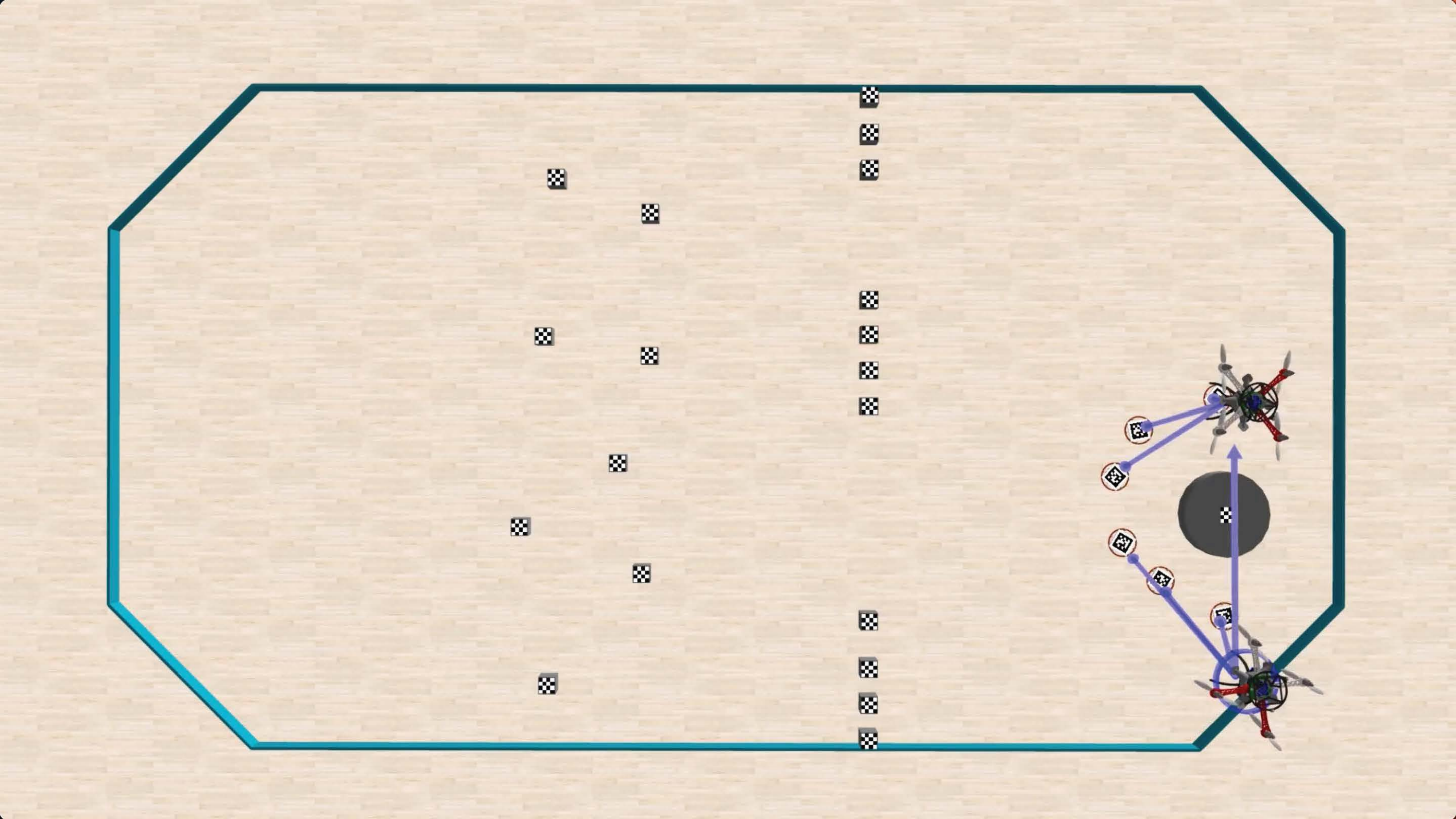}}\\
\vspace{-2mm}
\caption{{\bf Binary decision making: Key frames.} (a,b) Robots that have established their target SoNS begin moving across an environment, searching for the final destination object, and encounter a field of small obstacles. (c) After surpassing the obstacle field, robots sense a wall and collaboratively choose the largest opening. (d,e) The robots adjust the position and path of the SoNS to reposition in front of and then pass through the selected opening. (f,g) Robots sense that there are no longer walls constraining them, then sense the final destination object and change their target SoNS to surround it, and then the mission is complete.}
\label{fig:mission4-keyframes}
\end{figure}

\clearpage
\subsubsection*{Mission: Binary decision making}

\begin{figure}[ht!]
\centering
\includegraphics[trim=120 60 120 80,clip,width=0.38\textwidth]{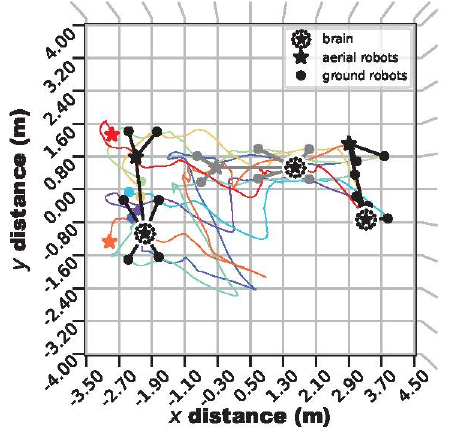}
\includegraphics[trim=20 0 40 20,clip,width=0.59\textwidth]{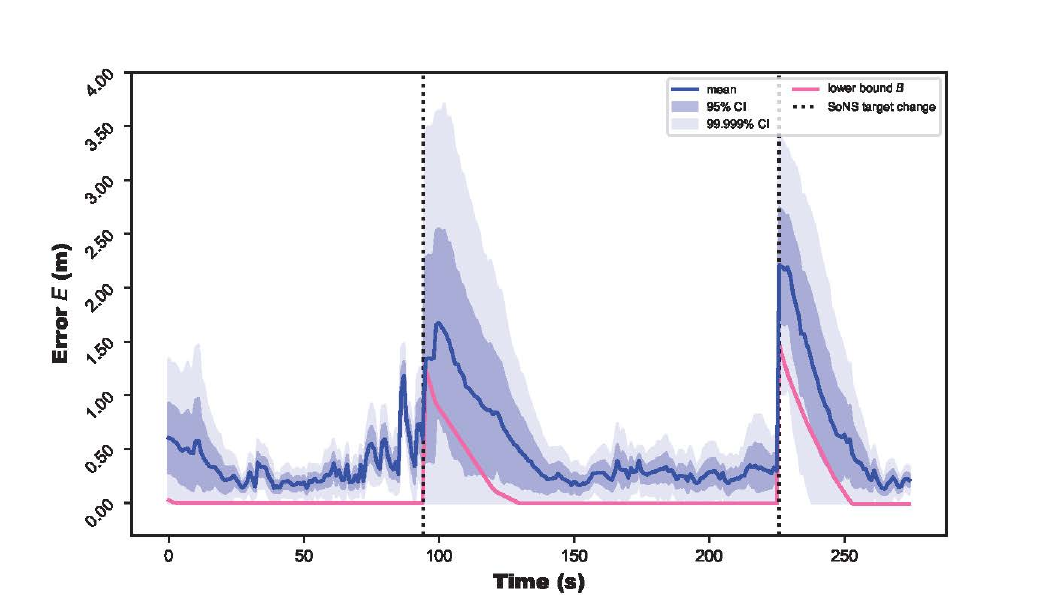}\\
\includegraphics[trim=120 60 120 80,clip,width=0.38\textwidth]{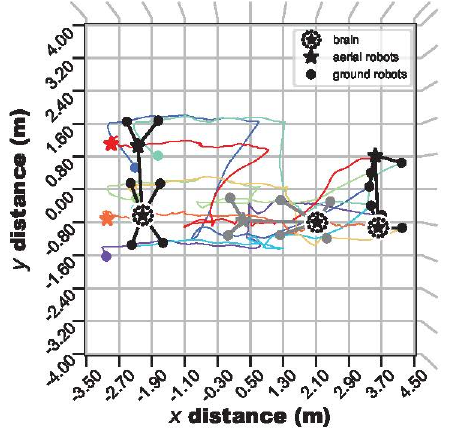}
\includegraphics[trim=20 0 40 20,clip,width=0.59\textwidth]{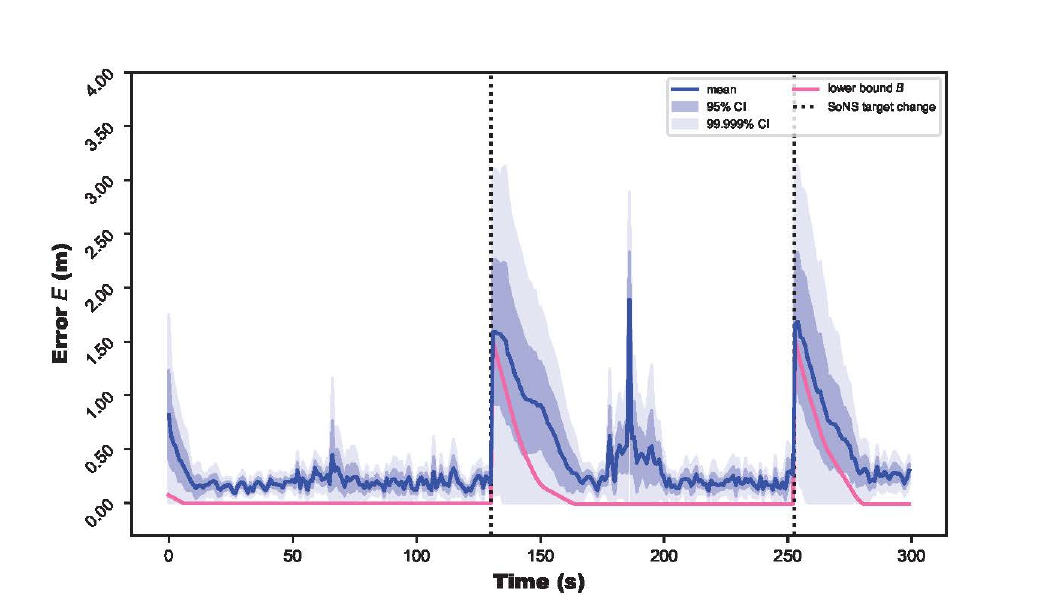}\\
\includegraphics[trim=120 60 120 80,clip,width=0.38\textwidth]{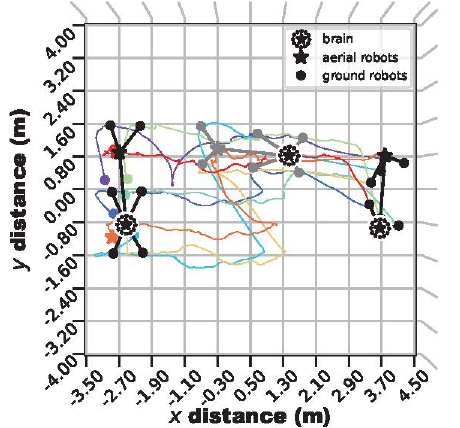}
\includegraphics[trim=20 0 40 20,clip,width=0.59\textwidth]{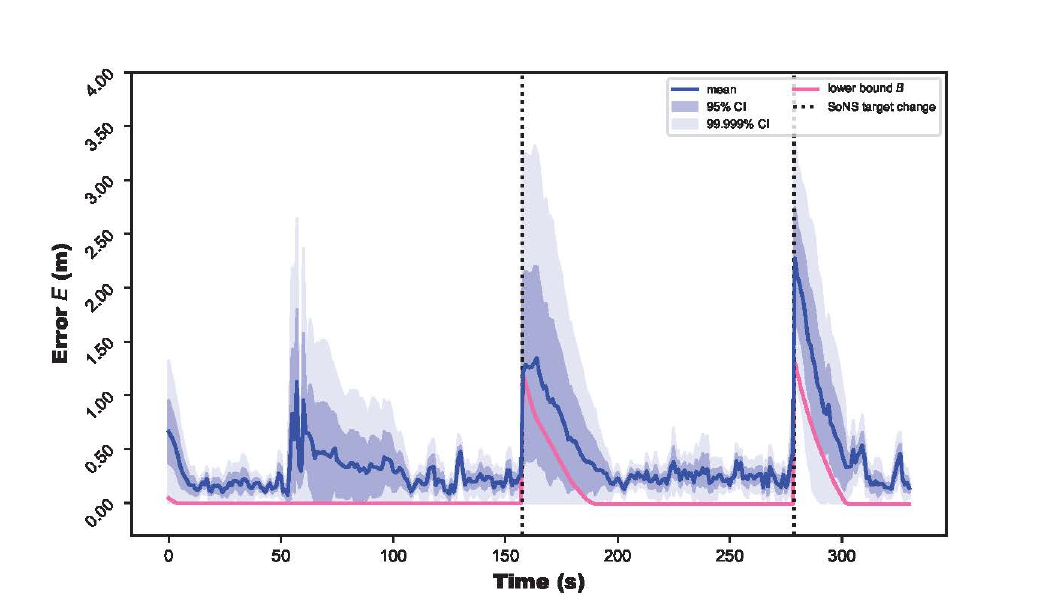}\\
\caption{{\bf Binary decision making: Real robot trials.} Five trials with real robots were conducted, each with eight robots {\it (figure continued on next page)}.}
\label{fig:mission4-hardware}
\end{figure}

\clearpage

\begin{figure}[ht!]
\ContinuedFloat
\centering
\includegraphics[trim=120 60 120 80,clip,width=0.38\textwidth]{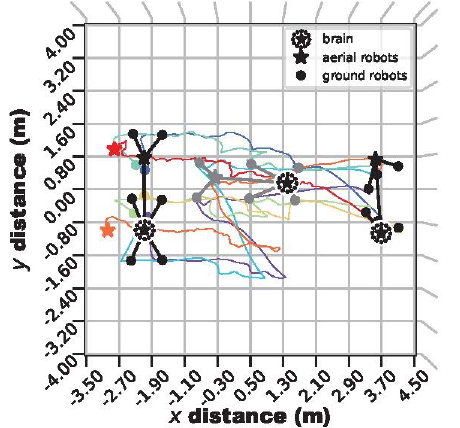}
\includegraphics[trim=20 0 40 20,clip,width=0.59\textwidth]{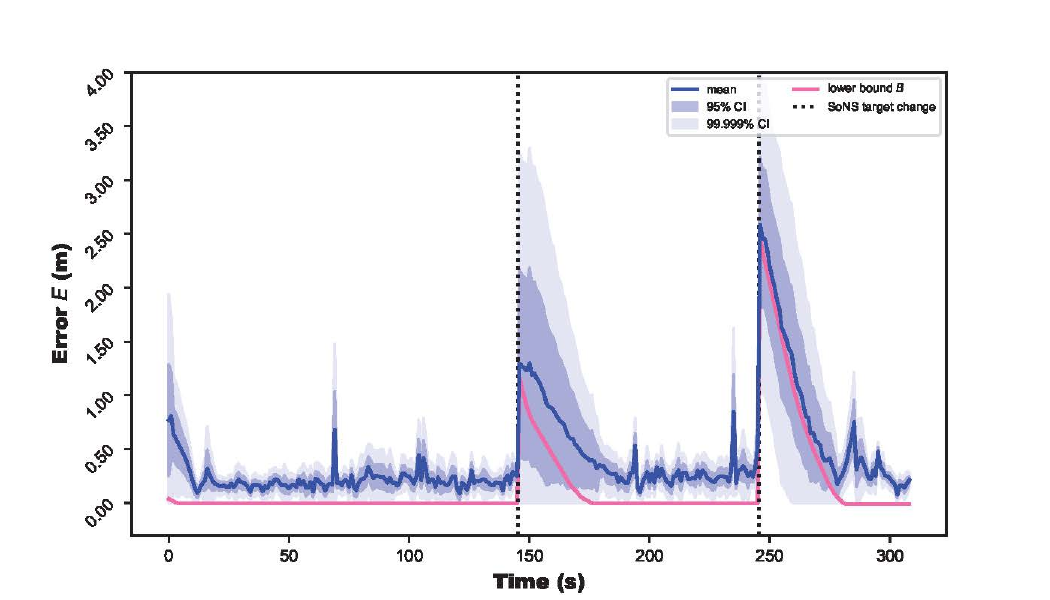}\\
\includegraphics[trim=120 60 120 80,clip,width=0.38\textwidth]{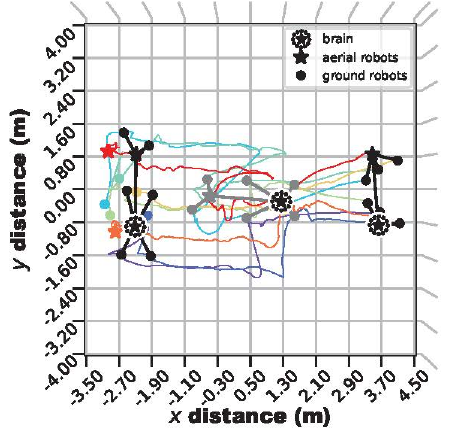}
\includegraphics[trim=20 0 40 20,clip,width=0.59\textwidth]{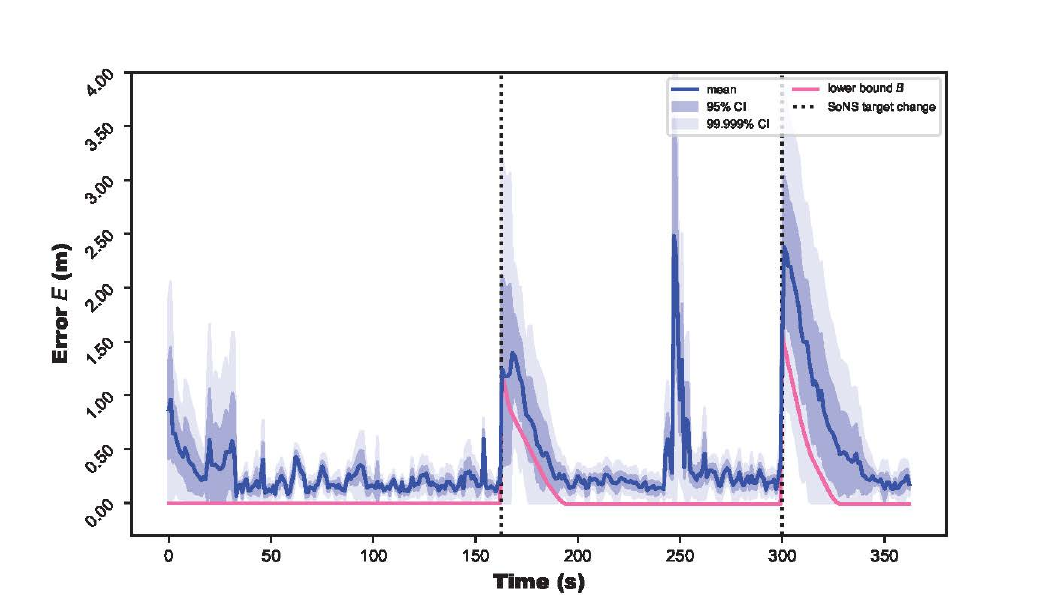}\\
\caption{{\it (cont'd)} {\bf Binary decision making: Real robot trials.} Five trials with real robots were conducted, each with eight robots.}
\label{fig:mission4-hardware}
\end{figure}

\vspace{7mm}
\noindent
{\it (Section continued on next page.)}

\clearpage
\subsection*{Mission: Splitting and merging systems (see Sec.~2.1.5 in the main paper)}
\rhead{Mission: Splitting and merging systems}

This mission includes three different variants, two run in experiments with real robots and one run in simulation.

\subsubsection*{Variant with real robots: Search and rescue}
\begin{figure}[h!]
\centering
\subfigure[]{
\includegraphics[trim=90 60 90 75,clip,width=0.4\textwidth]{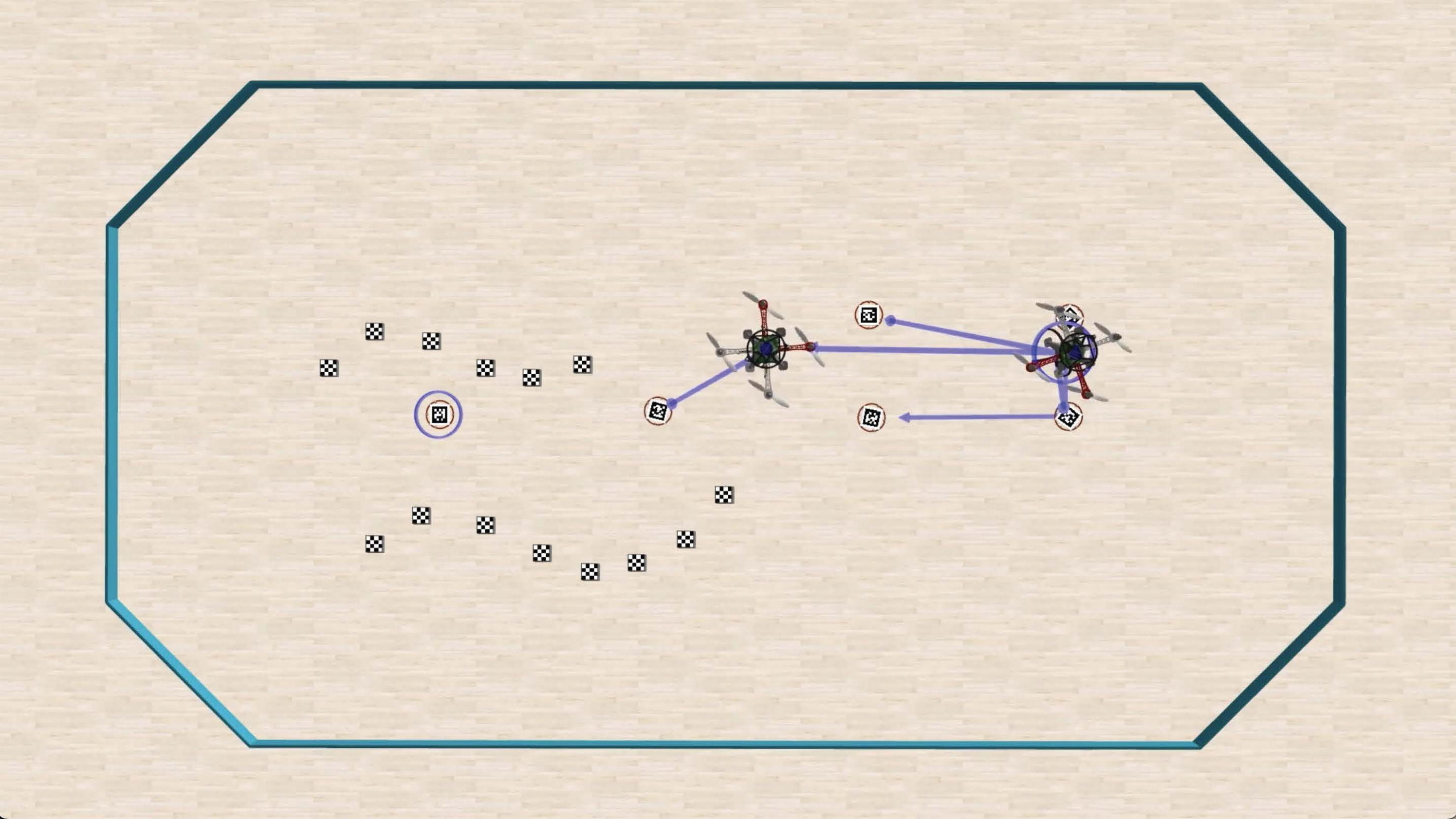}}
\subfigure[]{
\includegraphics[trim=90 60 90 75,clip,width=0.4\textwidth]{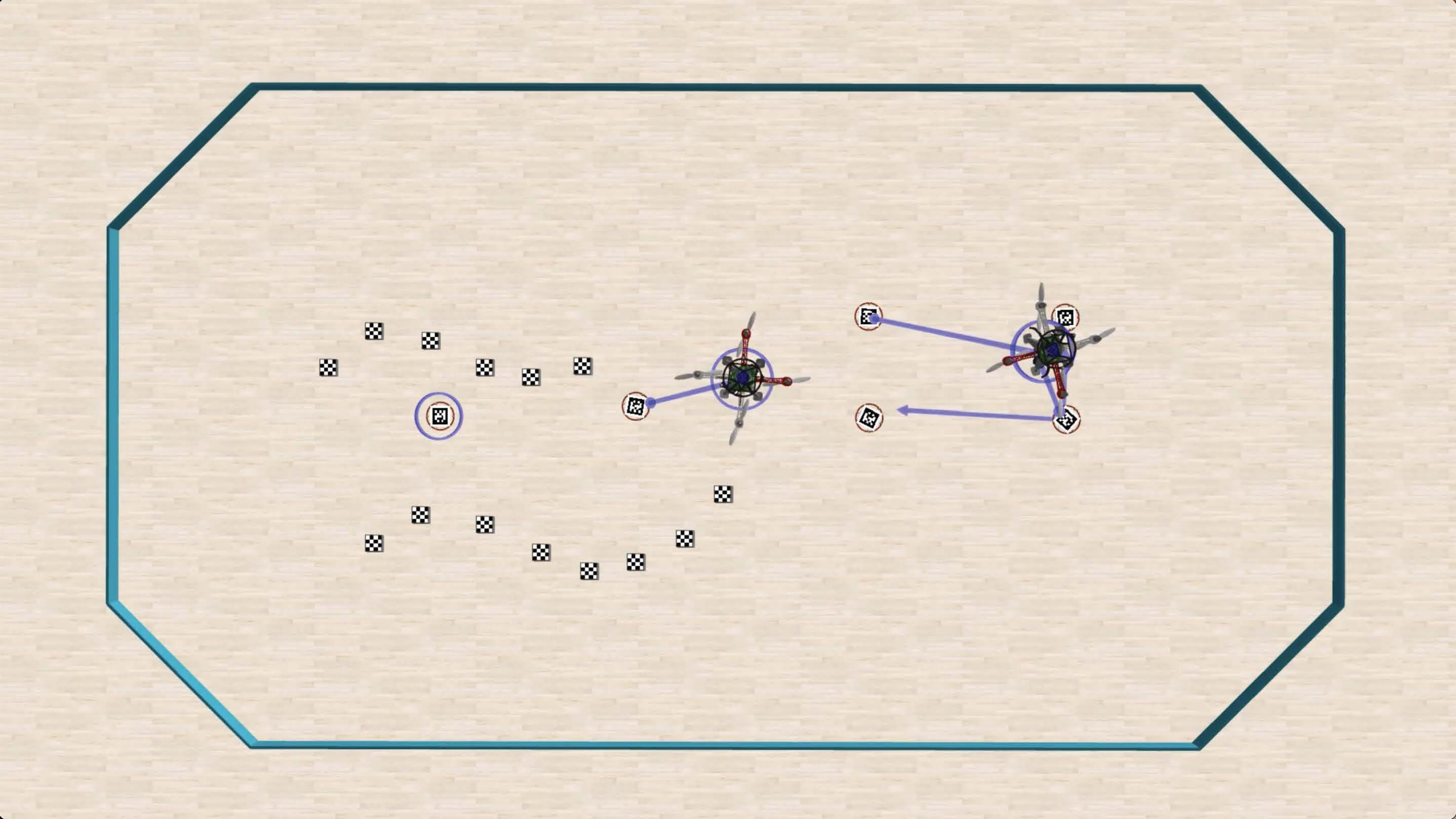}}\\
\vspace{-2mm}
\subfigure[]{
\includegraphics[trim=90 60 90 75,clip,width=0.4\textwidth]{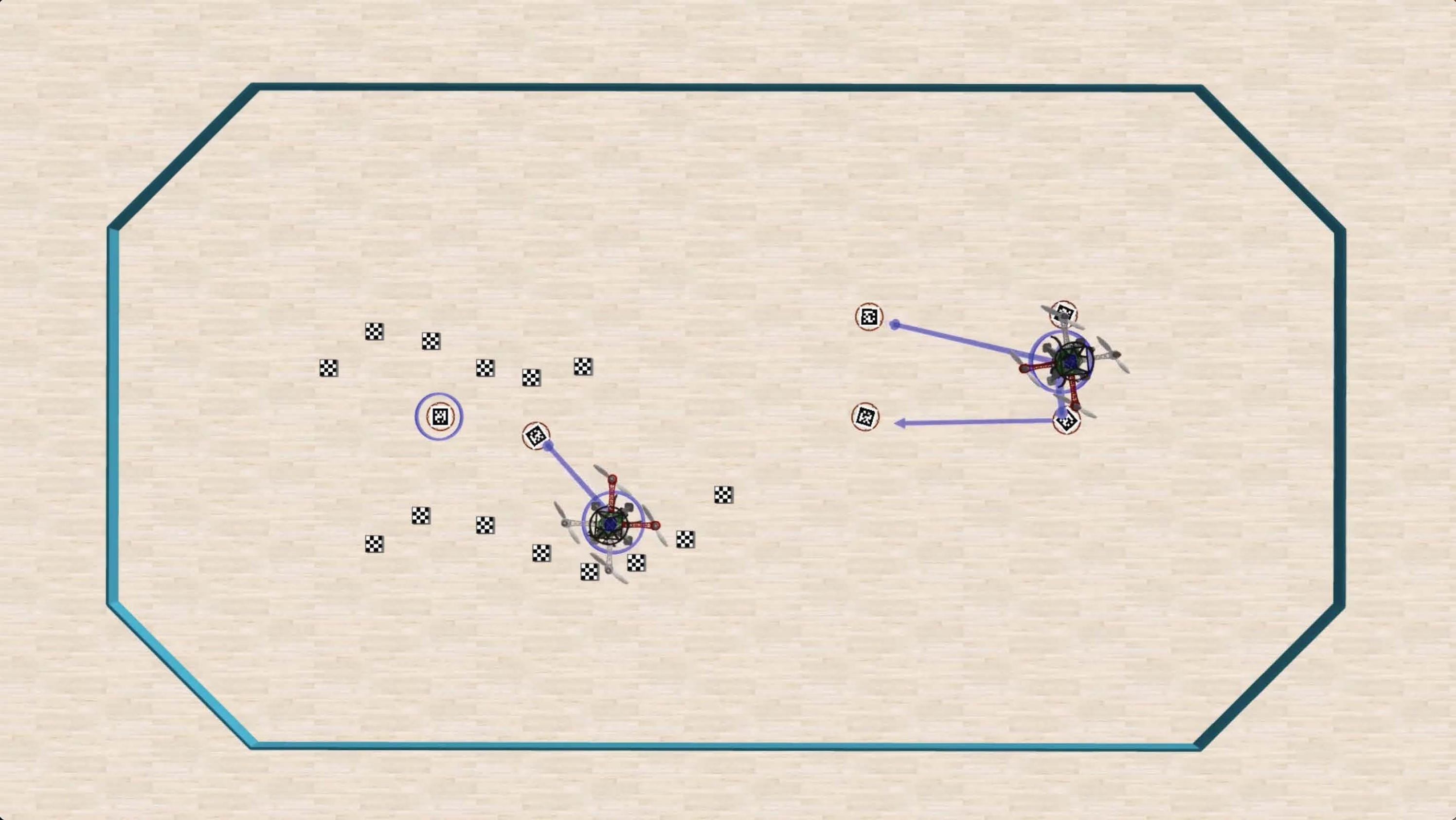}}
\subfigure[]{
\includegraphics[trim=90 60 90 75,clip,width=0.4\textwidth]{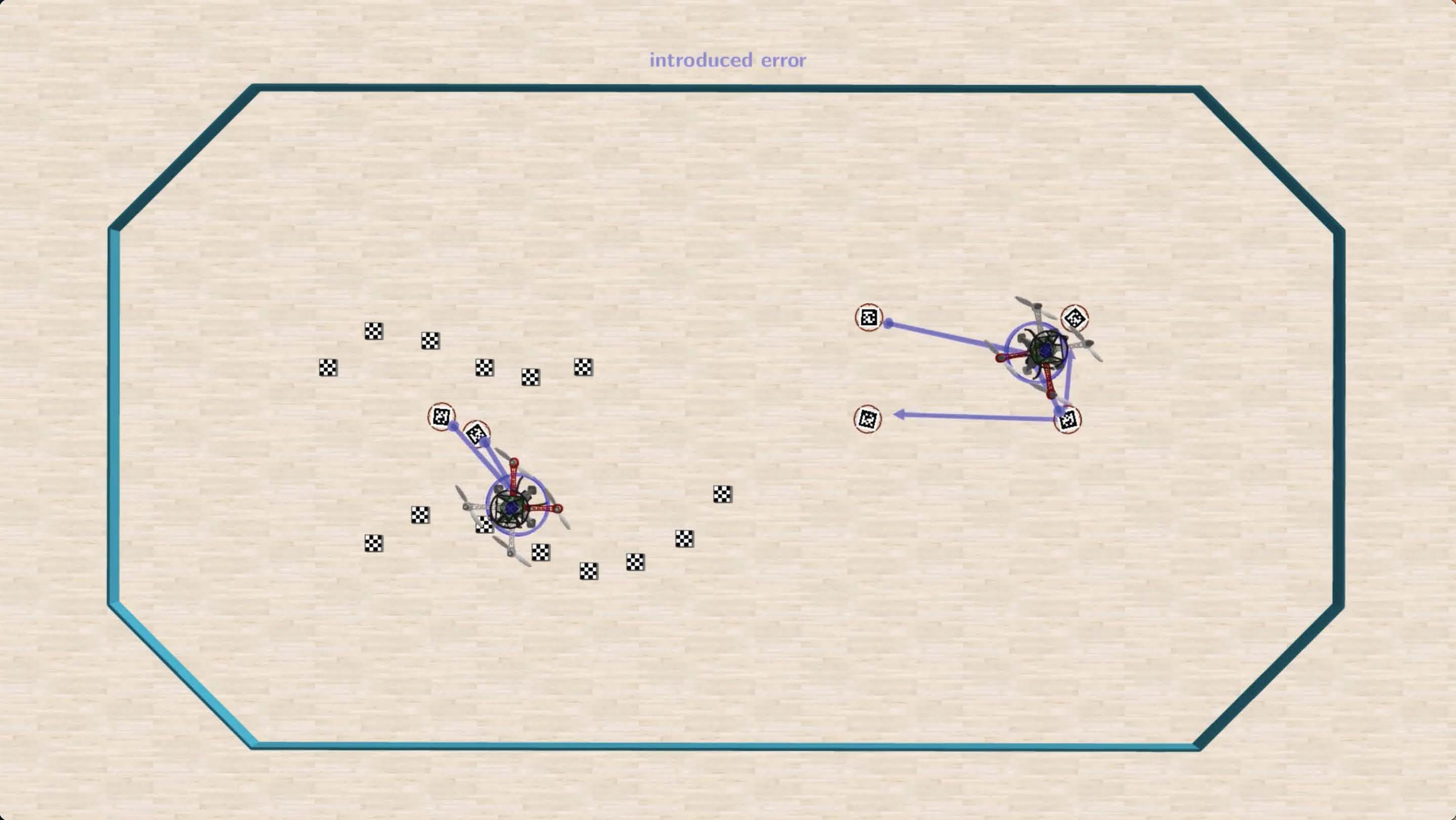}}\\
\vspace{-2mm}
\subfigure[]{
\includegraphics[trim=90 60 90 75,clip,width=0.4\textwidth]{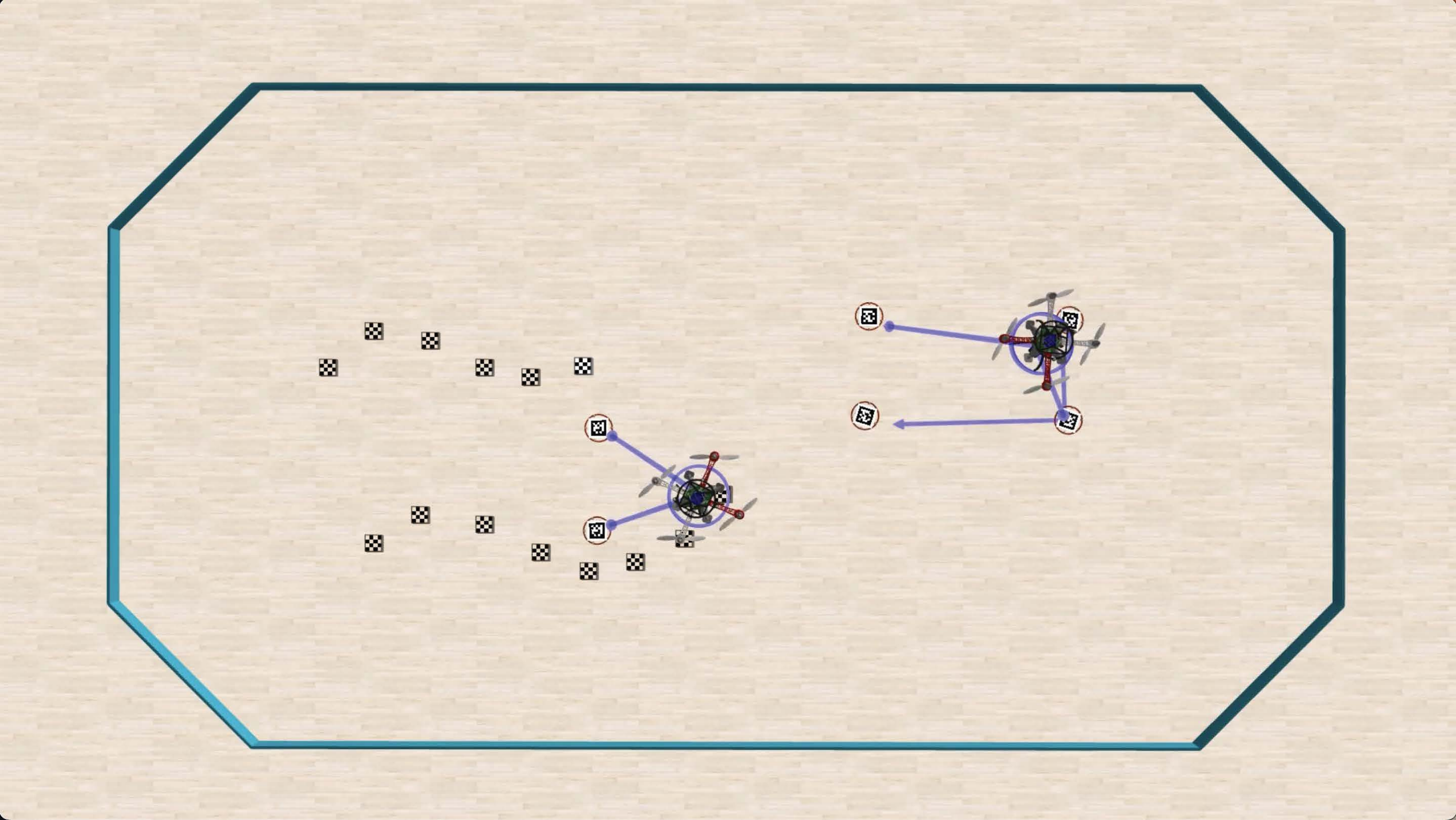}}
\subfigure[]{
\includegraphics[trim=90 60 90 75,clip,width=0.4\textwidth]{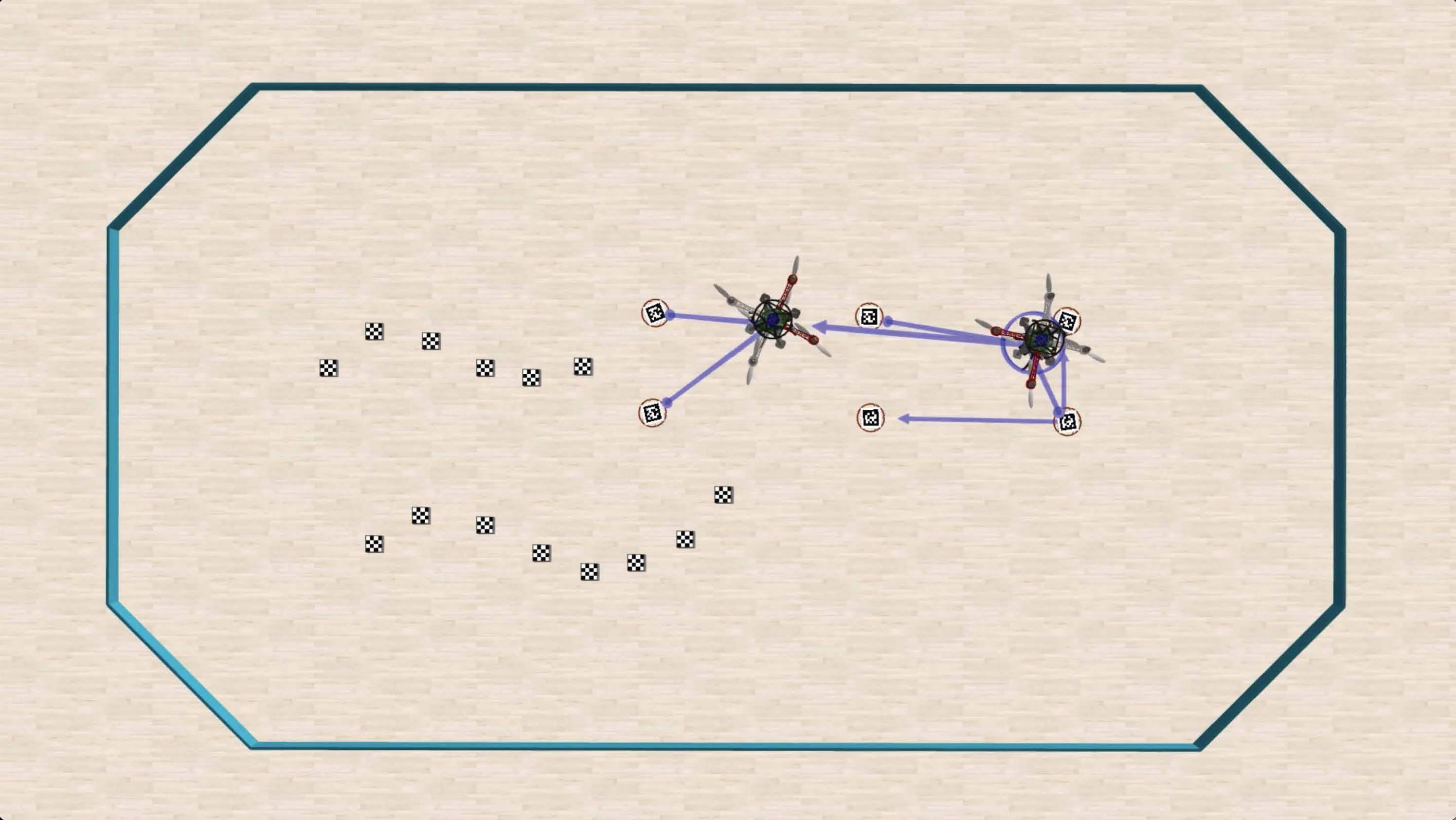}}\\
\vspace{-2mm}
\caption{{\bf Splitting and merging systems, search and rescue: Key frames.} (a) When robots start, they are in a SoNS that is missing one or two of its members. (b) The SoNS-brain instructs one of the robots to split from it and temporarily form its own multi-robot SoNS as a rescue team. (c,d) The rescue team SoNS searches the environment until it finds the missing robot(s) and merges with it/them. (e,f) The merged rescue team SoNS then returns to the location where it initially split off and re-merges with the remaining SoNS, so that all robots are reunited, and the mission is complete. 
Note that, in the experiments, there are more obstacles in the environment, forming the walls of an arbitrary passage that the robots need to navigate to complete the search and rescue mission.}
\label{fig:mission5-variant1-keyframes}
\end{figure}

\clearpage

\subsubsection*{Variant with real robots: Search and rescue}

\begin{figure}[h!]
\centering
\includegraphics[trim=120 60 120 80,clip,width=0.38\textwidth]{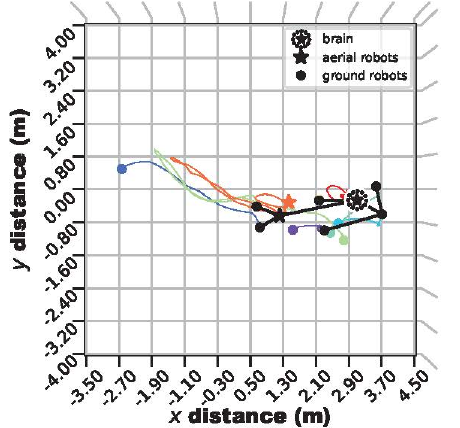}
\includegraphics[trim=20 0 40 20,clip,width=0.59\textwidth]{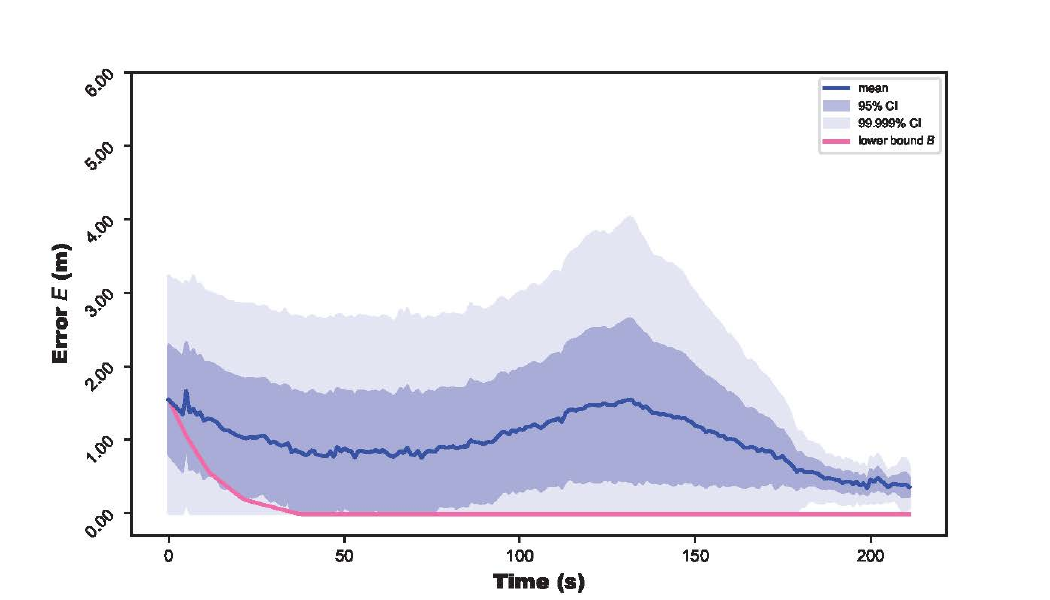}\\
\includegraphics[trim=120 60 120 80,clip,width=0.38\textwidth]{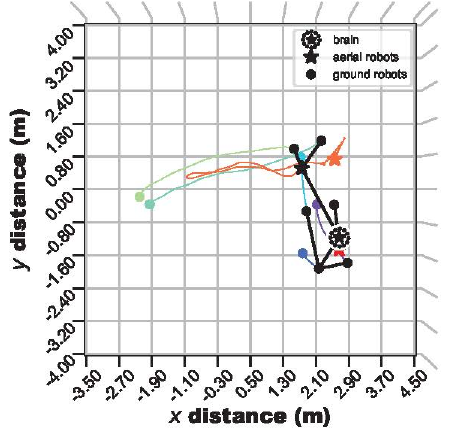}
\includegraphics[trim=20 0 40 20,clip,width=0.59\textwidth]{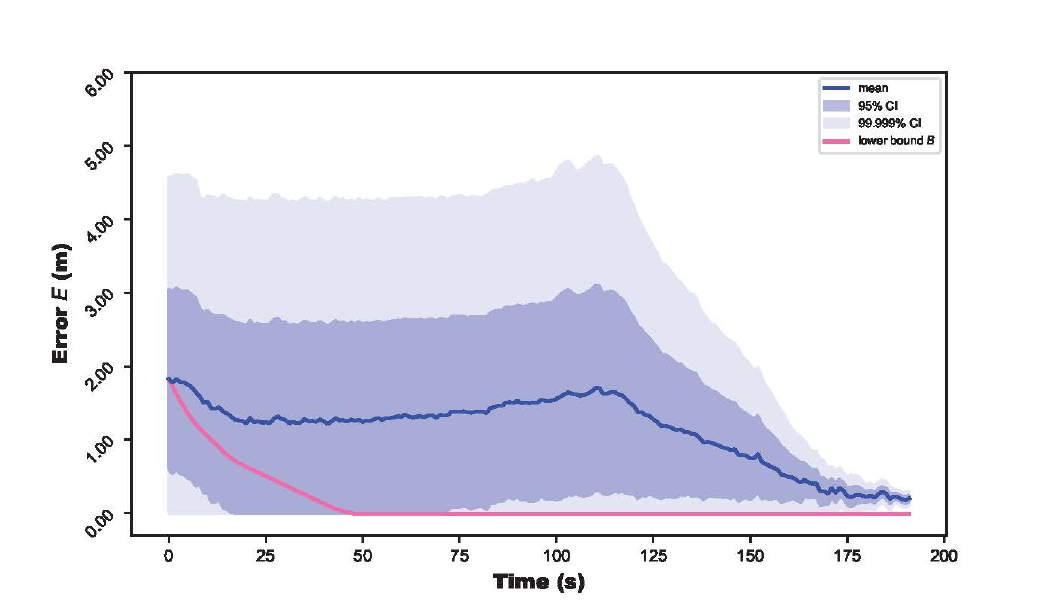}\\
\includegraphics[trim=120 60 120 80,clip,width=0.38\textwidth]{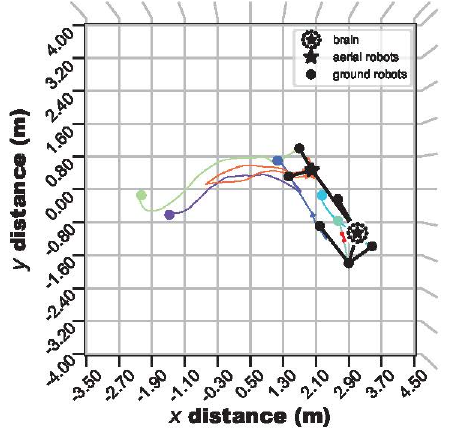}
\includegraphics[trim=20 0 40 20,clip,width=0.59\textwidth]{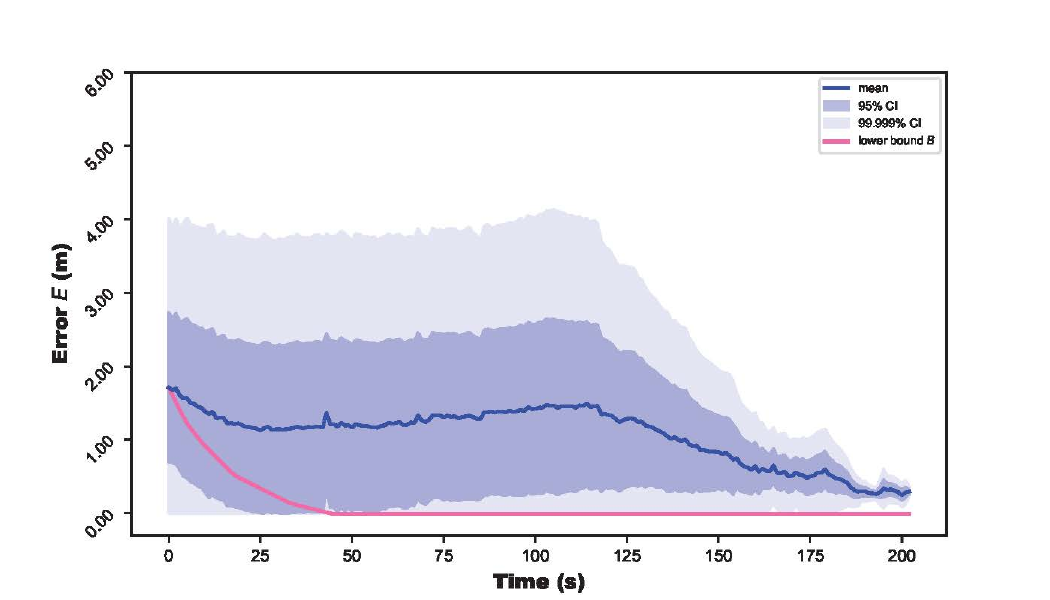}\\
\caption{{\bf Splitting and merging systems, search and rescue: Real robot trials.} Five trials with real robots were conducted, each with eight robots {\it (figure continued on next page)}.}
\label{fig:mission5-variant1-hardware}
\end{figure}

\clearpage

\begin{figure}[h!]
\ContinuedFloat
\centering
\includegraphics[trim=120 60 120 80,clip,width=0.38\textwidth]{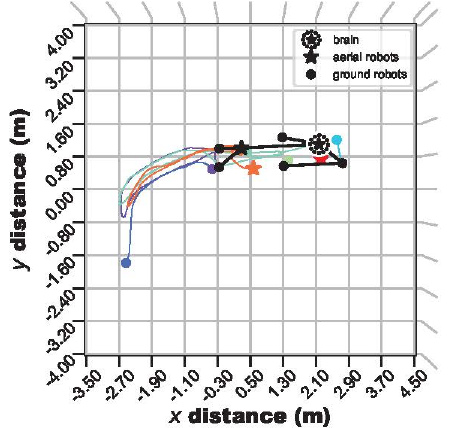}
\includegraphics[trim=20 0 40 20,clip,width=0.59\textwidth]{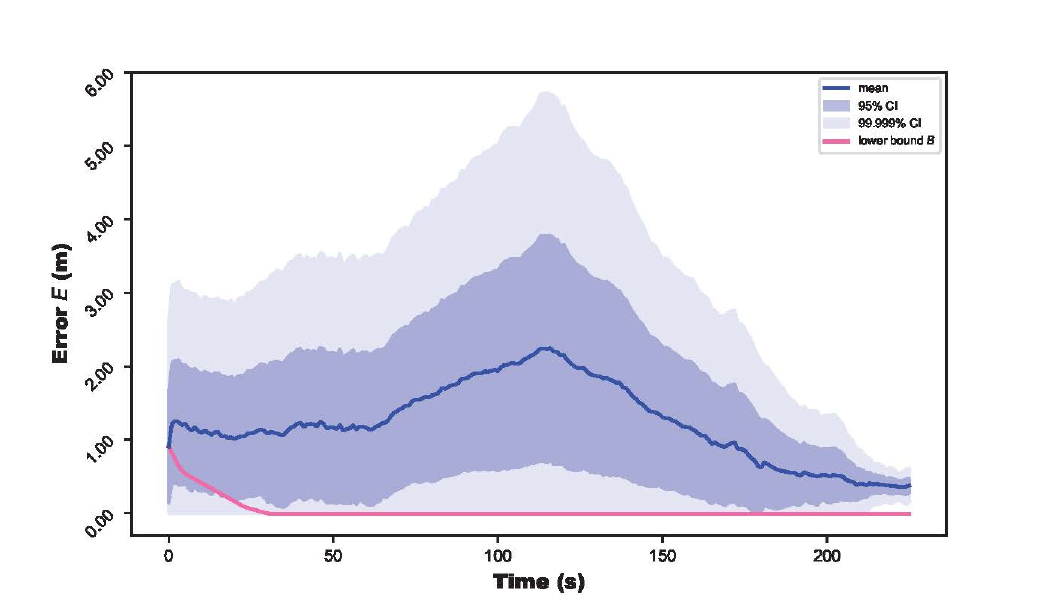}\\
\includegraphics[trim=120 60 120 80,clip,width=0.38\textwidth]{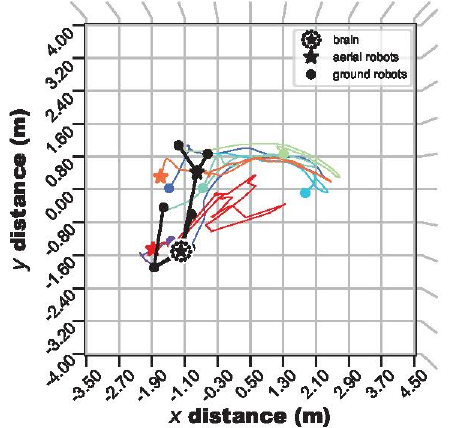}
\includegraphics[trim=20 0 40 20,clip,width=0.59\textwidth]{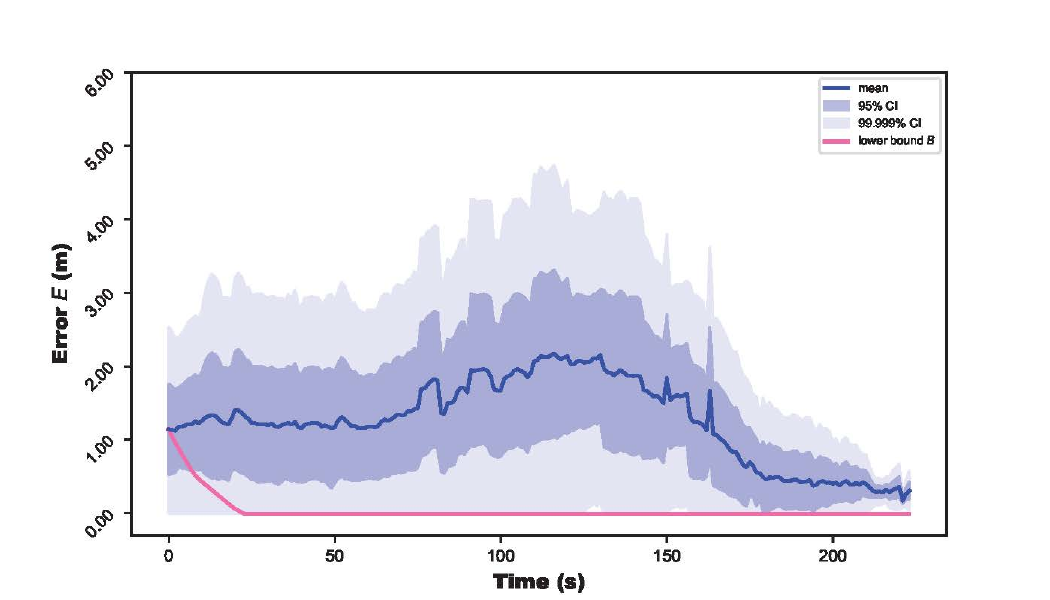}\\
\caption{{\it (cont'd)} {\bf Splitting and merging systems, search and rescue: Real robot trials.} Five trials with real robots were conducted, each with eight robots.}
\label{fig:mission5-variant1-hardware}
\end{figure}

\vspace{7mm}
\noindent
{\it (Section continued on next page.)}

\clearpage
\subsubsection*{Variant with real robots: Push away an obstruction}
\begin{figure}[h!]
\centering
\subfigure[]{
\includegraphics[trim=90 60 90 75,clip,width=0.4\textwidth]{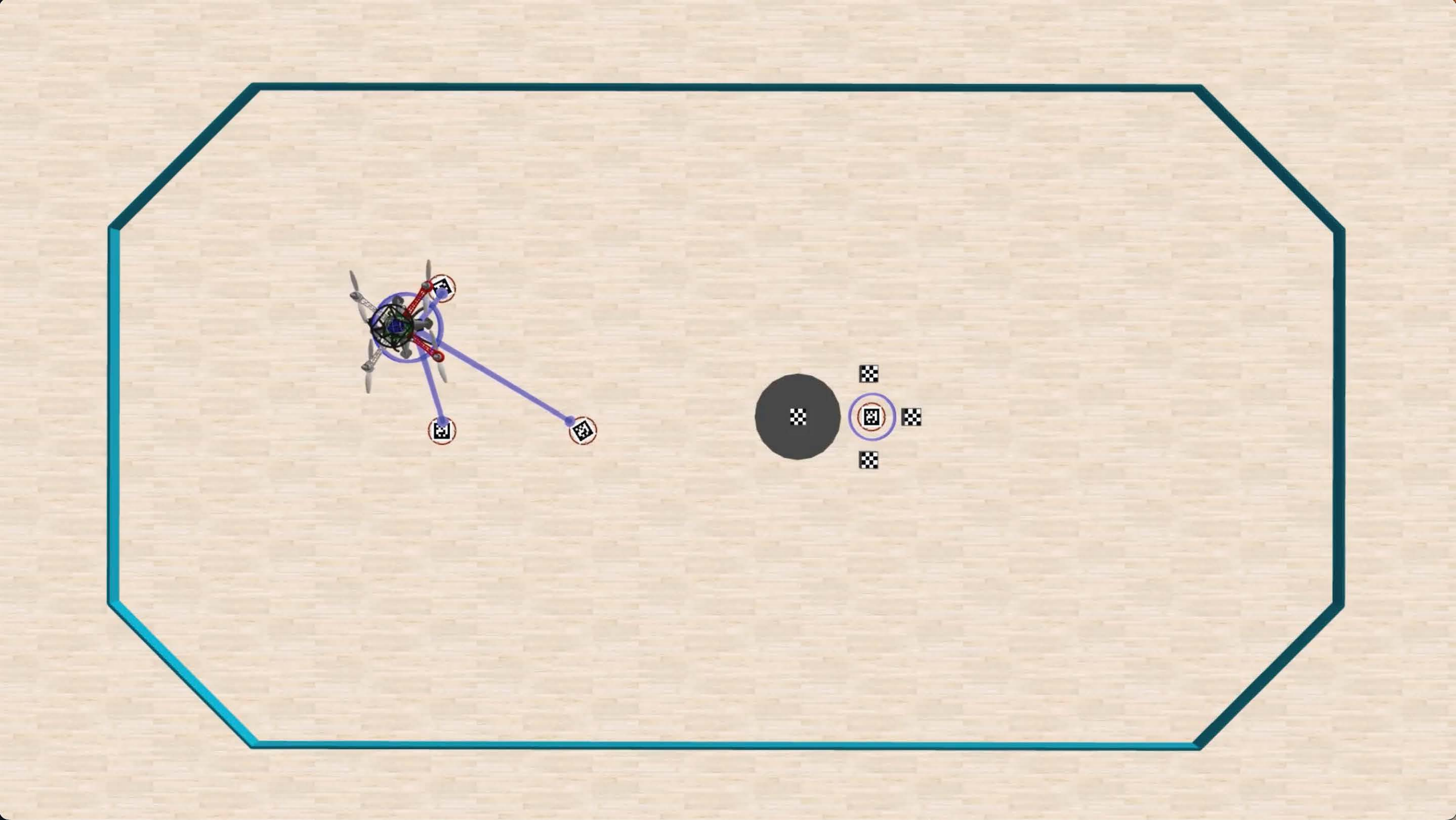}}
\subfigure[]{
\includegraphics[trim=90 60 90 75,clip,width=0.4\textwidth]{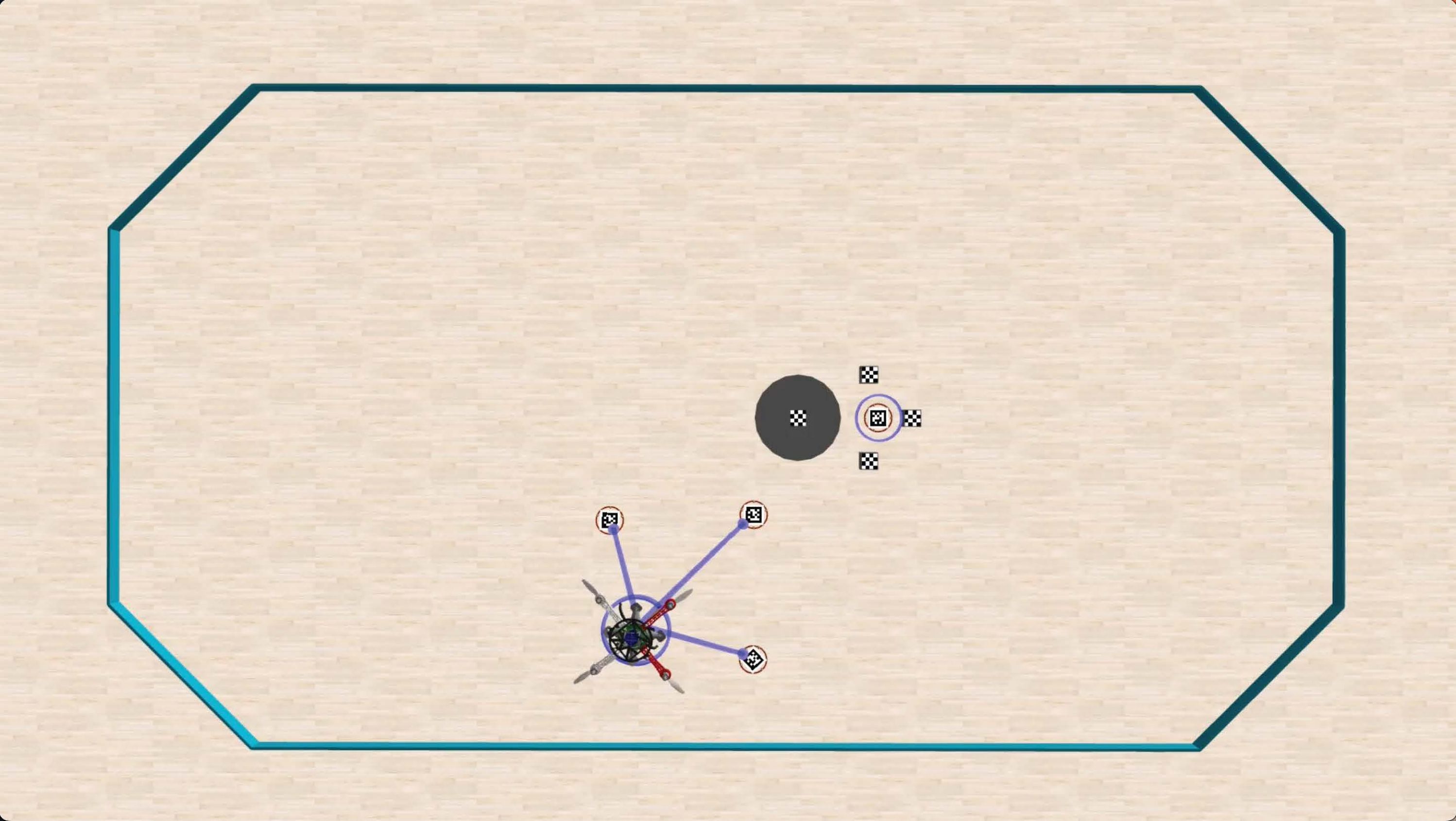}}\\
\vspace{-2mm}
\subfigure[]{
\includegraphics[trim=90 60 90 75,clip,width=0.4\textwidth]{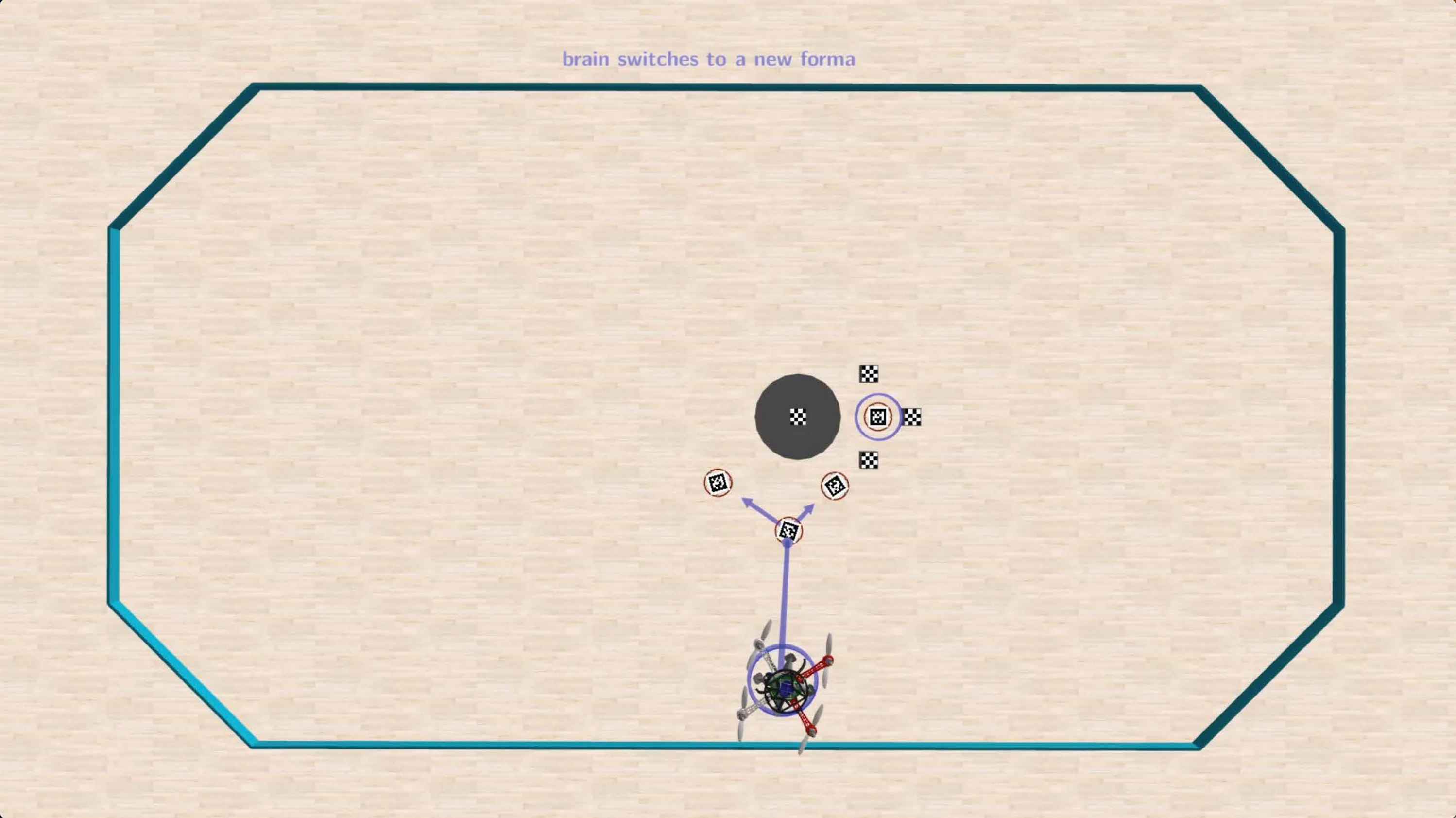}}
\subfigure[]{
\includegraphics[trim=90 60 90 75,clip,width=0.4\textwidth]{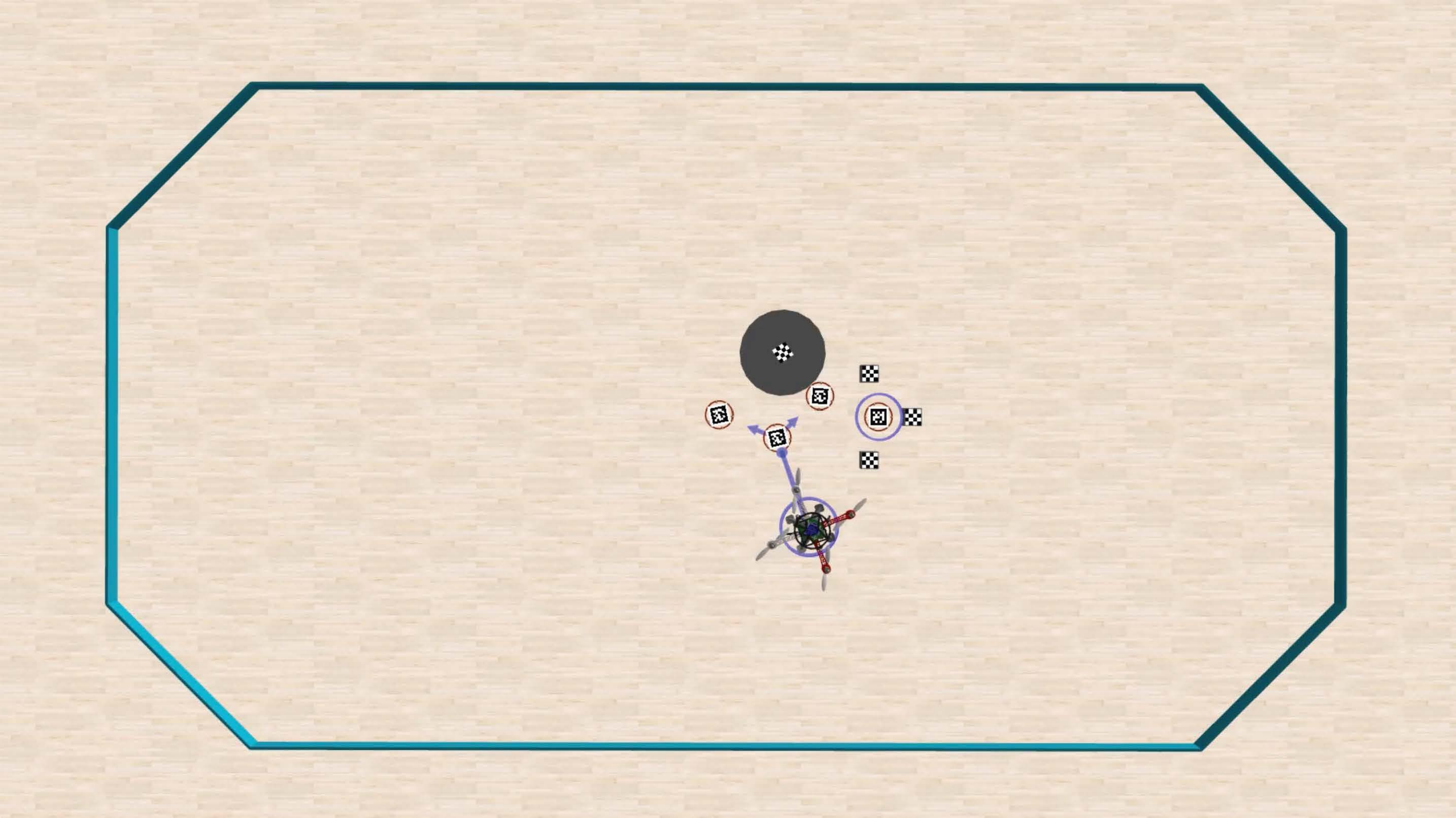}}\\
\vspace{-2mm}
\subfigure[]{
\includegraphics[trim=90 60 90 75,clip,width=0.4\textwidth]{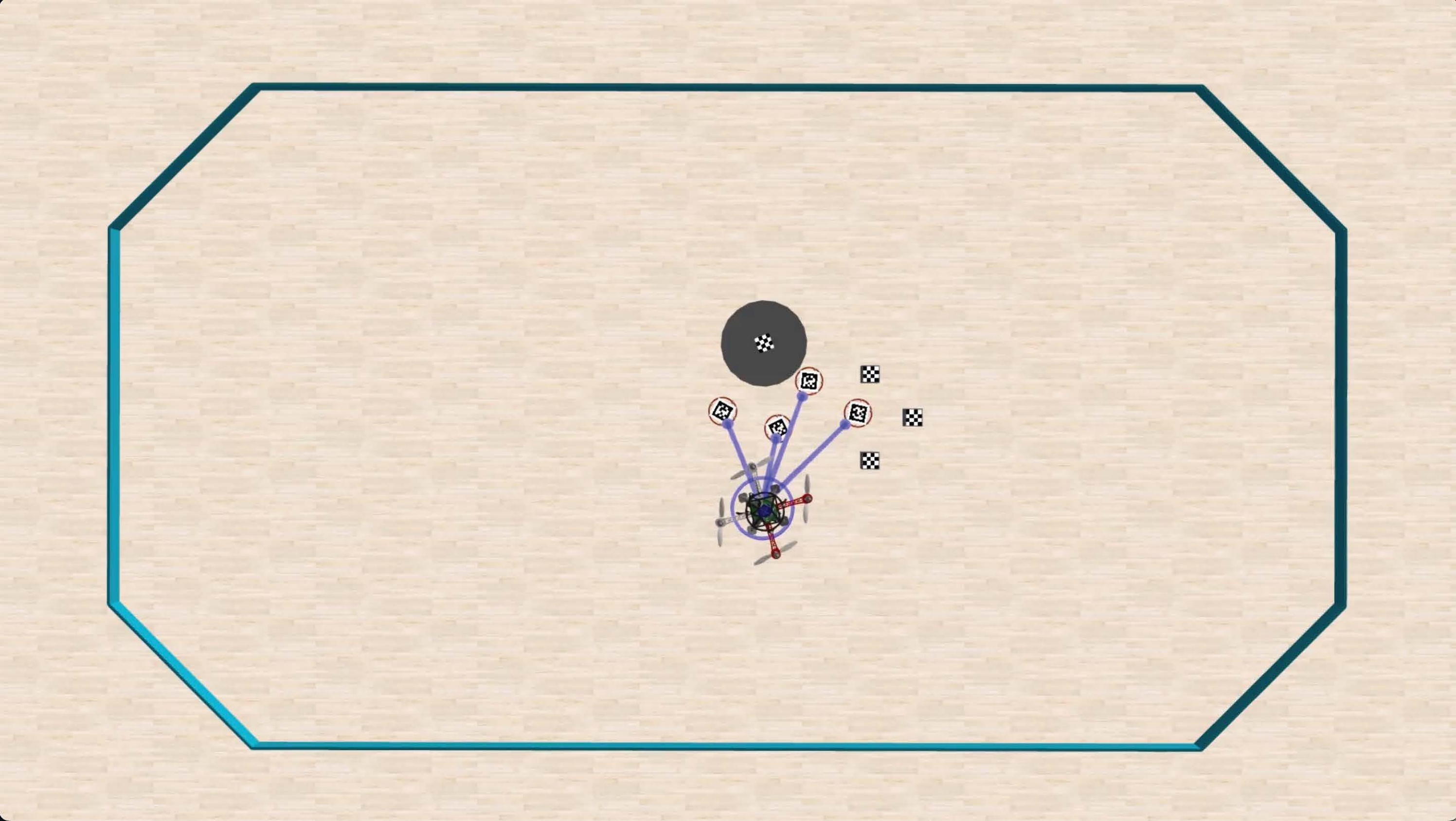}}
\subfigure[]{
\includegraphics[trim=90 60 90 75,clip,width=0.4\textwidth]{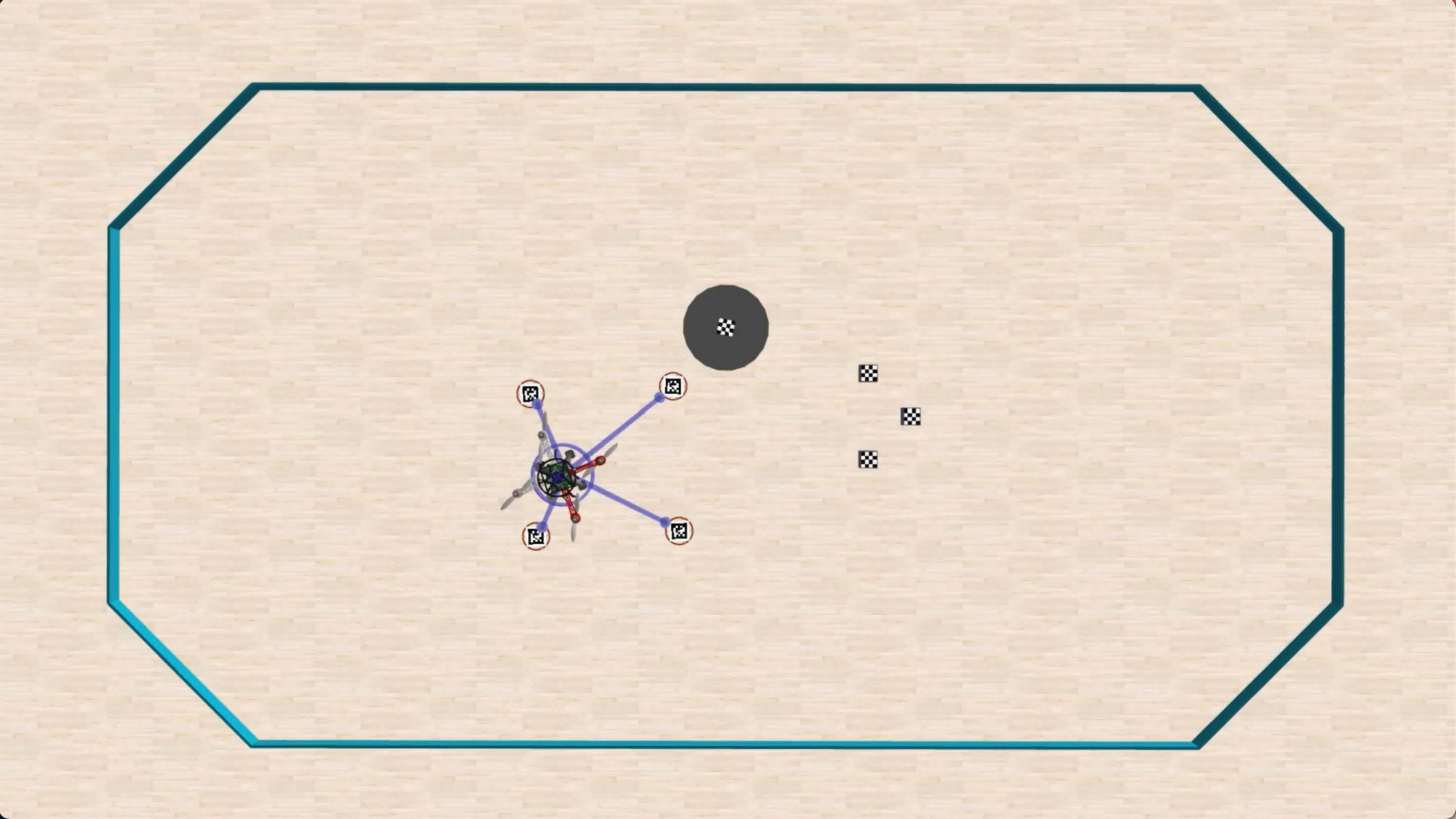}}\\
\vspace{-2mm}
\caption{{\bf Push away an obstruction: Key frames.} (a) When robots start, they are in a SoNS that is missing a member. (b) The SoNS searches the environment and finds the missing robot trapped by obstacles. (c) The SoNS reorganizes into a shape that allows its ground robots to collaboratively push a large obstruction. (d,e) The robots push the obstruction away and the SoNS merges with the missing robot, so that all robots are reunited. (f) The robots reorganize into the target SoNS and the mission is complete.}
\label{fig:mission5-variant2-keyframes}
\end{figure}

\vspace{7mm}
\noindent
{\it (Section continued on next page.)}

\clearpage

\subsubsection*{Variant with real robots: Push away an obstruction}

\begin{figure}[h!]
\centering
\includegraphics[trim=120 60 120 80,clip,width=0.38\textwidth]{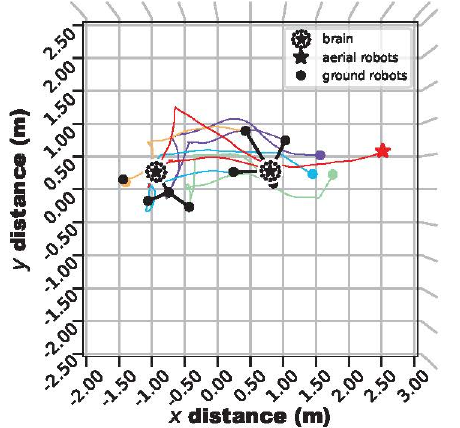}
\includegraphics[trim=20 0 40 20,clip,width=0.59\textwidth]{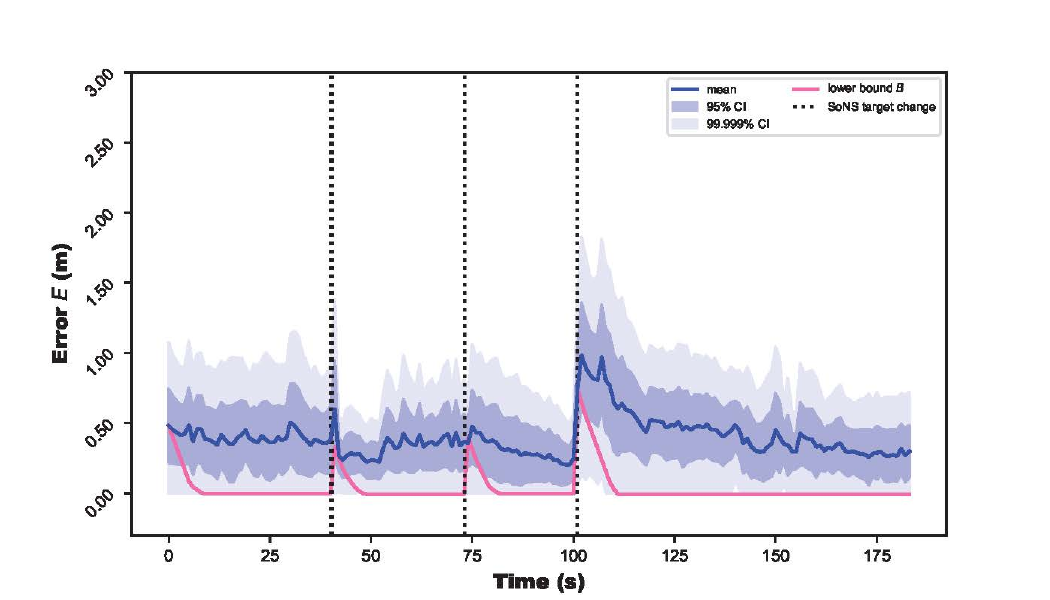}\\
\includegraphics[trim=120 60 120 80,clip,width=0.38\textwidth]{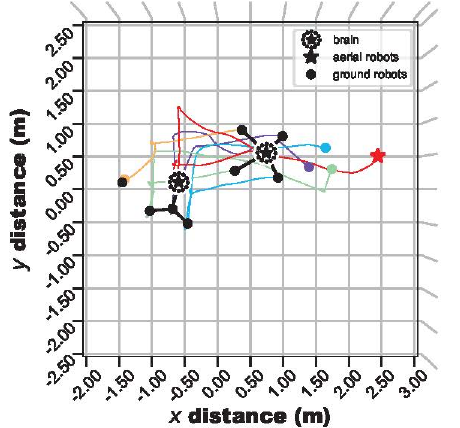}
\includegraphics[trim=20 0 40 20,clip,width=0.59\textwidth]{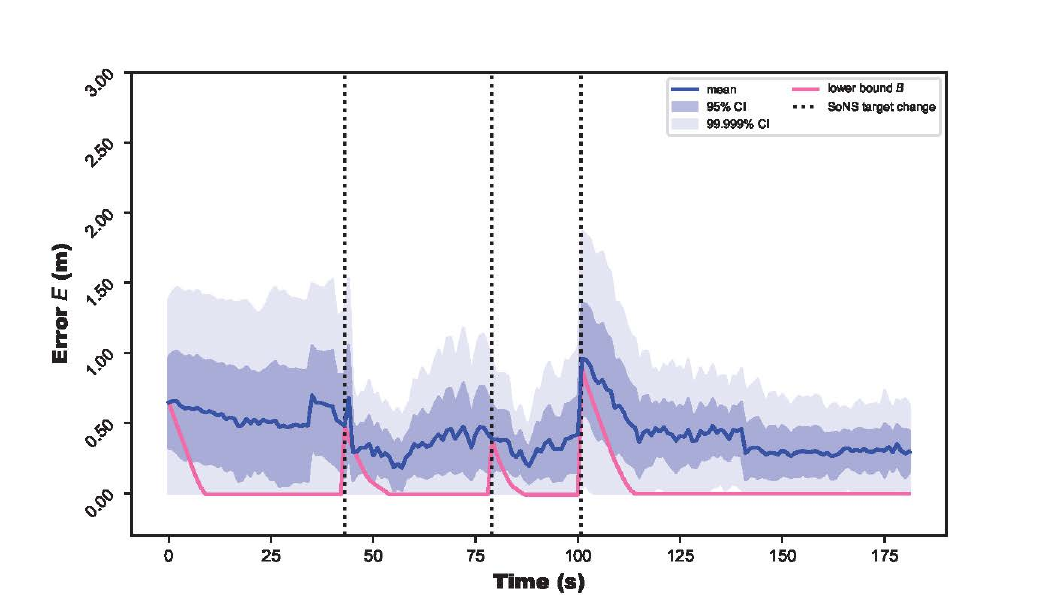}\\
\includegraphics[trim=120 60 120 80,clip,width=0.38\textwidth]{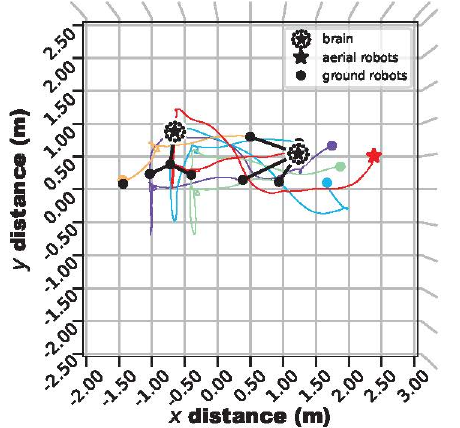}
\includegraphics[trim=20 0 40 20,clip,width=0.59\textwidth]{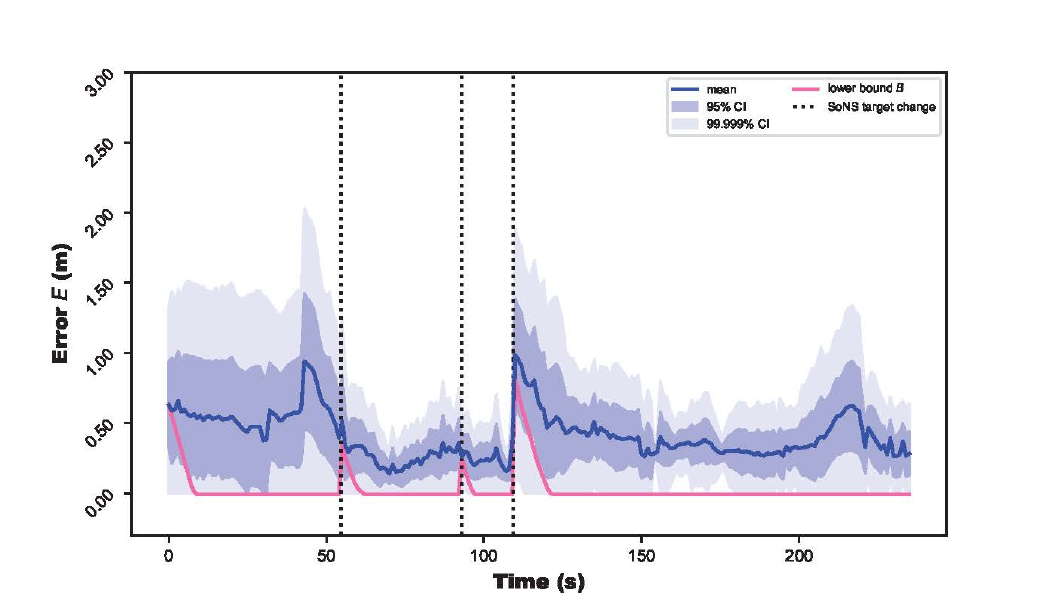}\\
\caption{{\bf Splitting and merging systems, push away an obstruction: Real robot trials.} Note that the highlighted keyframes shown here are from an intermediary time step, not the end of the trial. Five trials with real robots were conducted, each with eight robots {\it (figure continued on next page)}.}
\label{fig:mission5-variant2-hardware}
\end{figure}

\clearpage

\begin{figure}[h!]
\ContinuedFloat
\centering
\includegraphics[trim=120 60 120 80,clip,width=0.38\textwidth]{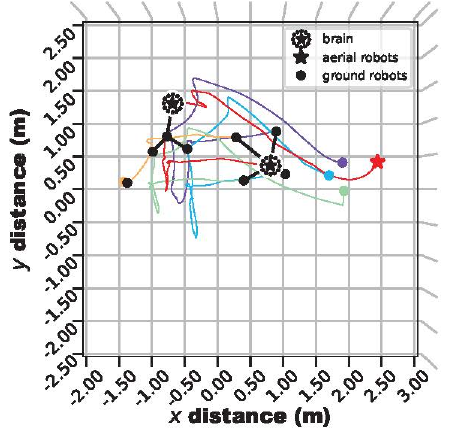}
\includegraphics[trim=20 0 40 20,clip,width=0.59\textwidth]{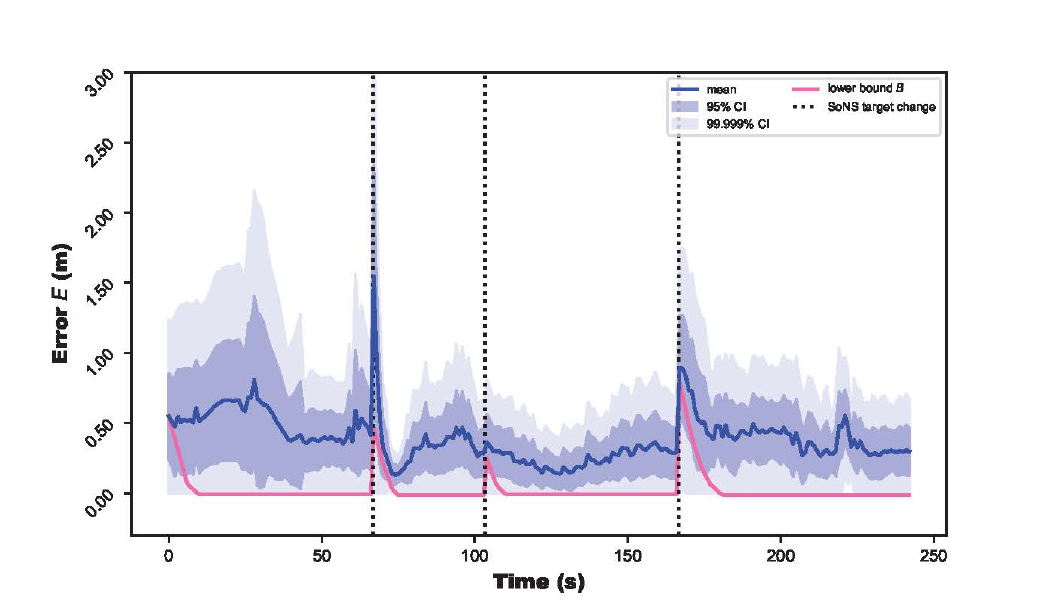}\\
\includegraphics[trim=120 60 120 80,clip,width=0.38\textwidth]{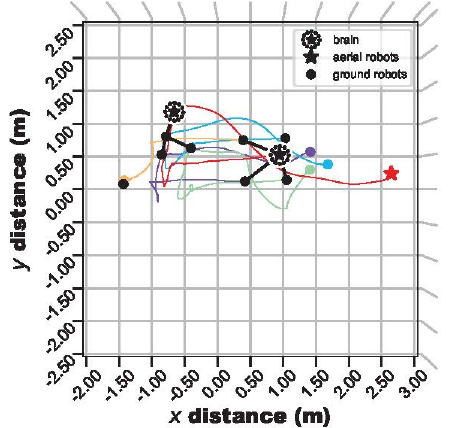}
\includegraphics[trim=20 0 40 20,clip,width=0.59\textwidth]{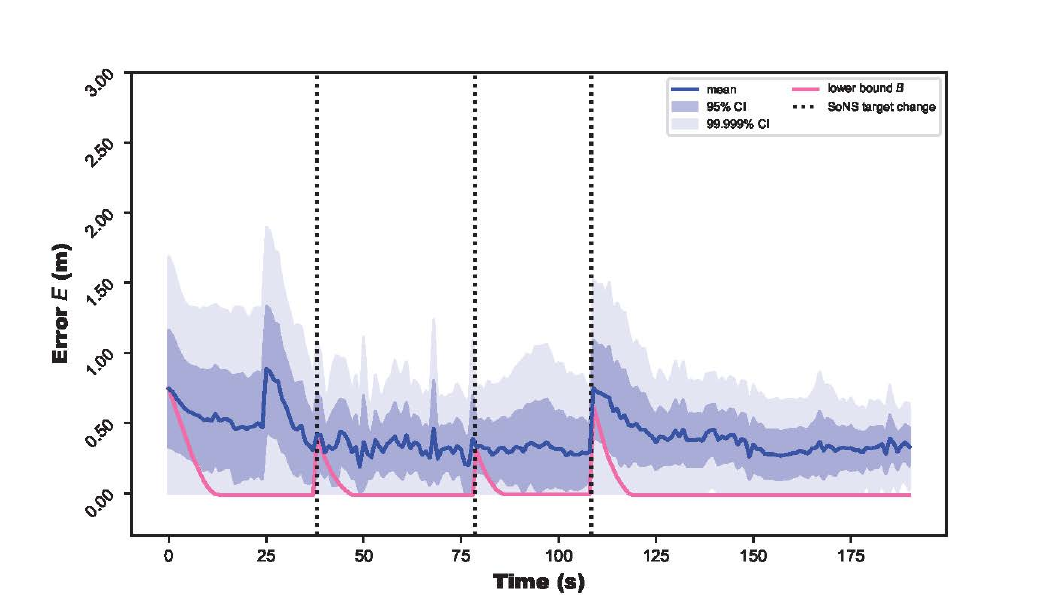}\\
\caption{{\it (cont'd)} {\bf Splitting and merging systems, push away an obstruction: Real robot trials.} Note that the highlighted keyframes shown here are from an intermediary time step, not the end of the trial. Five trials with real robots were conducted, each with eight robots.}
\label{fig:mission5-variant2-hardware}
\end{figure}

\vspace{7mm}
\noindent
{\it (Section continued on next page.)}

\clearpage
\subsubsection*{Simulation variant: Simple split and merge}

This variant consists simply of split and merge operations, without search-and-rescue or other constituent tasks.

\begin{figure}[h!]
\centering
\includegraphics[trim=120 60 120 80,clip,width=0.38\textwidth]{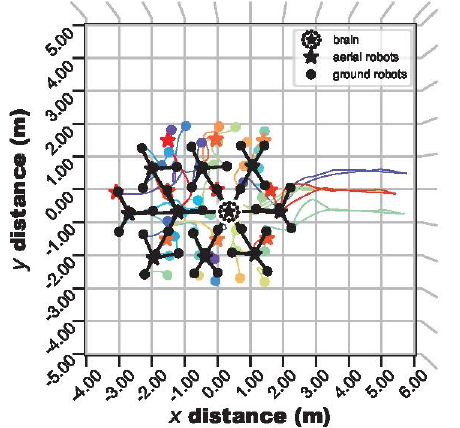}
\includegraphics[trim=20 0 40 20,clip,width=0.59\textwidth]{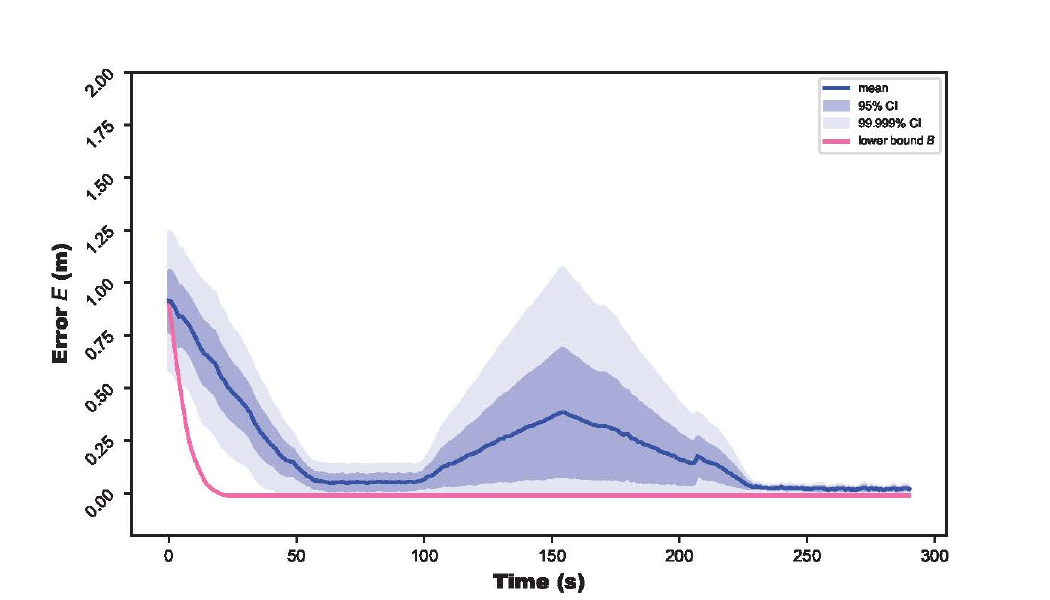}\\
\includegraphics[trim=120 60 120 80,clip,width=0.38\textwidth]{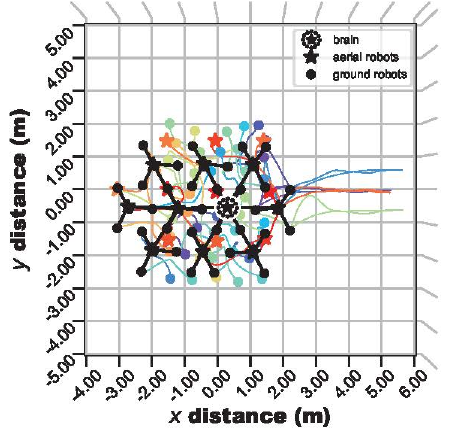}
\includegraphics[trim=20 0 40 20,clip,width=0.59\textwidth]{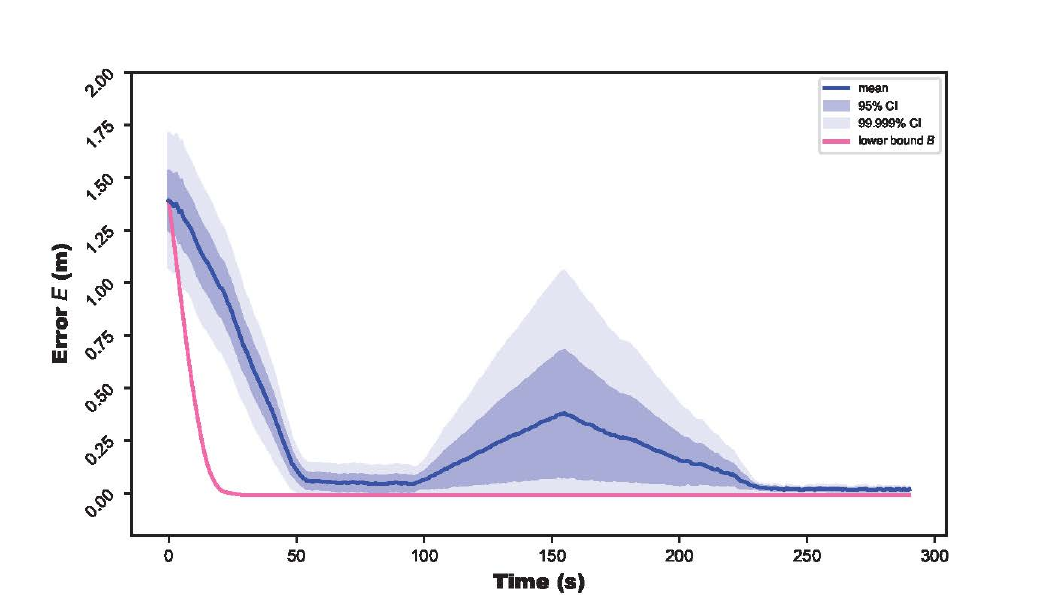}\\
\includegraphics[trim=120 60 120 80,clip,width=0.38\textwidth]{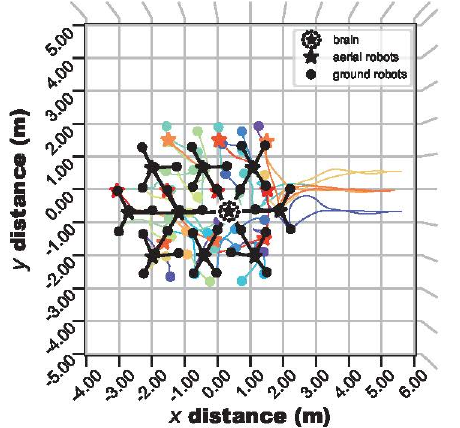}
\includegraphics[trim=20 0 40 20,clip,width=0.59\textwidth]{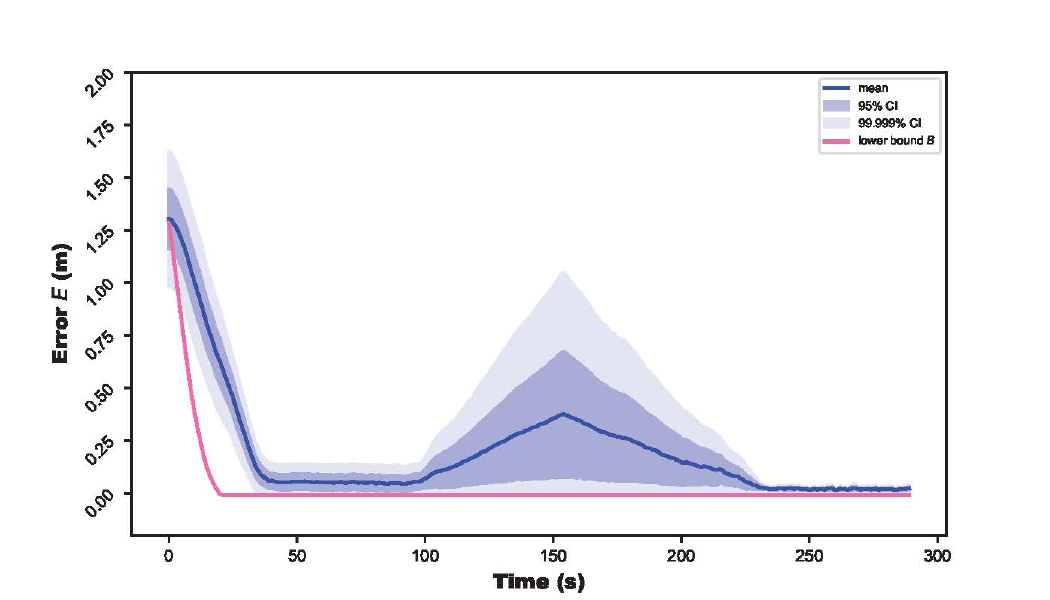}\\
\caption{{\bf Splitting and merging systems, simple split and merge: Example simulation trials.} 50 trials were conducted in simulation, each with 50 robots.}
\label{fig:mission5-variant3-simulation}
\end{figure}

\clearpage
\subsection*{Scalability setups (see Sec.~2.2 in the main paper)}
\rhead{Scalability setups}

The scalability setups include two variants based on two of the robot missions (those shown in Secs.~2.1.1 and~2.1.4 in the main paper), both run in simulation only.

\subsubsection*{Scalability in the binary decision-making mission}

\begin{figure}[h!]
\centering
\includegraphics[trim=120 60 120 80,clip,width=0.38\textwidth]{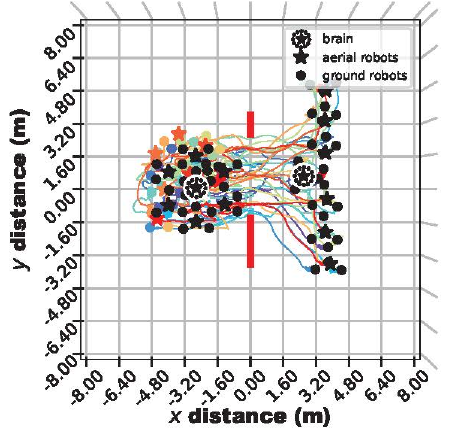}
\includegraphics[trim=20 0 40 20,clip,width=0.59\textwidth]{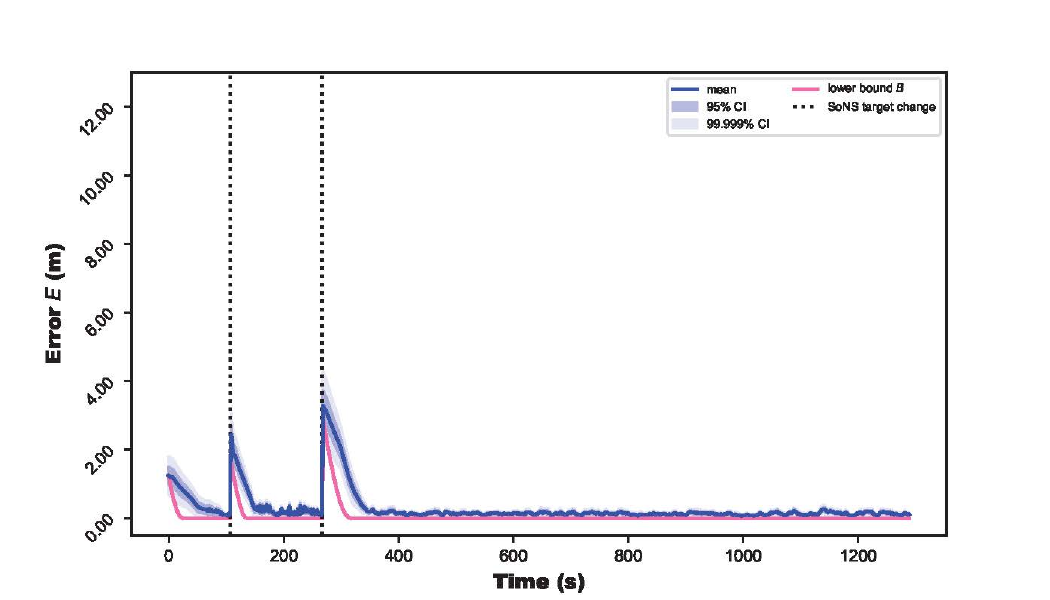}\\
\includegraphics[trim=120 60 120 80,clip,width=0.38\textwidth]{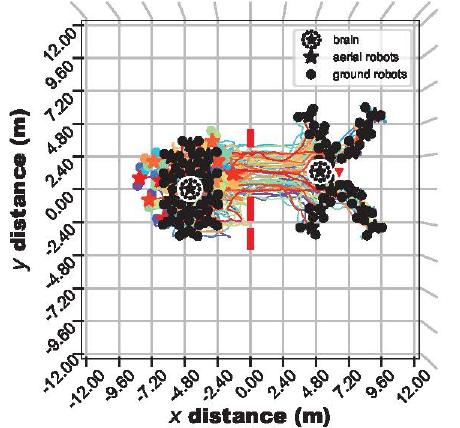}
\includegraphics[trim=20 0 40 20,clip,width=0.59\textwidth]{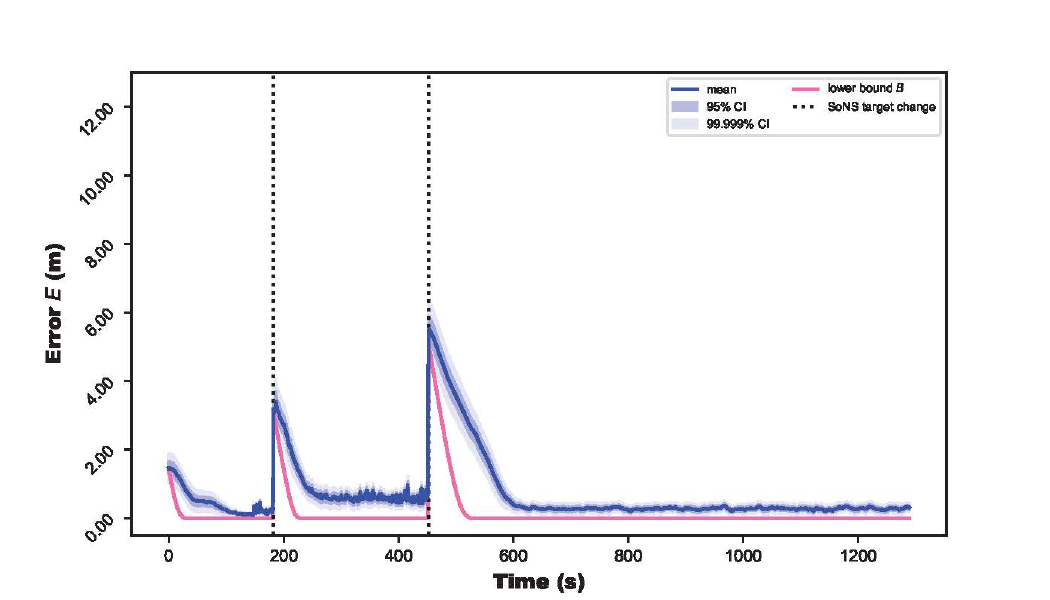}\\
\caption{{\bf Scalability in the binary decision-making mission:} Example simulation trials of four different system sizes (from top to bottom, 35, 65, 95, and 125 robots). 50 trials per system size were conducted in simulation, for four different system sizes that are all shown here {\it (figure continued on next page}).}
\label{fig:scalability-variant1-simulation}
\end{figure}

\vspace{7mm}
\noindent
{\it (Section continued on next page.)}

\clearpage

\begin{figure}[h!]
\ContinuedFloat
\centering
\includegraphics[trim=120 60 120 70,clip,width=0.38\textwidth]{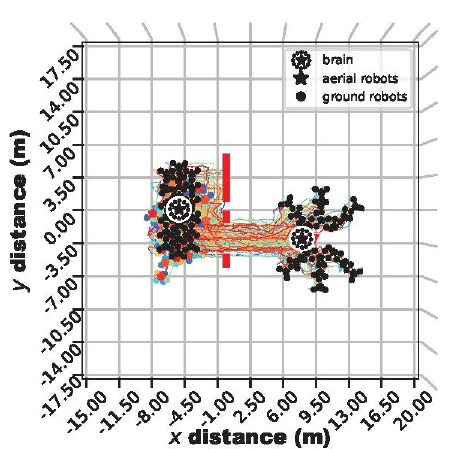}
\includegraphics[trim=20 0 40 20,clip,width=0.59\textwidth]{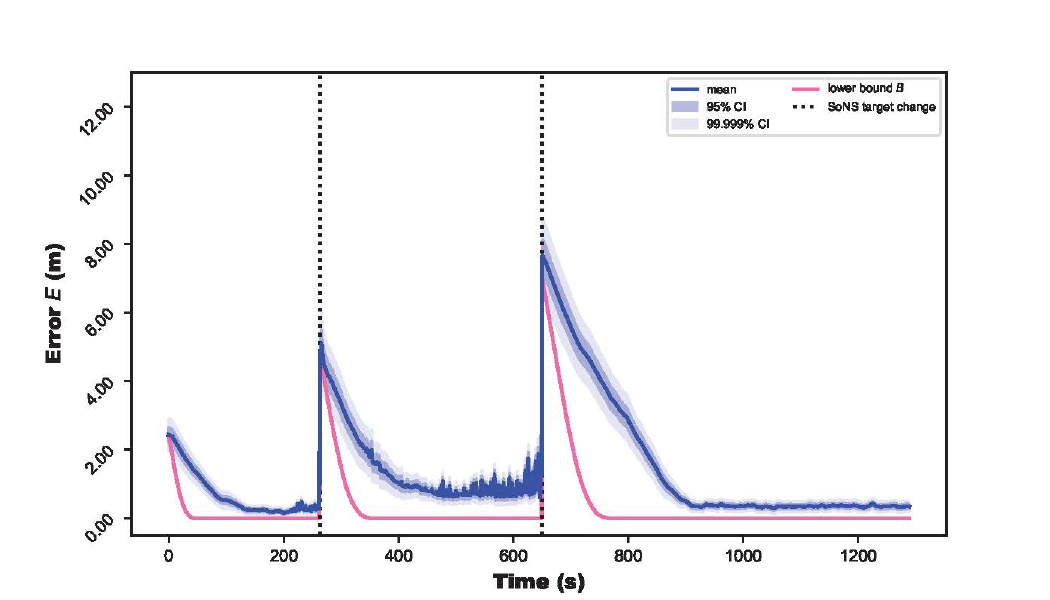}\\
\includegraphics[trim=120 60 120 70,clip,width=0.38\textwidth]{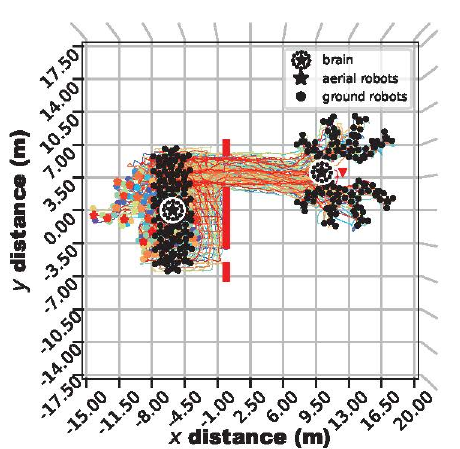}
\includegraphics[trim=20 0 40 20,clip,width=0.59\textwidth]{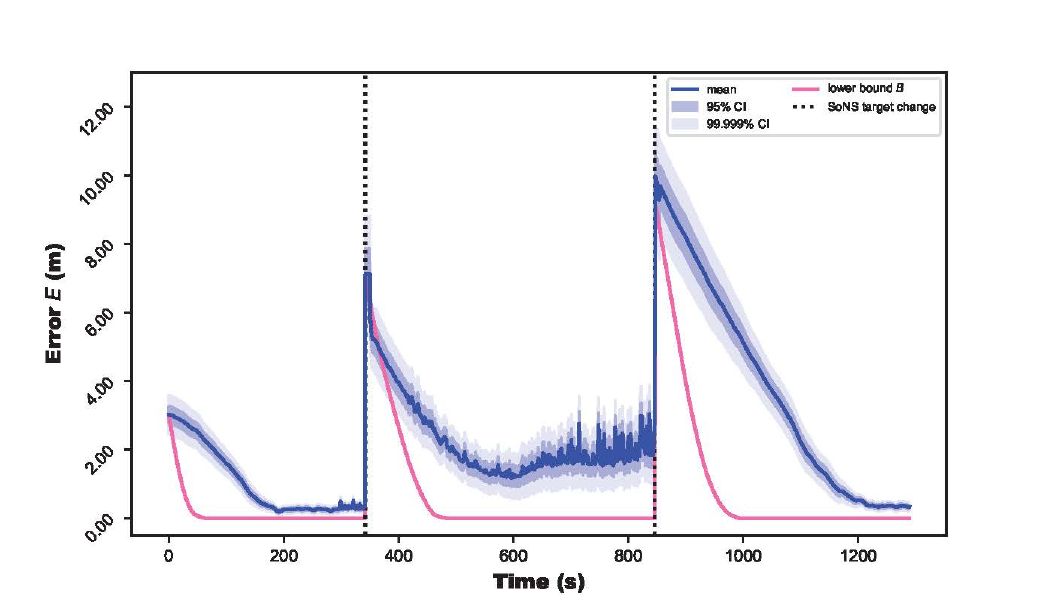}\\
\caption{{\it (cont'd)} {\bf Scalability in the binary decision-making mission:} Example simulation trials of four different system sizes (from top to bottom, 35, 65, 95, and 125 robots). 50 trials per system size were conducted in simulation, for four different system sizes that are all shown here.}
\label{fig:scalability-variant1-simulation}
\end{figure}

\vspace{7mm}
\noindent
{\it (Section continued on next page.)}

\clearpage
\subsubsection*{Scalability in the establishing self-organized hierarchy mission}

\begin{figure}[h!]
\centering
\includegraphics[trim=120 60 120 80,clip,width=0.38\textwidth]{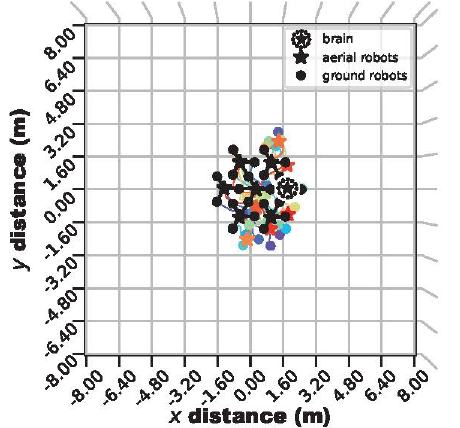}
\includegraphics[trim=20 0 40 20,clip,width=0.59\textwidth]{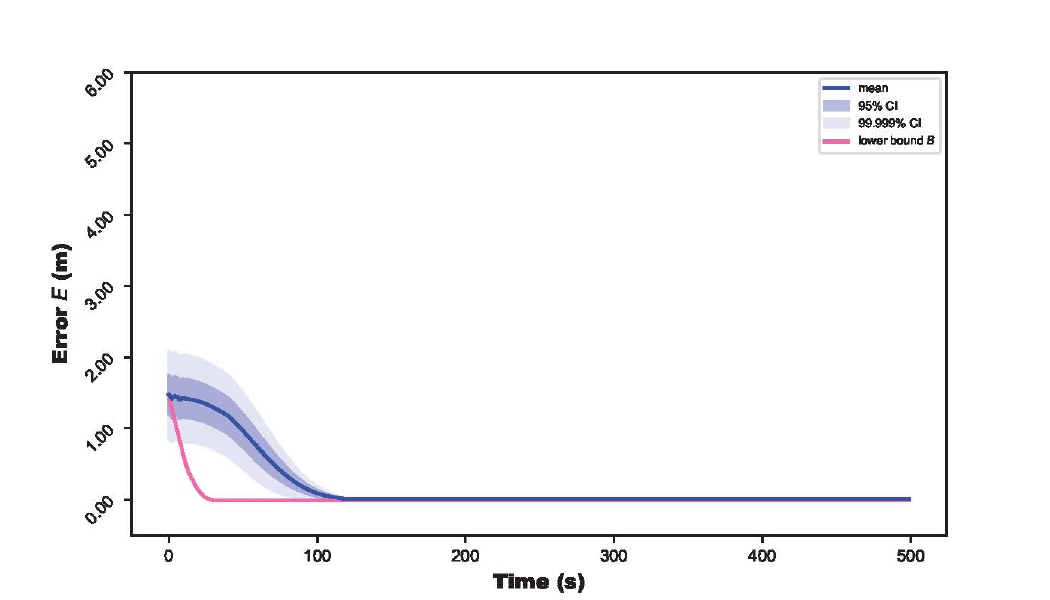}\\
\includegraphics[trim=120 60 120 80,clip,width=0.38\textwidth]{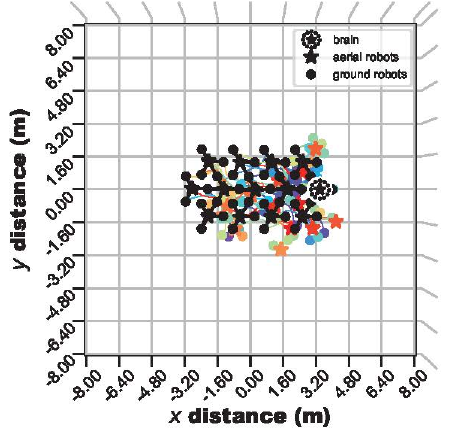}
\includegraphics[trim=20 0 40 20,clip,width=0.59\textwidth]{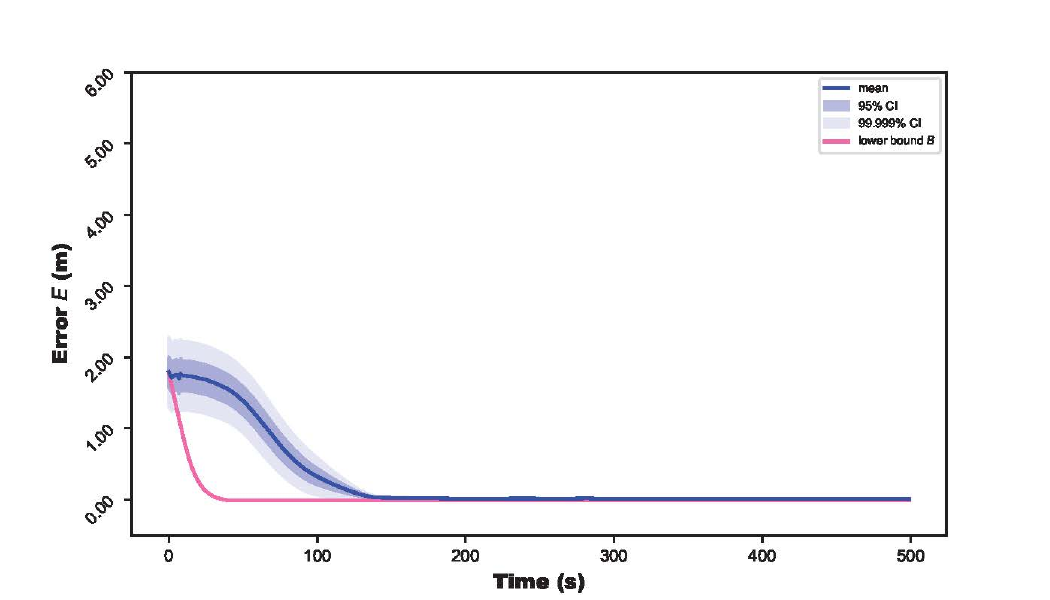}\\
\caption{{\bf Scalability in the establishing self-organized hierarchy mission:} Example simulation trials of four different system sizes (from top to bottom, 35, 65, 95, and 125 robots). 30 trials per system size were conducted in simulation, for 50 different system sizes in total {\it (figure continued on next page}).}
\label{fig:scalability-variant2-simulation}
\end{figure}

\vspace{7mm}
\noindent
{\it (Section continued on next page.)}

\clearpage

\begin{figure}[h!]
\ContinuedFloat
\centering
\includegraphics[trim=120 60 120 80,clip,width=0.38\textwidth]{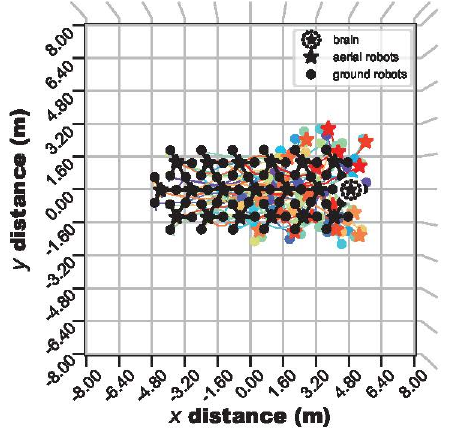}
\includegraphics[trim=20 0 40 20,clip,width=0.59\textwidth]{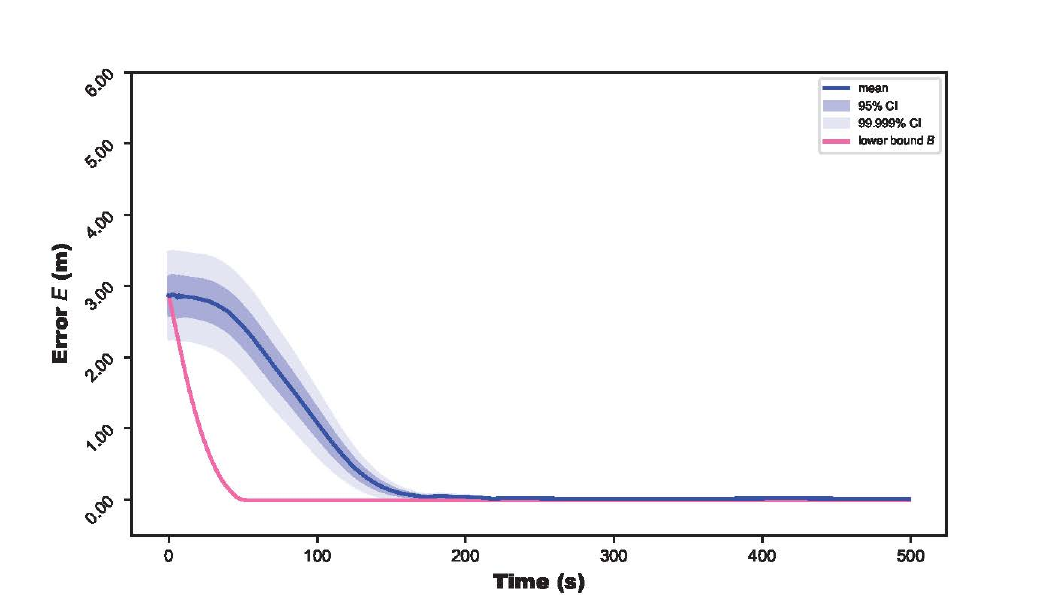}\\
\includegraphics[trim=120 60 120 80,clip,width=0.38\textwidth]{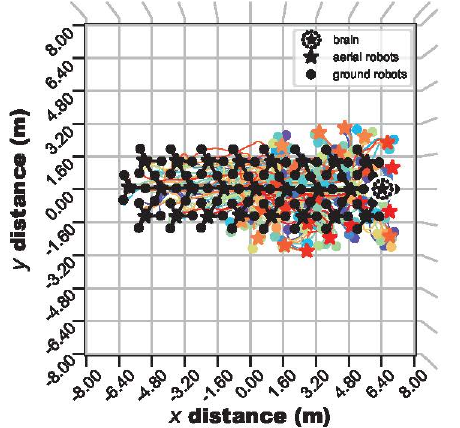}
\includegraphics[trim=20 0 40 20,clip,width=0.59\textwidth]{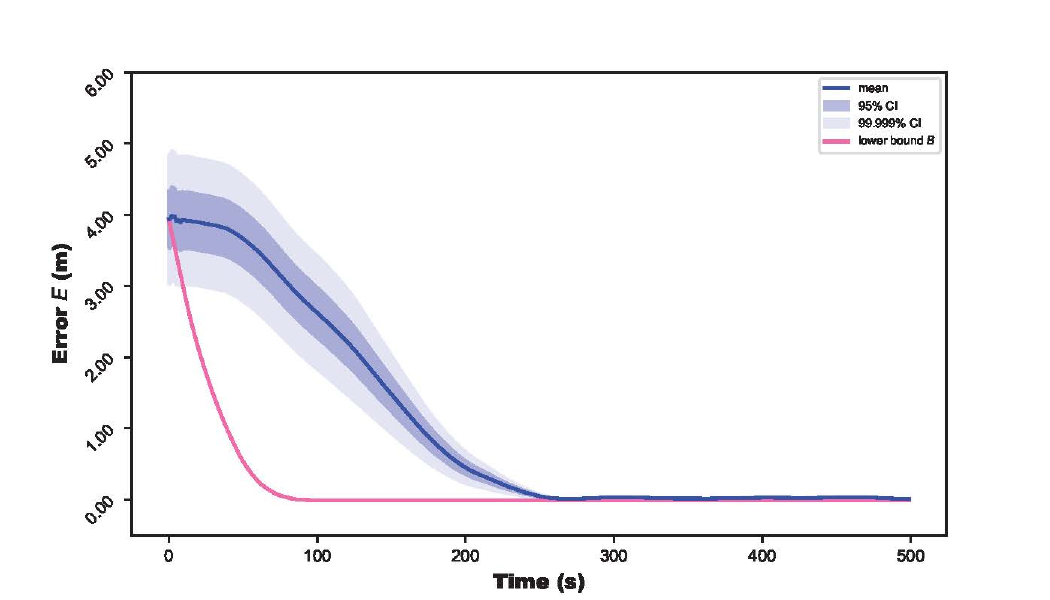}\\
\caption{{\it (cont'd)} {\bf Scalability in the establishing self-organized hierarchy mission:} Example simulation trials of four different system sizes (from top to bottom, 35, 65, 95, and 125 robots). 30 trials per system size were conducted in simulation, for 50 different system sizes in total.}
\label{fig:scalability-variant2-simulation}
\end{figure}

\vspace{7mm}
\noindent
{\it (Section continued on next page.)}

\clearpage
\subsection*{Fault tolerance setups (see Sec.~2.3 in the main paper)}
\rhead{Fault tolerance setups}

The fault tolerance setups include four variants that are all based on the same robot mission (shown in Sec.~2.1.4 in the main paper), with one variant run in experiments with real robots and the other three run in simulation.

\subsubsection*{Variant with real robots: Multiple permanent failures}
\begin{figure}[h!]
\centering
\includegraphics[trim=120 60 120 80,clip,width=0.38\textwidth]{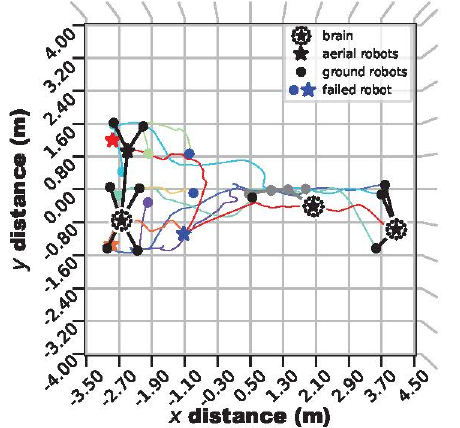}
\includegraphics[trim=20 0 40 20,clip,width=0.59\textwidth]{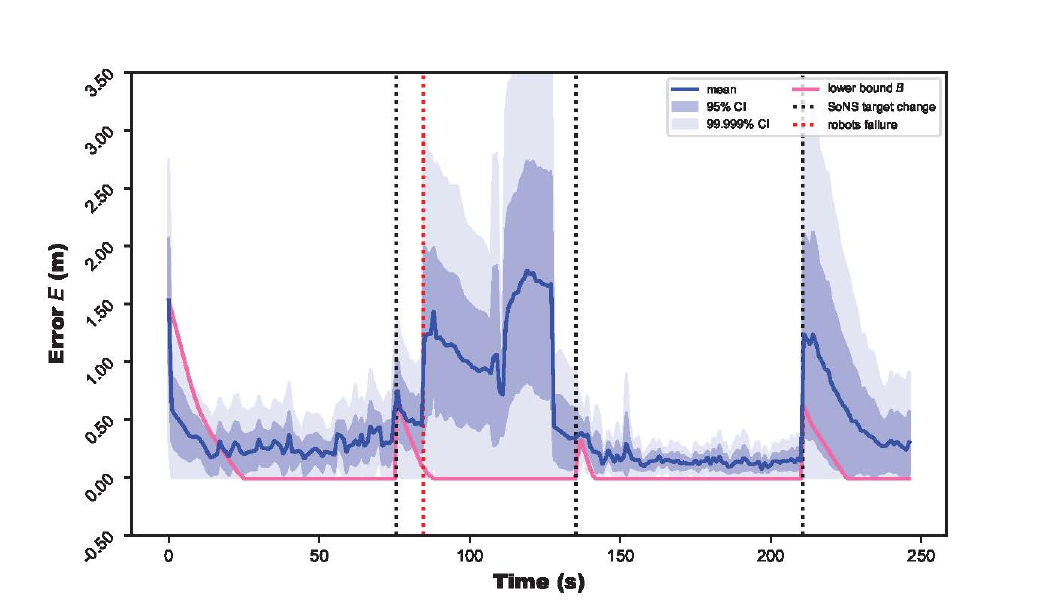}\\
\includegraphics[trim=120 60 120 80,clip,width=0.38\textwidth]{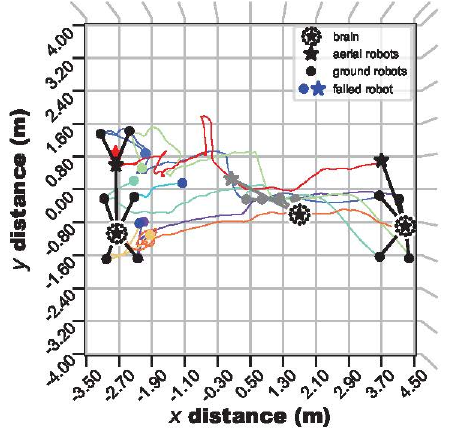}
\includegraphics[trim=20 0 40 20,clip,width=0.59\textwidth]{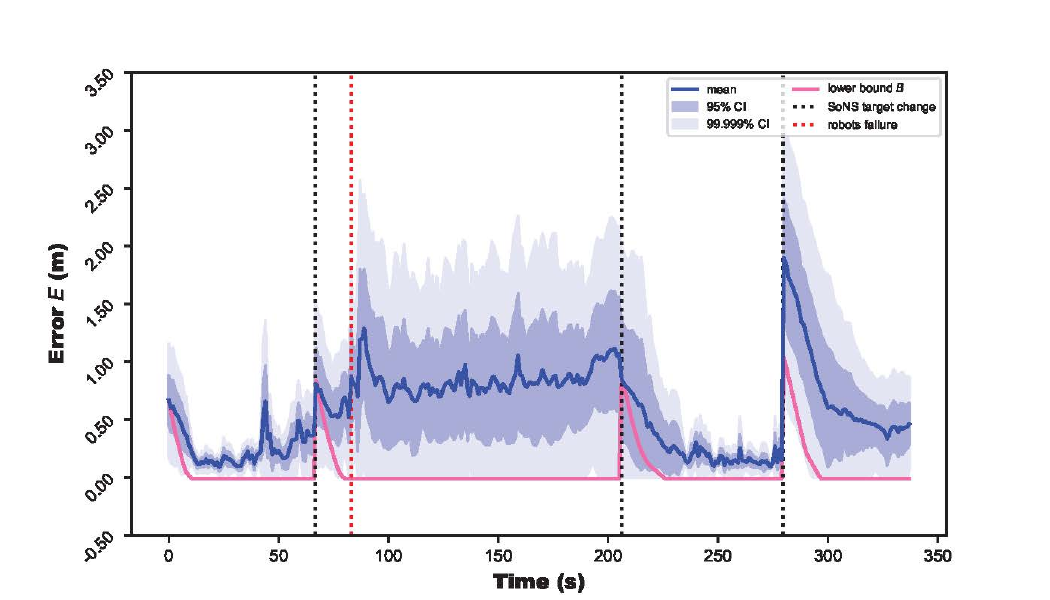}\\
\caption{{\bf Multiple permanent failures: Real robot trials.} Five trials with real robots were conducted, each with eight robots. Failures are triggered uniformly randomly. Note that in the third trial shown here, both aerial robots failed, and therefore the remaining functional robots (all ground robots) did not fulfill the requirements of the mission and did not continue with the mission after the failures occurred {\it (figure continued on next page)}.}
\label{fig:fault-hardware}
\end{figure}

\vspace{7mm}
\noindent
{\it (Section continued on next page.)}

\clearpage

\begin{figure}[h!]
\ContinuedFloat
\centering
\includegraphics[trim=120 60 120 80,clip,width=0.38\textwidth]{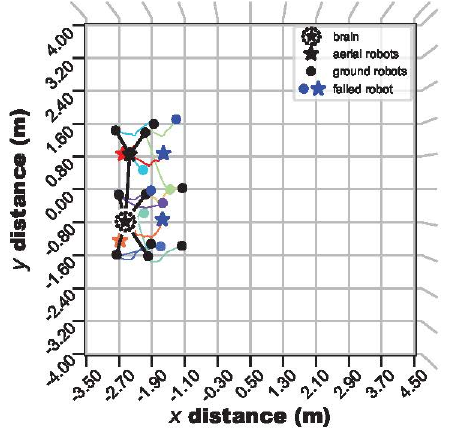}
\includegraphics[trim=20 0 40 20,clip,width=0.59\textwidth]{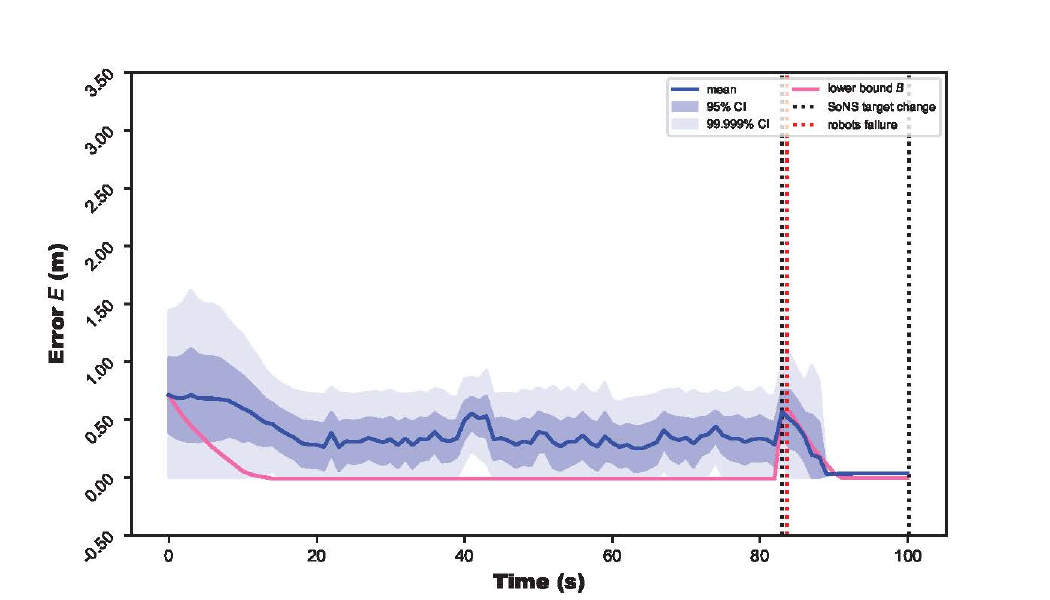}\\
\includegraphics[trim=120 60 120 80,clip,width=0.38\textwidth]{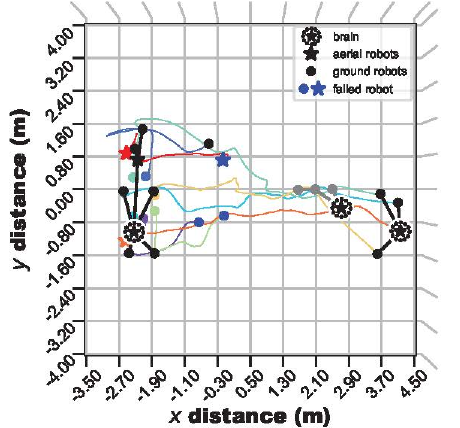}
\includegraphics[trim=20 0 40 20,clip,width=0.59\textwidth]{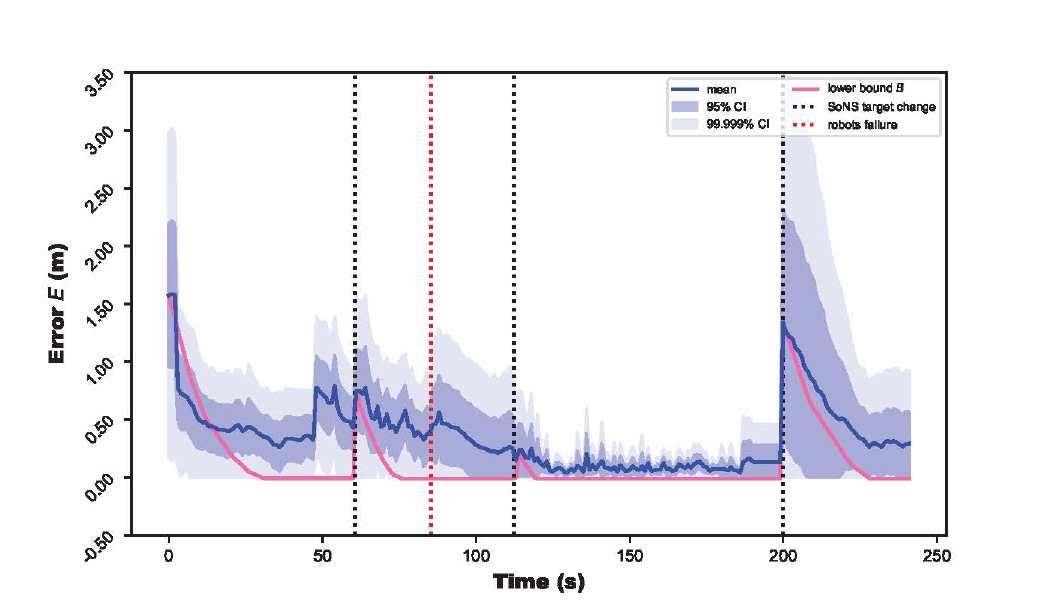}\\
\includegraphics[trim=120 60 120 80,clip,width=0.38\textwidth]{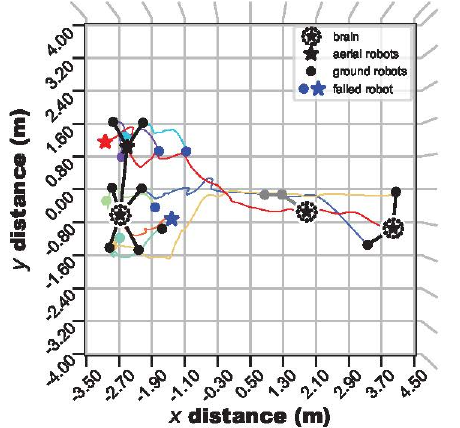}
\includegraphics[trim=20 0 40 20,clip,width=0.59\textwidth]{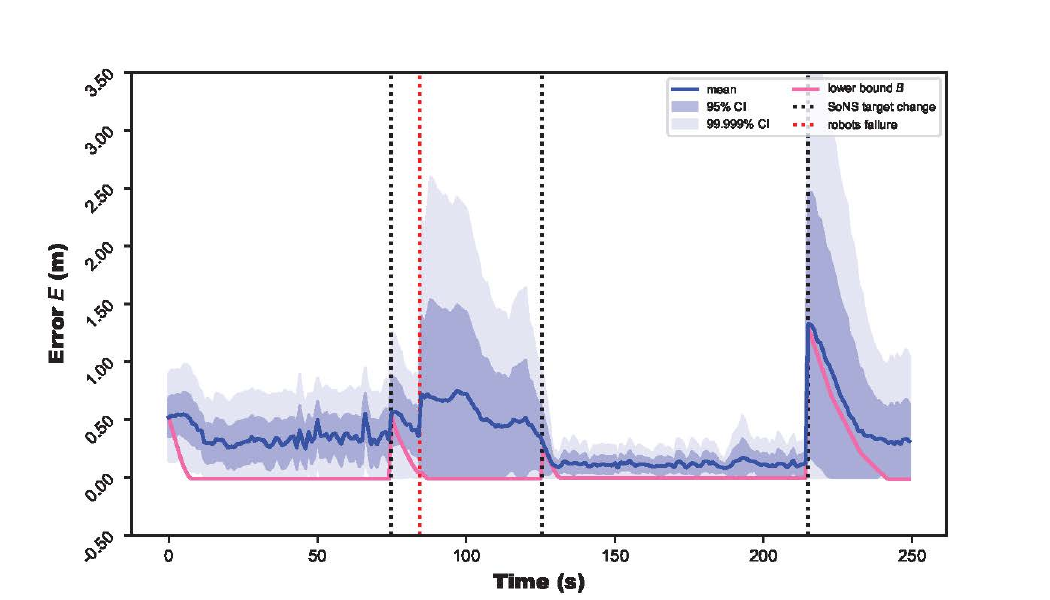}\\
\caption{{\it (cont'd)} {\bf Multiple permanent failures: Real robot trials.} Five trials with real robots were conducted, each with eight robots. Failures are triggered uniformly randomly. Note that in the third trial shown here, both aerial robots failed, and therefore the remaining functional robots (all ground robots) did not fulfill the requirements of the mission and did not continue with the mission after the failures occurred.}
\label{fig:fault-hardware}
\end{figure}

\clearpage
\subsubsection*{Simulation variant: High-loss conditions, 33.$\overline{3}$\% probability to fail}
\begin{figure}[h!]
\centering
\includegraphics[trim=120 60 120 70,clip,width=0.38\textwidth]{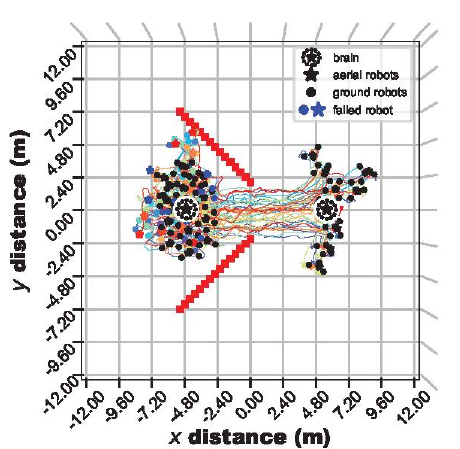}
\includegraphics[trim=20 0 40 20,clip,width=0.59\textwidth]{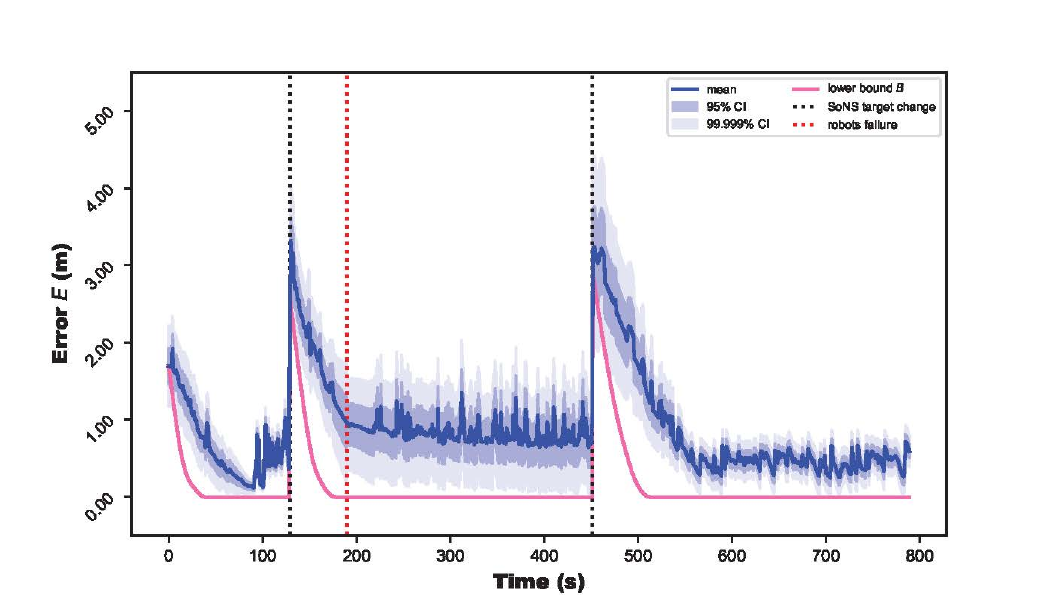}\\
\includegraphics[trim=120 60 120 70,clip,width=0.38\textwidth]{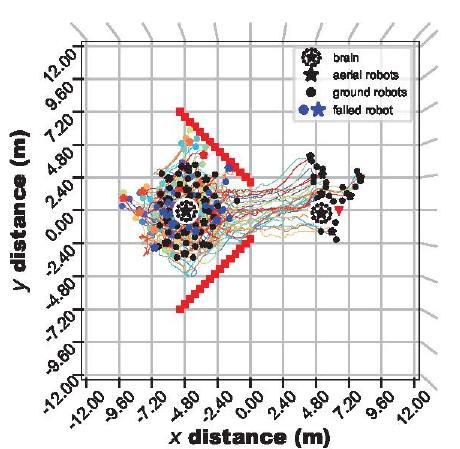}
\includegraphics[trim=20 0 40 20,clip,width=0.59\textwidth]{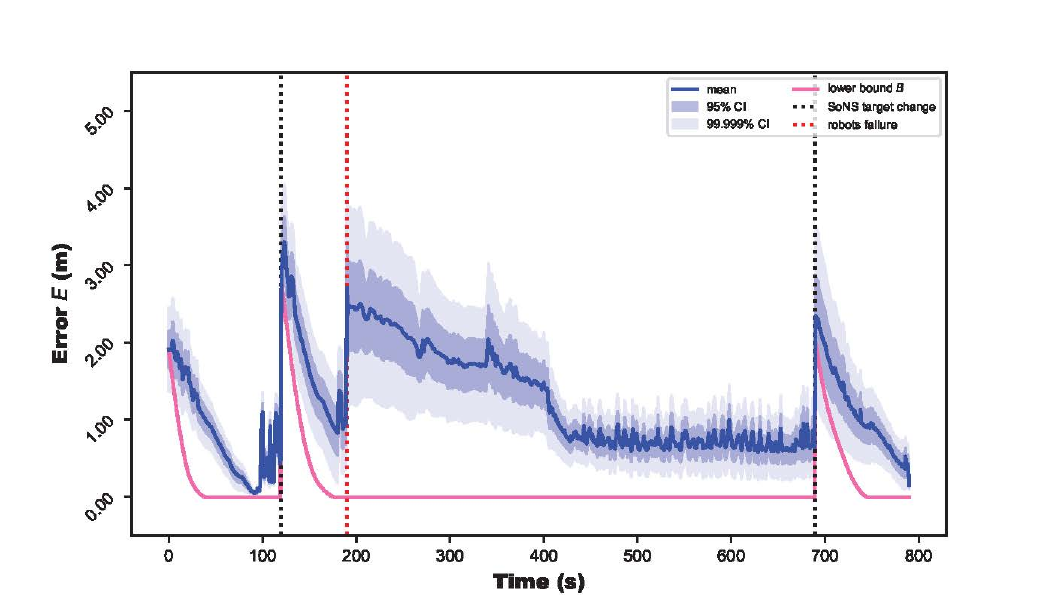}\\
\includegraphics[trim=120 60 120 70,clip,width=0.38\textwidth]{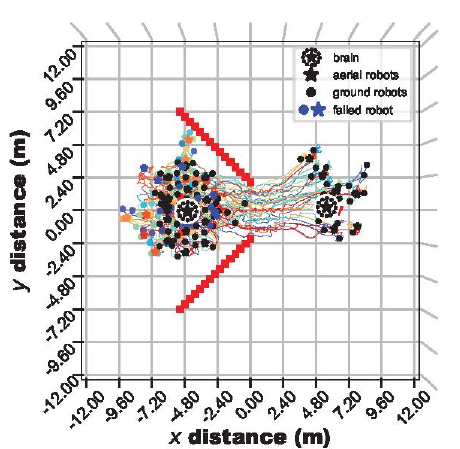}
\includegraphics[trim=20 0 40 20,clip,width=0.59\textwidth]{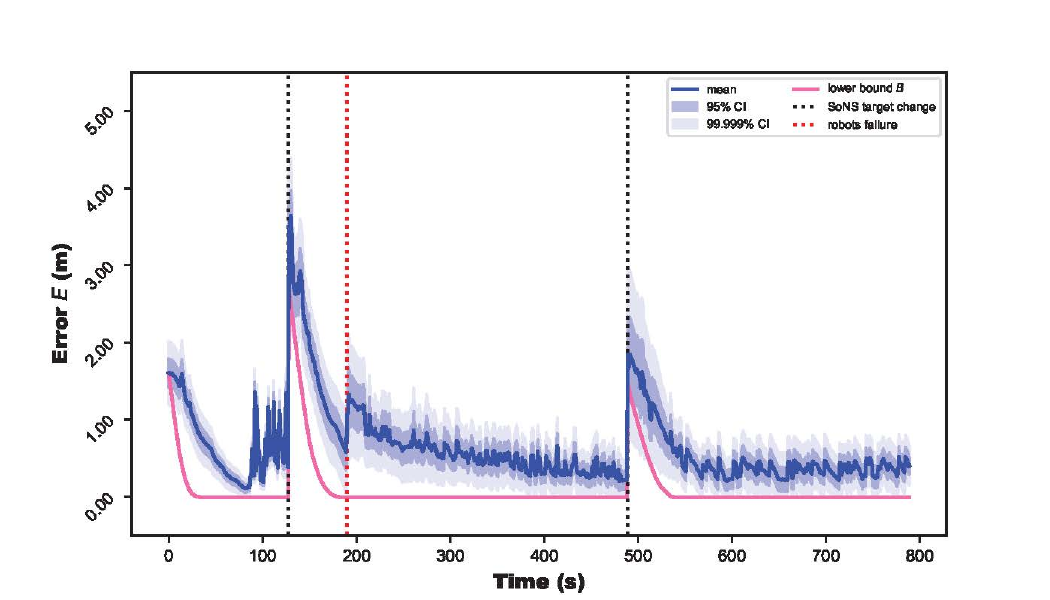}\\
\caption{{\bf High-loss conditions, 33.$\overline{3}$\% probability to fail: Example simulation trials.} 50 trials were conducted in simulation, each with 65 robots.}
\label{fig:fault-variant1-simulation}
\end{figure}

\clearpage
\subsubsection*{Simulation variant: High-loss conditions, 66.$\overline{6}$\% probability to fail}

\begin{figure}[h!]
\centering
\includegraphics[trim=120 60 120 70,clip,width=0.38\textwidth]{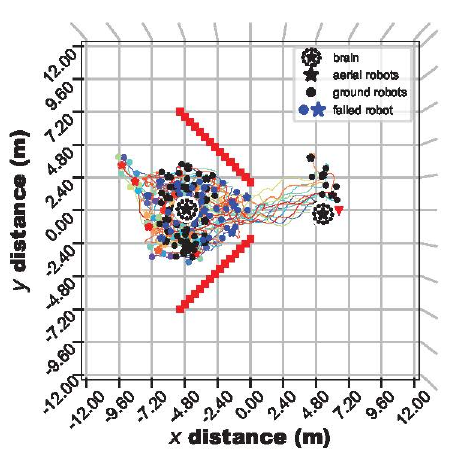}
\includegraphics[trim=20 0 40 20,clip,width=0.59\textwidth]{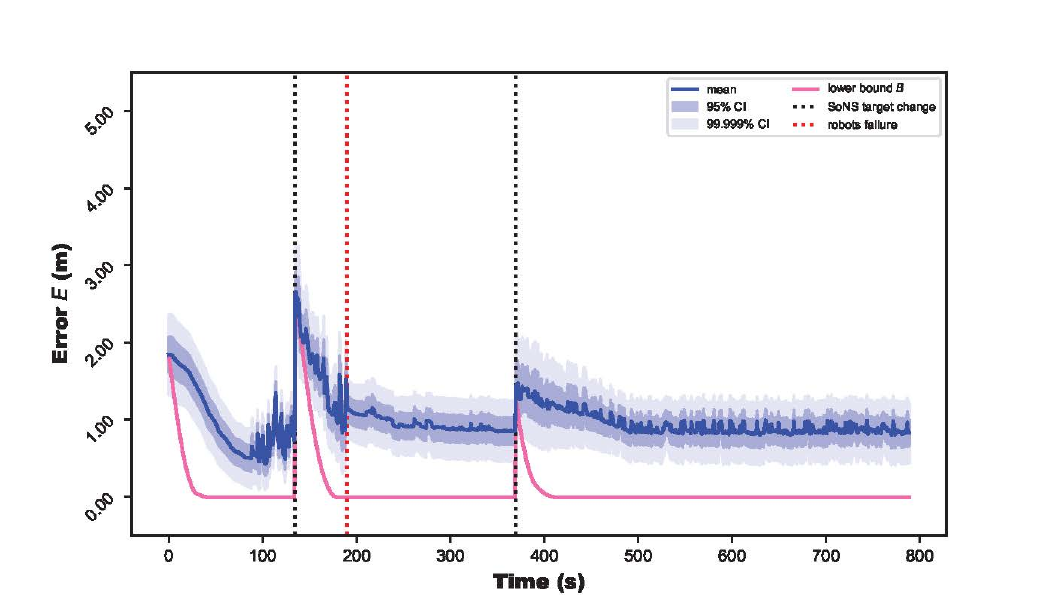}\\
\includegraphics[trim=120 60 120 70,clip,width=0.38\textwidth]{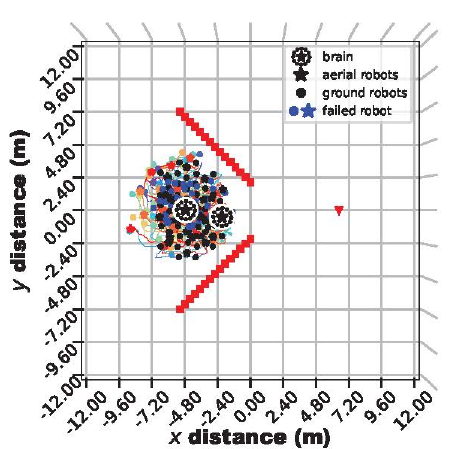}
\includegraphics[trim=20 0 40 20,clip,width=0.59\textwidth]{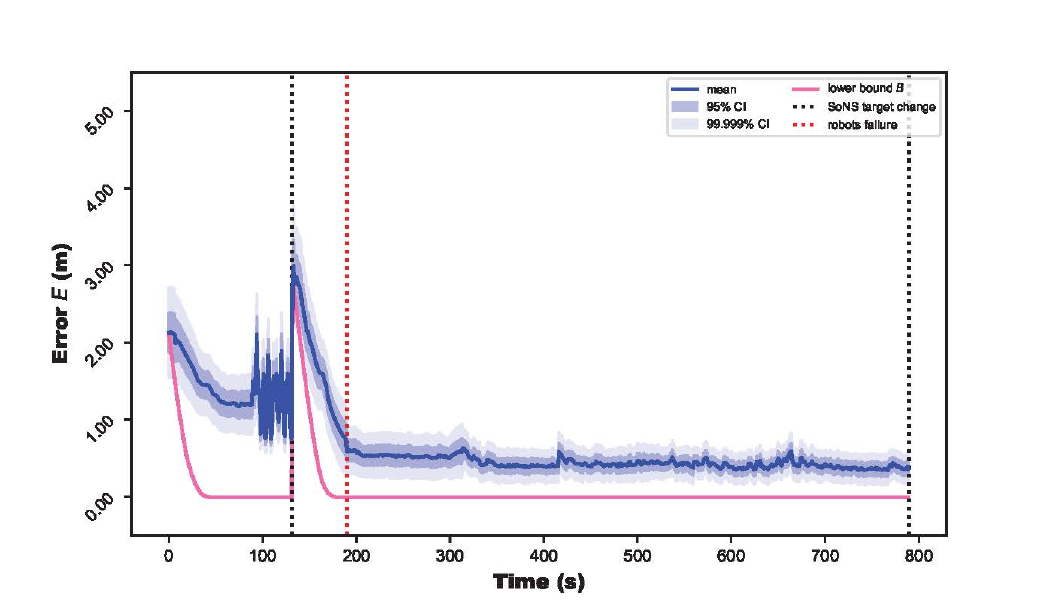}\\
\includegraphics[trim=120 60 120 70,clip,width=0.38\textwidth]{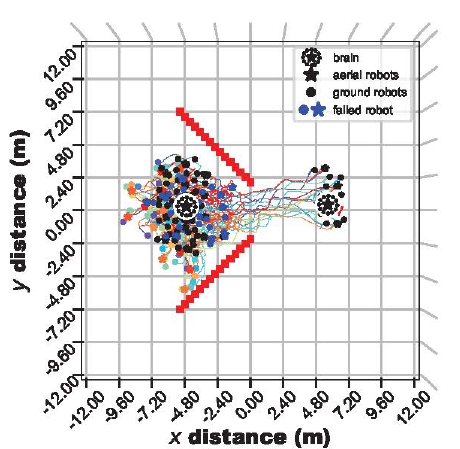}
\includegraphics[trim=20 0 40 20,clip,width=0.59\textwidth]{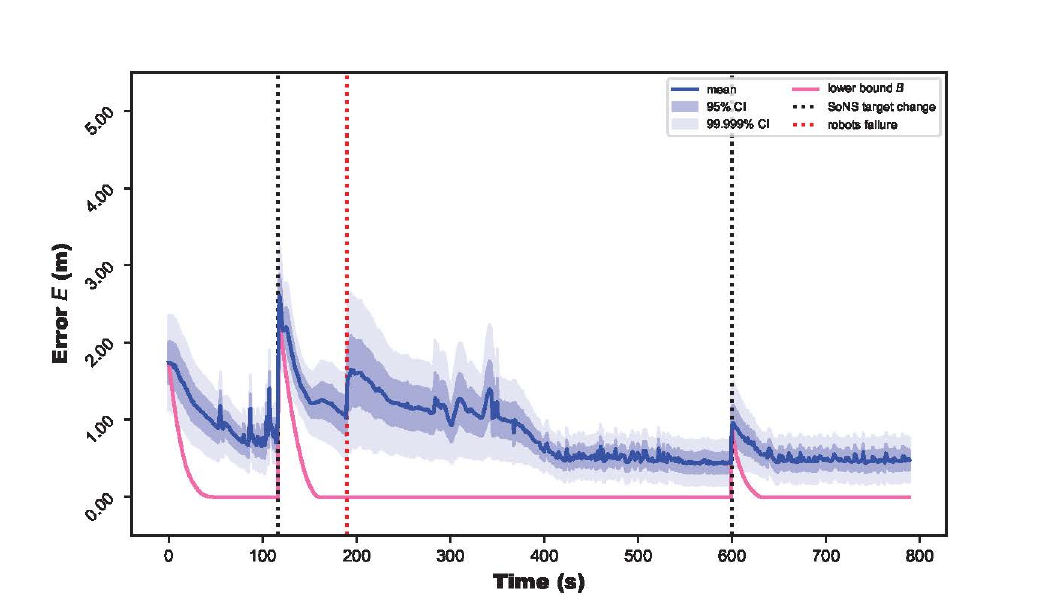}\\
\caption{{\bf High-loss conditions, 66.$\overline{6}$\% probability to fail: Example simulation trials.} Note that in some trials under this condition, the SoNS was able to re-establish itself after the failures occurred, but was not able to complete other portions of the mission due to, for example, functional ground robots in the SoNS becoming trapped by the failed ground robots (as in the middle trial shown here). 50 trials were conducted in simulation, each with 65 robots.}
\label{fig:fault-variant2-simulation}
\end{figure}

\clearpage
\subsubsection*{Simulation variant: Temporary system-wide vision failure}

\begin{figure}[h!]
\centering
\includegraphics[trim=120 60 120 70,clip,width=0.38\textwidth]{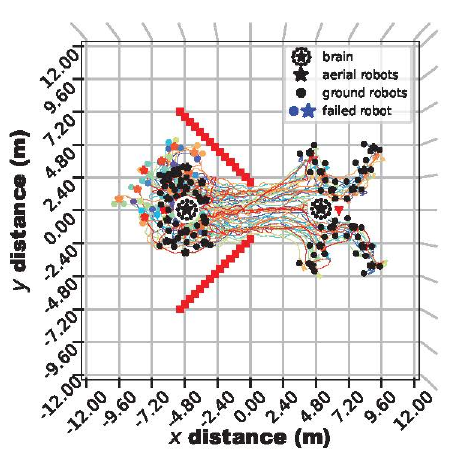}
\includegraphics[trim=20 0 180 20,clip,width=0.59\textwidth]{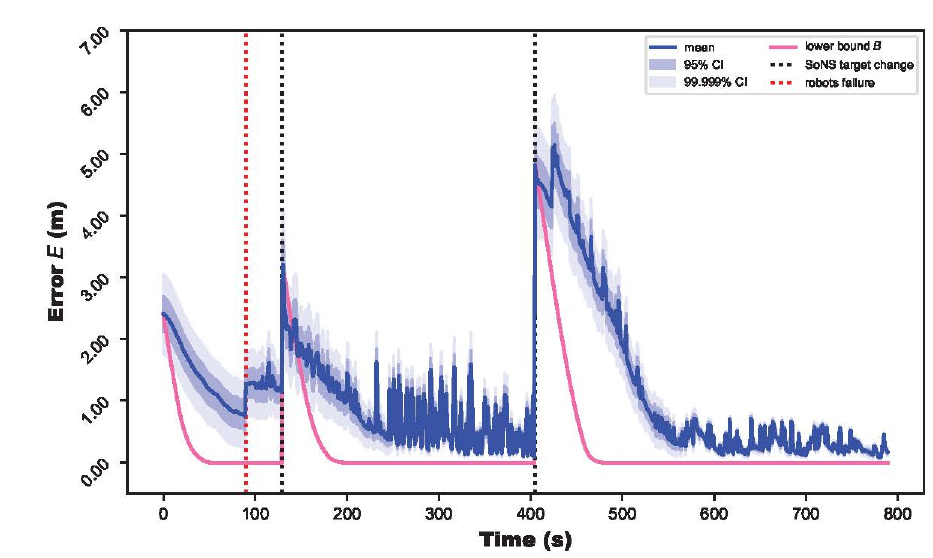}\\
\includegraphics[trim=120 60 120 70,clip,width=0.38\textwidth]{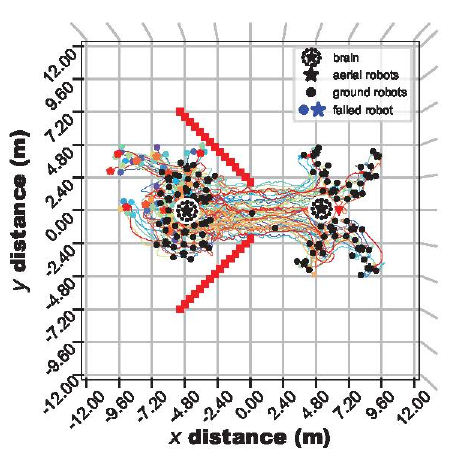}
\includegraphics[trim=20 0 180 20,clip,width=0.59\textwidth]{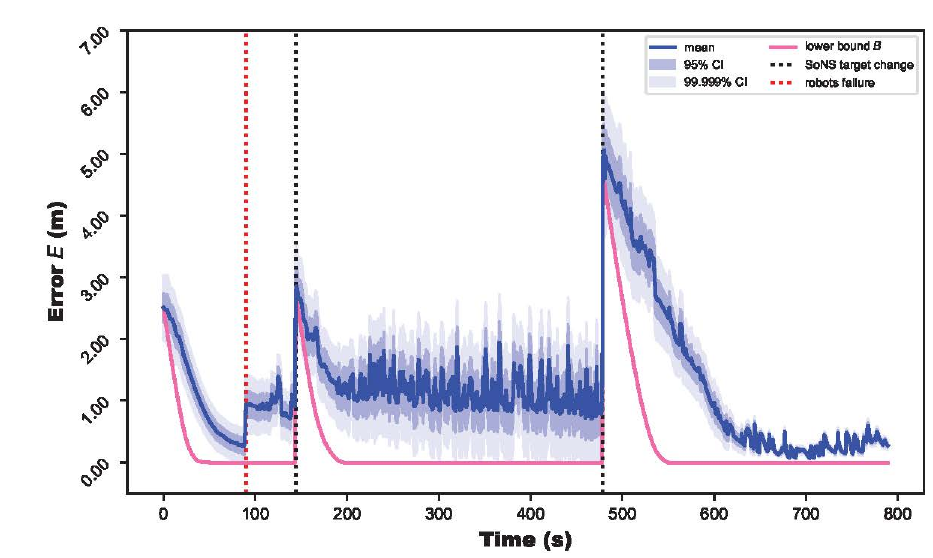}\\
\includegraphics[trim=120 60 120 70,clip,width=0.38\textwidth]{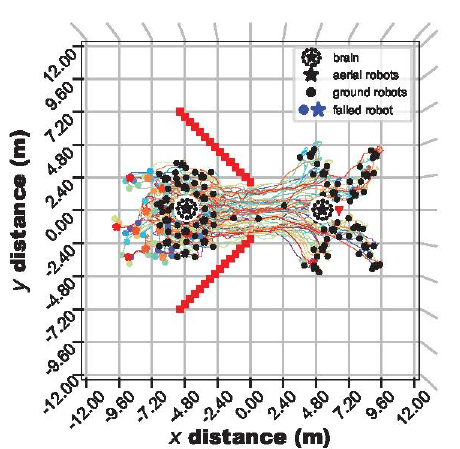}
\includegraphics[trim=20 0 180 20,clip,width=0.59\textwidth]{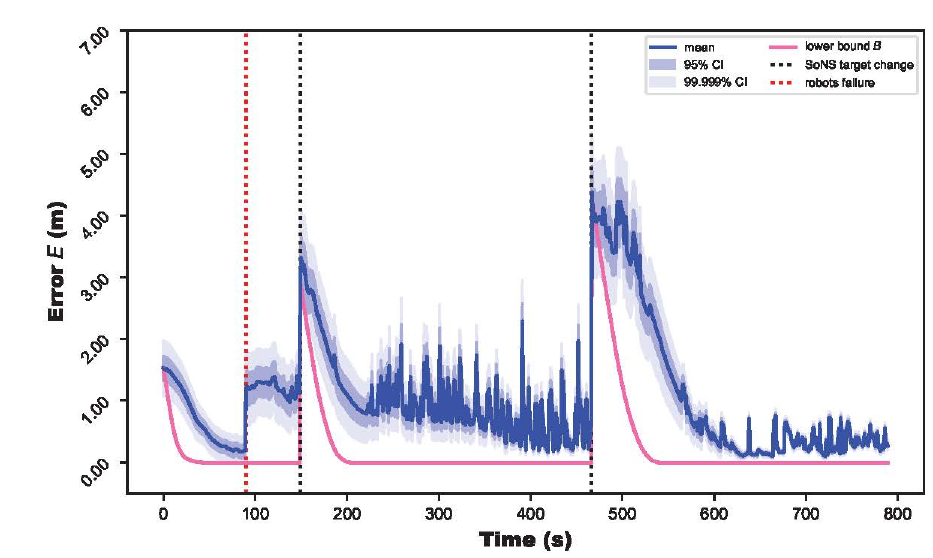}\\
\caption{{\bf Temporary system-wide vision failure: Example simulation trials of three different failure durations (from top to bottom, 0.5, 1.0, and 30~s).} 50 trials per duration were conducted in simulation, each with 65 robots.}
\label{fig:fault-variant3-simulation}
\end{figure}

\clearpage
\subsubsection*{Simulation variant: Temporary system-wide communication failure}

\begin{figure}[h!]
\centering
\includegraphics[trim=120 60 120 70,clip,width=0.38\textwidth]{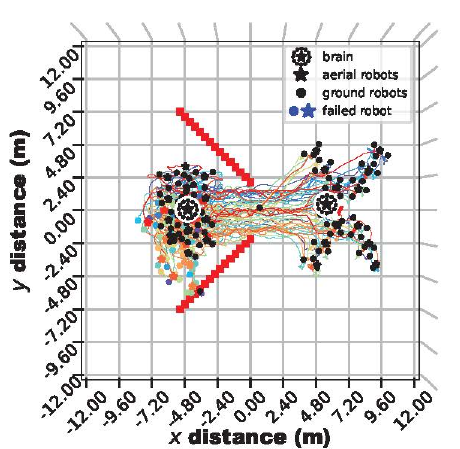}
\includegraphics[trim=20 0 180 20,clip,width=0.59\textwidth]{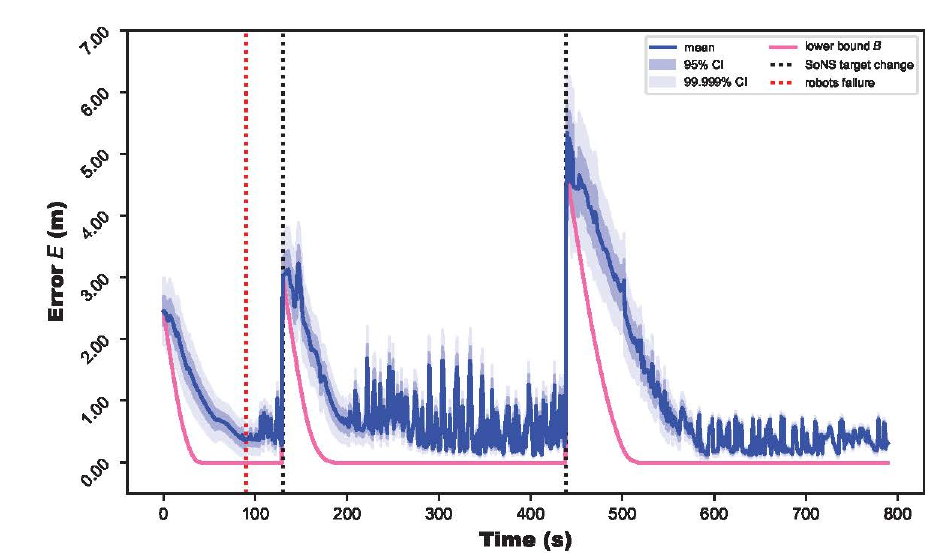}\\
\includegraphics[trim=120 60 120 70,clip,width=0.38\textwidth]{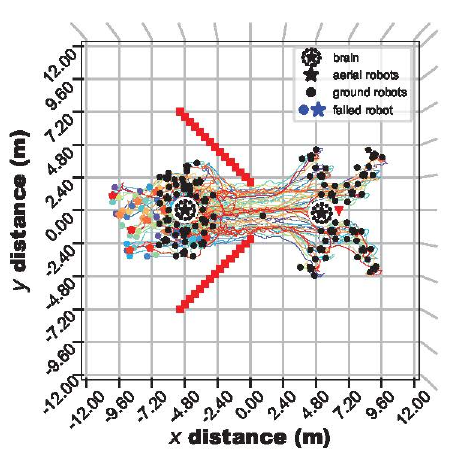}
\includegraphics[trim=20 0 180 20,clip,width=0.59\textwidth]{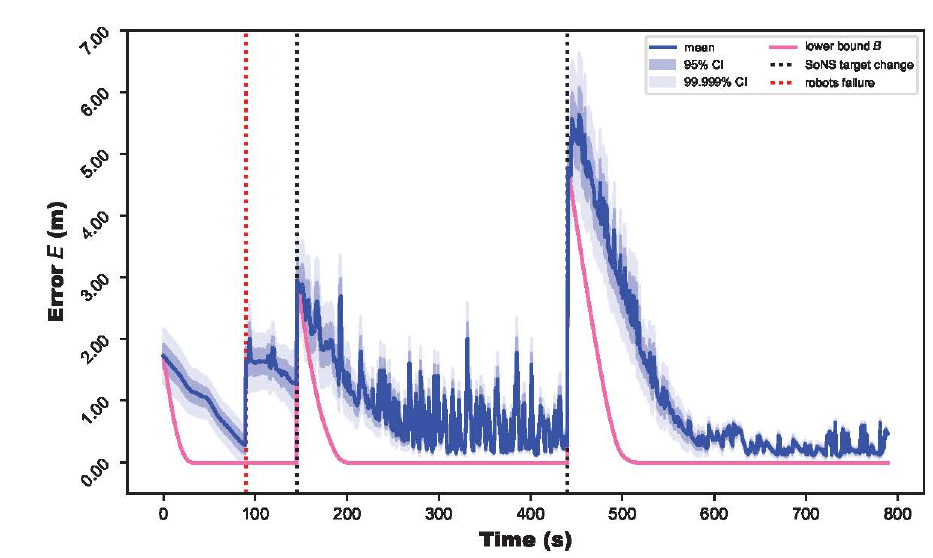}\\
\includegraphics[trim=120 60 120 70,clip,width=0.38\textwidth]{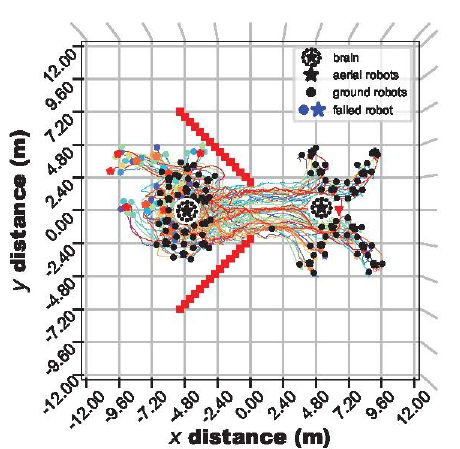}
\includegraphics[trim=20 0 180 20,clip,width=0.59\textwidth]{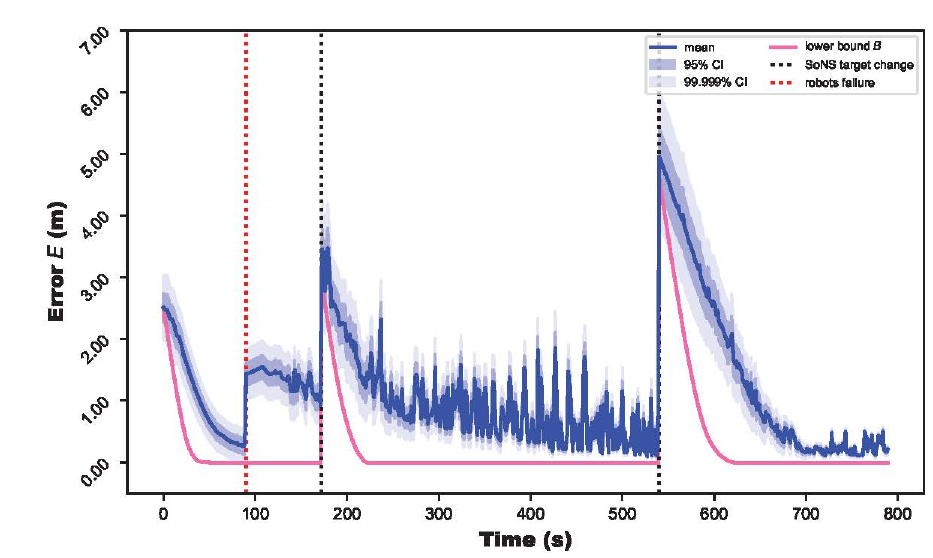}\\
\caption{{\bf Temporary system-wide communication failure: Example simulation trials of three different failure durations (from top to bottom, 0.5, 1.0, and 30~s).} 50 trials per duration were conducted in simulation, each with 65 robots.}
\label{fig:fault-variant4-simulation}
\end{figure}

\clearpage\rhead{}
\section*{Section \ref*{SM:cross-verify}. Cross-verification of the SoNS in simulation and on real hardware}
\lhead{Section \ref*{SM:cross-verify}. Cross-verification}
To cross-verify the SoNS behaviors in the simulator with those of the real robots, we ran simulation experiments (50 trials per setup) with the same setups as six of the experiments with real robots. This section provides results from an example simulation trial of each of the cross-verification setups as well as statistical comparisons of the 50 simulation trials and five real robot trials of the matching setup. The experiment data for all trials of all setups is available in the online data repository.

\subsection*{Example simulation trials that match real robot setups}

For each example trial included in this section, we provide the following results: (on the left) trajectories of the robots over time, with the final (and sometimes the initial and/or an intermediary) SoNS indicated in black; and (on the right) the mean and confidence interval per robot of the actuation error $E$ over time (see Eq.~1 in Sec.~4.2 in the main paper), with the lower bound $B$ (see Eq.~2 in Sec.~4.2 in the main paper) plotted for reference.

\begin{figure}[h!]
\centering
\includegraphics[trim=120 60 120 80,clip,width=0.38\textwidth]{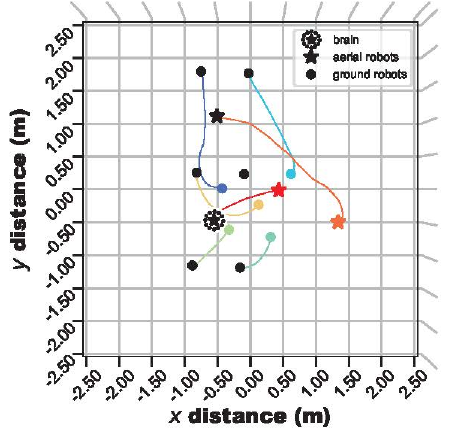}
\includegraphics[trim=20 0 40 20,clip,width=0.59\textwidth]{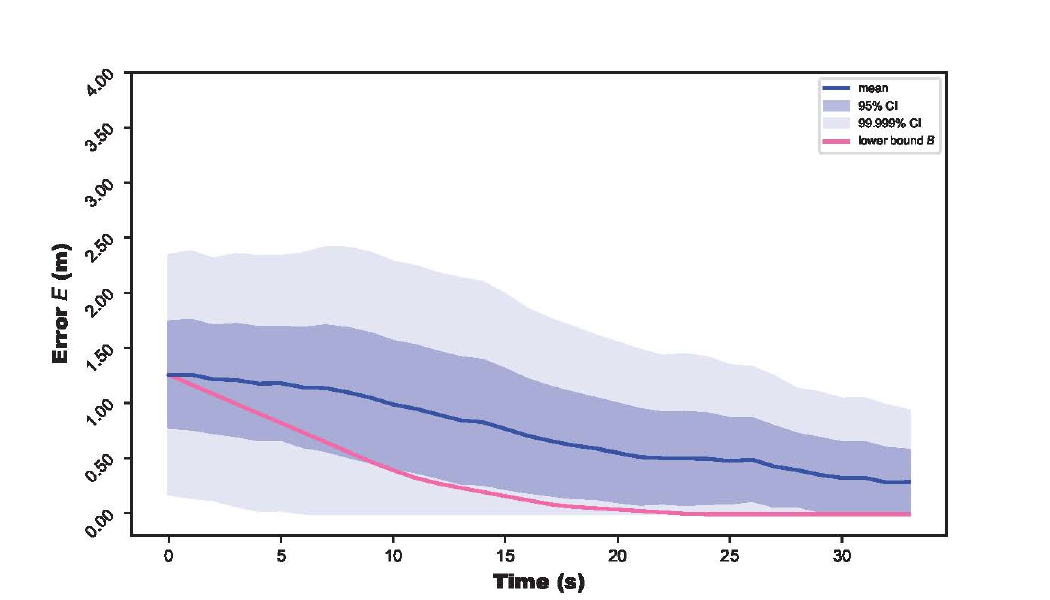}
\caption{{\bf Establishing self-organized hierarchy, clustered start:} Example simulation trial with eight robots that matches the setup of the real robot trials shown in Fig.~\ref{fig:mission1-variant1-hardware}.}
\label{fig:mission1-variant1-crossverify}
\end{figure}

\vspace{7mm}
\noindent
{\it (Section continued on next page.)}

\clearpage
\rhead{Example simulation trials}

\begin{figure}[h!]
\centering
\includegraphics[trim=120 60 120 80,clip,width=0.38\textwidth]{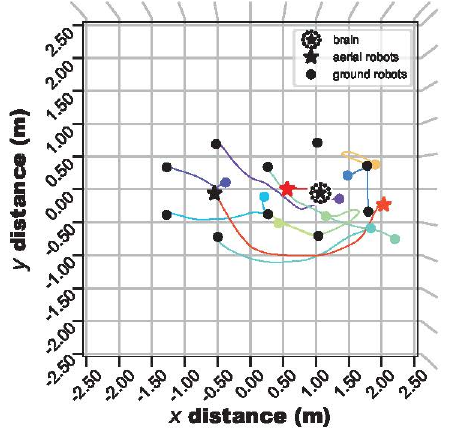}
\includegraphics[trim=20 0 40 20,clip,width=0.59\textwidth]{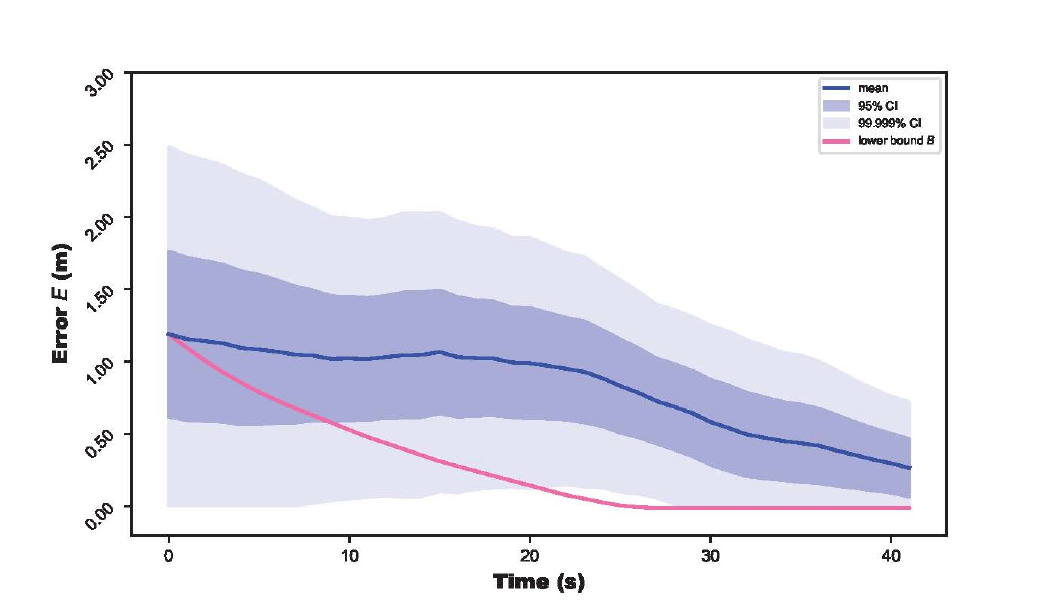}
\caption{{\bf Establishing self-organized hierarchy, scattered start:} Example simulation trial with 12 robots that matches the setup of the real robot trials shown in Fig.~\ref{fig:mission1-variant2-hardware}.}
\label{fig:mission1-variant2-crossverify}
\end{figure}

\begin{figure}[h!]
\centering
\includegraphics[trim=120 60 120 80,clip,width=0.38\textwidth]{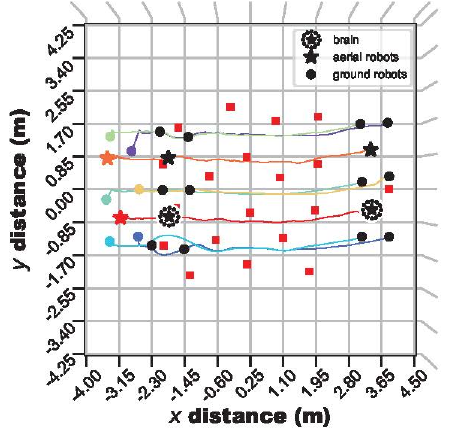}
\includegraphics[trim=20 0 40 20,clip,width=0.59\textwidth]{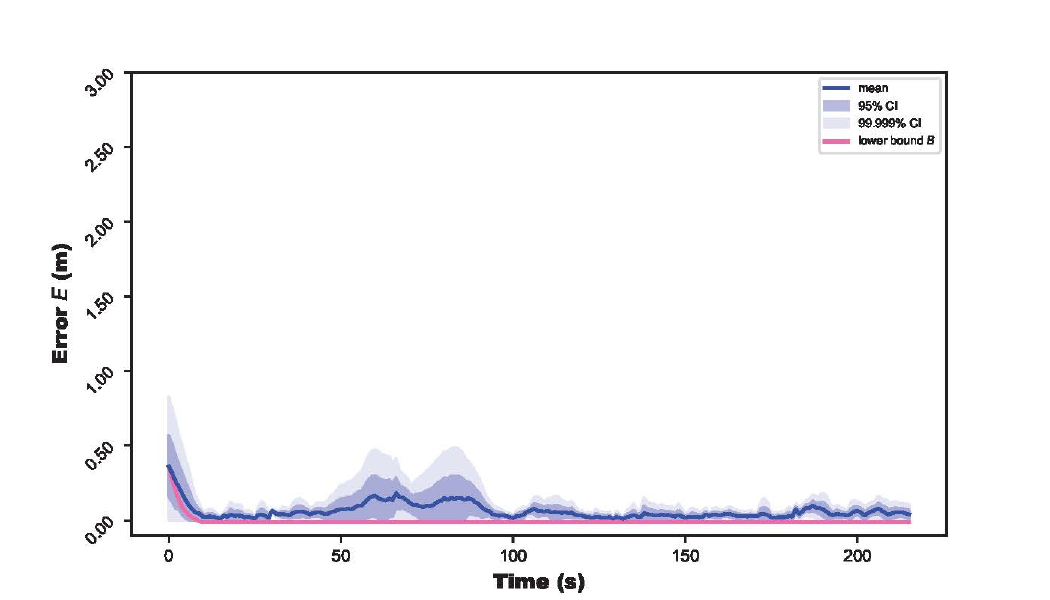}
\caption{{\bf Balancing global and local goals, with smaller, denser obstacles:} Example simulation trial with eight robots that matches the setup of the real robot trials shown in Fig.~\ref{fig:mission2-variant1-hardware}.}
\label{fig:mission2-variant1-crossverify}
\end{figure}

\vspace{7mm}
\noindent
{\it (Section continued on next page.)}

\clearpage

\begin{figure}[h!]
\centering
\includegraphics[trim=120 60 120 80,clip,width=0.38\textwidth]{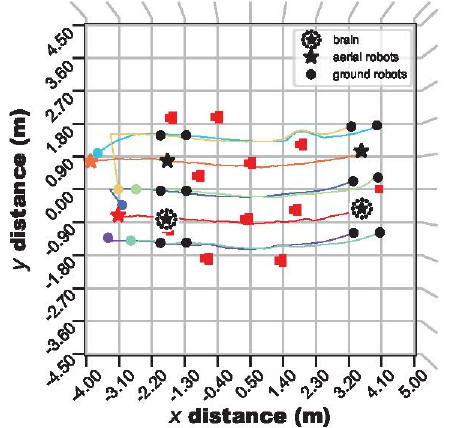}
\includegraphics[trim=20 0 40 20,clip,width=0.59\textwidth]{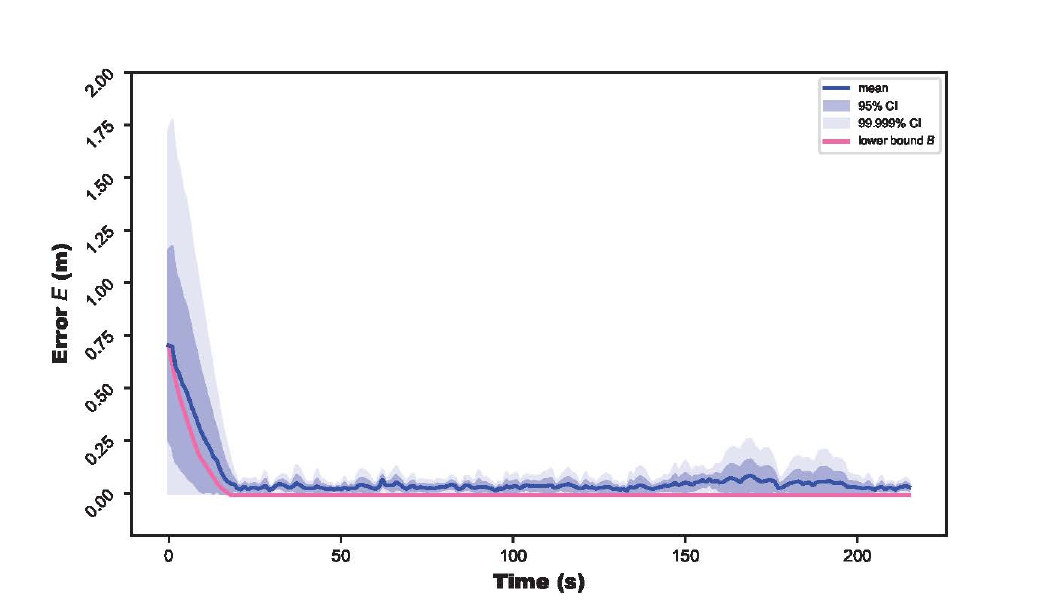}
\caption{{\bf Balancing global and local goals, with larger, less dense obstacles:} Example simulation trial with eight robots that matches the setup of the real robot trials shown in Fig.~\ref{fig:mission2-variant2-hardware}.}
\label{fig:mission2-variant2-crossverify}
\end{figure}

\begin{figure}[h!]
\centering
\includegraphics[trim=120 60 120 80,clip,width=0.38\textwidth]{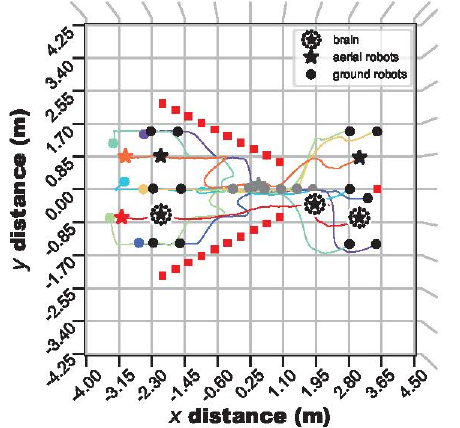}
\includegraphics[trim=20 0 40 20,clip,width=0.59\textwidth]{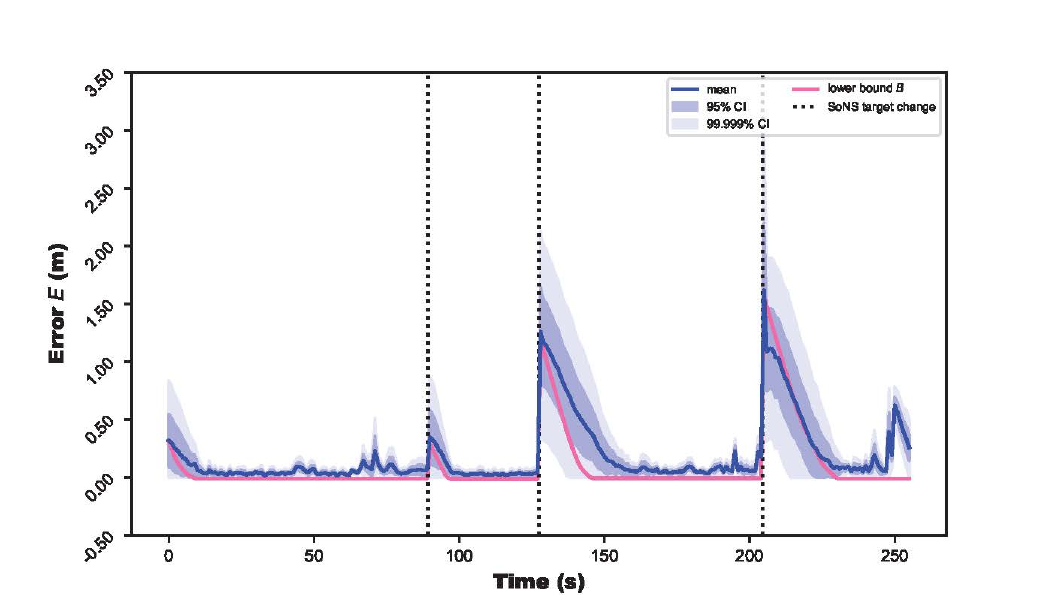}
\caption{{\bf Collective sensing and actuation:} Example simulation trial with eight robots that matches the setup of the real robot trials shown in Fig.~\ref{fig:mission3-hardware}.}
\label{fig:mission3-crossverify}
\end{figure}

\vspace{7mm}
\noindent
{\it (Section continued on next page.)}

\clearpage

\begin{figure}[h!]
\centering
\includegraphics[trim=120 60 120 80,clip,width=0.38\textwidth]{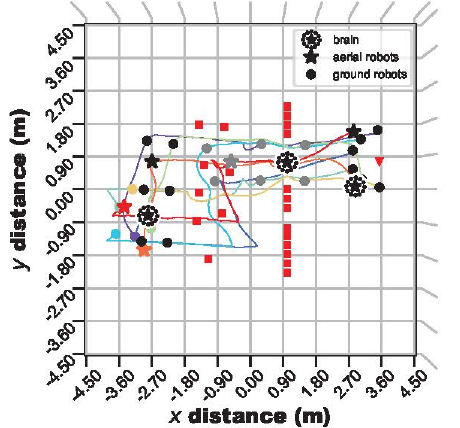}
\includegraphics[trim=20 0 40 20,clip,width=0.59\textwidth]{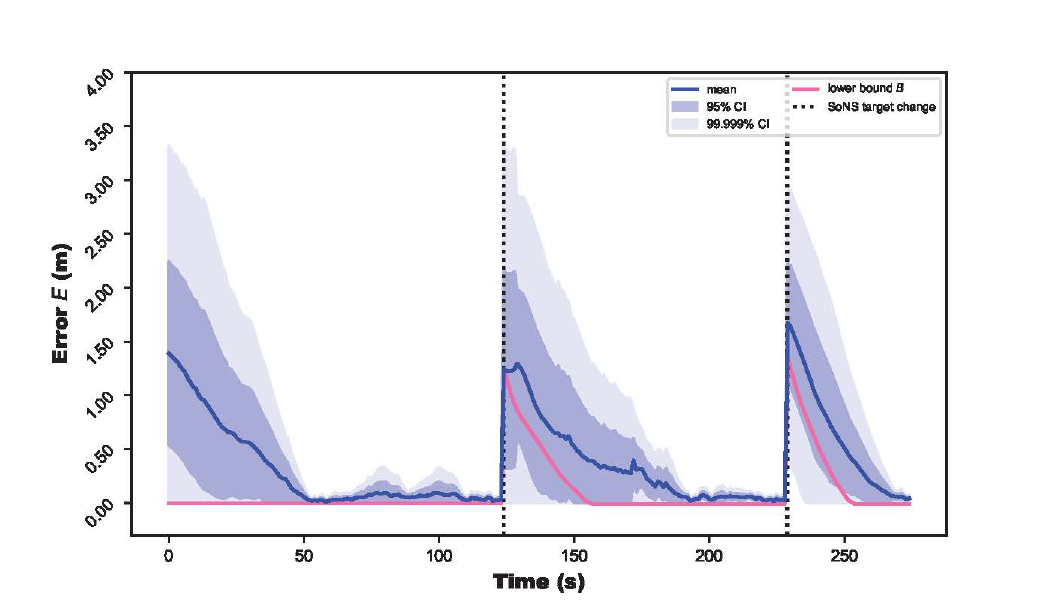}\\
\caption{{\bf Binary decision making:} Example simulation trial that matches the setup of the real robot trials shown in Fig.~\ref{fig:mission4-hardware}.}
\label{fig:mission4-crossverify}
\end{figure}

\vspace{10mm}

\subsection*{Comparisons of trials with real robots and trials in simulation}

The SoNS behaviors observed in the cross-verification simulation trials are very similar to those observed in their real-robot counterparts (see Figs~\ref{fig:mission1-variant1-crossverify}--\ref{fig:mission4-crossverify}). However, there is noticeably higher error in the real robot experiments than in the simulations. To assess the overall difference in error, we provide violin plots and Q--Q plots (quantile--quantile plots) comparing the real robot trials to their matching simulations (see Figs.~\ref{fig:violin-crossverify},~\ref{fig:qq-crossverify}).

It can be seen in the comparative violin plots that the difference in mean error between the real robots and simulation is very low (always less than $E = 0.5$\,m, usually less than $E = 0.25$\,m), but the error in the real robots is greater overall. It can also be seen that the biggest difference in mean error, between the real robots and simulation, occurs in the missions to establish self-organized hierarchy. Indeed, this is expected, as a large portion of experiment time in this mission is spent with the robots in a steady phase, and therefore with the aerial robots attempting to hover in place. Much of the difference in error here can be attributed to the simulation model of the aerial robot, rather than to the SoNS behaviors.
These observations are further confirmed by the Q--Q plots, which show a great difference in the probability distributions of the error between the real robots and simulation in the missions to establish self-organized hierarchy. In the other mission types, the real robots indeed always show greater error, but are much more similar to the simulations overall.

\vspace{7mm}
\noindent
{\it (Section continued on next page.)}

\clearpage

\rhead{Comparisons of simulation and reality}

\begin{figure}[h!]
\centering
\includegraphics[width=0.75\textwidth]{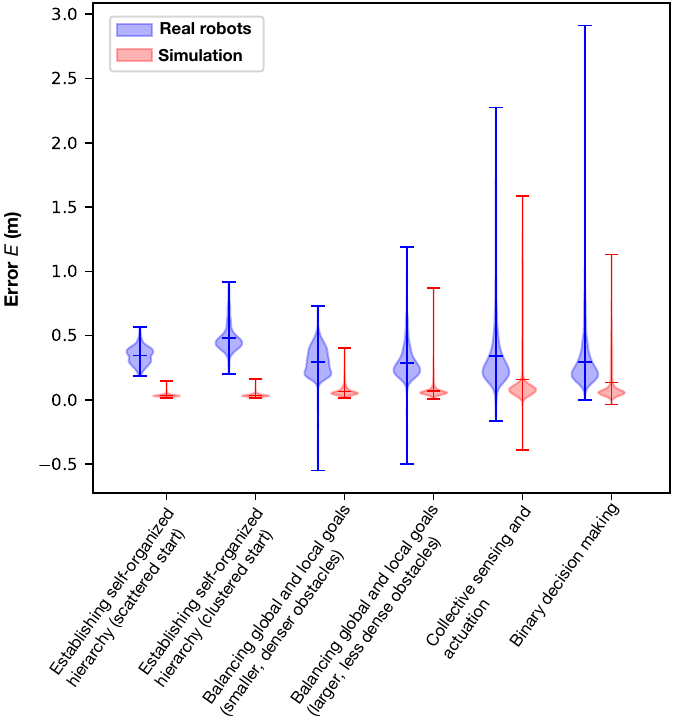}
\caption{{\bf Comparison of the real robot trials to their matching simulations.} Violin plots of the actuation error $E$ (mean and 95\% confidence interval per robot per second) in all trials of each experiment type.}
\label{fig:violin-crossverify}
\end{figure}

\vspace{7mm}
\noindent
{\it (Section continued on next page.)}

\clearpage

\begin{figure}[h!]
\centering
\subfigure[Establishing self-organized hierarchy, clustered start]{
\includegraphics[width=0.4\textwidth]{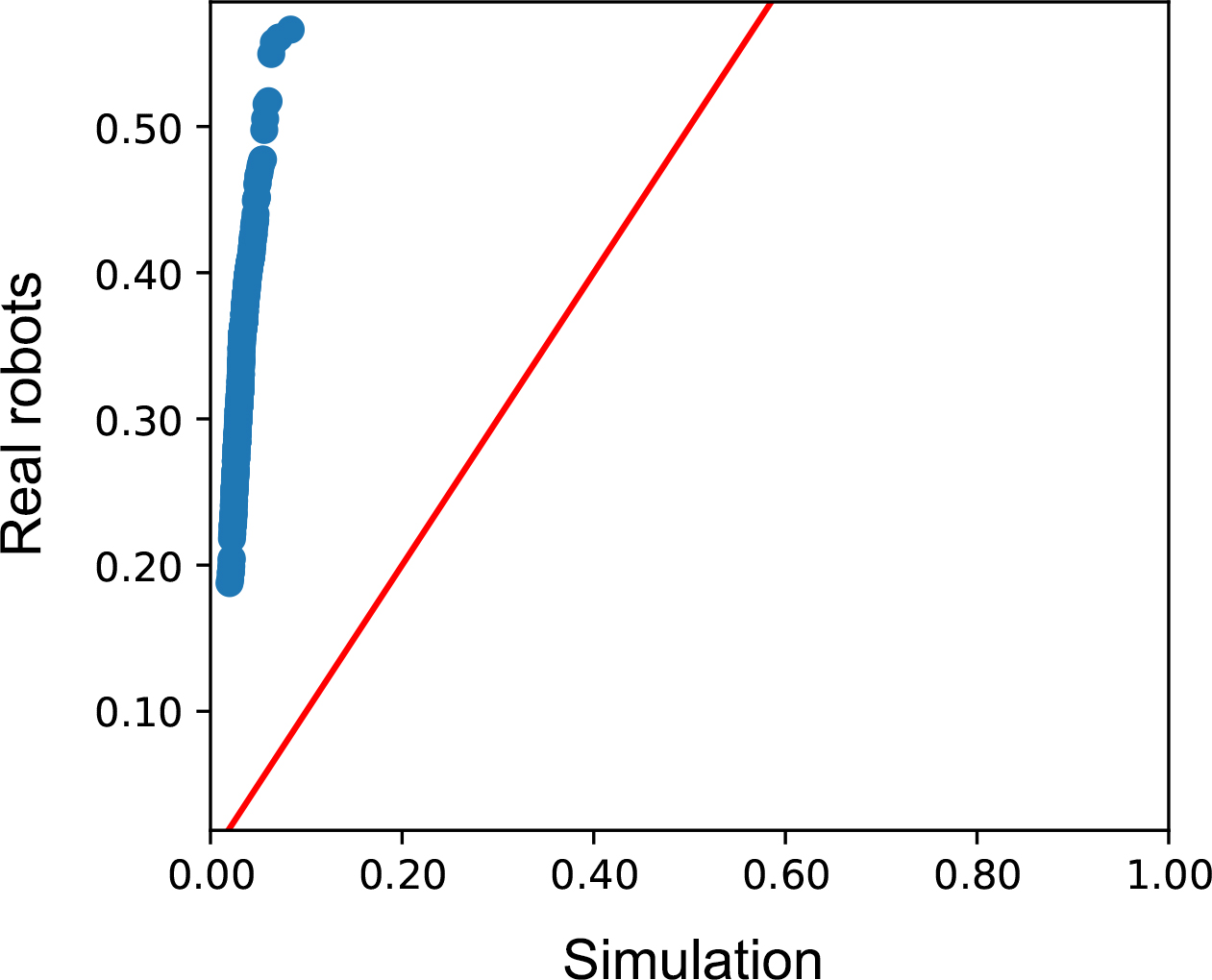}}
\hspace{5mm}
\subfigure[Establishing self-organized hierarchy, scattered start]{
\includegraphics[width=0.4\textwidth]{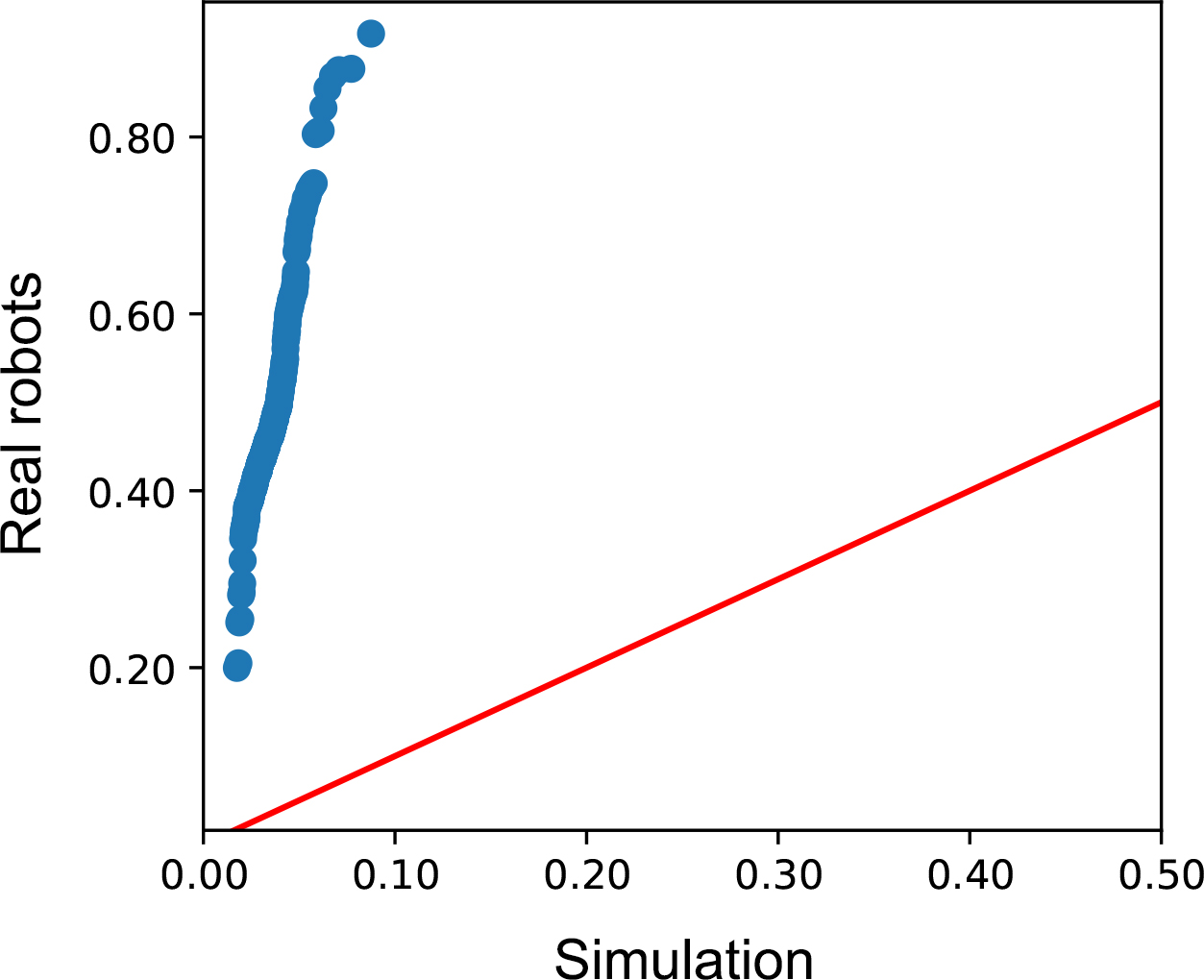}}
\subfigure[Balancing global and local goals, with smaller, denser obstacles]{
\includegraphics[width=0.4\textwidth]{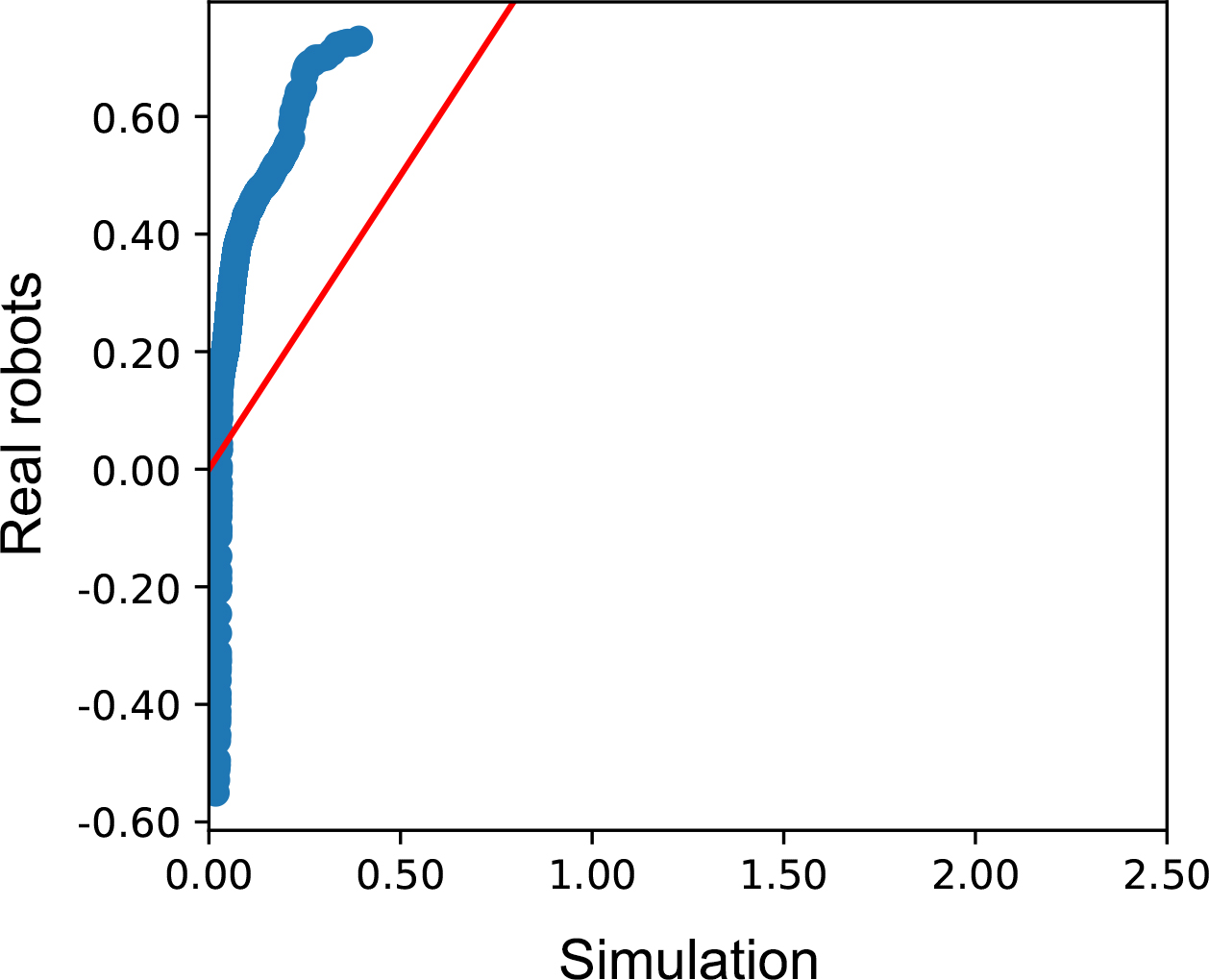}}
\hspace{5mm}
\subfigure[Balancing global and local goals, with larger, less dense obstacles]{
\includegraphics[width=0.4\textwidth]{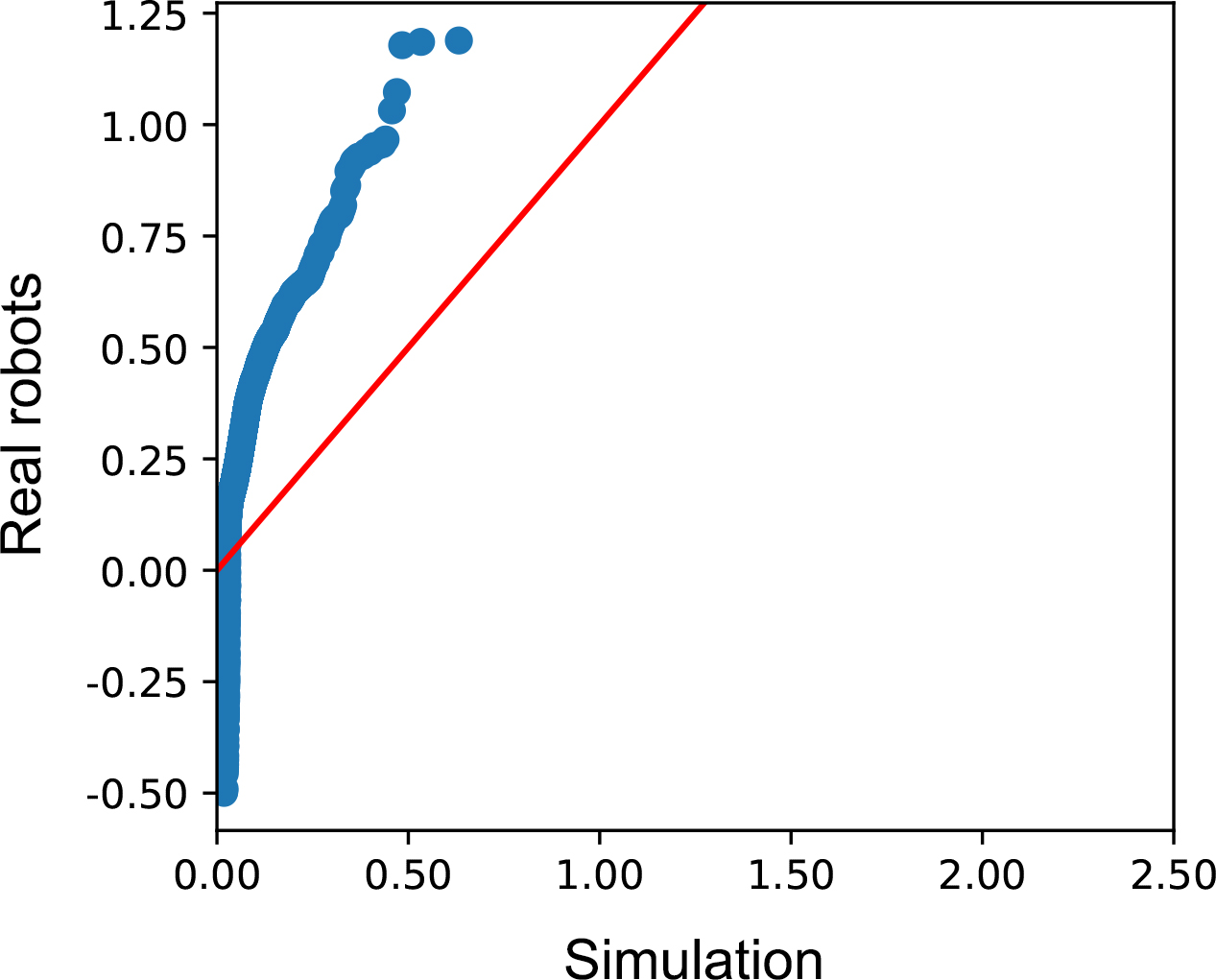}}
\subfigure[Collective sensing and actuation]{
\includegraphics[width=0.4\textwidth]{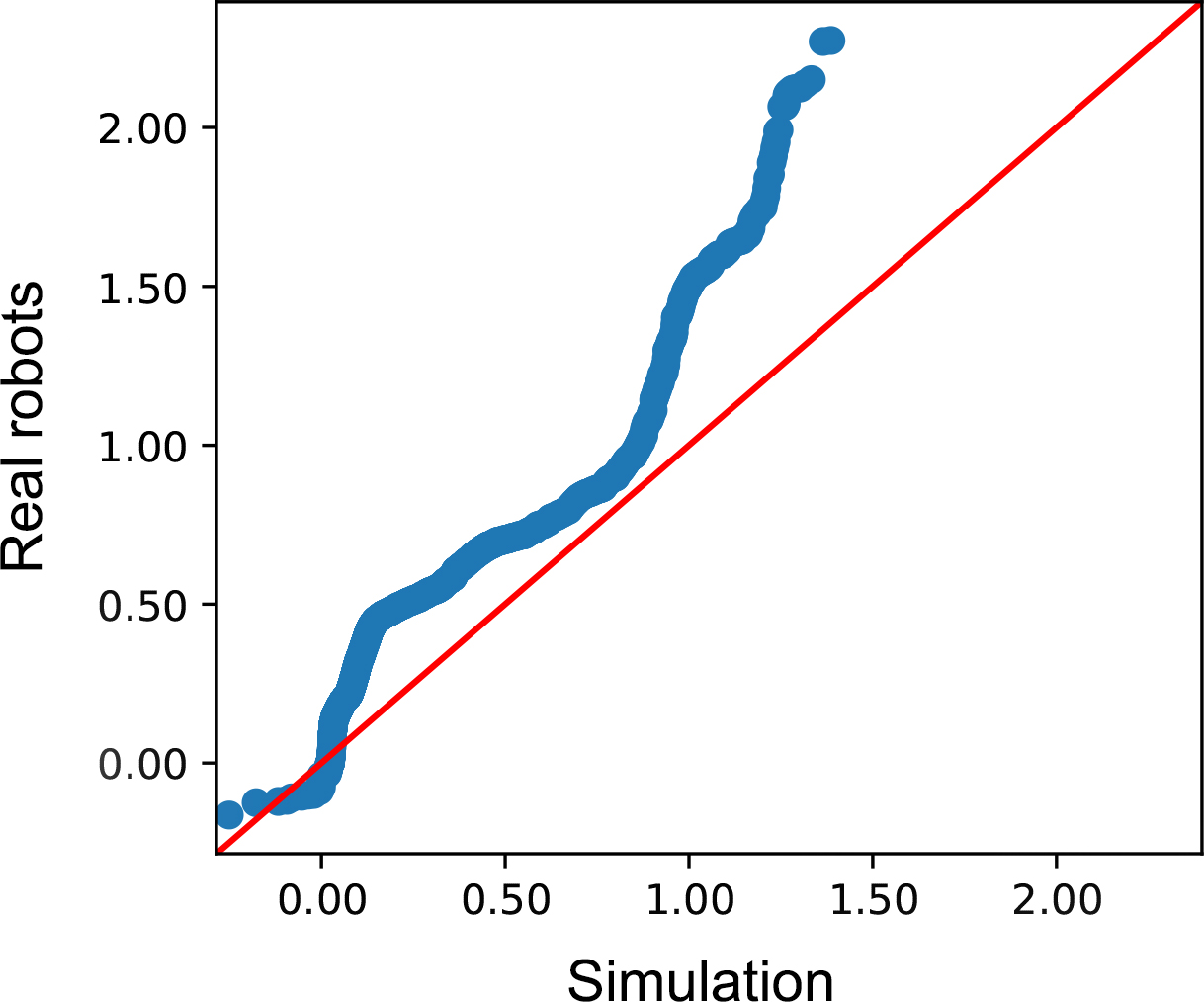}}
\hspace{5mm}
\subfigure[Binary decision making]{
\includegraphics[width=0.4\textwidth]{qqplot-5.jpg}}
\caption{{\bf Comparison of the real robot trials to their matching simulations.} Q--Q plots of the probability distributions of actuation error $E$ in all trials of the specified type.}
\label{fig:qq-crossverify}
\end{figure}

\clearpage\rhead{}
\section*{Section \ref*{SM:theory}. SoNS theoretical analysis and mathematical proofs}
\lhead{Section \ref*{SM:theory}. Theoretical analysis}
This section presents theoretical analyses and mathematical proofs related to the tracking of desired relative positions in a SoNS. We analyze convergence and stability with motionless and moving SoNS-brain robots. 

For analysis, we represent a SoNS of $n$ robots as a multi-robot formation and decompose it into $n-1$ sub-formations of leader-follower pairs.
To study the interactions and stability bounds of a SoNS, we analyze the local leader-follower formation tracking problems of these pairs. We use a standard reactive control law~\cite{oh2011formation, krick2009stabilisation} to generate appropriate inputs for the follower robot, in order to maintain the desired relative position with respect to its leader. This strictly reactive control law is the same approach we use in our real robot experiments; it establishes a performance baseline for the tracking of target positions within a SoNS.
We assess the convergence and closed-loop stability of position tracking in the system, utilizing the \textit{leader-to-formation stability} notion~\cite{tanner2004leader}. 
At the end of this section, we discuss other types of control laws that could be combined with the SoNS architecture to improve performance beyond the baseline established using a strictly reactive control law.

\subsection*{Modeling}
Consider a SoNS with $n$ robots (including both aerial robots and ground robots). The translational motion in $\mathbb{R}^2$ space of a robot $R_i$, $i \in \left\{ {1,2, \ldots ,n} \right\}$, is governed by~\cite{oh2011formation}:
\begin{equation}
\Dot{\Vec{p}}_i = \Vec{u}_i,
\label{eq:Motion_model}
\end{equation}
where $\Vec{p}_i = \begin{bmatrix} x_i ~~ y_i\end{bmatrix}^T \in \mathbb{R}^2$ is the absolute position in a global coordinate system $\mathcal{F}_\mathcal{I}$ used for analysis and $\Vec{u}_i \in \mathbb{R}^2$ is the control input.

\subsection*{Control of a leader-follower pair}

We first analyze the interactions of a single leader-follower pair, before extending to multiple leader-follower pairs.

\subsubsection*{Leader-follower kinematics}
We derive the leader-follower kinematics for a formation consisting of a single robot pair. Consider the leader-follower setup in Fig.~\ref{fig:LF_setup}, where $\Vec{p}_i = \begin{bmatrix} x_i & y_i \end{bmatrix}^T$ is the position of the leader robot $R_i$, ~$\Vec{p}_j = \begin{bmatrix} x_j & y_j\end{bmatrix}^T$ is the position of the follower robot $R_j$, $\Vec{d}_{ij} \in \mathbb{R}^2$ is the desired displacement of $R_j$ w.r.t. $R_i$ and is defined according to the control outputs of the SoNS software, all expressed in the global coordinate frame used for the analysis. 
From this setup, we can define $\Vec{z}_{ij} = \Vec{p}_i - \Vec{p}_j$ as the displacement of the follower $R_j$ with respect to the leader $R_i$, $\| \Vec{z}_{ij} \| \in \mathbb{R} $ as the Euclidean norm of that displacement, $\| \Vec{d}_{ij}\| = \| \Vec{d}_{ji}\| \in \mathbb{R}$ as the Euclidean norm of the desired displacement $\Vec{d}_{ij}$,
$\Vec{p}_j^{~d} = \Vec{p}_i - \Vec{d}_{ij}$ as the desired position of $R_j$, and $\Vec{e}_{ij} = \Vec{z}_{ij} - \Vec{d}_{ij} \in \mathbb{R}^2$ as the formation tracking error.

With this notation, the kinematics of a leader-follower pair can be defined as
\begin{equation}
    \begin{aligned}
        \Dot{\Vec{e}}_{ij} &= \Dot{\Vec{p}}_i - \Dot{\Vec{p}}_j \\
        &= \Vec{u}_i - \Vec{u}_j
    \end{aligned}~,
    \label{eq:LF_kinematics}
\end{equation}~.

\begin{remark}
Note that robots in a SoNS are assumed to not have access to any global position information or external reference frame. Each robot can only access the relative positions $\Vec{z}_{ij}$ of its neighbors $j \in \mathcal{N}_i$, with respect to its own local reference frame. Global position information is used exclusively for analysis.
\end{remark}

\begin{figure}[h]
    \centering
    \includegraphics[width=0.445\textwidth]{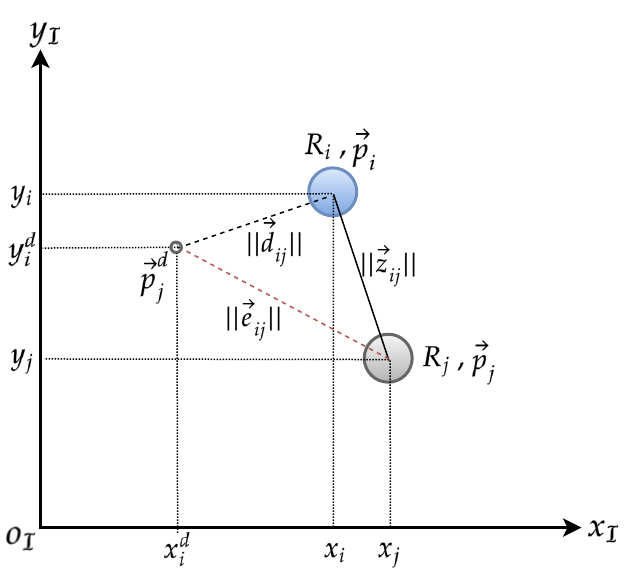}
    \caption{An example leader-follower pair.
    }
    \label{fig:LF_setup}
\end{figure}

\subsubsection*{Problem formulation}
The control law that generates inputs for the follower robot to move from its current position $\Vec{p}_j$ to the desired position $\Vec{p}_j^{~d}$ can be expressed as follows.

\begin{problem} 
The goal is to ensure that the formation tracking of the follower robot $R_j$ adheres to the motion model given in Eq.~\eqref{eq:Motion_model} and conforms to the leader-follower kinematics given in Eq.~\eqref{eq:LF_kinematics}, using the following control law $\Vec{u}_j$ for the follower robot: 
 \begin{equation}
           \Vec{u}_j = f(t, \Vec{z}_{ij}, \Vec{d}_{ij}) \in \mathbb{R}^2,
     \label{eq:Control_input}
    \end{equation}
    where $\Vec{z}_{ij}$ is the displacement with respect to the leader and $\Vec{d}_{ij}$ is the desired displacement. The function $f(\cdot)$ generates the required control inputs to move the follower robot from its current position $\Vec{p}_j$ to the desired position $\Vec{p}_j^{~d}$, such that the norm of the formation tracking error $\|\Vec{e}_{ij}\| \to 0$ and $\|\Vec{z}_{ij}\| \to \|\Vec{d}_{ij}\| $.
    \label{pr:Problem_1}
\end{problem}

\rhead{Leader-follower pair}

\subsubsection*{Control law design and analysis}
Based on the leader-follower kinematics given in Eq.~\eqref{eq:LF_kinematics}, we use the following standard reactive controller~\cite{oh2011formation, krick2009stabilisation} for the follower robot $R_j$:
\begin{equation}
    \begin{split}
        \Vec{u}_j &=   \boldsymbol{K}^j(\Vec{z}_{ij} - \Vec{d}_{ij}) \\
                 &= \boldsymbol{K}^j\Vec{e}_{ij}
    \end{split}~,
    \label{eq:Control_law}
\end{equation}
where $\boldsymbol{K}^j = (\boldsymbol{K}^j)^T = \begin{bmatrix} k_1^j & 0 \\  0 & k_2^j  \end{bmatrix} $ is a symmetric positive-definite matrix, the constants $k_1^j$ and $k_2^j$ are control gains, and $\Vec{u}_j \in \mathbb{R}^2$ is the input for maintaining the desired displacement $\|\Vec{d}_{ij}\|$ of the follower with respect to the leader. The control law is independent of any global position information; it uses only relative position information between the follower and the leader, making it independent of any global position information.

\subsubsection*{Input-to-state stability and bounding of formation tracking errors}
We first analyze the stability properties of the leader-follower kinematics Eq.~\eqref{eq:LF_kinematics} in the case of zero external input based on the following definition of Lyapunov stability.

\begin{definition}(\textbf{Exponential Stability}~\cite{khalil2002nonlinear})
For a time-invariant system $\Dot{x} = f(x)$ where $f:\mathcal{D} \subset \mathbb{R}^n \to \mathbb{R}^n $ is a locally Lipschitz function and $x = 0 \in \mathcal{D}$ is an equilibrium point of the system, exponential stability can be established through the use of a Lyapunov function $V(x)$ that satisfies the following conditions:
    \begin{itemize}
        \item $V(x)$ is $\mathcal{C}^1$ (i.e., continuously differentiable),
        \item $c_1\|x\|^a \leq V(x) \leq c_2\|x\|^a$,
        \item and $\Dot{V}(x) \leq -c_3\|x\|^a$,
    \end{itemize}
    where $a > 0$, $c_1 > 0$, $c_2 > 0$, and $c_3 > 0$. 
\end{definition}

\begin{theorem} 
For a leader-follower robot pair with the motion model Eq.~\eqref{eq:Motion_model}, an initial condition $\Vec{p}_i(0) \neq \Vec{p}_j(0)$, and control input $\Vec{u}_{i} = 0$, the controller Eq.~\eqref{eq:Control_law} guarantees exponential stability of $\Vec{e}_{ij} \to 0$ as $t \to \infty$, in the system given in Eq.~\eqref{eq:LF_kinematics}.
\label{tr:Theorem_1}
\end{theorem}

\begin{proof}
   By substituting the control law Eq.~\eqref{eq:Control_law} into the motion model Eq.~\eqref{eq:Motion_model}, we obtain the following closed-loop system:
    \begin{equation}
        \begin{aligned}
             \Dot{\Vec{p}}_j =  \boldsymbol{K}^j\Vec{e}_{ij}
        \end{aligned}\,. 
        \label{eq:Closed-loop}
    \end{equation} 
The leader-follower kinematics Eq.~\eqref{eq:LF_kinematics} can then be expressed as
    \begin{equation}
        \Dot{\Vec{e}}_{ij} = -\boldsymbol{K}^j\Vec{e}_{ij} + \Vec{u}_i \,.
        \label{eq:Closed-loop_error}
    \end{equation}
We consider the Lyapunov function candidate $V_1 = \frac{1}{2}\Vec{e}_{ij}^T\Vec{e}_{ij}$ that satisfies the following condition:
    \begin{equation}
        \begin{aligned}
           c_1^j\|\Vec{e}_{ij}\|\leq V_1  \leq c_2^j\|\Vec{e}_{ij}\| \\
        \end{aligned},
        \label{eq:Lyapunov}
    \end{equation}
where $a_j = 2$, ~$0 < c_1^j = \min(k_1^j, k_2^j)$, and $ 0 < c_2^j = \max(k_1^j, k_2^j)$. Note that the constants $k_1^j$ and $k_2^j$ are the gains of the matrix $\boldsymbol{K}^j$ defined in Eq.~\eqref{eq:Control_law}. 
Then, differentiating Eq.~\eqref{eq:Lyapunov} with respect to time yields
    \begin{equation}
        \begin{aligned}
          \Dot{V}_1 &=\phantom{-}\Vec{e}_{ij}^T\Dot{\Vec{e}}_{ij} \\
          &=\phantom{-}\Vec{e}_{ij}^T(\Vec{u}_i  - \boldsymbol{K}^j\Vec{e}_{ij})\\
          &=-\Vec{e}_{ij}^T\boldsymbol{K}^j\Vec{e}_{ij} + \Vec{e}_{ij}^T\Vec{u}_i \\
          & \leq -2c_1^j\|\Vec{e}_{ij}\|^2 + \|\Vec{e}_{ij}\|\|\Vec{u}_i\| \leq -c_3^j\| \Vec{e}_{ij} \|^2 \leq 0 , ~ \forall ~ \|\Vec{e}_{ij}\| \geq \frac{\|\Vec{u}_i\|}{2c_1^j\theta}~,
        \end{aligned} 
        \label{eq:Lyapunov_der}
    \end{equation}
where $c_3^j \triangleq 2c_1^j(1 - \theta), ~ \theta \in (0, 1)$.The proof is concluded by substituting $\Vec{u}_i$.
\label{prf:Proof_1}
\end{proof} 
With minor modification and following the lines of \cite{seibert1990global}, the same proof can be used also to show that if the control input $\Vec{u}_{i} \to 0$ then the formation tracking error $\|\Vec{e}_{ij}\|$ converges to $0$. 

If the leader is moving with other velocity regimes, then there is a lower bound that the formation tracking error can attain, according to $\Vec{u}_{i}$.
In order to ensure that the quadrotors maintain stability within some flight safety requirements, it is crucial to bound the error amplitudes in the worst-case scenario. We use the input-to-state stability (ISS) notion~\cite{sontag1995input,isidori1985nonlinear, khalil2002nonlinear} to establish an upper limit for the formation tracking error and an upper limit for the admissible leader input $\Vec{u}_i$ that can maintain flight safety at all times. We then use the analysis of the formation's ISS to establish a link between the magnitude of the leader's input and the evolution of the formation tracking errors, i.e., the error dynamics given in Eq.~\eqref{eq:Closed-loop_error}.

\begin{definition} (\textbf{Input-to-State-Stability}~\cite{sontag1995input})
Let a leader-follower pair be input-to-state stable. Then, there is a class $\mathcal{KL}$ function $\beta$ and a class $\mathcal{K}$ function $\gamma$ such that, for any initial formation tracking error $\Vec{e}_{ij}(0)$ and for any bounded input of the leader $\Vec{u}_i(t)$, the solution $\Vec{e}_{ij}(t)$ exists for all $0 \leq t$ and satisfies the following inequality~\cite{tanner2004leader}:
 \begin{equation}
    \|\Vec{e}_{ij}(t)\| \leq \beta_{ij}(q,t) + \gamma_{ij} \left( r \right),
    \label{eq:ISS}
\end{equation}
where the functions $\beta_{ij}(q,t)$ and $\gamma_{ij}(r)$ are transient and asymptotic ISS gain functions, respectively. These functions help to measure the impact of initial conditions and the leader's input on the formation tracking errors~\cite{tanner2002effect}.
\label{definition1}
\end{definition}

In order to further analyze the ISS properties of the leader-follower pair, we treat the error dynamics given by Eq.~\eqref{eq:Closed-loop_error} as a perturbed system and derive an upper bound on the error norm $\|\Vec{e}_{ij}(t) \|$. This bound provides a measure of the rate at which the formation tracking error converges and indicates that it is ISS with respect to the leader's velocity $\Vec{u}_i$~\cite{tanner2002effect, khalil2002nonlinear, isidori1985nonlinear}. Using the initial error norm $\|\Vec{e}_{ij}(0)\|$ we can then rewrite the inequality Eq.~\eqref{eq:ISS} as
\begin{equation}
\begin{aligned}
         \|\Vec{e}_{ij}(t)\| &\leq  \beta_{ij}(\|\Vec{e}_{ij}(0)\|,t) + \gamma_{ij} \left( \underset{0 \leq \tau \leq t}{\sup} \|\Vec{u}_i(\tau)\| \right) \\
                             &\leq \hat{\beta}_{ij}\|\Vec{e}_{ij}(0)\| e^{-\frac{c_3^j}{c_2^ja_j}t} +  \hat{\gamma}_{ij} \underset{\tau \leq t}{\sup} \|\Vec{u}_i(\tau)\|
\end{aligned},
        \label{eq:ISS_gains}
\end{equation}
where the terms $\hat{\beta}_{ij}$ and $\hat{\gamma}_{ij}$ are gain estimates that provide insight into the relationship between the initial error, the leader's input, and the observed interconnection errors observed. They are defined as
\begin{equation}
     \hat{\beta}_{ij} \triangleq \left( \frac{c_2^j}{c_1^j}\right)^\frac{1}{a_j} \hspace{-4mm}, ~~~~~ \hat{\gamma}_{ij} \triangleq  \frac{c_2^j}{c_1^j\theta}.
     \label{eq:ISS_est}
\end{equation}
By substituting the gain estimates from Eq.~\eqref{eq:ISS_est} into Eq.~\eqref{eq:ISS_gains}, we obtain
\begin{align}  
             \|\Vec{e}_{ij}(t)\| \leq \left( \frac{c_2^j}{c_1^j}\right)^\frac{1}{a_j} \|\Vec{e}_{ij}(0)\| e^{-\frac{c_3^j}{c_2^ja_j}t} + \frac{c_2^j}{c_1^j\theta}\underset{\tau \leq t}{\sup} \| 
 \Vec{u}_i(\tau)\| .
 \label{eq:ISS_gains2}
\end{align}

We use the formation ISS measure as a metric to help provide an upper bound on the leader's input and ensure that the formation remains within desired specifications. Additionally, we use it to compare the stability properties of different formation shapes and connection schemes.

\begin{definition} (\textbf{Formation ISS Measure}~\cite{tanner2002effect})
Consider a leader-follower pair that is ISS with gain functions $\beta_{ij}(q,t)$ and $\gamma_{ij}(r)$. Assume that $\gamma_{ij}(r) \in \mathcal{C}^1$ (i.e., it is continuously differentiable) and let $\mathcal{U} \subseteq \mathbb{R}^n$ be a compact neighborhood of the origin containing all $\Vec{u}_i \in \mathcal{U}$ that are of interest. The formation tracking error $\Vec{e}_{ij}$ always satisfies the following inequality if the leader-follower pair is ISS:
    \begin{equation}
        \lim_{t\to\infty} \|\Vec{e}_{ij}(t)\| \leq \gamma_{ij} \left( r \right).
    \end{equation}
If we consider a specification such as the first leader's input bounded inside a unit sphere where $r = 1$, we can derive the formation's ISS performance measure $P_{ISS} \in [0,~ 1]$, based on the performance measure of leader-to-follower stability $P_{LFS}$ in~\cite{tanner2004leader}, as follows:
\begin{equation}
        P_{ISS} \triangleq \frac{1}{\gamma_{ij}(1)}.
        \label{eq:formation_ISS}
    \end{equation}
\end{definition}

\subsubsection*{Numerical simulations: formation tracking error with bounded leader inputs}
Consider a leader-follower pair as depicted in Fig.~\ref{fig:LF_setup}, where the initial positions are $\Vec{p}_i(0) = [5,~10]^T$ and $\Vec{p}_j(0) = [0,~0]^T$, the desired displacement vector is given as $\Vec{d}_{ij} = [3,~4]^T$, and the controller gain matrix $\boldsymbol{K}^j$ has $k_1^j = 5$ and $k_2^j = 5$.  

\begin{remark}
    In the real-robot and simulated experiments in this study, we use a strictly reactive control law; there is no feedforward control nor preview of the reference signal. We use this reactive control with the aim of establishing a performance baseline for the tracking of target positions within a SoNS---in other words, to study the lower bounds of performance that the formation tracking error can attain. For a discussion of incorporating feedforward information in SoNS for improved performance, see the end of this section.
\end{remark}

Under a stationary leader (i.e., $\Vec{u}_i = [0,0]$), the formation tracking error exponentially converges to zero (see Fig.~\ref{fig:ErrorNorm}a). The upper bound defined in Eq.~\eqref{eq:ISS_gains} decays with respect to the initial formation tracking error and the lower bound defined in Eq.~\eqref{eq:Lyapunov_der} is always $0$ (the error norm converges to $0$).  

\begin{figure}[ht]
\centering
\subfigure[]{
\includegraphics[width=0.4451\textwidth]{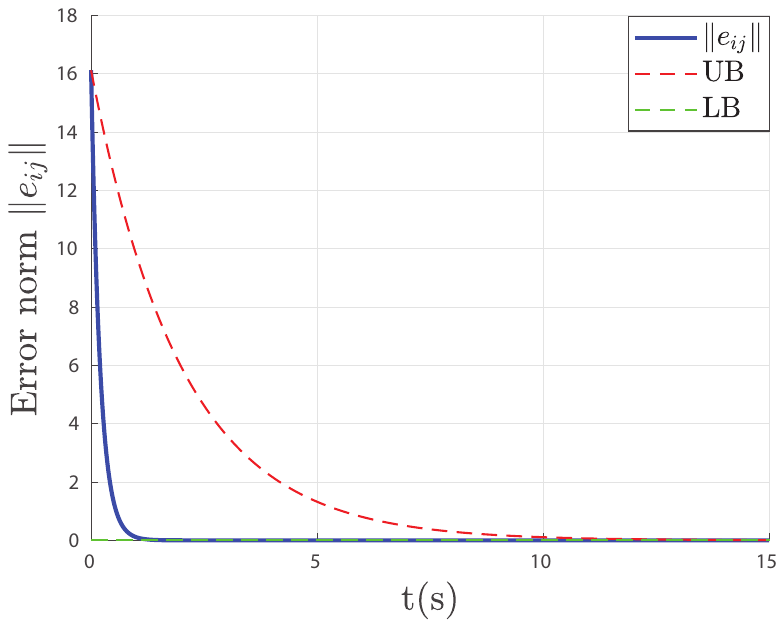}
}
\hspace{-10mm}
\subfigure[]{
\includegraphics[width=0.4451\textwidth]{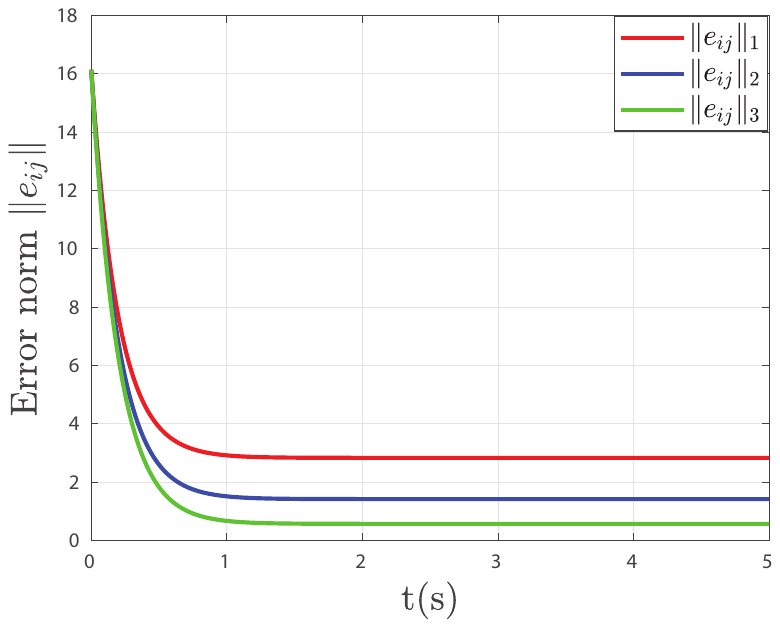}
}
\caption{Norm of the formation tracking error without feed-forward information. (a) Error norm under a stationary leader, with the upper bound (UB) and lower bound (LB). Note that the lower bound line (green) is along the bottom of the graph. (b) Error norm under a moving leader with different velocity regimes.}
\label{fig:ErrorNorm}
\end{figure}

For cases of moving leaders, we consider the lower bound that the error can attain under velocity regimes that span the conditions occurring in our real robot experiments. We consider the example error norms $\|e_{ij}\|_1$, $\|e_{ij}\|_2$, and $\|e_{ij}\|_3$, which represent the formation tracking errors under leader velocities $\Vec{u}_i^1  = [10,~10]^T$, $\Vec{u}_i^2  = [5,~5]^T$, and $\Vec{u}_i^3 = [2,~2]^T$, respectively. Under these velocity regimes, the error norms converge to a lower bound (see Fig.~\ref{fig:ErrorNorm}b). For instance, when the leader velocity is $\Vec{u}_i^2$ and $\theta = 0.5$, the norm of the formation tracking error $\|e_{ij}\|_2$ converges to a lower bound around $1.5$.  

\subsection*{Control of multiple leader-follower pairs}

We also assess how stability bounds are propagated in formations composed of multiple leader-follower pairs. 

\subsubsection*{Graph theory preliminaries}
For a SoNS, we define a target topology $G = (V, E, D)$, where $V = \{ v_1,\dots, v_n\}$ is a set of vertices, $E = \{\xi_{ij} = (\overrightarrow{v_j,v_i})\ |\ v_j,v_i \in V, v_i \neq v_j\}$ is a set of edges, and $D = {\Vec{p}_j^{~d}}$ is a set of formation attributes that includes information such as the desired positions. Each vertex $v_i$ has a set of neighbors $\mathcal{N}_i = \{v_j \in V\ |\ (i, j) \in E\}$. We also define an adjacency matrix $A_n = [a_{ij}] \in \mathbb{R}^{n \times n}$ to represent the connections between vertices. For example, if vertices $v_i$ and $v_j$ are connected, the corresponding entry $a_{ij}$ in the adjacency matrix will be non-zero. An entry of the adjacency matrix is defined as
\begin{equation}
     a_{ij} = 
\begin{cases}
              0, &~~~ i=j \\
              0, &~~~ (i,~j) \not\in E\\
              1, &~~~ (i,~j) \in E
\end{cases}~.           
\end{equation}
The Laplacian matrix $L = [l_{ij}] \in \mathbb{R}^{n \times n} $ related to the adjacency matrix $A$ for $i \in \left\{ {1,2, \ldots ,n} \right\}$ and $j \in \left\{ {1,2, \ldots ,n} \right\}$ is defined as 
\begin{equation}
     l_{ij} = 
\begin{cases}
               \sum_{k=1}^{n} a_{ij}, & i=j \\
              -a_{ij}, & i \neq j
\end{cases}~.
\end{equation}

\begin{remark}
    Note that, in the SoNS approach, a directed edge between vertices $v_j$ and $v_i$ represents a communication and control link between the corresponding follower robot $R_j$ and the leader $R_i$, and the indegree of each vertex is 1 (i.e., each follower has only one leader). Therefore, an $n$-node graph has $n-1$ edges and can be considered an $n$-robot formation with $n-1$ pairs of leaders and followers. For each pair, the control law Eq.~\eqref{eq:Control_law} drives the follower robot to its desired position. If Theorem \ref{tr:Theorem_1} is satisfied for all pairs, the formation of $n$ robots will be stable~\cite{das2002vision}.
\end{remark}

\rhead{Multiple leader-follower pairs}

\subsubsection*{Propagation of the stability bounds}
As we demonstrate in Eq.~\eqref{eq:ISS_gains2}, 
a leader-follower pair has input-to-state stability (ISS). Input-to-state stability is preserved in cascaded connections~\cite{isidori1985nonlinear}, such that ISS bounds are calculated from one robot to another, from the first leader to the last follower of the formation. However, the upper bound of the formation's ISS depends on the initial magnitude of the formation tracking error, $\|\Vec{e}_{ij}(0)\|$, and this error term tends to increase as more leader-follower pairs form complex structures. Thus, the formation's ISS depends on the longest path length of information passed from leaders to their respective followers (i.e., the length of the path from the first leader to the last follower) in the formation. Any formation can be constructed from two types of 3-robot primitives: either with cascaded connections or parallel connections. We analyze the ISS properties of these formation primitives.

\begin{figure}[h]
    \centering
    \subfigure[]{
    \includegraphics[width=0.3\linewidth]{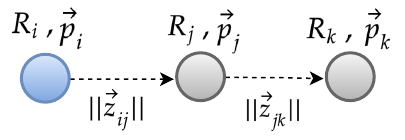}
    }
    \subfigure[]{
    \includegraphics[width=0.2\linewidth]{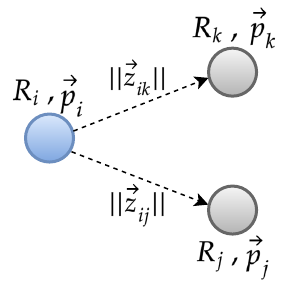}
    }
    \caption{Formation primitives. (a) Cascaded formation primitive. Robot $R_j$ follows the first leader of the formation $R_i$, and robot $R_k$ in turn follows robot $R_j$. Dashed arrows denote information passed from leader to follower. (b) Cascaded formation primitive. Robot $R_j$ follows the first leader of the formation $R_i$, and robot $R_k$ in turn follows robot $R_j$. Dashed arrows denote information passed from leader to follower.}
    \label{fig:primitives}
\end{figure}

In a formation primitive of three robots with cascaded connections (see Fig.~\ref{fig:primitives}a), the control laws for each follower robot $R_j$ and $R_k$ are defined as
\begin{equation}
\begin{aligned}
            \Vec{u}_j &=   \boldsymbol{K}^j(\Vec{z}_{ij} - \Vec{d}_{ij}) = \boldsymbol{K}^j\Vec{e}_{ij} \\
            \Vec{u}_k &=   \boldsymbol{K}^k(\Vec{z}_{jk} - \Vec{d}_{jk}) = \boldsymbol{K}^k\Vec{e}_{jk}~ \\
\end{aligned}.
\end{equation}
The ISS bounds of each pair can be expressed as follows:
\begin{align}
\begin{cases}
  \|\Vec{e}_{ij}(t)\| &\leq  \beta_{ij}(\|\Vec{e}_{ij}(0)\|,t) + \gamma_{ij} \left( \underset{0 \leq \tau \leq t}{\sup} \|\Vec{u}_i(\tau)\| \right) \\
                      &\leq \left( \frac{c_2^j}{c_1^j}\right)^\frac{1}{a_j} \|\Vec{e}_{ij}(0)\| e^{-\frac{c_3^j}{c_2^ja_j}t} + \frac{c_2^j}{c_1^j\theta}\underset{\tau \leq t}{\sup} \|\Vec{u}_i(\tau)\| \\
  \|\Vec{e}_{jk}(t)\| &\leq  \beta_{jk}(\|\Vec{e}_{jk}(0)\|,t) + \gamma_{jk} \left( \underset{0 \leq \tau \leq t}{\sup} \|\Vec{u}_j(\tau)\| \right) \\
                      &\leq \left( \frac{c_2^k}{c_1^k}\right)^\frac{1}{a_k} \|\Vec{e}_{jk}(0)\| e^{-\frac{c_3^k}{c_2^ka_k}t} + \frac{c_2^k}{c_1^k\theta}\underset{\tau \leq t}{\sup} \|\Vec{u}_j(\tau)\| 
  \end{cases}.
\end{align}

The proof of \textit{Proposition III.1} in~\cite{tanner2004leader} shows that the error norm $\|\Vec{e}_{jk}(t)\|$ between two followers can be expressed in terms of the first leader's velocity $\Vec{u}_i$, as follows:
\begin{equation}
\begin{aligned}
  \|\Vec{e}_{jk}(t)\| &\leq  \beta_{jk}(\|\Vec{e}_{jk}(0)\|,t) + \gamma_{jk} \left( \underset{0 \leq \tau \leq t}{\sup} \|\Vec{u}_i(\tau)\| \right) \\
\beta_{jk}(\|\Vec{e}_{jk}(0)\|,t) \triangleq   \left( \frac{c_2^k}{c_1^k}\right)^\frac{1}{a_k} &\|\Vec{e}_{jk}(0)\| e^{-\frac{c_3^k}{c_2^ka_k}t} ~,~~~~~~~~
\gamma_{jk} \left( \underset{0 \leq \tau \leq t}{\sup} \|\Vec{u}_i(\tau)\| \right) \triangleq \frac{c_2^kc_2^j}{c_1^k\theta}\underset{\tau \leq t}{\sup} \|\Vec{u}_i(\tau)\|~.
\end{aligned}
\end{equation}
To demonstrate that the ISS of the first leader-follower pair (robots $R_i$ and $R_j$) can be extended to the second pair (robots $R_j$ and $R_k$) in the cascaded formation primitive in Fig.~\ref{fig:primitives}a, we can construct the following composite error vector: 
\begin{align}
    \Vec{e}_{ik} \triangleq [ \Vec{e}_{ij} ~~ \Vec{e}_{jk}]^T~.
\end{align}

A system composed of two cascaded ISS systems is also ISS~\cite{isidori1985nonlinear,khalil2002nonlinear}. Therefore, the composite formation tracking error of two systems $\Vec{e}_{ik}$ satisfies the inequality
\begin{equation}
\begin{aligned}
         \|\Vec{e}_{ik}(t)\| &\leq  \beta_{ik}(\|\Vec{e}_{ik}(0)\|,t) + \gamma_{ik} \left( \underset{0 \leq \tau \leq t}{\sup} \|\Vec{u}_i(\tau)\| \right),
\end{aligned}
        \label{eq:ISS_composite}
\end{equation}
where
\begin{equation}
\begin{split}
 &\beta_{ik}(\|\Vec{e}_{ik}(0)\|,t) =  \beta_{ik1}(\|\Vec{e}_{ik}(0)\|,t)  + \beta_{ik2}(\|\Vec{e}_{ik}(0)\|,t) \\
 &\beta_{ik1}(\|\Vec{e}_{ik}(0)\|,t) = \beta_{jk}\left(\left(2\beta_{jk}(\|\Vec{e}_{ik}(0)\|,\frac{t}{2})
 + \gamma_{jk}(2\beta_{ij}(\|\Vec{e}_{ik}(0)\|,\frac{t}{2})) + 2\gamma_{jk}(2\beta_{ij}(\|\Vec{e}_{ik}(0)\|,0))\right), \frac{t}{2}\right) \\
 &\beta_{ik2}(\|\Vec{e}_{ik}(0)\|,t) = \beta_{ij}(\|\Vec{e}_{ik}(0)\|,t)~,
\end{split}
\label{eq:ISS_gains_composite}
\end{equation}
and
\begin{equation}
\begin{split}
 \gamma_{ik} \left( \underset{0 \leq \tau \leq t}{\sup} \|\Vec{u}_i(\tau)\| \right) = ~ &\gamma_{jk}(2\gamma_{ij}(\sup\|\Vec{u}_i(\tau)\|) + 2\sup\|\Vec{u}_i(\tau)\|) + \beta_{ik}(2\gamma_{jk}(2\gamma_{ij}(\sup\|\Vec{u}_i(\tau)\|)  \\ &+2\sup\|\Vec{u}_i(\tau)\|), 0) + \gamma_{ij}(\sup\|\Vec{u}_i(\tau)\|)~.
\end{split}
\label{eq:ISS_gains_composite2}
\end{equation}
Then, Eqs.~\eqref{eq:ISS_gains_composite} and \eqref{eq:ISS_gains_composite2} can be transformed into
\begin{equation}
\begin{split}
 &\beta_{ik}(q,t) = \beta_{jk}\left(2\beta_{jk}(q,\frac{t}{2})
 + 2\gamma_{jk}(2\beta_{ij}(q,0)), \frac{t}{2}\right) + \gamma_{jk}\left(2\beta_{ij}(q,\frac{t}{2})\right)  + \beta_{ij}(q,t) \\
 &\gamma_{ik} (r) = \gamma_{jk}(2\gamma_{ij}(r) + 2r) + \beta_{ik}(2\gamma_{jk}(2\gamma_{ij}(r))  +2r), 0) + \gamma_{ij}(r)~,
\end{split}
\label{eq:ISS_gains_compositeSimple}
\end{equation}
where $q = \|\Vec{e}_{ik}(0)\| $ and $r = \sup\|\Vec{u}_i(\tau)\|$.

Because a system composed of two ISS systems is also ISS~\cite{isidori1985nonlinear,khalil2002nonlinear}, the composite formation tracking error satisfies the inequality
\begin{equation}
\begin{aligned}
         \|\Vec{e}_{ijk}(t)\| &\leq  \beta_{ijk}(\|\Vec{e}_{ijk}(0)\|,t) + \gamma_{ijk} \left( \underset{0 \leq \tau \leq t}{\sup} \|\Vec{u}_i(\tau)\| \right),
\end{aligned}
\end{equation}
where
\begin{equation}
\begin{split}
&\beta_{ijk}(\|\Vec{e}_{ijk}(0)\|,t) = \beta_{ij}(\|\Vec{e}_{ij}(0)\|,t) + \beta_{ik}(\|\Vec{e}_{ik}(0)\|,t) \\ 
&\gamma_{ijk} \left( \underset{0 \leq \tau \leq t}{\sup} \|\Vec{u}_i(\tau)\| \right) = \gamma_{ij}(\sup \|\Vec{u}_i(\tau)\|) + \gamma_{ik}(\sup \|\Vec{u}_i(\tau)\|).
\end{split}
\label{eq:ISS_gains_composite3}
\end{equation}
Then, Eq.~\eqref{eq:ISS_gains_composite3} can be rewritten as
\begin{equation}
\begin{split}
&\beta_{ijk}(q,t) = \beta_{ij}(q,t) + \beta_{ik}(q,t) \\ 
&\gamma_{ijk} (r) = \gamma_{ij}(r) + \gamma_{ik}(r)~.
\end{split}
\label{eq:ISS_gains_composite4}
\end{equation}

In a formation primitive of three robots with parallel connections (see Fig.~\ref{fig:primitives}b), both follower robots $R_j$ and $R_k$ can be assumed to be equivalent to the first follower robot $R_j$ in the cascaded formation primitive shown in Fig.~\ref{fig:primitives}a.

\subsubsection*{Numerical simulation of error norm between the first leader and the first follower}

\begin{figure}
\centering
\subfigure[]{
\includegraphics[width=0.445\textwidth]{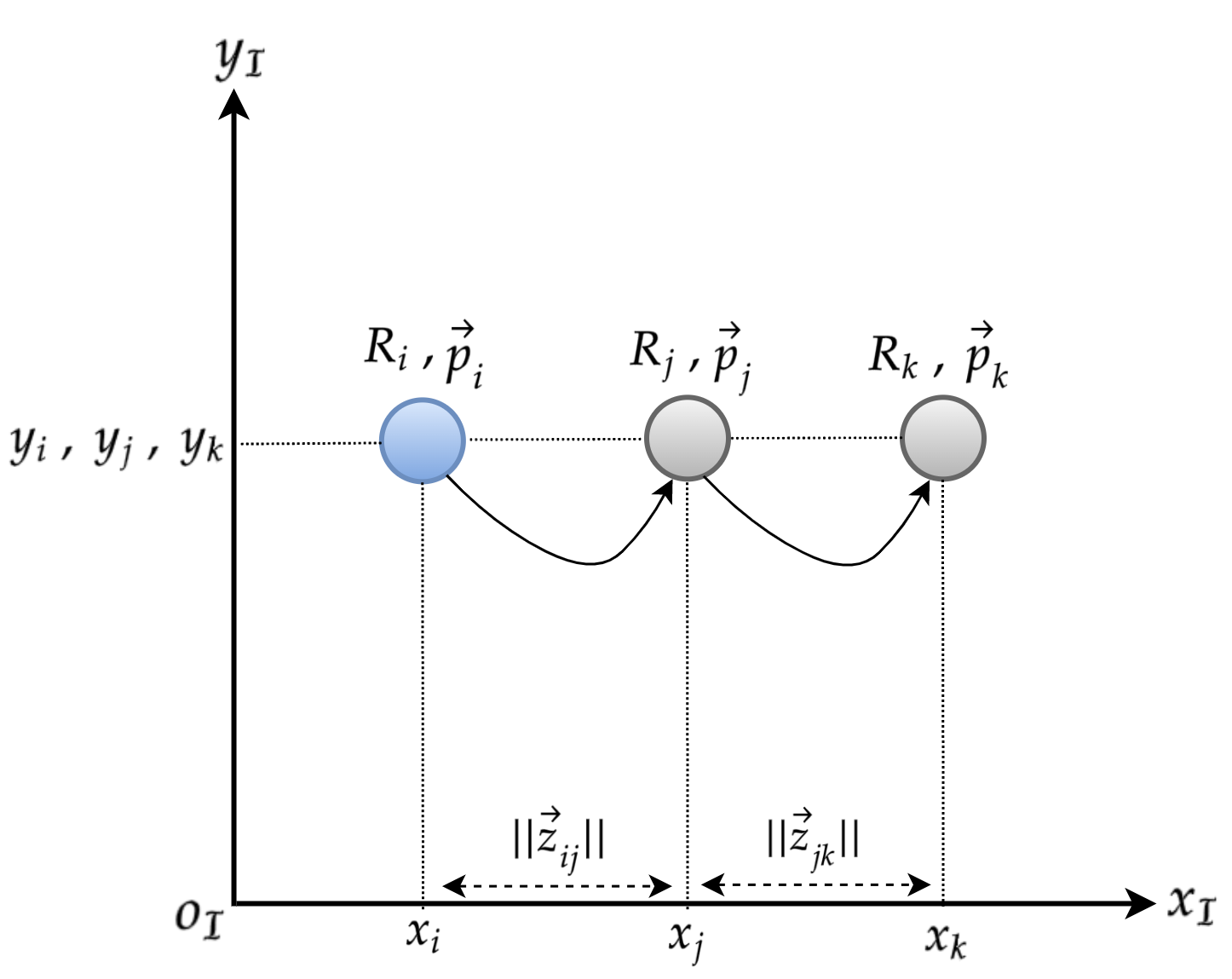}
}
\subfigure[]{
\includegraphics[width=0.445\textwidth]{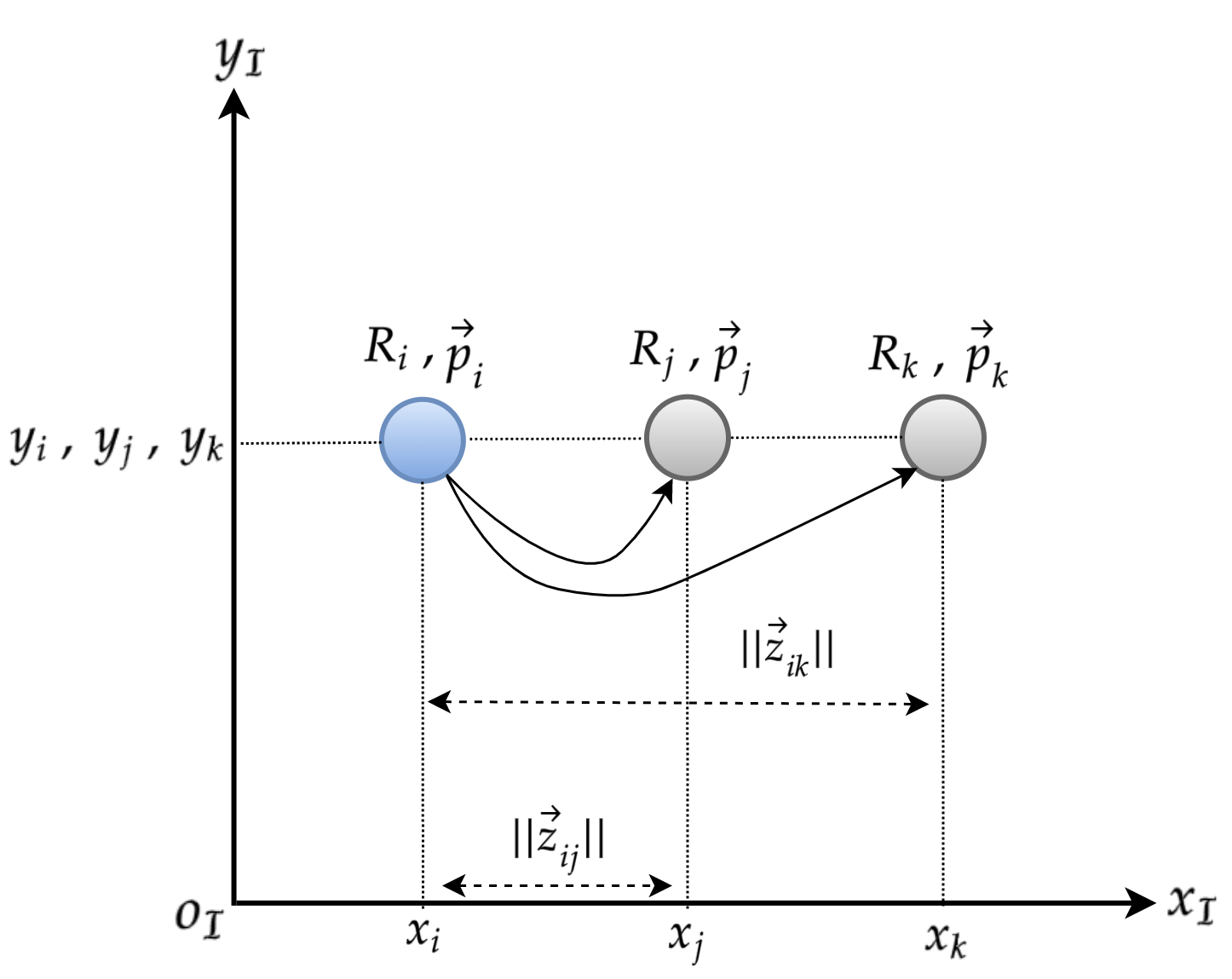}
}
\caption{Two example formations of three robots. The arrows denote graph connections, where information is passed from leader to follower. (a) Cascaded connections. (b) Parallel connections.}
    \label{fig:Compare}
\end{figure}

We assume that both formations move along the $y$-axis. The closed-loop kinematics for the connections can be given as
\begin{equation}
\begin{aligned}
    \Vec{u}_j &=   \boldsymbol{K}^j(\Vec{z}_{ij} - \Vec{d}_{ij}) \\ 
    \Vec{u}_k &=   \boldsymbol{K}^k(\Vec{z}_{jk} - \Vec{d}_{jk}) \\
\end{aligned}~,
\end{equation}
where $\boldsymbol{K}^j$ and $\boldsymbol{K}^k$ are controller gain matrices and $k_1^j = k_2^j = k_1^k = k_2^k =5 $. Note that the initial positions of the robots are $\Vec{p}_i(0) = [5,~10]^T$, $\Vec{p}_j(0) = [1,~3]^T$, and $\Vec{p}_j(0) = [3,~-2]^T$. 
We set the constant reference velocity and desired displacement vectors as~$\Vec{u}_i = [0, ~ 10]^T$, $\Vec{d}_{ij} = [6, ~0]^T$, $\Vec{d}_{jk} = [6, ~0]^T$, $\Vec{d}_{ik} = [12, ~0]^T$. 

To calculate the formation ISS measure $P_{ISS}$ for the cascaded connections (shown in Fig.~\ref{fig:Compare}a), we redefine the asymptotic ISS gain Eq.~\eqref{eq:ISS_gains_compositeSimple} by setting $r = 1$, $q = 1$, and $t = 0$. This bounds the inputs of the first leader inside the unit ball and ensures that $P_{ISS}$ varies in the range of $[0,~ 1]$. This results in the following expression:
\begin{equation}
\begin{split}
\gamma_{ik} (1)  &= \gamma_{jk}(2\gamma_{ij}(1) + 2) + \beta_{ik}(2\gamma_{jk}(2\gamma_{ij}(1) + 2), 0) + \gamma_{ij}(1) \\
&= \gamma_{jk}(2\hat{\gamma}_{ij} + 2) + \beta_{jk}(2\gamma_{jk}(2\hat{\gamma}_{ij} + 2), 0) + \hat{\gamma}_{ij} \\
&= 2\hat{\gamma}_{jk}\hat{\gamma}_{ij} + 2\hat{\gamma}_{jk} + \beta_{jk}(4\hat{\gamma}_{jk}\hat{\gamma}_{ij} + 4\hat{\gamma}_{jk}), 0) + \hat{\gamma}_{ij} \\
&= 2\hat{\gamma}_{jk}\hat{\gamma}_{ij} + 2\hat{\gamma}_{jk} + (4\hat{\beta}_{jk}\hat{\gamma}_{jk}\hat{\gamma}_{ij} + 4\hat{\beta}_{jk}\hat{\gamma}_{jk})q+ \hat{\gamma}_{ij} \\
&= 2\hat{\gamma}_{jk}\hat{\gamma}_{ij} + 2\hat{\gamma}_{jk} + 4\hat{\beta}_{jk}\hat{\gamma}_{jk}\hat{\gamma}_{ij} + 4\hat{\beta}_{jk}\hat{\gamma}_{jk}+ \hat{\gamma}_{ij} 
\end{split}~,
\end{equation}
where $\hat{\beta}_{ij}  \triangleq \left( \frac{c_2^j}{c_1^j}\right)^\frac{1}{a_j}$, $\hat{\beta}_{jk}  \triangleq \left( \frac{c_2^k}{c_1^k}\right)^\frac{1}{a_k}$, $\hat{\gamma}_{ij} \triangleq  \frac{c_2^j}{c_1^j\theta}$, $\hat{\gamma}_{jk} \triangleq  \frac{c_2^k}{c_1^k\theta}$, $c_1^j \triangleq  \min(k_1^j, k_2^j)$, $ c_2^j \triangleq  \max(k_1^j, k_2^j)$, $c_1^k \triangleq  \min(k_1^k, k_2^k)$, $ c_2^k \triangleq  \max(k_1^j, k_2^k)$, and $a_j = a_k = 2$.
This results in the following $P_{ISS}$ value:
\begin{equation}
    P_{ISS}^c = \frac{1}{1 + \gamma_{ik} (1)} = \frac{\theta^2}{ \theta^2 + 7\theta + 6}~,
\end{equation}
where $\theta \in (0,~1)$.

The follower robots $R_j$ and $R_k$ with parallel connections (see Fig.~\ref{fig:primitives}b) can be assumed to be equivalent to the first follower robot $R_j$ with cascaded connections (see Fig.~\ref{fig:primitives}a).
A follower that is directly connected to the first leader has a lower magnitude of relative errors with respect to the first leader's velocity than a follower that is instead connected to another follower.
Fig.~\ref{fig:CompareRes}a shows a comparison of the error norm for the first leader and its first follower in the two formations. In both formations, $R_j$ is in the first hierarchy layer and therefore the error norms $\| e_{ij} \|$ for $R_j$ are the same. However, in the cascaded connections, robot $R_k$ follows $R_j$, while in the parallel connections, $R_k$ directly follows $R_i$. 
As a result, the error norm $\| e_{ik} \|$ for $R_k$ in the cascaded connections is higher than that for all other followers in both formations. 

\begin{figure}
    \centering
\subfigure[]{
    \includegraphics[width=0.445\textwidth]{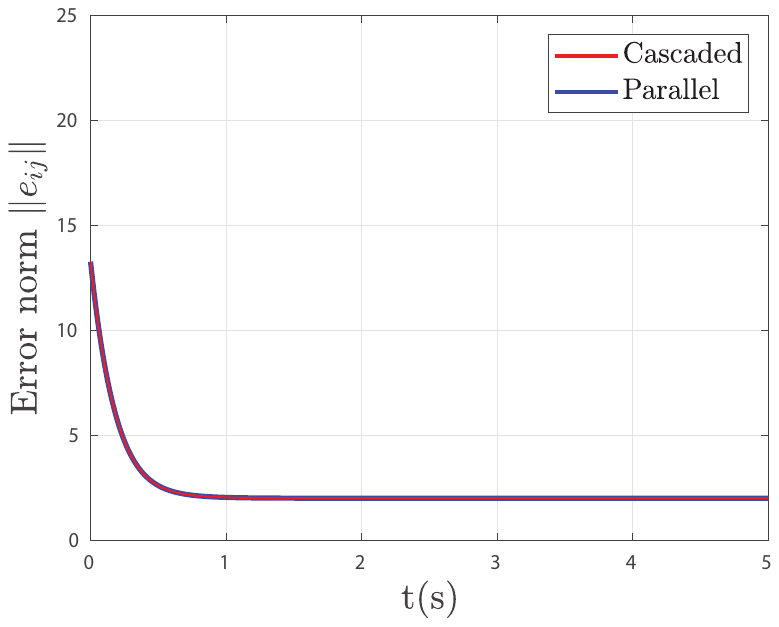}
}
\subfigure[]{
\includegraphics[width=0.445\textwidth]{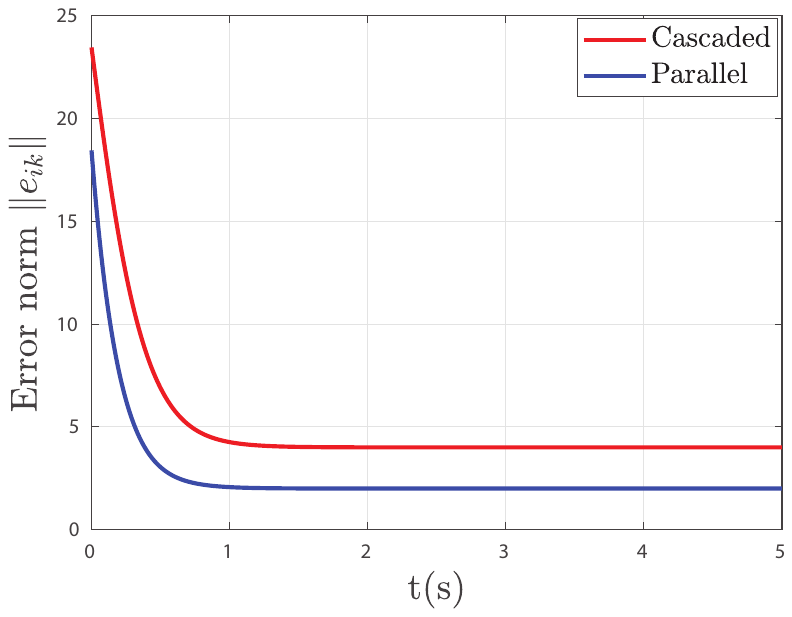}
}   
\caption{Comparison of two formations' tracking error evolutions. (a) Error norm between the first leader $R_i$ and first follower $R_j$. Note that the lines for the cascaded and parallel cases are the same (i.e., the red and blue lines of the graph are in the same position, and therefore not fully visible). (b) Error norm between the first leader $R_i$ and last robot $R_k$. }
\label{fig:CompareRes}
\end{figure}

\subsubsection*{ISS gains calculation for $n$-robot formations}
To obtain the total ISS gains for a formation of multiple leader--follower pairs, we start from the terminal nodes (i.e., vertices with outdegree of 0) and apply Eqs.~\eqref{eq:ISS_gains_composite}, \eqref{eq:ISS_gains_composite2}, and \eqref{eq:ISS_gains_composite3} iteratively based on the algorithm given in~\cite{tanner2004leader}, thus reducing the original graph to a depth of one.

To illustrate this calculation for a specific target topology, consider a graph $G = (V, E, D)$ with an $n \times n$ adjacency matrix $A$, where $a_i$ represents the $i$-th row. Begin by defining the following vectors:
\begin{align}
    &\hat{\beta}^0 \triangleq [\hat{\beta_1}^0 \ldots \hat{\beta_n}^0]^T \\
    &\hat{\gamma}^0 \triangleq [\hat{\gamma_1}^0 \ldots \hat{\gamma_n}^0]^T~,
\end{align}
where $\hat{\beta}$ and $\hat{\gamma}$ are the gain estimates defined in Eq.~\eqref{eq:ISS_est}. 

After $k+1$ iterations, we obtain:
\begin{equation}
\begin{split}
    &\hat{\beta}^{k+1} \triangleq [\hat{\beta_1}^{k+1} \ldots \hat{\beta_n}^{k+1}]^T \\
    &\hat{\gamma}^{k+1} \triangleq [\hat{\gamma_1}^{k+1} \ldots \hat{\gamma_n}^{k+1}]^T
\end{split}~.
\end{equation}
Then, $\hat{\beta}^{k+1}$ and $\hat{\gamma}^{k+1}$ can be calculated as
\begin{equation}
\begin{split}
    &\hat{\beta}^{k+1} =  \hat{\beta}^k + \eta_{n-k}c_{\beta}^k \\
    &\hat{\gamma}^{k+1} =  \hat{\gamma}^k + \eta_{n-k}c_{\gamma}^k \\
    &\eta_{n-k} = [\underbrace{0~ \ldots ~0}_{\text{$n-k-1$}} ~ 1 ~\ldots ~ 0]^T \\
    &c_{\beta}^k = a_{n-k}\hat{\beta}^ka_{n-k}\hat{\gamma}^k\hat{\beta}_{n-k}^k + (a_{n-k}\hat{\beta}^k)^2 \\
    &c_{\gamma}^k = a_{n-k}\hat{\beta}^ka_{n-k}\hat{\gamma}^k\hat{\gamma}_{n-k}^k + a_{n-k}\hat{\gamma}^k\hat{\gamma}_{n-k}^k
\end{split}~.
\end{equation}

For any formation, the algorithm that is iteratively applied, Eqs.~\eqref{eq:ISS_gains_composite}, \eqref{eq:ISS_gains_composite2}, and \eqref{eq:ISS_gains_composite3}, will terminate in at most $n-1$ steps (i.e., the maximum path length in a graph with $n$ vertices). 

It is important to note that the depth of a formation's graph will affect its stability: the higher the depth of the graph, the larger the ISS gains will be. Similarly, the depth of the graph will affect its robustness~\cite{wang2009robustness} in response to noisy local information (i.e., under random disturbances in the information transferred between robots).

\subsection*{Discussion: incorporating feed-forward information for improved performance}

If a follower robot receives velocity information (feed-forward) from its leader without any time delay, the control law Eq.~\eqref{eq:Control_input} can be rewritten as follows:
\begin{equation}
          \begin{aligned}
           \Vec{u}_j &=  f(t, \Vec{z}_{ij}, \Vec{d}_{ij}) + \Vec{u}_i  \\
           &= \boldsymbol{K}^j(\Vec{z}_{ij} - \Vec{d}_{ij}) + \Vec{u}_i \\
           &= \boldsymbol{K}^j\Vec{e}_{ij} + \Vec{u}_i
        \end{aligned}~,
     \label{eq:Control_law2}
\end{equation} 
where the first term in the control input is used to maintain the desired displacement from the leader and the second term is used to follow the leader's reference velocity. 
 
If we substitute the control law given by Eq.~\eqref{eq:Control_law2} into the equation for the motion model given by Eq.~\eqref{eq:Motion_model}, we obtain
     \begin{align}
     \Dot{\Vec{p}}_j =  \boldsymbol{K}^j\Vec{e}_{ij} + \Vec{u}_i~.
     \end{align}
Then, the leader-follower kinematics given by Eq.~\eqref{eq:LF_kinematics} can be rewritten as
    \begin{equation}
        \Dot{\Vec{e}}_{ij} = -\boldsymbol{K}^j\Vec{e}_{ij}~.
        \label{eq:Closed-loop_error2}
    \end{equation}
    
By solving the differential equation Eq.~\eqref{eq:Closed-loop_error2}, we then obtain
    \begin{equation}
        \Vec{e}_{ij}(t) = \mathcal{C}e^{-(\boldsymbol{K}^j)^Tt}~,
        \label{eq:diff_eq_sol}
    \end{equation}
where $\mathcal{C}$ is any positive constant. Since $(\boldsymbol{K}^j)^T = \boldsymbol{K}^j$ is positive definite, then $\|\Vec{e}_{ij}(t)\| \to 0$. As $t \to \infty$, $\Dot{\Vec{p}}_j$ (or $\Vec{u}_j$) approaches $\Vec{u}_i$, such that the leader can steer the follower with the velocity $\Vec{u}_i$.

\begin{remark}
    Unlike the control law given by Eq.~\eqref{eq:Control_law}, the control law in Eq.~\eqref{eq:Control_law2} is not influenced by the leader's velocity, due to the use of feed-forward information. This means that the stability of the leader-follower pair is not affected by whether the leader is stationary or moving. 
\end{remark}

To demonstrate the benefit of using feed-forward information, we conduct a simulation using the same parameters as the previous case, but with the follower robot receiving the leader's velocity information without any delay. The results in Fig.~\ref{fig:ErrorNormFF}a,b, demonstrate that the error norm indeed approaches zero regardless of the leader's velocity.

\rhead{Discussion: feed-forward information}

\begin{figure}[h]
\centering
\subfigure[]{
\includegraphics[width=0.455\textwidth]{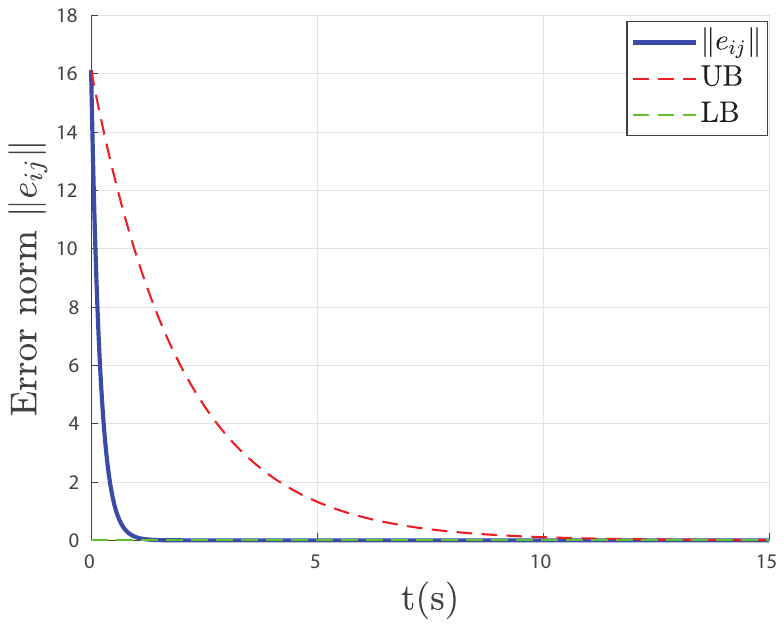}
}
\subfigure[]{
\includegraphics[width=0.475\textwidth]{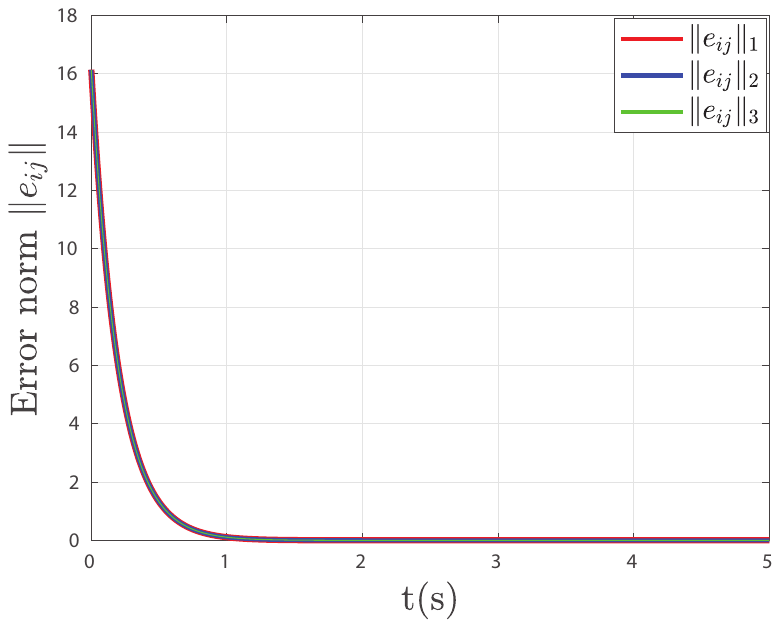}
}
\caption{Norm of the relative formation tracking error with feed-forward information. (a) Stationary leader with the upper bound (UB) and lower bound (LB) of the error norm. Note that the lower bound line (green) is along the bottom of the graph. (b) Moving leader under different velocity regimes. Note that the lines for the three velocity regimes are the same (i.e., the red, blue, and green lines of the graph are all in the same position, and therefore not fully visible). }
\label{fig:ErrorNormFF}
\end{figure}

\clearpage\rhead{}
\section*{Section \ref*{SM:algorithm}. SoNS control algorithm details}
\lhead{Section \ref*{SM:algorithm}. SoNS algorithm}
The open-source code for the SoNS control algorithm and all experiment setups is available in an online repository.

\subsection*{SoNS preliminaries}
\rhead{SoNS preliminaries}

\subsubsection*{Graph preliminaries}

A SoNS is a directed rooted tree, denoted by a directed graph $G=\{V,E\}$ with sets of attributes $A_V$ and $A_E$ associated to $V$ and $E$ respectively, where $V$ represents the set of robots $r_i$ ($|V|=n$), $E$ represents the links $e_{ij}$ between parent robots $r_i$ and child robots $r_j$, and the robot at the root node $r_1$ is the SoNS-brain.   
In other words, a parent of robot $r_j$ is the robot immediately upstream from $r_j$ in the SoNS graph $G$, and similarly, a child of robot $r_i$ is a robot immediately downstream from $r_i$ in the SoNS graph $G$. The subgraph of graph $G$ that includes robot $r_i$ and all the robots downstream from it (i.e., all of its children, all of its children's children, and so on) is denoted as graph $H_{r_i}$.

The set of children of a robot $r_i$ is denoted by $C_{r_i}$ and the set of all robots that are directly connected (whether child or parent) to robot $r_i$ is denoted by $F_{r_i}$.
Note that, because the SoNS graph $G$ is a tree (i.e., a connected acyclic graph), a robot cannot have more than one parent and no cycles are present in the SoNS.

The set of attributes $A_V$ associated to $V$ includes one categorical variable $\textsc{type}_i = \tau_\ell$ for each robot $r_i$, where $\tau_\ell \in \{\tau_1,\tau_2,\ldots,\tau_n\}$.\footnote{In our current implementation, because we use aerial and ground robots and do not make other distinctions in robot type, $\textsc{type}$ is a binary condition, where $\tau_\ell \in \{\tau_1,\tau_2\}$.}
The set of attributes $A_E$ associated to $E$ includes two attributes for each link $e_{ij}$, which are: the relative position of $r_j$ w.r.t.~$r_i$, denoted by the displacement vector $\boldsymbol{d}_{r_i r_j}$, and the relative orientation of $r_j$ w.r.t. $r_i$, denoted by the unit quaternion $\boldsymbol{q}_{r_i r_j}$.

The target subgraph that a robot $r_i$ is attempting to have built downstream from it is denoted as $H_{r_i}^{\boldsymbol{*}}$, 
and its components are denoted likewise (e.g., a target robot node is denoted as $r_i^{\boldsymbol{*}}$, a target link as $e_{ij}^{\boldsymbol{*}}$, and a target displacement as $\boldsymbol{d}_{r_i r_j}^{\boldsymbol{*}}$).
Note that the target subgraph $H_{r_1}^{\boldsymbol{*}}$ of the SoNS-brain at root node $r_1$ is equivalent to $G^{\boldsymbol{*}}$, which denotes the overall target graph of a SoNS.

\begin{remark}
    No robot in a SoNS has access to the graph $G=\{V,E\}$ and attributes $A_V, A_E$ that would represent the current state of the SoNS (nor the subgraph $H$ that would represent the current state of any branch of the SoNS) at any point in time. The current SoNS graph $G$ and the overall target graph $G^{\boldsymbol{*}}$ are used only for analysis of the experiment results (see calculations of error and lower bound in Sec.~4.2 in the main paper). However, each robot $r_i$ has access to the target subgraph $H_{r_i}^{\boldsymbol{*}}$ that it is attempting to have built downstream from it. Each robot also tracks the attributes of some of its immediate neighbors, i.e., robots in $F_{r_i}$, using strictly local communication and local sensing.
\end{remark}

\subsubsection*{Robot initialization requirements}

Independent of a SoNS, each robot is required to initialize with three static identifiers: a non-unique robot type $\textsc{robotTYPE} = \tau_\ell$, where $\tau_\ell \in \{\tau_1,\tau_2,\ldots,\tau_n\}$; a unique robot identifier $\textsc{robotID} \in \{1,2,\ldots,n\}$; and a unique robot rank $\textsc{robotRANK}$ that can be assigned manually or can be assigned randomly according to the uniform distribution $U(0,1)$.\footnote{In our current implementation, the robot ranks are normally generated randomly, with the generated number having 15 decimal digits of precision (double precision floating point) and the seed used for the pseudorandom generator being the microsecond of the UTC time at the moment of initialization. For safety considerations in the constrained indoor robot arena, we sometimes manually assign a robot the rank of 1, which helps ensure that the self-organized processes of the SoNS do not result in robots getting too close to the boundary of the indoor arena, see below for details.}
Note that $\textsc{robotTYPE}$, $\textsc{robotID}$ and $\textsc{robotRANK}$ are all independent of a robot's vertex position $r_i$ in the SoNS graph $G$.

The SoNS control algorithm considers all robots to have the same motion and robot-to-robot sensing capabilities.
All robots are required to accept omnidirectional motion control inputs. All robots are also required to be capable of sensing the relative position, relative orientation, and robot identifier $\textsc{robotID}_i$ of robots in their sensing range (e.g., using computer vision and unique fiducial markers) and to be capable of mutual sensing (i.e., if robot A can sense robot B, then robot B can also sense robot A). If not all of the robot platforms being used have all of these features by default, then any gap in capability and/or heterogeneity of robot platforms is handled by external control layers that are specific to the individual robots (see Secs.~\ref{SM:aerial} and~\ref{SM:ground} respectively for the aerial and ground robot control layers used in this study).

Each robot initializes as $r_1$ of a graph $G$ with one node and no links. In other words, it initializes as the SoNS-brain of its own single-robot SoNS. Each robot also initializes with a default target $H_{r_1}^{\boldsymbol{*}}$, which is equivalent to $G^{\boldsymbol{*}}$, and with a certain configuration of the SoNS control algorithm to use as its local copy, detailed below.

\subsection*{Self-organized node attributes}

Each robot updates its individual node attributes at each time step in a self-organized way, using strictly local communication and local sensing. Because the shared information used to inform these node attributes is updated asynchronously, the node attributes are usually updated with some delay. 
Using the SoNS control algorithm, the desired overall SoNS behaviors execute correctly without all (or even most) of the node attributes having to be fully up-to-date.

\subsubsection*{SoNS identifier and SoNS rank}

The SoNS identifier of robot $r_i$ is denoted as $\textsc{SoNSrootID}_i$. The SoNS identifier $\textsc{SoNSrootID}_i$ of robot $r_i$ is equivalent to the robot identifier of its respective SoNS-brain, i.e., of the root of graph $G$, denoted as $\textsc{robotID}(r_1)$. 
At each time step, each child robot $r_j$ receives a SoNS identifier $\textsc{SoNSrootID}_i$ from its parent and updates its own $\textsc{SoNSrootID}_j \leftarrow \textsc{SoNSrootID}_i$ accordingly, which it will then send to its own children, if it has any. 
If a robot $r_i$ has no parent (i.e., it is the SoNS-brain), then it takes its own robot identifier as its SoNS identifier, such that $\textsc{SoNSrootID}_i \leftarrow \textsc{robotID}(r_i)$. 

Similarly to the SoNS identifer $\textsc{SoNSrootID}$, $\textsc{SoNSrootRANK}_i$ denotes the SoNS rank of robot $r_i$ and is equivalent the robot rank of its respective SoNS-brain $\textsc{robotRANK}(r_1)$. Like the SoNS identifier, each robot takes the SoNS rank received from its parent (Eq.~\ref{eq:rank-sons}) or takes its own robot rank as its SoNS rank if it has no parent (Eq.~\ref{eq:rank-brain}):
\begin{equation}
\label{eq:rank-sons}
\textsc{SoNSrootRANK}_j \leftarrow~\textsc{SoNSrootRANK}_i~,
\end{equation}
\begin{equation}
\label{eq:rank-brain}
\textsc{SoNSrootRANK}_i \leftarrow \textsc{robotRANK}(r_i)~. 
\end{equation}
In other words, every robot in a SoNS has the same SoNS identifier and SoNS rank, once all robots' attributes are up-to-date.

Note that, when the robot occupying the role of SoNS-brain changes, the SoNS identifier and SoNS rank of robot $r_i$ will take $\mu_1$ time steps to reflect the change, where $\mu_1$ is equal to the path length between robot $r_i$ and the new SoNS-brain.

Each robot $r_i$ also stores its most recent former SoNS identifier, if it has any former ones, denoted by $\textsc{SoNSrootID}^\textsc{OLD}_i$, and the number of time steps $t^\textsc{OLD}_i$ passed since it had that former SoNS identifier, according to its own internal clock.

\rhead{Self-organized node attributes}

\subsubsection*{Downstream vertex cardinality and vertex height}

Recall that the subgraph $H_{r_i}$ is the subgraph of robot $r_i$ and all robots downstream from it.
The vertex cardinality (i.e., the order) of the subgraph $H_{r_i}$ of robot $r_i$ is denoted by
$\lambda^{\textsc{card}}(r_i)$.

The cardinality of vertices in subgraph $H_{r_i}$ that have the vertex attribute $\textsc{type}_i = \tau_\ell$ is denoted by $\lambda^{\textsc{card}}_{\tau_\ell}(r_i)$, and is defined as 
\begin{equation}
\lambda^{\textsc{card}}_{\tau_\ell}(r_i) = |r_j \in V(H_{r_i}) : \textsc{type}_j = \tau_\ell|~,
\end{equation}
where $V(H_{r_i})$ denotes the set of vertices of $H_{r_i}$.
At each time step, each child robot $r_j$ sends its parent $r_i$ its respective $\lambda^{\textsc{card}}_{\tau_\ell}(r_j)$ values---one for each $\tau_\ell$ represented in the range set $\{ A_V(\textsc{type}): \textsc{type} = \tau_{1}, \ldots, \tau_{n} \}$ of its subgraph $H_{r_j}$. 
In other words, each robot $r_i$ that has children receives at least one $\lambda^{\textsc{card}}_{\tau_\ell}(r_j)$ value from each child. 
Each robot $r_i$ calculates its own $\lambda^{\textsc{card}}_{\tau_\ell}(r_i)$ value(s) at each time step as
\begin{equation}
    \lambda^{\textsc{card}}_{\tau_\ell}(r_i) \leftarrow
    \left\{ 
    \begin{aligned}
    & \sum_{r_j \in C_{r_i}}\lambda^{\textsc{card}}_{\tau_\ell}(r_j) + 1 ~~~~~\text{if}~\textsc{type}_i = \tau_\ell\\
    & \sum_{r_j \in C_{r_i}}\lambda^{\textsc{card}}_{\tau_\ell}(r_j) ~~~~~~~~~~\text{otherwise}\\
    \end{aligned} \right.~,
\end{equation}
which it will then send to its own parent, if it has one. 

The vertex height (i.e., the longest path from the respective vertex to any downstream leaf node) of robot $r_i$ is denoted as $\lambda^{\textsc{high}}(r_i)$, defined as
\begin{equation}
\lambda^{\textsc{high}}(r_i) = |[r_i, \ldots, r_{\alpha}]|~,
\end{equation}
where $r_{\alpha}$ is the furthest leaf node from $r_i$ in subgraph $H_{r_i}$. 
The vertex height $\lambda^{\textsc{high}}(r_i)$ is updated similarly to the downstream vertex cardinality $\lambda^{\textsc{card}}_{\tau_\ell}(r_i)$. At each time step, each child robot $r_j$ sends its vertex height $\lambda^{\textsc{high}}(r_j)$ to its parent and each robot $r_i$ calculates its own $\lambda^{\textsc{high}}(r_i)$ value as
\begin{equation}
    \lambda^{\textsc{high}}(r_i) \leftarrow \max(\lambda^{\textsc{high}}(r_j) | r_j \in C_{r_i}) + 1~,
\end{equation}
which it will then send to its own parent, if it has one. 

Thus, if a robot $r_i$ with type $\tau_\ell$ has no children, its $\lambda^{\textsc{card}}_{\tau_\ell}(r_i) = 1$ and its $\lambda^{\textsc{high}}(r_i) = 1$.
Note that, when a change occurs to the subgraph $H_{r_i}$, the downstream vertex cardinality and vertex height of robot $r_i$ will take at least $\mu_2$ time steps to reflect the change, where $\mu_2$ is equal to the path length between robot $r_i$ and the closest added or removed robot. In some cases, more time steps might be needed, if the closest added or removed robot also has experienced simultaneous changes to its own downstream subgraph.

\subsubsection*{SoNS targets}

Each robot $r_i$ always has a target subgraph $H_{r_i}^{\boldsymbol{*}}$ that it is attempting to have built downstream from it.
At each time step, each child robot $r_j$ receives its target subgraph $H_{r_j}^{\boldsymbol{*}}$ from its parent $r_i$. If a robot $r_i$ has no parent (i.e., it is the SoNS-brain), then it uses its default target subgraph $H_{r_1}^{\boldsymbol{*}}$.
Also at each time step, each robot $r_i$ that has at least one target child $r_j^{\boldsymbol{*}} \in V(H_{r_i}^{\boldsymbol{*}})$ subdivides its $H_{r_i}^{\boldsymbol{*}}$ into a new target subgraph for each target child $r_j^{\boldsymbol{*}}$, as follows:
\begin{equation}
H_{r_j}^{\boldsymbol{*}} = \left\{ V(H_{r_j}^{\boldsymbol{*}}),~ E(H_{r_j}^{\boldsymbol{*}})\right\} =
\left\{ R_{H_{r_i}^{\boldsymbol{*}}}^{+}(r_j^{\boldsymbol{*}}) \cup r_j^{\boldsymbol{*}} \in V(H_{r_i}^{\boldsymbol{*}}),~~~
e_{ij} | i \in V(H_{r_j}^{\boldsymbol{*}}) \right\}
\end{equation}
where $R_{H_{r_i}^{\boldsymbol{*}}}^{+}(r_j^{\boldsymbol{*}})$ denotes all vertices in the directed graph ${H_{r_i}^{\boldsymbol{*}}}$ reachable from vertex $r_j^{\boldsymbol{*}}$, i.e., all downstream vertices.
When robot $r_i$ becomes connected to a new child $r_j$ that matches its target $r_j^{\boldsymbol{*}} \in V(H_{r_i}^{\boldsymbol{*}}$), it sends the child the respective target subgraph $H_{r_j}^{\boldsymbol{*}}$.

From its own target subgraph $H_{r_i}^{\boldsymbol{*}}$ and the target subgraphs $H_{r_j}^{\boldsymbol{*}}$ it calculates for its children, each robot $r_i$ also calculates the target downstream vertex cardinality and target vertex height for itself, $\lambda^{\textsc{card}}_{\tau_\ell}(r_i)^{\boldsymbol{*}}$ and $\lambda^{\textsc{high}}(r_i)^{\boldsymbol{*}}$, and for each of its children, $\lambda^{\textsc{card}}_{\tau_\ell}(r_j)^{\boldsymbol{*}}$ and $\lambda^{\textsc{high}}(r_j)^{\boldsymbol{*}}$.

\subsubsection*{Neighbor information}

For each of its current children $r_j \in C_{r_i}$, robot $r_i$ stores the robot identifier $\textsc{robotID}(r_j)$, vertex position $r_j$ in subgraph $H_{r_i}$, and the most recent downstream vertex cardinality $\lambda^{\textsc{card}}_{\tau_\ell}(r_j)$ and vertex height $\lambda^{\textsc{high}}(r_j)$ values it has received.

At each time step, robot $r_i$ also stores the most recent displacement $\boldsymbol{d}_{r_i r_j}$ and the relative orientation $\boldsymbol{q}_{r_i r_j}$ for each current child or parent robot $r_j \in F_{r_i}$, according to its own sensor information.
Recall that the target subgraph $H_{r_i}^{\boldsymbol{*}}$ of robot $r_i$ includes the target displacement $\boldsymbol{d}_{r_i r_j}^{\boldsymbol{*}}$ and the target relative orientation $\boldsymbol{q}_{r_i r_j}^{\boldsymbol{*}}$ w.r.t.~robot $r_i$ for each child $r_j \in F_{r_i}$. Thus, the robot $r_i$ also has access to this information for its children (but not for its parent).

\subsubsection*{Hierarchically-organized sensor information}

When a child robot $r_j$ senses an environmental feature $a$, it uses its own sensor information and local decisions to determine whether to respond to this feature itself (detailed in the next subsection), whether to send this information upstream to be considered by its parent $r_i$, or both.
At each time step that $r_j$ sends feature information to its parent $r_i$, it sends a displacement ${\boldsymbol{d}_{r_j a}}$ and relative orientation ${\boldsymbol{q}_{r_j a}}$ for each feature $a$ w.r.t. $r_j$. At each time step that $r_i$ receives information about a feature $a$ from its child, it first converts the information into its own coordinate system, such that
\begin{equation}
\begin{aligned}
    {\boldsymbol{d}_{r_i a}} &= \boldsymbol{d}_{r_i r_j} + \textsc{RT}(\boldsymbol{d}_{r_j a}, \boldsymbol{q}_{r_i r_j}) \\
    \boldsymbol{q}_{r_i a} &= \textsc{H}(\boldsymbol{q}_{r_i r_j}, \boldsymbol{q}_{r_j a})
\end{aligned}~,
\end{equation}
where $\textsc{RT}(\boldsymbol{x}, \boldsymbol{y})$ is a function to rotate vector $\boldsymbol{y}$ by unit quaternion $\boldsymbol{x}$ using the Euler–Rodrigues formula, with the Euler parameters given by the coefficients of quaternions $\boldsymbol{y}^p = (0,\boldsymbol{y})$ and $\boldsymbol{x}$, and $\textsc{H}(\boldsymbol{x}, \boldsymbol{y})$ takes the Hamilton product of two quaternions $\boldsymbol{x}$ and $\boldsymbol{y}$. 
Then, robot $r_i$ makes its own decision about whether to respond to this feature, forward it upstream to be considered by its respective parent, or both.

\subsubsection*{Hierarchically-organized actuation instructions}

In the current SoNS implementation, all actuation is motion-based. Motion instructions are communicated asynchronously in a SoNS using six reference vectors, which each robot then uses to calculate its motion control outputs at each time step.

Each robot $r_i$ has three linear velocity reference vectors $\boldsymbol{v}_{r_i}^{\textsc{local}}$, $\boldsymbol{v}_{r_i}^{\textsc{hierarchical}}$, and $\boldsymbol{v}_{r_i}^{\textsc{global}}$; and three angular velocity reference vectors $\boldsymbol{\omega}_{r_i}^{\textsc{local}}$, $\boldsymbol{\omega}_{r_i}^{\textsc{hierarchical}}$, and $\boldsymbol{\omega}_{r_i}^{\textsc{global}}$. 

The reference vectors $\boldsymbol{v}_{r_i}^{\textsc{local}}$ and $\boldsymbol{\omega}_{r_i}^{\textsc{local}}$ are dedicated to local goals and can only be updated by $r_i$ according to its own sensor information and local decisions. 
At each time step, robot $r_i$ can either set its $\boldsymbol{v}_{r_i}^{\textsc{local}} = [0,0,0]$ or define its $\boldsymbol{v}_{r_i}^{\textsc{local}}$ according to the displacement of a target ${\boldsymbol{d}_{r_i a}}$ based on its own sensor information and its goal to either reach or avoid some position, object, or other robot in the environment, such that
\begin{equation}
  \boldsymbol{v}_{r_i}^{\textsc{local}} =
    \begin{cases}
      \boldsymbol{v}^{\textsc{maximum}} \times \widehat{\boldsymbol{d}_{r_i a}} 
         & \text{if $\delta < k_1$}\\
      \min\left( \boldsymbol{v}^{\textsc{maximum}}, ~~\lim_{\delta \to\ {k_1}} -\log(\delta) \times \frac{\delta - k_1}{k_3} \times k_2 \right) \times \widehat{\boldsymbol{d}_{r_i a}}~~ 
         & \text{if $k_1 < \delta < k_3$}\\
      [0,0,0] 
         & \text{otherwise}
    \end{cases}       
\label{eq:avoider}
\end{equation}
where $\boldsymbol{v}^{\textsc{maximum}}$ is a vector constant, $\widehat{\boldsymbol{d}_{r_i a}}$ denotes the unitized vector $\boldsymbol{d}_{r_i a}$, $\delta$ is the Euclidean distance between the robot $r_i$ and the position that it is targeting to be reached or avoided, and $k_1$, $k_2$, $k_3$ are scalar constants, with all constants to be defined in the SoNS control algorithm before robots are initialized.
Then, the relative orientation of a target ${\boldsymbol{q}_{r_i a}}$ is translated into an angular vector and
an equivalent of Eq.~\ref{eq:avoider}, where $\boldsymbol{\omega}^{\textsc{maximum}}$ and $\widehat{\boldsymbol{q}_{r_i a}}$ are used instead of $\boldsymbol{v}^{\textsc{maximum}}$ and $\widehat{\boldsymbol{d}_{r_i a}}$ respectively, is used by $r_i$ to calculate $\boldsymbol{\omega}_{r_i}^{\textsc{local}}$.

The reference vectors $\boldsymbol{v}_{r_i}^{\textsc{hierarchical}}$ and $\boldsymbol{\omega}_{r_i}^{\textsc{hierarchical}}$ are dedicated to global goals that are defined by the SoNS-brain and can only be updated by $r_i$ according to information received from its parent, in the following way. 
At each time step, for each child robot $r_j$, a parent robot $r_i$ first converts the target displacement $\boldsymbol{d}_{r_i r_j}^{\boldsymbol{*}}$ and target relative orientation $\boldsymbol{q}_{r_i r_j}^{\boldsymbol{*}}$ into the coordinate system of $r_j$, producing the new target displacement $\boldsymbol{d}_{r_j {r_j}'}^{\boldsymbol{*}}$ and target relative orientation $\boldsymbol{q}_{r_j {r_j}'}^{\boldsymbol{*}}$, which it then sends to its child $r_j$. The target displacement is then used by $r_j$ to calculate $\boldsymbol{v}_{r_j}^{\textsc{hierarchical}}$, such that
\begin{equation}
     \boldsymbol{v}_{r_j}^{\textsc{hierarchical}} =
    \begin{cases}
      \boldsymbol{v}^{\textsc{default}} \times \widehat{\boldsymbol{d}_{r_j {r_j}'}^{\boldsymbol{*}}}         
           & \text{if $|\boldsymbol{d}_{r_j {r_j}'}^{\boldsymbol{*}}| > k_4 $}\\
      \boldsymbol{v}^{\textsc{default}} \times \frac{(|\boldsymbol{d}_{r_j {r_j}'}^{\boldsymbol{*}}| - k_5)}{k_4} \times \widehat{\boldsymbol{d}_{r_j {r_j}'}^{\boldsymbol{*}}}   
           & \text{if $k_5 < |\boldsymbol{d}_{r_j {r_j}'}^{\boldsymbol{*}}| < k_4$}\\
      [0,0,0]          
           & \text{otherwise}
    \end{cases}~,       
\label{eq:v-hierarchical}
\end{equation}
where $\boldsymbol{v}^{\textsc{default}}$ is a vector constant, $\widehat{\boldsymbol{d}_{r_j {r_j}'}^{\boldsymbol{*}}}$ denotes the unitized vector $\boldsymbol{d}_{r_j {r_j}'}^{\boldsymbol{*}}$, and $k_4$ and $k_5$ are scalar constants, with all constants to be defined in the SoNS control algorithm before robots are initialized.
Then, an equivalent of Eq.~\ref{eq:v-hierarchical}, where $\boldsymbol{\omega}^{\textsc{default}}$ and $\boldsymbol{q}_{r_j {r_j}'}^{\boldsymbol{*}}$ are used instead of $\boldsymbol{v}^{\textsc{default}}$ and $\boldsymbol{d}_{r_j {r_j}'}^{\boldsymbol{*}}$ respectively, is used by $r_j$ to calculate $\boldsymbol{\omega}_{r_j}^{\textsc{hierarchical}}$.

The reference vectors $\boldsymbol{v}_{r_i}^{\textsc{global}}$ and $\boldsymbol{\omega}_{r_i}^{\textsc{global}}$ are dedicated to global goals defined by any robot in the same SoNS as $r_i$ and can be updated according to the sensor information of $r_i$ or according to messages received from any robot that $r_i$ is connected to (i.e., parent or child). Likewise, robot $r_i$ can also send messages of this type to any robot it is connected to.
At each time step, each robot $r_i$ first converts any $\boldsymbol{v}_{r_j}^{\textsc{global}}$ it has received from a robot $r_j$ (whether parent or child) into its own coordinate system, using the current relative orientation $\boldsymbol{q}_{r_i r_j}$ it has stored for robot $r_j$, such that
\begin{equation}
    \boldsymbol{v}_{r_i}^{\textsc{global}} = \textsc{RT}(\boldsymbol{v}_{r_j}^{\textsc{global}}, \boldsymbol{q}_{r_i r_j})~.
\end{equation}
Robot $r_i$ then sends the produced $\boldsymbol{v}_{r_i}^{\textsc{global}}$ to any robot $r_j \in F_{r_i}$ from which it did not originally receive the respective $\boldsymbol{v}_{r_j}^{\textsc{global}}$.
If a robot $r_i$ receives $\boldsymbol{v}_{r_j}^{\textsc{global}}$ vectors from multiple robots $r_j$ in one time step, it sums all its $\boldsymbol{v}_{r_i}^{\textsc{global}}$ from that time step, producing a new ${\boldsymbol{v}_{r_i}^{\textsc{global}}}'$ for its own use.
Likewise, the same series of operations is applied to $\boldsymbol{\omega}_{r_j}^{\textsc{global}}$ to produce ${\boldsymbol{\omega}_{r_i}^{\textsc{global}}}'$.

If a robot is currently a SoNS-brain, it only uses reference vectors $\boldsymbol{v}_{r_i}^{\textsc{global}}$ and $\boldsymbol{\omega}_{r_i}^{\textsc{global}}$ (all other reference vectors are set to $[0,0,0]$). 

At each time step, the SoNS control algorithm implemented on robot $r_i$ produces two motion control outputs. All robots, regardless of their type, receive the same style of omnidirectional motion control outputs from the SoNS control algorithm. These omnidirectional outputs are translated into the appropriate motor inputs by the motion control layer of the individual robot (see Secs.~\ref{SM:aerial} and~\ref{SM:ground}, respectively, for details of the motion control layer of the aerial robots and ground robots used in this study).

The motion control outputs that each robot $r_i$ produces are a target linear velocity vector $\boldsymbol{v}_{r_i}^{\boldsymbol{*}}$ and target angular velocity vector $\boldsymbol{\omega}_{r_i}^{\boldsymbol{*}}$, updated at each time step according to the current reference vectors of robot $r_i$, as follows:
\begin{equation}
\begin{aligned}
\boldsymbol{v}_{r_i}^{\boldsymbol{*}} &= \boldsymbol{v}_{r_i}^{\textsc{hierarchical}} + \boldsymbol{v}_{r_i}^{\textsc{local}} + \boldsymbol{v}_{r_i}^{\textsc{global}}\\
\boldsymbol{\omega}_{r_i}^{\boldsymbol{*}} &= \boldsymbol{\omega}_{r_i}^{\textsc{hierarchical}} + \boldsymbol{\omega}_{r_i}^{\textsc{local}} + \boldsymbol{\omega}_{r_i}^{\textsc{global}}\\
\end{aligned}~.
\end{equation}

\subsection*{Establishing and reconfiguring SoNS connections}

A SoNS connection $e_{ij}$ between a parent robot $r_i$ and a child $r_j$ is established when one robot successfully recruits another robot, in the following way.

When two robots mutually sense each other, both robots send recruitment messages to each other and reach a consensus about which one of them should become the parent. Depending on the configuration of the SoNS algorithm, the robot that becomes the parent $r_i$ of the new link is either the robot with the larger
SoNS rank $\textsc{SoNSrootRANK}$, larger downstream vertex cardinality $\lambda^{\textsc{card}}$, or larger combination of the two values. 
When the new link $e_{ij}$ is formed, if the child robot $r_j$ had a former parent $r_k$, it breaks its link $e_{kj}$ with its former parent and also sets a timer to ignore any new recruitment messages from any robot with its former SoNS identifier $\textsc{SoNSrootID}^\textsc{OLD}_i$ until the time steps passed since it had that former SoNS identifier $t^\textsc{OLD}_i$ is equal to $\lambda^{\textsc{high}}(r_j)$. Note that, because the $r_j$ and its former parent $r_k$ are no longer connected and are likely to mutually sense each other, they will likely try to recruit each other, and in this case the robot $r_j$ will become the parent of $r_k$ if the SoNS rank $\textsc{SoNSrootRANK}$ is the recruitment metric, because $r_j$ would then have the new (larger) SoNS rank $\textsc{SoNSrootRANK}$ that it inherited from its new parent $r_i$.

A SoNS connection $e_{ij}$ can be broken for reasons other than a new recruitment (either because of a local decision by robot $r_i$ or $r_j$, or because of a disturbance or other error). When such a break occurs, the former child $r_j$ will no longer have a parent. It then becomes the SoNS-brain of its own SoNS, returns to its original defaults with which it initialized, and randomly generates a new $\textsc{robotRANK}$ according to the defined uniform distribution. If it has children, it sends them updated information accordingly.

At any SoNS connection $e_{ij}$, the parent robot $r_i$ can choose to transfer its child $r_j$ to another robot $r_k$, in an operation called a ``handover." When parent robot $r_i$ hands over its child to $r_k$, a new link $e_{kj}$ is established and then the former link $e_{ij}$ is broken.

\rhead{SoNS connections}

\subsection*{Node allocation}

For each parent robot $r_i$, each of its children robots $r_j$ need to be allocated to one of its target child positions $r_j^{\boldsymbol{*}} \in C_{r_i}^{\boldsymbol{*}}$. Each node allocation operation occurs in the following way.
For all children $r_j \in C_{r_i}$ and all target child positions $r_j^{\boldsymbol{*}} \in C_{r_i}^{\boldsymbol{*}}$ of parent robot $r_i$, define
\begin{equation}
\begin{aligned}
    \textit{Source displacements:}&~~~~~~ \bm{\mathcal{S}}^{\textsc{d}} &=& ~\boldsymbol{d}_{r_i r_j} ~\forall r_j \in C_{r_i}  \\
    \textit{Source downstream cardinalities:}&~~~~~~ \bm{\mathcal{S}}^{\textsc{card}} &=& ~\lambda^{\textsc{card}}(r_j) ~\forall r_j \in C_{r_i}  \\ 
    \textit{Target displacements:}&~~~~~~ \bm{\mathcal{T}}^{\textsc{d}} &=& ~\boldsymbol{d}_{r_i r_j}^{\boldsymbol{*}} ~\forall r_j^{\boldsymbol{*}} \in C_{r_i}^{\boldsymbol{*}}  \\
    \textit{Target downstream cardinalities:}&~~~~~~ \bm{\mathcal{T}}^{\textsc{card}} &=& ~\lambda^{\textsc{card}}(r_j^{\boldsymbol{*}}) ~\forall r_j^{\boldsymbol{*}} \in C_{r_i}^{\boldsymbol{*}}  \\
    \textit{Displacement costs:}&~~~~~~ \bm{\mathcal{W}}^{\textsc{d}} &=& ~|| \mathcal{S}_j^{\textsc{d}} - \mathcal{T}_j^{\textsc{d}}|| ~\forall~ \textsc{CB}(r_j \in C_{r_i} ,~ r_j^{\boldsymbol{*}} \in C_{r_i}^{\boldsymbol{*}}) \\
    \textit{Cardinalities costs:}&~~~~~~ \bm{\mathcal{W}}^{\textsc{card}} &=& ~|| \mathcal{S}_j^{\textsc{card}} - \mathcal{T}_j^{\textsc{card}}|| ~\forall~ \textsc{CB}(r_j \in C_{r_i} ,~ r_j^{\boldsymbol{*}} \in C_{r_i}^{\boldsymbol{*}})
\end{aligned}~,
\label{eq:costs}
\end{equation}
where $\textsc{CB}(\boldsymbol{x},\boldsymbol{y})$ denotes all unique combinations of entries in $\boldsymbol{x}$ and $\boldsymbol{y}$.
Then, using the matrices defined by Eq.~\ref{eq:costs}, use an algorithm to allocate nodes. In the current implementation, we use the following algorithm:
\begin{algorithm}[h!]
\caption{Node Allocation}
\label{alg:nodes}
\begin{algorithmic}[1]
\Procedure{allocate}{$\bm{\mathcal{S}}^{\textsc{d}}, \bm{\mathcal{S}}^{\textsc{card}}, \bm{\mathcal{T}}^{\textsc{d}}, \bm{\mathcal{T}}^{\textsc{card}}, \bm{\mathcal{W}}^{\textsc{d}}, \bm{\mathcal{W}}^{\textsc{card}}$}
\State Sort the values in $\bm{\mathcal{W}}^{\textsc{d}}$ and apply the same sorting transformation to $\bm{\mathcal{W}}^{\textsc{card}}$
\State Construct a network flow graph $G_{\textsc{flow}}$ using the defined matrices
\State Apply the Network Flow Algorithm in~\cite{wayne1999generalized,cormen2001section} to obtain the network with the maximum possible flow rate and select the source in $\bm{\mathcal{S}}^{\textsc{d}}, \bm{\mathcal{S}}^{\textsc{card}}$ that should be matched to each target in $\bm{\mathcal{T}}^{\textsc{d}}, \bm{\mathcal{T}}^{\textsc{card}}$ accordingly
\EndProcedure
\end{algorithmic}
\end{algorithm}

Robots in a SoNS continually redistribute themselves by adjusting their node allocations based on currently sensed conditions.
For example, when at least one new SoNS connection $e_{ij}$ is being established and the parent robot $r_i$ currently has more than one target child position $r_j^{\boldsymbol{*}} \in C_{r_i}^{\boldsymbol{*}}$, or vice versa (i.e., more than one connection being established and at least one target child position), each incoming child robot $r_j$ needs to be allocated to a target child position $r_j^{\boldsymbol{*}}$. At this time, any existing children $r_j$ of robot $r_i$ are likewise (re)allocated.
In any SoNS containing more than two robots, the robots continually redistribute themselves at every time step (even when no new robots are being added), in the following way.

\rhead{Node allocation}

At each time step, each robot $r_i$ that has a parent $r_k$ and has children $r_j$ runs the node allocation algorithm (see Algorithm~\ref{alg:nodes}) twice. 
First, robot $r_i$ runs the node allocation algorithm considering itself and all its downstream robots as members of the source matrices, and likewise considering all target positions of its parent's downstream subgraph $H_{r_k}$ as members of the target matrices. If a child robot $r_j$ matches with a single target, robot $r_i$ instructs it to move towards that target position, or hands it over to robot $r_k$ if the target is on a branch that is not reachable by $r_i$. If a child $r_j$ matches with multiple targets, robot $r_i$ hands it over to robot $r_k$ so that $r_j$ and its respective downstream robots can be reallocated to the multiple targets. If robot $r_i$ matches a target, it moves towards that target position, otherwise, it moves towards its parent $r_k$.
Second, if robot $r_i$ has any remaining robots $r_j$ that are still its children, it checks if any of its children should substitute it, based on the inequality
\begin{equation}
\boldsymbol{d}_{r_k r_j} \times \frac{\boldsymbol{d}^*_{r_k r_j}}{\|\boldsymbol{d}^*_{r_k r_j}\|} < \boldsymbol{d}_{r_k r_i} \times \frac{\boldsymbol{d}^*_{r_k r_i}}{\|\boldsymbol{d}^*_{r_k r_i}\|}~.
\end{equation}
If a child $r_j$ is in a better position, robot $r_i$ hands it over to $r_k$ so that it can be reassigned.
Third, robot $r_i$ runs the node allocation algorithm for its remaining children $r_j$. If one child matches one target, robot $r_i$ instructs it to move to the target position. If multiple children match the same target, robot $r_i$ selects the nearest child, instructs it to move to the target position, and then hands over to it the other children that matched with its target. If a child matches with multiple targets, robot $r_i$ sends it the list of targets so that it can reallocate its own downstream robots accordingly. If a child does not meet any of these conditions, robot $r_i$ hands it over to its parent $r_k$. 
At any point during these reallocations, if a robot $r_i$ senses that one of its links $e_{ij}$ spatially intersects with another link $e_{k \ell}$ in its sensing range, it triggers hand over operations such that $r_i$ and $r_k$ will swap children, resulting in the new links $e_{i \ell}$ and $e_{kj}$.

\clearpage\rhead{}
\section*{Section \ref*{SM:simulator}. Simulator setup}
\lhead{Section \ref*{SM:simulator}. Simulator}

We conduct our simulated experiments in the ARGoS multi-robot simulator~\cite{pinciroli2012argos}, a widely used simulator for swarm robotics research, using simulation models---custom-developed for this study~\cite{All2022:techreport-002,All2022:techreport-001}---of the aerial and ground robots used in our real experiments (see Fig.~\ref{fig:robots}). 

\begin{figure}[hbtp]
\centering
\subfigure[]{
\includegraphics[width=0.495\textwidth]{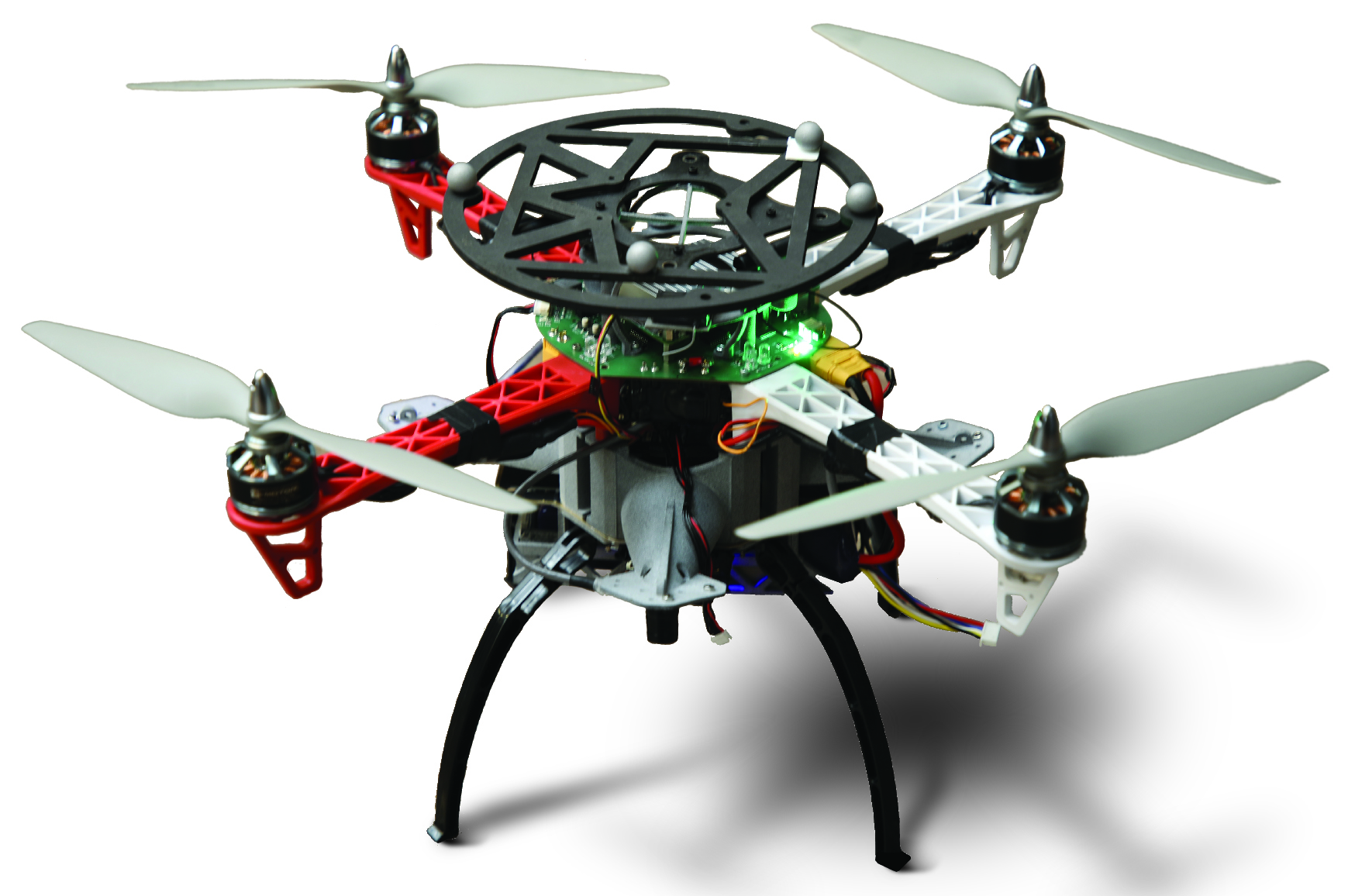}
}\hspace{-5mm}
\subfigure[]{
\includegraphics[width=0.495\textwidth]{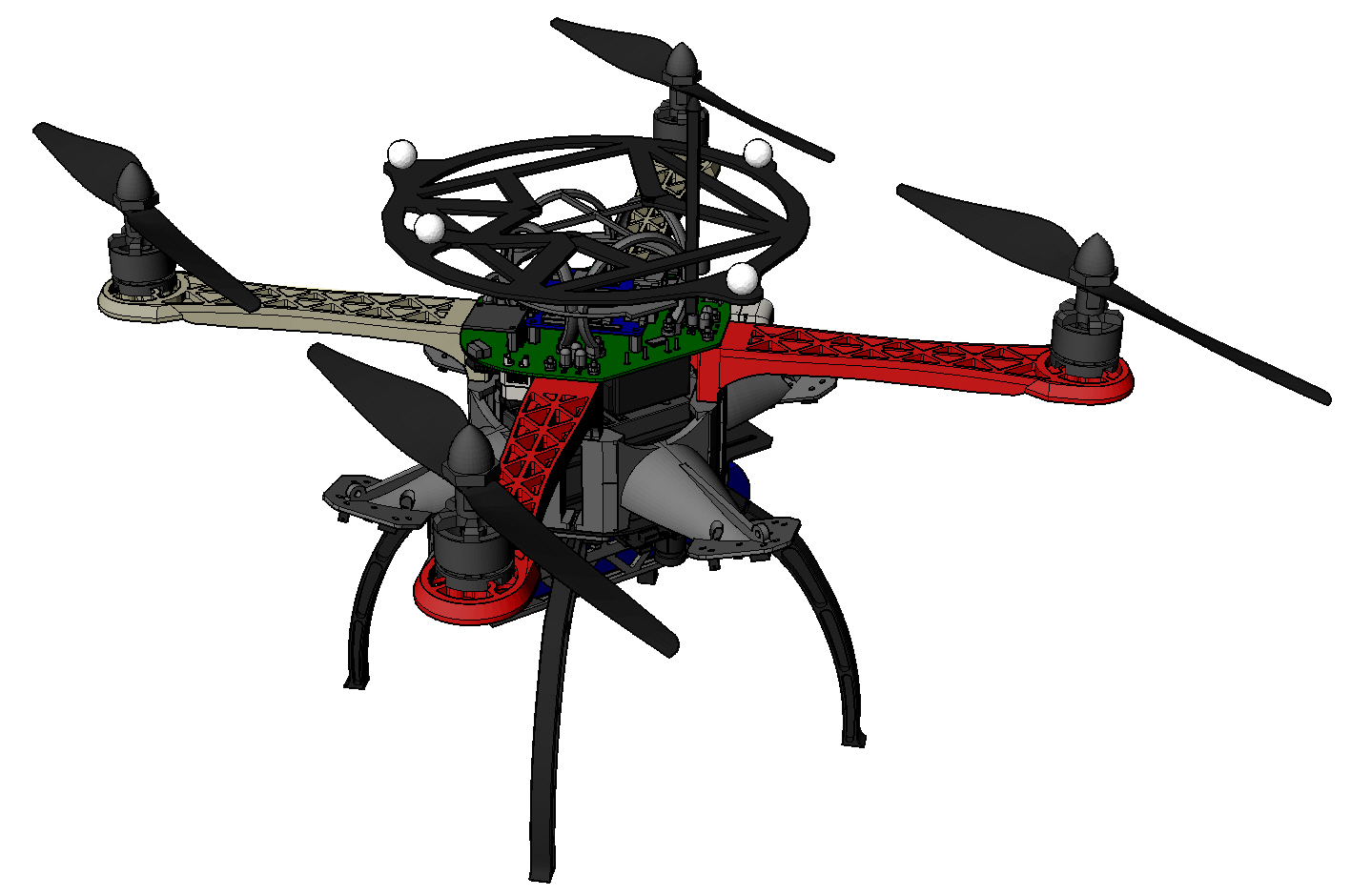}
}
\subfigure[]{
\includegraphics[width=0.35\textwidth]{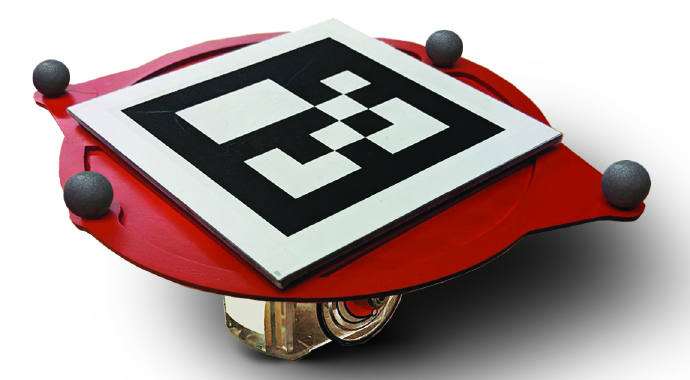}
}
\subfigure[]{
\includegraphics[width=0.35\textwidth]{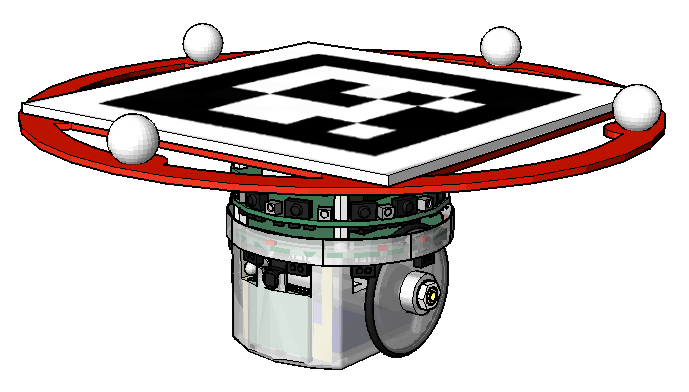}
}
\subfigure[]{
\includegraphics[width=0.25\textwidth]{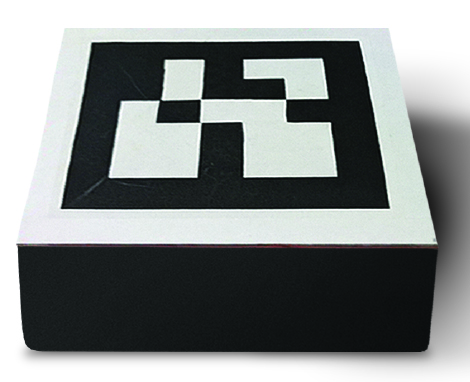}
}
\subfigure[]{
\includegraphics[width=0.25\textwidth]{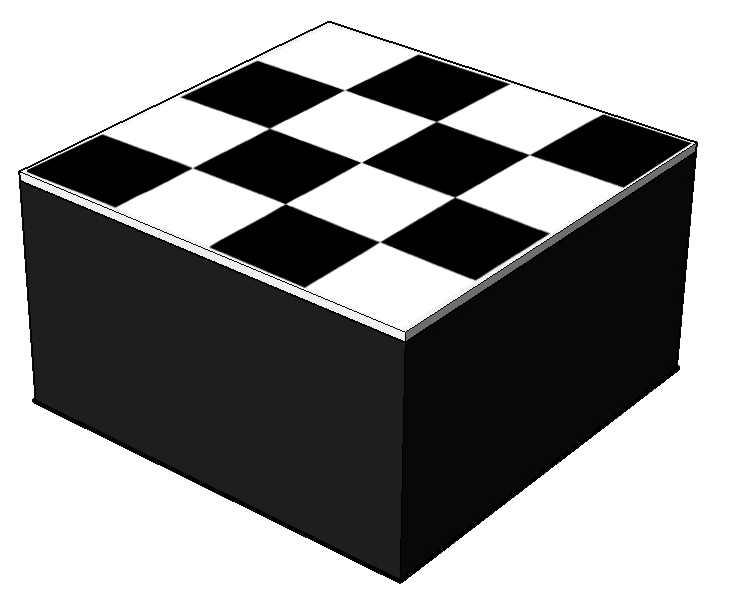}
}
\caption{The (a) real and (b) simulated aerial robot. The (c) real and (d) simulated ground robot. A (e) real and (f) simulated physical obstacle; not to size. (The spherical markers visible in (c) and (d) are not used by the SoNS software; they are used for data logging in the real experiments, see Sec.~S8.)}
\label{fig:robots}
\end{figure}

In ARGoS, sensors and actuators are plug-ins that have either read-only access or the ability to modify specific entities in the simulated 3D space~\cite{pinciroli2012argos}.
In our setup, to control the aerial and ground robots in a way that is replicable in simulation, we use executables based on the libraries of ARGoS that we have custom-developed for this study~\cite{All2022:techreport-002,All2022:techreport-001}. These executables initialize the sensors and actuators of the robots in such a way that the provided high-level control interface matches that of the sensor and actuator plug-ins of our robot models in the ARGoS simulator. In other words, each robot's control interface provided by the ARGoS libraries is an abstraction layer on top of the physical hardware (similar to the layer in~\cite{allwright2019open}) and we have used this same abstraction layer to create simulation models of the robots in ARGoS. Therefore, we can run exactly the same control software on both the real robots and the simulated robots, using C++ and the Lua scripting language. 

The robot motion models used in simulation as well as the motion control layers used to execute outputs from the SoNS software in both simulation and reality are detailed in Sec.~\ref{SM:aerial} for the aerial robots and Sec.~\ref{SM:ground} for the ground robots. 

For the models of the robot sensors and actuators in simulation, we conducted a trial-and-error calibration process to tune the speed and noise parameters. The speed parameters of the robot actuators in simulation have been tuned so that the speeds of the real robots match those of the simulated robots when controlled by the same control scripts. This calibration is straightforward, thus the maximum and average speeds of the robots in our simulated and real experiments are equivalent when running the same controllers. By contrast, noise is influenced by many (often unknown) factors and therefore the noise parameters of simulated sensors and actuators are much more challenging to calibrate and the result often underestimates the noise present in reality. To help compensate for this shortcoming, we tuned the noise parameters in simulation so that the simulated robots displayed greater noise in simple behaviors than the real robots when running the same controllers. However, in more complex missions (i.e., those used in our experiments), the noise displayed in reality is still noticeably greater than the noise displayed in simulation. The difference in error in our simulator and in reality is detailed in Sec.~\ref{SM:cross-verify}. The difference is most noticeable in portions of the missions when robots are in a steady phase and trying to remain stationary: in these steady phases, the observed error can be primarily attributed to the noise present in the aerial robot's hovering behavior, which is noticeably greater in reality than in simulation (see Sec.~\ref{SM:cross-verify}). (Note that in our cross-verification between the simulator and reality in Sec.~\ref{SM:cross-verify}, the observed differences are associated primarily with the noise of the robots, not with the SoNS behaviors.)

For the interactions between the robots and their environment, the ground robots are simulated using ARGoS's 3D-dynamics engine based on the ODE library~\cite{pinciroli2012argos}, because they can physically interact with objects on the ground, and the aerial robots are simulated using ARGoS's more lightweight 3D particle engine~\cite{pinciroli2012argos}.
Objects on the ground (such as obstacles to be avoided, see Fig~\ref{fig:robots}e,f) are also simulated using ARGoS's 3D-dynamics engine based on the ODE library, so that the ground robots can physically interact with them.

\begin{figure}[hbpt]
\centering
\subfigure[]{
\includegraphics[width=0.9\textwidth]{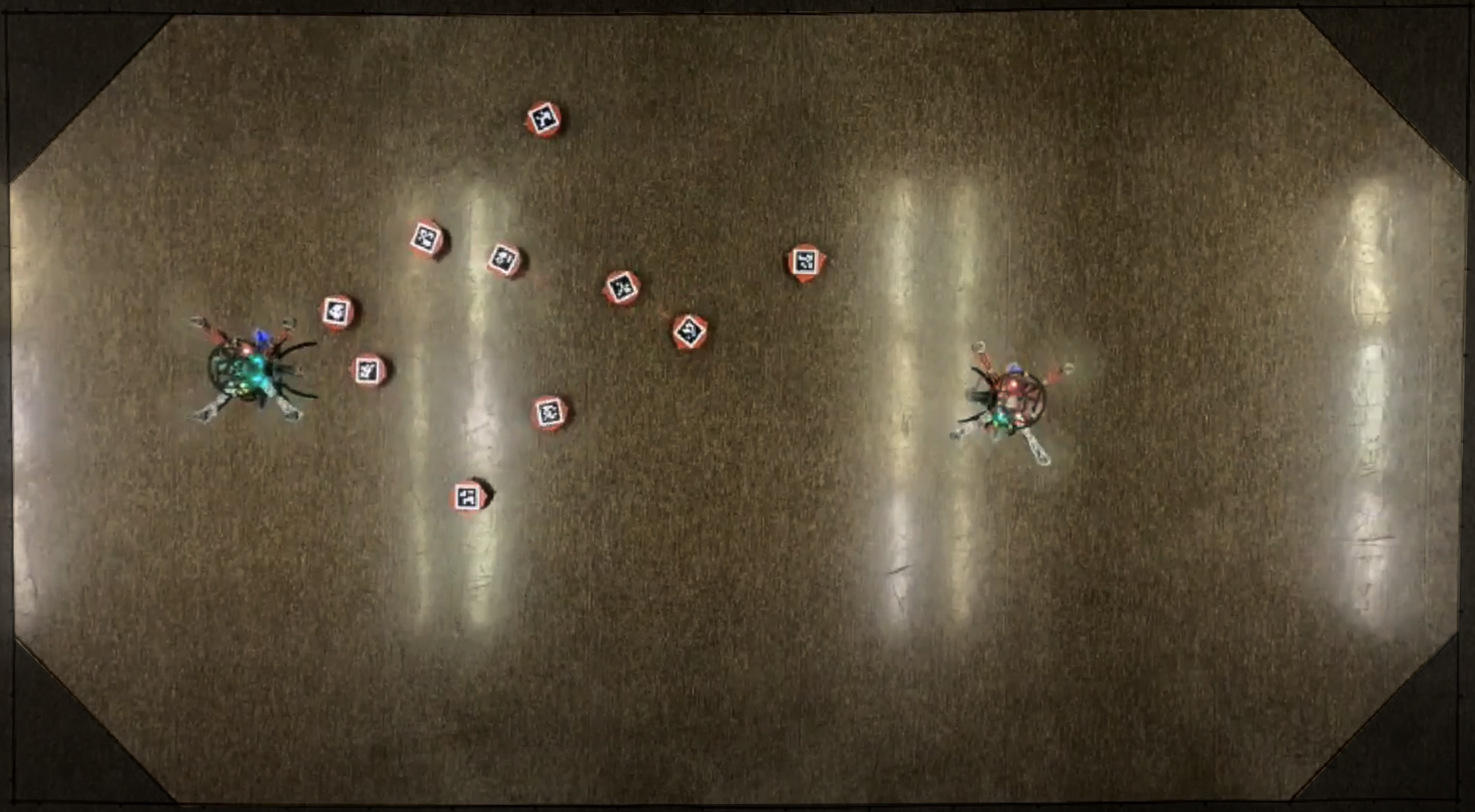}
}
\subfigure[]{
\includegraphics[width=0.9\textwidth]{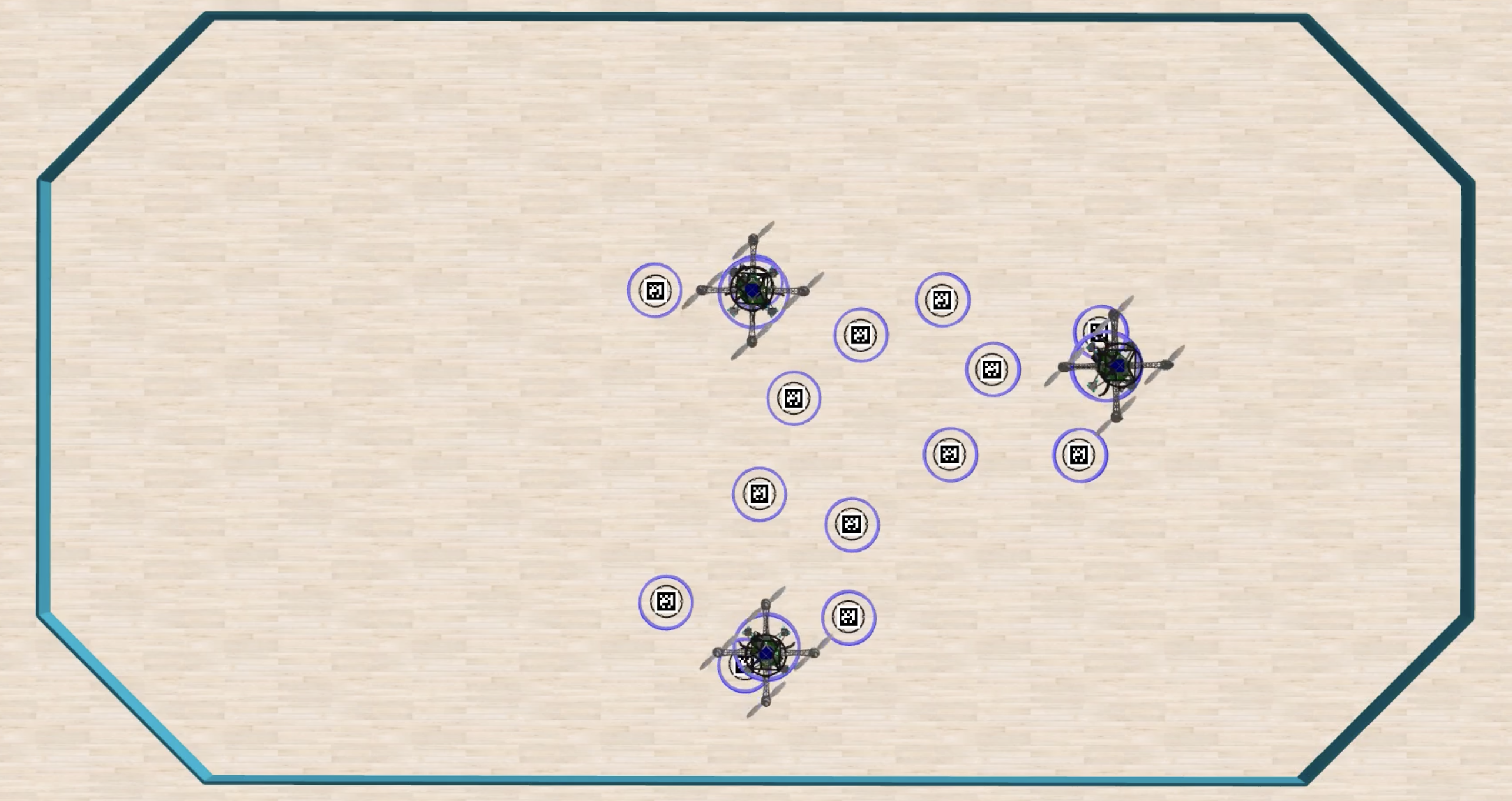}
}
\caption{The (a) real and (b) simulated arenas. Note that the blue circles around robots in (b) are for visualization purpose only and are not part of the robot models. Without these blue circles, the dimensions of the robots proportional to the arenas in the real and simulated environments can be seen to be equivalent.}
\label{fig:arenas}
\end{figure}

For the different experiment setups, customized environments in ARGoS are created using a python script to generate the .argos setup file. This generation can include arena walls, environmental features composed of obstacles that occur at certain positions relative to the size of the arena, and uniformly random initial positions of robots and positions of obstacles. 
In simulated experiment setups that are supposed to match a real experiment exactly (see Fig.~\ref{fig:arenas}), the generated arena walls in ARGoS match the border dimensions of the real arena floor. In other simulated experiments, in which more robots are used than in any real experiments, the arena size is generated to be large enough for the respective system size.
We also display some objects (blue arrows to denote SoNS connections, blue circles to denote SoNS-brains) for visualization only, when recording videos directly from ARGoS (see key frames of missions in Sec.~\ref{SM:results}, for example in Fig.~\ref{fig:mission1-keyframes}).
All simulation models and experiment setups are available in the online code repository.

In all experiments, the time step in the simulator is equivalent to 0.2~seconds.
Data logging is executed at each time step and records the global position and orientation of each robot as well as the SoNS information held locally by each robot. The experiment logs are available in the online data repository.

Depending on the system size, simulated experiments were run either on the experimenter's local machine or on our in-house IRIDIA computing cluster\footnote{\url{https://majorana.ulb.ac.be/wordpress/cluster-composition/}}, composed of 36 computational nodes for a total of 1536 CPU cores, in four logical racks. In analysis of our experiment results, computation work has only been assessed using CPU clock cycles, which is not affected by the machine on which the simulation is run.

\clearpage\rhead{}
\section*{Section \ref*{SM:aerial}. Aerial robot setup}
\lhead{Section \ref*{SM:aerial}. Aerial robot}

The aerial robot used in the experiments is the \textit{S-drone} (Swarm-drone) quadrotor platform, which we custom-developed for this study.

The full hardware details of the \textit{S-drone} quadrotor are available in an open-access technical report and open-source repositories~\cite{OguHeiAllZhuWahGarDor2022:techreport-010}\footnote{The open-access technical report of the \textit{S-drone} hardware, which includes URLs to the open-source repositories: \hyperlink{https://iridia.ulb.ac.be/IridiaTrSeries/link/IridiaTr2022-010.pdf}{https://iridia.ulb.ac.be/IridiaTrSeries/link/IridiaTr2022-010.pdf}.}, including the hardware description and specifications; design files and bill of materials; instructions for assembly, operation, tuning, and camera calibration for detection and tracking of fiducial markers; and example routines.
As described in~\cite{OguHeiAllZhuWahGarDor2022:techreport-010}, a Linux operating system compiled by Yocto is installed on the UpCore single-board computer of the S-drone~\cite{salvador2014embedded, All2022:techreport-002}.
The control software running in Linux is ARGoS~\cite{pinciroli2012argos}. The SoNS software is comprised of Lua scripts loaded and executed by ARGoS.

The \textit{S-drone} quadrotor platform has several possible operation modes. In this study, we use its operation mode for fully autonomous flight control, based on autonomous vision-based navigation and relative positioning. These capabilities are primarily supported by the quadrotor's single-board computer for onboard processing, its downward-facing optical flow smart camera module paired with single-point LiDAR for relative position estimation, and its four downward-facing camera modules for detection of fiducial markers and for relative position estimation.

In the remainder of Sec.~\ref*{SM:aerial}, we present the modeling and flight control of the \textit{S-drone} quadrotor. The section is organized in four subsections in which we give:
\begin{itemize}
\item the modeling preliminaries, including the reference frames, rotation matrix, quadrotor states, sensor modeling, and motor dynamics,
\item the nonlinear model of the quadrotor system, including its kinematics and dynamics,
\item a state space representation for system behavior analysis, 
\item the linearized model used for flight control, and \item the design of the position and attitude controllers.\end{itemize} 

The presented quadrotor modeling is used in the simulated experiments (see Sec.~\ref{SM:simulator} of the supplementary materials for details about the simulator setup).
The presented flight controllers and support for system behavior analysis (which are both based on the quadrotor modeling) are used for the simulated experiments and experiments with real robots.
The flight controllers and simulation model are also available in open-source code repositories~\cite{OguHeiAllZhuWahGarDor2022:techreport-010}\footnote{The open-access technical report of the \textit{S-drone} hardware, which includes URLs to the open-source repositories: \hyperlink{https://iridia.ulb.ac.be/IridiaTrSeries/link/IridiaTr2022-010.pdf}{https://iridia.ulb.ac.be/IridiaTrSeries/link/IridiaTr2022-010.pdf}.}.

\subsection*{Modeling preliminaries}

The preliminary information needed for the quadrotor modeling includes the reference frames, the rotation matrix for transformations between reference frames, the quadrotor states, and the sensor and motor information. 

\begin{remark}
In this study, there is no use of remote control, GPS, or other methods for off-board control or absolute positioning. 
All positioning is relative.
The only absolute measurements used in the study are those of the quadrotor's on-board magnetometer, which gives measurements in a fixed inertial frame, but these measurements are not shared between quadrotors. The flight control inputs also make partial use of the fixed inertial frame (during the calculation of waypoints). However, the quadrotor only has access to its own calculations in its fixed inertial frame, not the measurements and calculations of other quadrotors, and only uses them to control its flight based on autonomous navigation commands it has generated onboard using relative positioning. Hence, there is no absolute reference synchronized between quadrotors or otherwise used to coordinate navigation.
Navigation is strictly self-organized, using exclusively local communication and relative positioning.
\end{remark}

\subsubsection*{Reference frames}

\rhead{Modeling preliminaries}

The quadrotor is modeled using the body frame (denoted $B$) and the inertial frame (denoted $\mathcal{I}$), as shown in Fig.~\ref{fig:Ref_Frames}.

The body frame is a relative coordinate system that represents the body of the quadrotor. The origin of the frame is the quadrotor center of mass, the $x$-axis of the frame is the quadrotor roll axis (i.e., longitudinal axis, directed to the front), the $y$-axis is the pitch axis (i.e., transverse axis, directed to the right), and the $z$-axis is the yaw axis (i.e., vertical axis, directed to the bottom).

The inertial frame is a fixed coordinate system defined at an arbitrary point on the Earth's surface and can be defined at any point on the surface. The inertial frame uses a North-East-Down (NED) configuration, in which the $x$-axis is directed northward, the $y$-axis is directed eastward, and the $z$-axis is directed downward.

\begin{figure}
    \centering
    \includegraphics[width=0.9\textwidth]{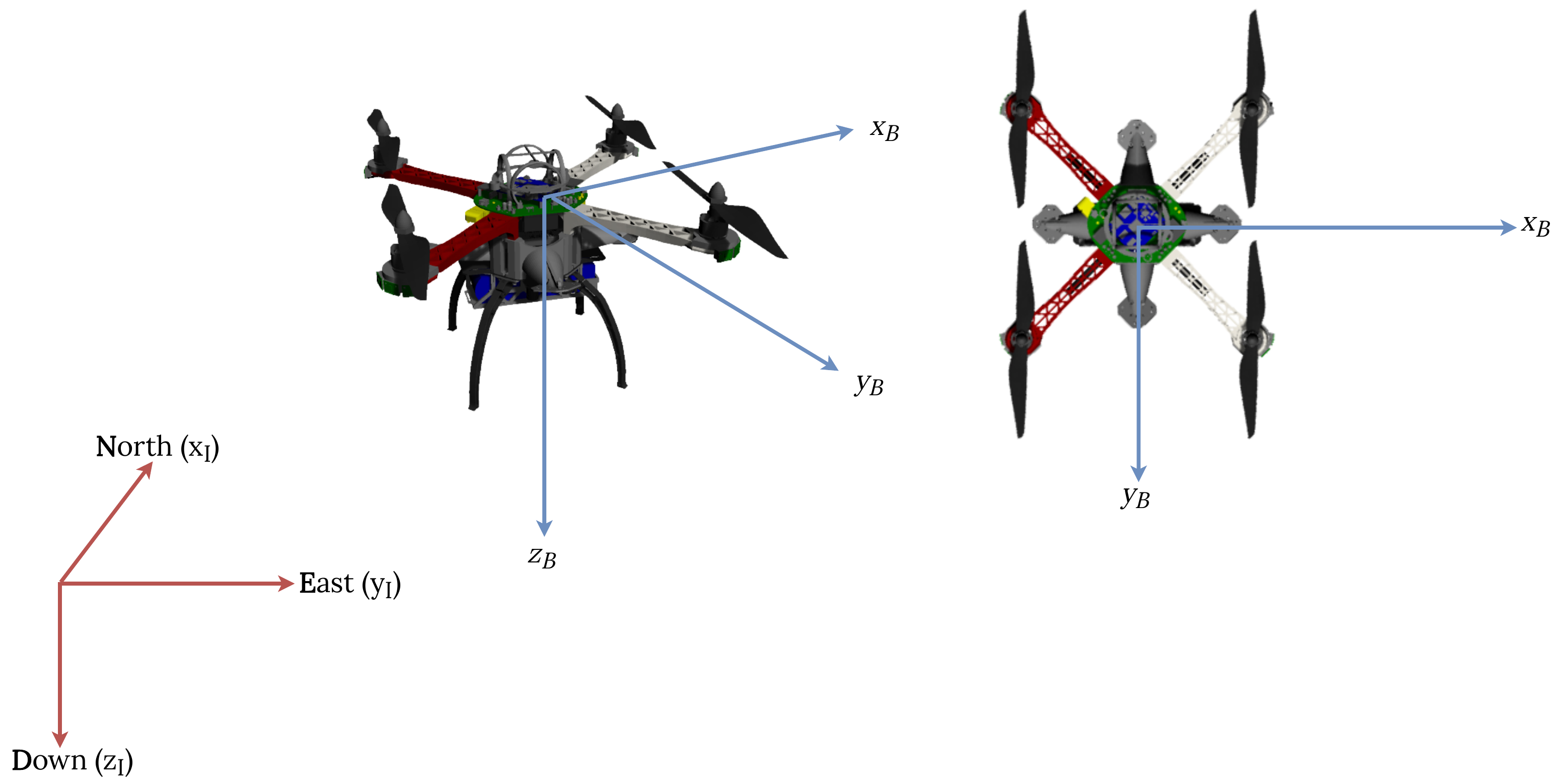}
    \caption{Reference frames.}
    \label{fig:Ref_Frames}
\end{figure}

\subsubsection*{Rotation matrix}

To transform vectors defined in the body frame $B$ into the inertial frame $\mathcal{I}$, we construct a $zyx$ rotation matrix using Euler angles (roll, pitch, yaw).
A $zyx$ rotation involves three rotations (see Fig.~\ref{fig:Rotation}): the frame is rotated around the $z$-axis (yaw rotation), then around the $y$-axis (pitch rotation), and lastly around the $x$-axis (roll rotation).
The $zyx$ rotation matrix $\tensor[^{\mathcal{I}}]{\boldsymbol{R}}{_{\mathcal{B}}}(\Phi, \Theta, \Psi) = \boldsymbol{R}_z(\Psi)\boldsymbol{R}_y(\Theta)\boldsymbol{R}_x(\Phi)$ is defined as:
\begin{equation}
    \label{eq:Rotation_xyz}
    \tensor[^{\mathcal{I}}]{\boldsymbol{R}}{_{\mathcal{B}}}(\Phi, \Theta, \Psi) =
       \begin{bmatrix*}[r]
            \phantom{-}\cos\Psi\cos\Theta  &  \cos\Psi \sin\Theta\sin\Phi  - \sin\Psi\cos\Phi  & \cos\Phi\sin\Theta\cos\Psi + \sin\Phi\sin\Psi \\
            \phantom{-}\sin\Psi\cos\Theta  &  \sin\Psi\sin\Theta\sin\Phi  + \cos\Psi\cos\Phi   & \cos\Phi\sin\Theta\sin\Psi - \sin\Phi\cos\Psi \\
            -\sin\Theta & \sin\Phi\cos\Theta & \cos\Phi\cos\Theta
        \end{bmatrix*}.
\end{equation}

\begin{figure}[ht]
    \centering
    \includegraphics[width=0.95\textwidth]{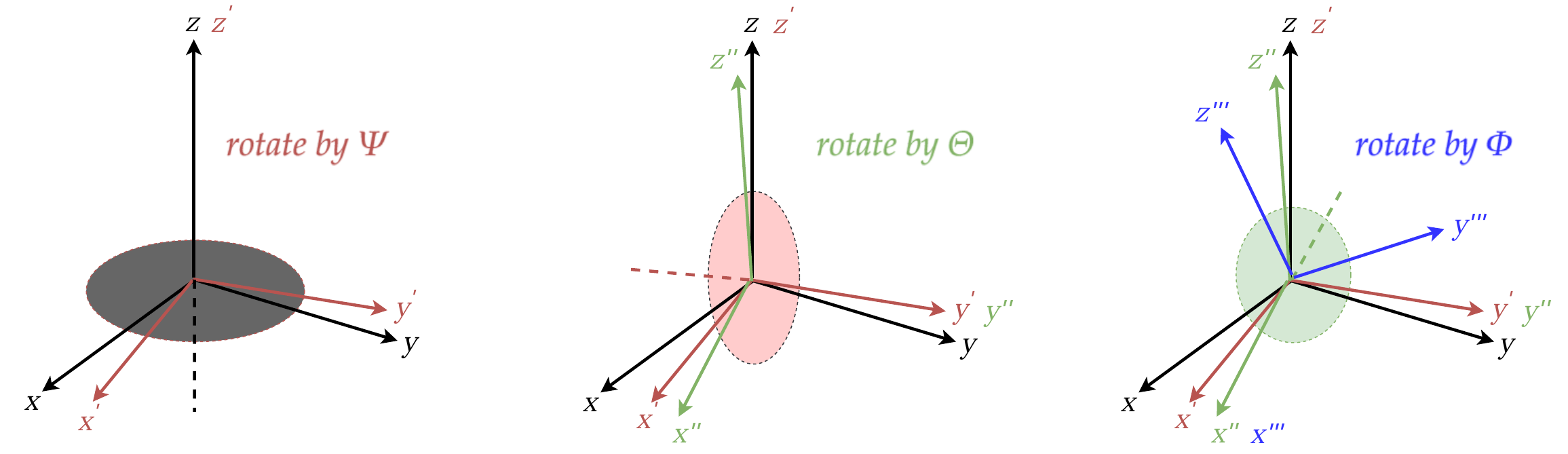}
    \caption{The $zyx$ rotation configuration.}
    \label{fig:Rotation}
\end{figure}

\subsubsection*{Quadrotor states}
The quadrotor has both inertial and body frame states, which can be used to describe its position, attitude, and velocity (see Table \ref{tab:states}). The inertial frame states, which are defined relative to a fixed reference point (the reference point is fixed locally and used only for onboard calculations, not shared or synchronized between quadrotors), include the quadrotor's position coordinates ($x, y, z$), its rotational angles ($\phi, \theta, \psi$), and its linear and angular velocities ($\Dot{x}, \Dot{y}, \Dot{z}, \Dot{\phi}, \Dot{\theta}, \Dot{\psi}$). The body frame states, which are defined relative to the quadrotor itself, include the body frame linear velocities ($u, v, w$) and body frame angular velocities ($p, q, r$). These states are important for understanding the quadrotor's motion and internal dynamics.

\begin{table}[h]
\footnotesize
\centering 
\begin{tabular}{l c c rrrrrrr} 
\hline
\textbf{Vector} &  \textbf{Description}
\\ [0.5ex]
\hline  
 \\[-1ex]
$\boldsymbol{\zeta}_I = \begin{bmatrix} x & y & z \end{bmatrix} $ & Positions in the inertial frame  \\ 
 \\[-1ex]
\hline 
 \\[-1ex]
$\boldsymbol{\eta}_I = \begin{bmatrix} \phi & \theta & \psi \end{bmatrix} $ & Euler angles (roll, pitch, yaw) in the inertial frame \\
 \\[-1ex]
\hline
 \\[-1ex]
$\boldsymbol{V}_I = \begin{bmatrix} \Dot{x} & \Dot{y} & \Dot{z} \end{bmatrix} $ & Linear velocities in the inertial frame \\
 \\[-1ex]
\hline 
 \\[-1ex]
$\boldsymbol{\omega}_I = \begin{bmatrix} \Dot{\phi} & \Dot{\theta}  & \Dot{\psi} \end{bmatrix} $ & Angular velocities in the inertial frame \\
 \\[-1ex]
\hline 
 \\[-1ex]
$\boldsymbol{V}_B = \begin{bmatrix} u & v & w \end{bmatrix} $ & Linear velocities in the body frame \\
 \\[-1ex]
\hline 
 \\[-1ex]
$\boldsymbol{\omega}_B = \begin{bmatrix} p & q & r \end{bmatrix} $ & Angular velocities in the body frame \\  [1ex] 
 \\[-1ex]
\hline
\end{tabular}
\caption{State vectors.}
\label{tab:states}
\end{table}

\subsubsection*{Sensor modeling}
To control the quadrotor, it is necessary to observe the states. The simulation model of the drone uses the following measurement models for the inertial sensors: a three-axis accelerometer, a three-axis gyroscope, and a three-axis magnetometer.

\begin{remark}
The \textit{S-drone} does not rely on external localization information, such as GPS signals, so we do not consider those types of sensors. Instead, we use inertial sensors in our quadrotor to measure the platform's orientation, velocity, and acceleration. These sensors provide independent and non-jammable measurements of the platform, as described in~\cite{farrell2008aided}. 
\end{remark}

The gyroscope measures angular rotation around the body frame axis by measuring the Coriolis Force, which acts on objects in a rotating reference frame~\cite{farrell2008aided}. 
In our measurement model, we assume that the gyroscope measures body frame angular velocities ($p,q,r$) directly, which can be integrated over time to compute the orientation of the sensor ($\phi, \theta,\psi$). The measurement model for the gyroscope is as follows:
\begin{equation}
    \begin{aligned}
        \boldsymbol{\omega}_B^m &= \boldsymbol{\omega}_B + \boldsymbol{b}(t) + \boldsymbol{\mu}  \\ 
        \Dot{\boldsymbol{b}}(t) &= \boldsymbol{\nu}_g \boldsymbol{1}_{3x1}
    \end{aligned} ,
\end{equation}
where $\boldsymbol{\omega}_B^m \in \mathbb{R}^3$ represents the measurement value of the angular velocities ($p,q,r$) in the body frame, 
the slowly time-varying bias term $\boldsymbol{b}(t) \in \mathbb{R}^3$ changes with white Gaussian noise $\boldsymbol{\nu}_g$, $\boldsymbol{\mu} \in \mathbb{R}^3 $ is the measurement noise term, and $\boldsymbol{1}_{3x1}$ is a matrix of 1s in the shape 3 by 1.

The accelerometer detects the forces present and uses them to calculate acceleration, according to the mass of the object and D'Alembert's force principle~\cite{farrell2008aided}. 
Our measurement model for the accelerometer determines the forces acting on the quadrotor and calculates its acceleration as follows:
\begin{equation}
    \begin{aligned}
        \boldsymbol{a}_B^m &= \boldsymbol{a}_B^l + \tensor[^{\mathcal{B}}]{\boldsymbol{R}}{_{\mathcal{I}}}\boldsymbol{a}_I^g \boldsymbol{b}(t) + \boldsymbol{\mu} \\ 
        \Dot{\boldsymbol{b}}(t) &= \boldsymbol{\nu}_a \boldsymbol{1}_{3x1} 
    \end{aligned} ~,
\end{equation}
where $\boldsymbol{a}_B^l \in \mathbb{R}^3$ is the linear acceleration vector in the body frame, $\boldsymbol{a}_I^g \in \mathbb{R}^3$ is the gravitational acceleration vector in the inertial frame, the slowly time-varying bias term  $\boldsymbol{b}(t) \in \mathbb{R}^3$ changes with the white Gaussian noise $\boldsymbol{\nu}_a$, and $\boldsymbol{\mu} \in \mathbb{R}^3 $ is the measurement noise term. 

The magnetometer measures the strength and direction of the earth's magnetic field, is used to help the quadrotor maintain a stable hover, and is used in the autonomous navigation and positioning processes happening onboard. (It is not used as a reference for coordination between quadrotors.) Our measurement model for the magnetometer is:
\begin{equation}
    \begin{aligned}
    \boldsymbol{H}_B^m &= \tensor[^{\mathcal{B}}]{\boldsymbol{R}}{_{\mathcal{I}}}\boldsymbol{H}_I + \boldsymbol{\mu}\\
    \boldsymbol{H}_I = \begin{bmatrix}
        H_x \\ H_y \\ H_z 
    \end{bmatrix}
    &= \| \boldsymbol{H}_I \| \cdot \begin{bmatrix}
        \cos\beta \\ 0 \\ \sin\beta
    \end{bmatrix}
    \end{aligned} ~,
\end{equation}
where $\boldsymbol{H}_I \in \mathbb{R}^3 $ is the magnetic field vector and $\beta$ is the magnetic inclination. 
If the measurements of the roll and pitch angles are already known, the yaw angle can then be calculated as:
\begin{equation}
    \begin{aligned}
        \psi = \arctan\frac{H_x^m\cos\theta + (H_y^m\sin\phi + H_z^m\cos\phi\sin\theta)}{H_x^m\sin\phi - H_y^m\cos\phi}
    \end{aligned}~.
\end{equation}

\subsubsection*{Motor dynamics}
We represent the dynamics of the motors using a first-order transfer function~\cite{bouabdallah2007full}, which allows for modeling the time-varying behavior of the motors and thereby accurately predicting the response of the quadrotor to control inputs. The transfer function describes the relationship between the input and output of the motor and can be used to design control algorithms that accurately regulate the speed and torque of the motors. The model is given as
\begin{equation}
    \begin{aligned}
        \frac{\Omega}{\Omega_d}  = \frac{1}{T_{rot}s + 1}~,
    \end{aligned}
\end{equation}
where $T_{rot}$ is a time constant, $\Omega_d$ is the desired motor speed, and $\Omega$ is the calculated motor speed.

\subsection*{Nonlinear model of the quadrotor}

For our nonlinear model of the quadrotor's dynamic behavior, we make the following assumptions, as in~\cite{bouabdallah2007full}, to simplify some of the modeling calculations:
\begin{itemize}
    \item the quadrotor body structure is rigid,
    \item the quadrotor is symmetrical in all axes, 
    \item the propeller structure is rigid and the oscillation on the quadrotor body does not show motion in the vertical direction, 
    \item the body frame coincides with the quadrotor's center of gravity, 
    \item motors are identical, 
    \item motors are positioned perpendicular to the body frame, 
    \item propeller thrust and drag moment are directly proportional to the square of the motor speed, 
    \item the ground effect is neglected,
    \item the gyroscopic effect of the motors is neglected.
\end{itemize}

\subsubsection*{Quadrotor kinematics} 

\rhead{Nonlinear quadrotor model}

To study the motion of the quadrotor using kinematics, we use the rotation matrix \eqref{eq:Rotation_xyz} to transform linear velocities measured in the body frame to the inertial frame, as follows:
\begin{equation}
    \begin{aligned}
       \Dot{\boldsymbol{\zeta}}_I = \boldsymbol{V}_I  &= \tensor[^{\mathcal{I}}]{\boldsymbol{R}}{_{\mathcal{B}}} \boldsymbol{V}_B \\ 
       \tensor[^{\mathcal{B}}]{\boldsymbol{R}}{_{\mathcal{I}}} &= (\tensor[^{\mathcal{I}}]{\boldsymbol{R}}{_{\mathcal{B}}})^T 
    \end{aligned}~.
    \label{eq:rotation_relation}
\end{equation}

Then, to transform the angular velocities from the body frame to the inertial frame, we use the angular transformation matrix $\tensor[^{\mathcal{I}}]{\boldsymbol{T}}{_{\mathcal{B}}}$, defined as
\begin{equation}
    \begin{aligned}
       \tensor[^{\mathcal{I}}]{\boldsymbol{T}}{_{\mathcal{B}}}
        =
        \begin{bmatrix}
            1     & \sin\phi\tan\theta    & \phantom{-}\cos\phi\tan\theta   \\
            0     & \cos\phi              & -\sin\phi      \\
            0     & \sin\phi\sec\theta    & \phantom{-}\cos\phi\sec\theta
        \end{bmatrix} 
        ~~\text{and}~~
        \tensor[^{\mathcal{B}}]{\boldsymbol{T}}{_{\mathcal{I}}} &= (\tensor[^{\mathcal{I}}]{\boldsymbol{T}}{_{\mathcal{B}}})^T
    \end{aligned}~.
    \label{eq:angular_trans2}
\end{equation}
We can then transform the angular velocities from the body frame to the inertial frame, as
\begin{equation}
    \begin{aligned}
        \Dot{\boldsymbol{\eta}}_I & = \tensor[^{\mathcal{I}}]{\boldsymbol{T}}{_{\mathcal{B}}} \boldsymbol{\omega}_B 
    \end{aligned}~,
    \label{eq:angular_trans}
\end{equation}
and write the quadrotor kinematics as
\begin{equation}
    \begin{aligned}
        \begin{bmatrix}
             \Dot{\boldsymbol{\zeta}}_I  \\  \Dot{\boldsymbol{\eta}}_I
        \end{bmatrix} 
        =
        \begin{bmatrix}
            \tensor[^{\mathcal{I}}]{\boldsymbol{R}}{_{\mathcal{B}}}  & \boldsymbol{0}_{3x3} \\ 
            \boldsymbol{0}_{3x3}&  \tensor[^{\mathcal{I}}]{\boldsymbol{T}}{_{\mathcal{B}}}
        \end{bmatrix}
        \begin{bmatrix}
            \boldsymbol{V}_B \\ \boldsymbol{\omega}_B 
        \end{bmatrix}          
    \end{aligned}~,
    \label{eq:kinematics}
\end{equation}
where $\boldsymbol{0}_{3x3}$ is a matrix of 3s in the shape 3 by 3 and
\begin{equation}
  \Dot{\boldsymbol{\zeta}}_I =
  \begin{cases}
              \Dot{x} &= w[\sin\phi\sin\psi + \cos\phi\cos\psi\sin\theta] -v[\cos\phi\sin\psi + \cos\psi\sin\phi\sin\theta] + u[\cos\psi\cos\theta] \\
        \Dot{y} &= -w[\cos\psi\sin\phi - \cos\phi\sin\psi\sin\theta] + v[\cos\phi\cos\psi + \sin\phi\sin\psi\sin\theta] + u[\cos\theta\sin\psi] \\
        \Dot{z} &= w[\cos\phi\cos\theta] + v[\cos\theta\sin\phi]  - u\sin\theta \\
  \end{cases}
        \label{eq:kinematics2} 
\end{equation}
and
\begin{equation}
  \Dot{\boldsymbol{\eta}}_I =
  \begin{cases}
        \Dot{\phi} &= p + r[\cos\phi\tan\theta] + q[\sin\phi\tan\theta] \\
        \Dot{\theta} &= q\cos\phi - r\sin\phi \\
        \Dot{\psi} &= r\frac{\cos\phi}{\cos\theta} + q\frac{\sin\phi}{\cos\theta}
  \end{cases}~.
        \label{eq:kinematics3} 
\end{equation}

\subsubsection*{Quadrotor dynamics} For the dynamics of the quadrotor, we need to consider the mass and inertia, the forces, and the torques acting on the body. 

To derive the differential equations describing the quadrotor dynamics using the Newton-Euler method~\cite{fossen1999guidance}, we start with the following equalities:
\begin{equation}
    \begin{aligned}
       \begin{bmatrix}
           \boldsymbol{F}_B \\ \boldsymbol{\tau}_B 
       \end{bmatrix}
       =
       \begin{bmatrix}
            m\boldsymbol{I}_{3} & \boldsymbol{0}_{3x3} \\
            \boldsymbol{0}_{3x3} & \boldsymbol{J}
        \end{bmatrix}
        \begin{bmatrix}
            \Dot{\boldsymbol{V}}_B \\ \Dot{\boldsymbol{\omega}}_B 
        \end{bmatrix}
        +
        \begin{bmatrix}
            \boldsymbol{\omega}_B \times m\boldsymbol{V}_B \\
            \boldsymbol{\omega}_B \times \boldsymbol{J}\boldsymbol{V}_B
        \end{bmatrix}
    \end{aligned}
    \label{eq:dynamics}
\end{equation}
\begin{equation*}
\boldsymbol{F}_B = 
   \begin{cases}
        F_x &= m(\Dot{u} + qw -rv) \\
        F_y &= m(\Dot{v} - pw -ru) \\
        F_z &= m(\Dot{w} + pv -qu) \\
    \end{cases} 
    ,~~ 
    \boldsymbol{\tau}_B  = 
   \begin{cases}
        \tau_x &= \Dot{p}J_x - qrJ_y + qrJ_z \\
        \tau_y &= \Dot{q}J_y + prJ_x - prJ_z \\
        \tau_z &= \Dot{r}J_z + pqJ_x + pqJ_y 
    \end{cases}~,
\end{equation*}
where $\boldsymbol{I}_3$ is a $3$ by $3$ identity matrix, $\boldsymbol{J}$ is a $3$ by $3$ diagonal inertia matrix, $m$ is the mass of the quadrotor, and $\boldsymbol{F}_B = \begin{bmatrix} F_x & F_y & F_z \end{bmatrix}^T  \in \mathbb{R}^3, \boldsymbol{\tau}_B = \begin{bmatrix} \tau_x & \tau_y & \tau_z \end{bmatrix}^T \in \mathbb{R}^3$ are force and torque vectors that act on the body. The force vector ($\boldsymbol{F}_B$) is composed of the gravitational force and the thrust force generated by the propulsion system. Similarly, the torque vector ($\boldsymbol{\tau}_B$) is composed of the roll, pitch, and yaw moments as well as the gyroscopic moments generated by the propulsion system. They can be expressed as
\begin{equation}
    \begin{aligned}
        &\boldsymbol{F}_B = \overbrace{mg\tensor[^{\mathcal{B}}]{\boldsymbol{R}}{_{\mathcal{I}}}.\boldsymbol{\hat{e}}_z}^\text{the gravity effect}
        - \overbrace{U_1.\boldsymbol{\hat{e}}_3}^\text{the thrust force} \\
         &\boldsymbol{\tau}_B = \underbrace{\boldsymbol{U}_{\tau}}_{\text{the moment vector}}
        - \underbrace{\boldsymbol{\tau}_g}_{\text{the gyroscopic moment}} \\
    \end{aligned}~,
    \label{eq:dynamics3}
\end{equation}
where $g$ is the gravitational acceleration, $\boldsymbol{\hat{e}}_z$ is the $z$-axis unit vector in the inertial frame, $\boldsymbol{\hat{e}}_3$ is the $z$-axis unit vector in the body frame, and the total thrust force $U_1$ (i.e., the combination of forces $F_1$, $F_2$, $F_3$, and $F_4$) provides the lift necessary for the quadrotor to ascend or descend (see Fig.~\ref{fig:ForcesMoments}). The moment vector $\boldsymbol{U}_{\tau} = \begin{bmatrix} U_2 & U_3 & U_4\end{bmatrix}^T$ includes moments $U_2$, $U_3$, and $U_4$. $U_2$ and $U_3$ are generated through the arms of the quadrotor and correspond to roll and pitch moments, respectively, whereas $U_4$ is the total yaw moment, which is the combination of moments $M_1$, $M_2$, $M_3$, and $M_4$ (see Fig.~\ref{fig:ForcesMoments}). The gravitational force $mg$ acting on the quadrotor is defined in the inertial frame and transformed to the body frame using the rotation matrix $\tensor[^{\mathcal{B}}]{\boldsymbol{R}}{_{\mathcal{I}}}$. In addition to these forces, the quadrotor's propulsion system generates a gyroscopic moment $\boldsymbol{\tau}_g$, which arises due to the change in the direction of the angular momentum vector of the motors as the quadrotor rolls and pitches, that can be expressed as
\begin{equation}
   \begin{aligned}
      \boldsymbol{\tau}_g = \sum_{i=1}^{4} \boldsymbol{J}(\boldsymbol{\omega}_B \times \boldsymbol{\hat{e}}_3)(-1)^{i+1}\Omega_i 
    \end{aligned}~,
    \label{eq:gyroMoment}
\end{equation}
where $\boldsymbol{J}$ is the inertia of the motors and is one of the moments acting on the quadrotor, and $\Omega_i$ is the angular velocity of the $i$-th motor. Because two of the motors on the quadrotor rotate clockwise and the other two rotate counterclockwise, the total angular velocity of the motors can be expressed as $\Omega_T = \Omega_1 - \Omega_2 + \Omega_3 - \Omega_4$.
Then, we can rewrite the gyroscopic moment \eqref{eq:gyroMoment} as
\begin{equation}
   \begin{aligned}
      \boldsymbol{\tau}_g = \boldsymbol{J} \left( \begin{bmatrix} p \\ q \\ r \end{bmatrix} \times \begin{bmatrix} 0 \\ 0 \\ \Omega_T \end{bmatrix} \right)
    \end{aligned}~.
    \label{eq:gyroMoment2}
\end{equation}

\begin{figure}
    \centering
    \includegraphics[width=0.7\textwidth]{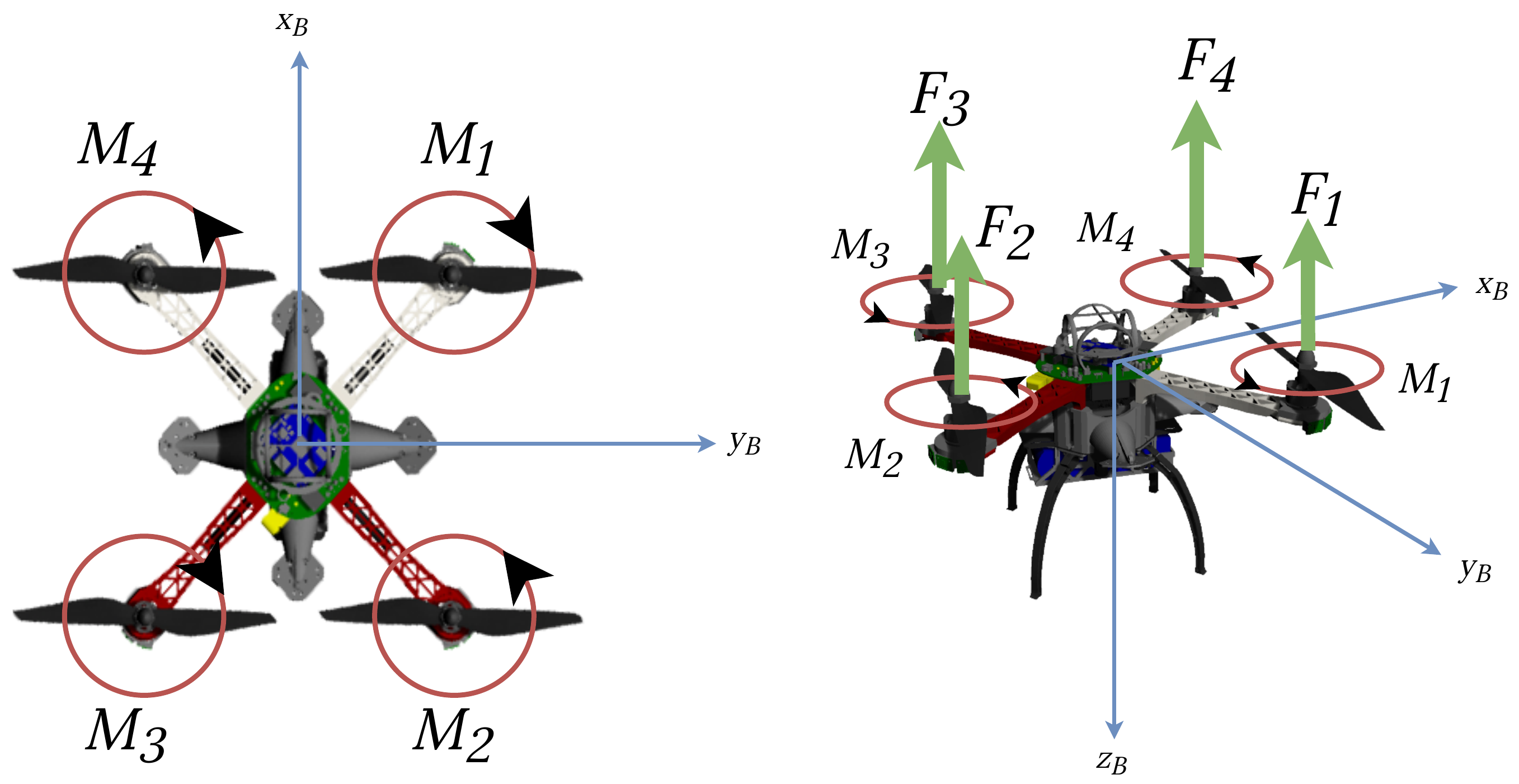}
    \caption{The moments and forces acting on the quadrotor body.
}
    \label{fig:ForcesMoments}
\end{figure}

The motor inertia $\boldsymbol{J}$ is minimal compared to the other moments acting on the quadrotor~\cite{bouabdallah2007full}. Therefore, when building the model of the quadrotor, it is generally safe to neglect the gyroscopic effect of the motors, the ground effect during takeoff or landing, the dynamic effects of the propellers' airflow and blade flapping, or other aerodynamic effects that might be observed during flight. 
Using these simplifications, and under the assumptions previously mentioned, we can construct the dynamics model of the quadrotor as follows:
\begin{equation}
    \begin{aligned}
        -mg\sin\theta &= m(\Dot{u} + qw -rv) \\
        mg\cos\theta\sin\phi &= m(\Dot{v} - pw -ru) \\
        mg\cos\theta\cos\phi - U_1 &= m(\Dot{w} + pv -qu) \\
        U_2 &= \Dot{p}J_x - qrJ_y + qrJ_z \\
        U_3 &= \Dot{q}J_y + prJ_x - prJ_z \\
        U_4 &= \Dot{r}J_z + pqJ_x + pqJ_Y 
    \end{aligned}~,
    \label{eq:dynamics4}
\end{equation}

In this model, $U_1$ is the total thrust force and only acts in the $z$-axis direction of the body frame. Assuming $K_T$ is a constant thrust coefficient, $U_1$ can be defined as
\begin{equation}
    \begin{aligned}
        U_1 = 
        \begin{bmatrix}
           0 \\ 0 \\ K_T(\Omega_1^2 + \Omega_2^2 + \Omega_3^2 + \Omega_4^2) 
        \end{bmatrix}
    \end{aligned}~,
    \label{eq:U1}
\end{equation}
such that the total thrust force is a function of the square of the rotational speed of the motors. 
$U_2$ is the roll angular moment, which is generated about the $x$-axis of the body frame, and is defined as
\begin{equation}
    \begin{aligned}
        U_2 = 
           l_xK_T(-\Omega_1^2 - \Omega_2^2 + \Omega_3^2 + \Omega_4^2)\sin45^{\circ}
    \end{aligned}~,
    \label{eq:U2}
\end{equation}
where $l_x$ is the moment arm distance of the motor frame from the $x$-axis of the body frame.

$U_3$ is the pitch angular moment, which is generated about the $y$-axis of the body frame, and is defined as
\begin{equation}
    \begin{aligned}
        U_3 = 
           l_yK_T(\Omega_1^2 - \Omega_2^2 - \Omega_3^2 + \Omega_4^2)\sin45^{\circ}
    \end{aligned}~, 
    \label{eq:U3}
\end{equation}
where $l_y$ is the moment arm distance of the motor frame from the $y$-axis of the body frame.

$U_4$ is the yaw angular momentum, which is generated about the $z$ axis of the body frame, and is defined as
\begin{equation}
    \begin{aligned}
        U_4 = 
           K_M(\Omega_1^2 - \Omega_2^2 + \Omega_3^2 - \Omega_4^2)
    \end{aligned}~,
    \label{eq:U4}
\end{equation}
where $K_M$ is a constant moment coefficient that reflects the inherent moment-generating capabilities of the motor-propeller system. Note that there is no moment arm coefficient included here because, in our quadrotor, there is no distance between the point at which the moment is applied and the yaw axis. 

By combining equations \eqref{eq:U1}, \eqref{eq:U2}, \eqref{eq:U3}, and \eqref{eq:U4}, the forces and moments generated by the propulsion system of the quadrotor can be obtained as
\begin{equation}
    \begin{aligned}
        \begin{bmatrix}
            U_1 \\ \boldsymbol{U}_{\tau}
        \end{bmatrix}
        = \begin{bmatrix} U_1 \\ U_2 \\ U_3 \\ U_4  \end{bmatrix} = 
        \begin{bmatrix}
            K_T & K_T & K_T & K_T \\
            -K_Tl_x\frac{\sqrt{2}}{2} & -K_Tl_x\frac{\sqrt{2}}{2} & -K_Tl_x\frac{\sqrt{2}}{2} & -K_Tl_x\frac{\sqrt{2}}{2} \\
            K_Tl_y\frac{\sqrt{2}}{2} & -K_Tl_y\frac{\sqrt{2}}{2} & -K_Tl_y\frac{\sqrt{2}}{2} & K_Tl_y\frac{\sqrt{2}}{2} & \\
            K_M & -K_M & K_M & -K_M 
        \end{bmatrix}
        \begin{bmatrix}
            \Omega_1^2 \\ \Omega_2^2 \\ \Omega_3^2 \\ \Omega_4^2  
        \end{bmatrix}
    \end{aligned}~.
    \label{eq:momentVec}
\end{equation}
The inverse of the transformation matrix in Eq.~\eqref{eq:momentVec} can then be used to determine the motor angular velocities needed to produce the desired forces and moments, as follows:
\begin{equation}
    \begin{aligned}
        \begin{bmatrix}
            \Omega_1^2 \\ \Omega_2^2 \\ \Omega_3^2 \\ \Omega_4^2  
        \end{bmatrix} = 
        \begin{bmatrix}
            \frac{1}{4K_T} & -\frac{1}{4K_Tl_x} & \frac{1}{4K_Tl_y} & \frac{1}{4K_M} \\
            \frac{1}{4K_T} & -\frac{1}{4K_Tl_x} & -\frac{1}{4K_Tl_y} & -\frac{1}{4K_M} \\
            \frac{1}{4K_T} & \frac{1}{4K_Tl_x} & -\frac{1}{4K_Tl_y} & \frac{1}{4K_M} \\
            \frac{1}{4K_T} & \frac{1}{4K_Tl_x} & \frac{1}{4K_Tl_y} & -\frac{1}{4K_M} \\
        \end{bmatrix}
        \begin{bmatrix} U_1 \\ U_2 \\ U_3 \\ U_4  \end{bmatrix}
    \end{aligned}~.
    \label{eq:momentVec2}
\end{equation}

\begin{remark}
The motor angular velocities serve as inputs that determine the thrust and torque generated by the motors, which in turn affect the motion of the quadrotor. These inputs are determined by an onboard controller, which receives input from sensors and processes it to compute the desired motor angular velocities. The controller uses these velocities to drive the motors, which produce the necessary thrust and torque to achieve the desired behavior of the quadrotor. In other words, the motor angular velocities are an important factor in the control and operation of the quadrotor.
\end{remark}

\subsection*{State space modeling}

To construct the quadrotor model we use for control and simulation, we now express the model in state space form, which combines the kinematic and dynamic equations. 
To create the state space representation of the quadrotor's model, we define the following state vector:
\begin{equation}
    \begin{aligned}
        \boldsymbol{X} = 
            [~ x ~~ y ~~ z ~~ u ~~ v~~ w ~~ \phi ~~ \theta ~~ \psi ~~ p ~~ q ~~ r ~]^T
    \end{aligned}~.
    \label{eq:state}
\end{equation}
Then, the complete model of the quadrotor can be written using the equation sets in Eq.~\eqref{eq:kinematics2} \eqref{eq:kinematics3} and \eqref{eq:dynamics4}, as
\begin{equation}
   \begin{aligned}
\Dot{\boldsymbol{X}} = 
\begin{bmatrix}
        \Dot{x} \\ \Dot{y} \\ \Dot{z} \\ \Dot{u} \\ \Dot{v} \\ \Dot{w} \\ \Dot{\phi} \\ \Dot{\theta} \\  \Dot{\psi}  \\ \Dot{p} \\ \Dot{q} \\       \Dot{r} 
\end{bmatrix}
=
\begin{bmatrix}
            w(\sin\phi\sin\psi + \cos\phi\cos\psi\sin\theta) -v(\cos\phi\sin\psi + \cos\psi\sin\phi\sin\theta) + u\cos\psi\cos\theta \\
         -w(\cos\psi\sin\phi - \cos\phi\sin\psi\sin\theta) + v(\cos\phi\cos\psi + \sin\phi\sin\psi\sin\theta) + u\cos\theta\sin\psi \\
         w\cos\phi\cos\theta + v\cos\theta\sin\phi  - u\sin\theta \\
         rv - qw - g\sin\theta \\
         pw - ru + g\cos\theta\sin\phi \\
         qu - pv + g\cos\theta\cos\phi - \frac{U_1}{m}\\
         p + r\cos\phi\tan\theta + q\sin\phi\tan\theta \\
         q\cos\phi - r\sin\phi \\
         r\frac{\cos\phi}{\cos\theta} + q\frac{\sin\phi}{\cos\theta}\\
         qr\frac{I_y - I_z}{I_x} + \frac{U_2}{I_x}\\
         pr\frac{I_z - I_x}{I_y} + \frac{U_3}{I_y} \\
         pq\frac{I_x - I_y}{I_z} + \frac{U_4}{I_z}
   \end{bmatrix}
    \end{aligned}~.
    \label{eq:mathmodel}
\end{equation}

The model given in Eq.~\eqref{eq:mathmodel} includes states that are defined in the body frame, namely translational acceleration states ($\Dot{u},~\Dot{v},~\Dot{w}$) and rotational acceleration states ($\Dot{p},~\Dot{q},~\Dot{r}$). Because it is more convenient to work with calculations in the inertial frame for control studies~\cite{bouabdallah2007full}, we redefine the body frame states of the state vector in the inertial frame. 
To express the translational states ($\Dot{u},\Dot{v},\Dot{w}$) in the inertial frame, we apply the rotation matrix $\tensor[^{\mathcal{I}}]{\boldsymbol{R}}{_{\mathcal{B}}}$ as follows:
\begin{equation}
    \begin{aligned}
    &\boldsymbol{F}_B = 
     m\begin{bmatrix}
         \Dot{u} & \Dot{v} & \Dot{w}
     \end{bmatrix}^T
     \\ 
 &\boldsymbol{F}_B = {mg\tensor[^{\mathcal{B}}]{\boldsymbol{R}}{_{\mathcal{I}}}.\boldsymbol{\hat{e}}_z}
        - {U_1.\boldsymbol{\hat{e}}_3}
\\
   &  \tensor[^{\mathcal{I}}]{\boldsymbol{R}}{_{\mathcal{B}}}\boldsymbol{F}_B = 
     mg\boldsymbol{\hat{e}}_z - U_1 \tensor[^{\mathcal{I}}]{\boldsymbol{R}}{_{\mathcal{B}}}\boldsymbol{\hat{e}}_3
    \end{aligned}~.
    \label{eq:state02}
\end{equation}
Then we can obtain the translational states in the inertial frame, denoted as ($\Ddot{x},\Ddot{y},\Ddot{z}$):
\begin{equation}
    \begin{aligned}
    \tensor[^{\mathcal{I}}]{\boldsymbol{R}}{_{\mathcal{B}}}\boldsymbol{F}_B &= 
     mg\boldsymbol{\hat{e}}_z - U_1 \tensor[^{\mathcal{I}}]{\boldsymbol{R}}{_{\mathcal{B}}}\boldsymbol{\hat{e}}_3 \\
      m\Dot{\boldsymbol{V}}_I 
      &= 
     mg\boldsymbol{\hat{e}}_z - U_1 \tensor[^{\mathcal{I}}]{\boldsymbol{R}}{_{\mathcal{B}}}\boldsymbol{\hat{e}}_3 
     \\   
      m\begin{bmatrix}
        \Ddot{x} \\ \Ddot{y} \\ \Ddot{z} 
      \end{bmatrix} 
      &=
      mg\boldsymbol{\hat{e}}_z - U_1 \tensor[^{\mathcal{I}}]{\boldsymbol{R}}{_{\mathcal{B}}}\boldsymbol{\hat{e}}_3
      \\
     \begin{bmatrix}
        \Ddot{x} \\ \Ddot{y} \\ \Ddot{z} 
    \end{bmatrix} 
    &=  
    g\boldsymbol{\hat{e}}_z - \frac{U_1 \tensor[^{\mathcal{I}}]{\boldsymbol{R}}{_{\mathcal{B}}}\boldsymbol{\hat{e}}_3}{m} 
    \\
    &= 
    \begin{bmatrix}
        -\frac{U_1}{m}(\sin\phi\sin\psi + \cos\phi\cos\psi\sin\theta) \\
        -\frac{U_1}{m}(-\cos\psi\sin\phi + \cos\phi\sin\psi\sin\theta) \\ 
        g - \frac{U_1}{m}\cos\phi\cos\theta
    \end{bmatrix}
    \end{aligned}~.
    \label{eq:state2}
\end{equation}

To express the rotational states  ($\Dot{p},~\Dot{q},~\Dot{r}$) in the inertial frame, we establish the relationship between the body frame and the inertial frame, using equation~\eqref{eq:angular_trans}, as
\begin{equation}
    \begin{aligned}
\begin{bmatrix}
    p \\ q \\ r
\end{bmatrix}
=
\tensor[^{\mathcal{I}}]{\boldsymbol{T}}{_{\mathcal{B}}}
\begin{bmatrix}
    \Dot{\phi} \\ \Dot{\theta} \\ \Dot{\psi} 
\end{bmatrix}
&,~~
\tensor[^{\mathcal{I}}]{\boldsymbol{T}}{_{\mathcal{B}}} = \begin{bmatrix}
    1 & 0 & 0 \\ 0 & 1 & 0 \\ 0 & 0 &1 
\end{bmatrix}
\\
\begin{bmatrix}
    p \\ q \\ r
\end{bmatrix}
&=
\begin{bmatrix}
    \Dot{\phi} \\ \Dot{\theta} \\ \Dot{\psi} 
\end{bmatrix}
    \end{aligned}~.
    \label{eq:state3}
\end{equation}

\begin{remark}
   In the equation, \eqref{eq:state3}, we make the assumption that the quadrotor primarily hovers or moves with small angles, meaning that the roll and pitch angles are approximately zero. This assumption, known as the small-angle approximation~\cite{bouabdallah2007full}, allows us to simplify the angular transformation matrix given $\tensor[^{\mathcal{I}}]{\boldsymbol{T}}{_{\mathcal{B}}}$ to a 3 by 3 identity matrix. 
\end{remark}

\rhead{State space modeling}

With the relationship established and the corresponding equalities for ($\Dot{p},\Dot{q},\Dot{r}$) in Eq.~\eqref{eq:mathmodel}, we can obtain the transformed rotational states ($\Ddot{\phi},\Ddot{\theta},\Ddot{\psi}$) as
\begin{equation}
    \begin{aligned}
        \begin{bmatrix}
        \Ddot{\phi} \\
        \Ddot{\theta} \\
        \Ddot{\psi} 
\end{bmatrix}
=
    \begin{bmatrix}
         \Dot{\psi}\Dot{\theta}\frac{I_y - I_z}{I_x} + \frac{U_2}{I_x}\\
         \Dot{\phi}\Dot{\psi}\frac{I_z - I_x}{I_y} + \frac{U_3}{I_y} \\
         \Dot{\phi}\Dot{\theta}\frac{I_x - I_y}{I_z} + \frac{U_4}{I_z}
\end{bmatrix}
    \end{aligned}~.
    \label{eq:state4}
\end{equation}

Then, the model in Eq.~\eqref{eq:mathmodel} becomes
\begin{equation}
 \begin{aligned}
\Dot{\boldsymbol{X}} 
&= g(\boldsymbol{X}, \boldsymbol{U}) \\
&= f(\boldsymbol{X}) + \sum_{i=1}^{4} \boldsymbol{a}_i(\boldsymbol{X})U_i
\end{aligned}~,
    \label{eq:mathmodel2}
\end{equation}
where
\begin{equation}
 \begin{aligned}
   &\boldsymbol{X} = [~ x ~~ y ~~ z ~~ \phi ~~ \theta~~ \psi ~~ \Dot{x} ~~ \Dot{y} ~~ \Dot{z} ~ \Dot{\phi} ~~ \Dot{\theta} ~~ \Dot{\psi}  ~]^T \\
&\Dot{\boldsymbol{X}} 
=
[~ \Dot{x} ~~ \Dot{y} ~~ \Dot{z} ~~ \Dot{\phi} ~~  \Dot{\theta} ~~  \Dot{\psi} ~~  \Ddot{x} ~~  \Ddot{y} ~~  \Ddot{z} ~~  \Ddot{\phi} ~~  \Ddot{\theta} ~~   \Ddot{\psi} ~]^T ,~~ \\
&\boldsymbol{f}(\boldsymbol{X})
= [~ \Dot{x} ~~ \Dot{y} ~~ \Dot{z} ~~ \Dot{\phi} ~~  \Dot{\theta} ~~  \Dot{\psi} ~~ 0 ~~ 0 ~~ g ~~ (\Dot{\psi}\Dot{\theta})\frac{I_y - I_z}{I_x} ~~ (\Dot{\phi}\Dot{\psi})\frac{I_z - I_x}{I_y} ~~  (\Dot{\phi}\Dot{\theta})\frac{I_x - I_y}{I_z} ~]^T
\end{aligned}
\end{equation}
and
\begin{equation}
 \begin{aligned}
\boldsymbol{a}_1 &=  [~ 0 ~~ 0 ~~ 0 ~~ 0 ~~ 0~~ 0 ~~ a_1^7 ~~ a_1^8 ~~ a_1^9 ~~ 0 ~~ 0 ~~ 0 ~]^T ,~~ \\ 
a_1^7 &= -\frac{1}{m}(\sin\phi\sin\psi + \cos\phi\cos\psi\sin\theta) , ~~ \\
a_1^8 &= -\frac{1}{m}(-\cos\psi\sin\phi + \cos\phi\sin\psi\sin\theta) , ~~ \\
a_1^9 &= - \frac{1}{m}\cos\phi\cos\theta, ~~ \\
\boldsymbol{a}_2 &=  [~ 0 ~~ 0 ~~ 0 ~~ 0 ~~ 0~~ 0 ~~ 0 ~~ 0~~ 0 ~~ \frac{1}{I_x} ~~ 0 ~~ 0 ~]^T , ~~ \\
\boldsymbol{a}_3 &=  [~ 0 ~~ 0 ~~ 0 ~~ 0 ~~ 0~~ 0 ~~ 0 ~~ 0~~ 0 ~~ 0 ~~ \frac{1}{I_y} ~~ 0 ~]^T , ~~ \\
\boldsymbol{a}_4 &=  [~ 0 ~~ 0 ~~ 0 ~~ 0 ~~ 0~~ 0 ~~ 0 ~~ 0~~ 0 ~~ 0 ~~ 0 ~~ \frac{1}{I_z} ~]^T , ~~ \\
\end{aligned}~.
\end{equation}

\subsection*{Linearized modeling for flight control}

In order to design a linear flight controller, for instance, a proportional–integral–derivative (PID) controller, it is necessary to linearize the model \eqref{eq:mathmodel2} around an operating point that brings the system to an equilibrium state. This process is described, following~\cite{koo1998output}, as
\begin{equation}
    \begin{aligned}
     \Dot{\boldsymbol{X}}  = g(\Bar{\boldsymbol{X}},\Bar{\boldsymbol{U}}) \to \boldsymbol{0}
    \end{aligned}~,
\end{equation}
where the control input vector $\boldsymbol{U} = \begin{bmatrix} U_1 & U_2 & U_3 & U_4\end{bmatrix}^T$ consists of the total thrust $U_1$ and control moments. $\Bar{\boldsymbol{X}}$ represents an operating point with a constant input $\Bar{\boldsymbol{U}}$, which is known as the trim condition and is defined as
\begin{equation}
    \begin{aligned}
     \Bar{\boldsymbol{U}}  = \begin{bmatrix} mg & 0 & 0 &0 \end{bmatrix}^T
    \end{aligned}~,
\end{equation}
where $mg$ represents the total thrust required to counteract the inertial force along the $+z$-axis (due to gravity and to the weight of the quadrotor) and maintain the hover state of the quadrotor.

To get the linearized model of Eq.~\eqref{eq:mathmodel2} around the operating point $ \Bar{\boldsymbol{X}}  = [~ \Bar{x}~~\Bar{y}~~\Bar{z}~~0~~0~~0~~0~~0~~0~]^T$ with the constant input $\Bar{\boldsymbol{U}}$, we use Taylor expansion as follows:
\begin{equation}
    \begin{aligned}
        \Dot{\boldsymbol{X}} = g(\Bar{\boldsymbol{X}}, \Bar{\boldsymbol{U}}) 
        + \frac{\delta g}{\delta \boldsymbol{X}}|_{\Bar{\boldsymbol{X}}}\delta \boldsymbol{X} 
        + \frac{\delta g}{\delta \boldsymbol{U}}|_{\Bar{\boldsymbol{U}}}\delta \boldsymbol{U} 
        + \frac{1}{2}\frac{\delta^2 g}{\delta^2 \boldsymbol{X}}|_{\Bar{\boldsymbol{X}}}\delta^2 \boldsymbol{X} 
        + \frac{1}{2}\frac{\delta^2 g}{\delta^2 \boldsymbol{U}}|_{\Bar{\boldsymbol{U}}}\delta^2 \boldsymbol{U}  + \ldots
    \end{aligned}~,
    \label{eq:taylor}
\end{equation}
where $\delta \boldsymbol{X}$ and $\delta \boldsymbol{U}$ represent deviations from the trim condition, with $\delta \boldsymbol{X} = \boldsymbol{X} - \Bar{\boldsymbol{X}}$ and $\delta \boldsymbol{U} = \boldsymbol{U} - \Bar{\boldsymbol{U}}$. 

To simplify the Taylor expansion, we can ignore higher-order terms and consider only the first-order terms. Since $g(\Bar{\boldsymbol{X}}, \Bar{\boldsymbol{U}}) \to 0$, Eq.~\eqref{eq:taylor} becomes
\begin{equation}
    \begin{aligned}
        \Dot{\boldsymbol{X}} &\approx \frac{\delta g}{\delta \boldsymbol{X}}|_{\Bar{\boldsymbol{X}}}\delta \boldsymbol{X} 
        + \frac{\delta g}{\delta \boldsymbol{U}}|_{\Bar{\boldsymbol{U}}}\delta \boldsymbol{U}  \\
        &\approx \boldsymbol{A}\boldsymbol{X} + \boldsymbol{B}\boldsymbol{U}
    \end{aligned}~,
        \label{eq:taylor2}
\end{equation}
where
\begin{equation*}
\small
    \begin{aligned}
       \boldsymbol{A} =  \begin{bmatrix*}[r]
            0 & 0 & 0 & 0 & 0 & 0 & 0 & 1 & 0 & 0 & 0 & 0 \\
            0 & 0 & 0 & 0 & 0 & 0 & 0 & 0 & 1 & 0 & 0 & 0 \\
            0 & 0 & 0 & 0 & 0 & 0 & 0 & 0 & 0 & 1 & 0 & 0 \\
            0 & 0 & 0 & 0 & 0 & 0 & 0 & 0 & 0 & 0 & 1 & 0 \\
            0 & 0 & 0 & 0 & 0 & 0 & 0 & 0 & 0 & 0 & 0 & 1 \\
            0 & 0 & 0 & 0 & -g & 0 & 0 & 0 & 0 & 0 & 0 & 0 \\
            0 & 0 & 0 & g & 0 & 0 & 0 & 0 & 0 & 0 & 0 & 0 \\
            0 & 0 & 0 & 0 & 0 & 0 & 0 & 0 & 0 & 0 & 0 & 0 \\
            0 & 0 & 0 & 0 & 0 & 0 & 0 & 0 & 0 & 0 & 0 & 0 \\
            0 & 0 & 0 & 0 & 0 & 0 & 0 & 0 & 0 & 0 & 0 & 0
       \end{bmatrix*} ~~ \text{and} ~~
       \boldsymbol{B} = \begin{bmatrix*}[r]
            0 & 0 & 0 & 0  \\
            0 & 0 & 0 & 0  \\
            0 & 0 & 0 & 0  \\
            0 & 0 & 0 & 0  \\
            0 & 0 & 0 & 0  \\
            0 & 0 & 0 & 0  \\
            0 & 0 & 0 & 0  \\
            0 & 0 & 0 & 0  \\
            \frac{1}{m} & 0 & 0 & 0  \\
            0 & \frac{1}{I_x} & 0 & 0  \\
            0 & 0 & \frac{1}{I_y} & 0  \\
            0 & 0 & 0 & \frac{1}{I_z}  \\
       \end{bmatrix*} 
    \end{aligned}~.
\end{equation*}

\begin{remark}
    The differential equation given in Eq.~\eqref{eq:taylor2} drives the deviation variables $\delta \boldsymbol{X}$ and $\delta \boldsymbol{U}$ towards zero as long as they remain small. This linearization procedure is known as Jacobian linearization \cite{isidori1985nonlinear}. 
\end{remark}

\rhead{Linearized modeling for flight control}

\subsection*{Design of position and attitude controllers}

The overall structure of the flight control system for the simulation studies is shown in Fig.~\ref{fig:FCS}. This system is based on the controllers used in the real quadrotor. We use a cascaded control structure that effectively decouples the position (position and linear velocity) and attitude (angle and angular velocity) control loops. 

\begin{figure}[hbt]
    \centering
    \includegraphics[width=0.99\textwidth]{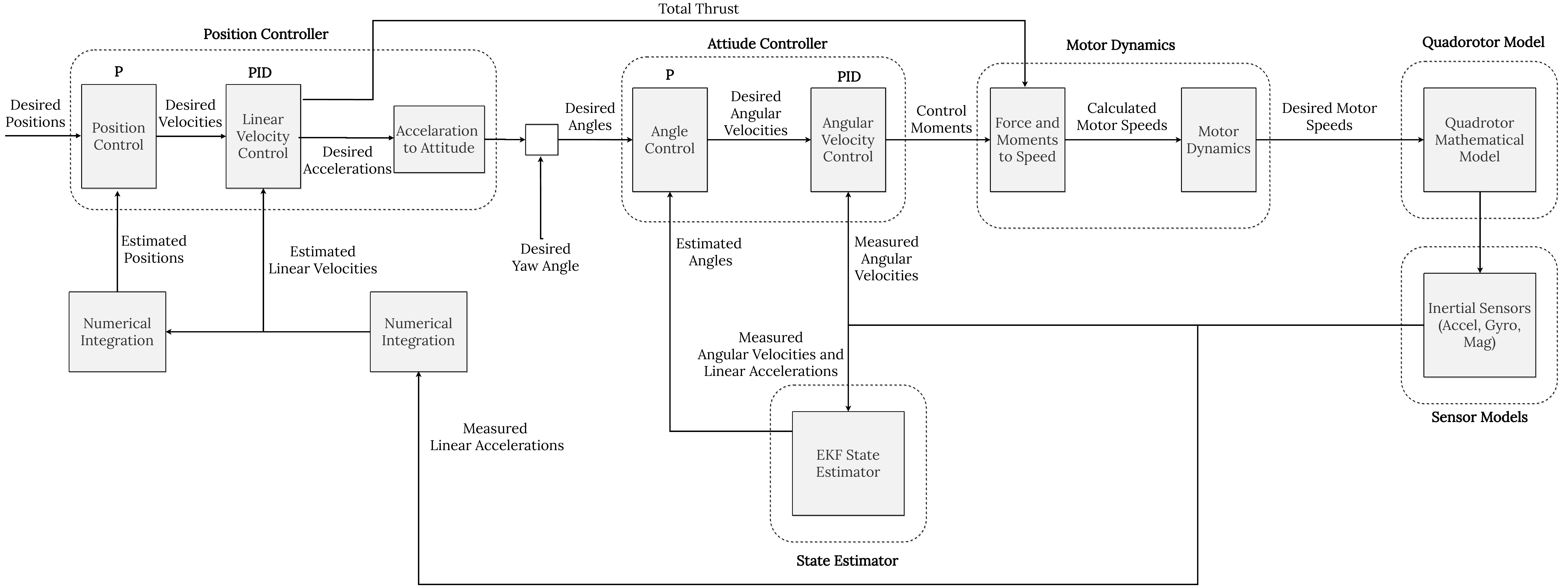}
    \caption{The flight control system.}
    \label{fig:FCS}
\end{figure}

\subsubsection*{Position control}
We use a position controller with a cascaded P/PID structure, with the first stage being a 3D position controller and the second stage being a linear velocity controller.
Note that the quadrotor position control system is underactuated~\cite{hua2013introduction}, therefore, the position control for the $xy$-axes can be achieved through attitude control.
The position controller uses a P controller and its outputs for the $xy$ axes produce the desired acceleration values $\Ddot{x}^d$, $\Ddot{y}^d$ in the inertial frame.
The outputs of the position controller for the $xy$ axes produce the desired acceleration values $\Ddot{x}^d$, $\Ddot{y}^d$ in the inertial frame. These values are then used to generate the desired roll $\phi^d$ and pitch $\theta^d$ angles for the attitude controller to use. The overall structure of the position controller (see Fig.~\ref{fig:posControl}) is given as
\begin{equation}
    \begin{aligned}
        \textbf{Altitude control}: U_1 = K_{Pz} e_{Vz} + K_{Iz}\int e_{Vz} - K_{Dz}\hat{\Dot{z}} + mg \\
       \textbf{Desired roll and pitch angles}: 
                                      \phi^d = \frac{m}{U_1} (-\Ddot{x}^d\sin\hat{\psi} + \Ddot{y}^d\cos\hat{\psi} ) \\
                                      \theta^d = \frac{m}{U_1} (-\Ddot{x}^d\cos\hat{\psi}- \Ddot{y}^d\sin\hat{\psi}) \\
                    \text{where}   ~~~  \Ddot{x}^d =  K_{Px} e_{Vx} + K_{Ix}\int e_{Vx} - K_{Dx}\hat{\Dot{x}} \\ 
                                      \Ddot{y}^d =  K_{Py} e_{Vy} + K_{Iy}\int e_{Vy} - K_{Dy}\hat{\Dot{y}} \\
    \end{aligned}~.
    \label{eq:poscont}
\end{equation}

\begin{figure}[hbt]
    \centering
    \includegraphics[width=0.85\textwidth]{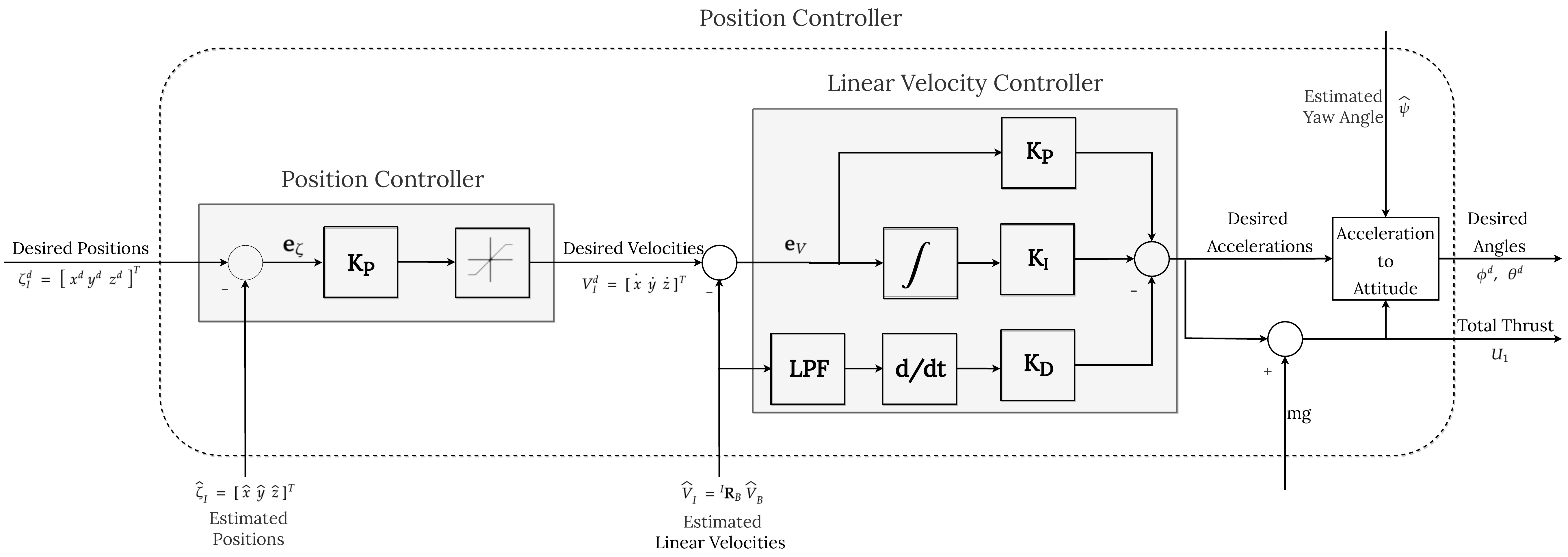}
   \caption{The cascaded position controller.}
    \label{fig:posControl}
\end{figure}

\subsubsection*{Attitude Control}

We use an attitude controller with a cascaded P/PID (proportional/proportional-integral-derivative) structure, with the first stage being an angle controller and the second stage being an angular velocity controller.

The angle controller uses a P controller that produces an output signal proportional to the error between the desired and actual values. The output of the angle controller becomes the desired value for the angular velocity controller, which uses a PID controller. 

The outputs of the angular velocity controller, $U_2$, $U_3$, and $U_4$, are the control moments that are used to manipulate the quadrotor's attitude. The overall structure of the attitude controller (see Fig.~\ref{fig:AttitudeControl}) is given as
\begin{equation}
    \begin{aligned}
        &\textbf{Roll control}: U_2 = K_{P\phi} e_{\omega\phi} + K_{I\phi}\int e_{\omega\phi} - K_{D\phi}\hat{\Dot{\phi}} \\
        &\textbf{Pitch control}: U_3 = K_{P\theta} e_{\omega\theta} + K_{I\theta}\int e_{\omega\theta} - K_{D\theta}\hat{\Dot{\theta}} \\ 
        &\textbf{Yaw control}: U_4   = K_{P\psi} e_{\omega\psi} + K_{I\psi}\int e_{\omega\psi} - K_{D\psi}\hat{\Dot{\psi}}
    \end{aligned}~.
    \label{eq:attcont}
\end{equation}

\rhead{Position and attitude controllers}

\begin{figure}[hbt]
    \centering
    \includegraphics[width=0.8\textwidth]{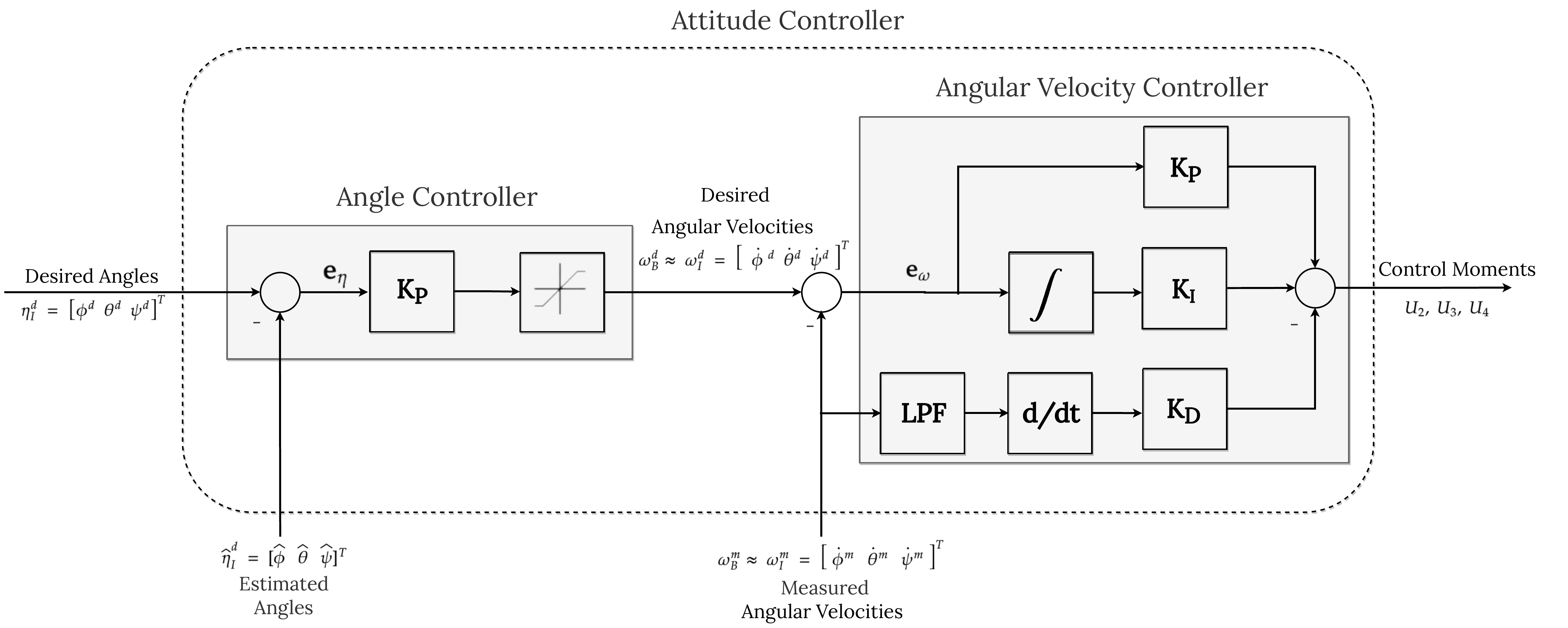}
    \caption{The cascaded attitude controller}
    \label{fig:AttitudeControl}
\end{figure}

These controllers are used for the quadrotors in all real and simulated experiments included in this study.

\subsection*{Sensing robots and objects in the environment}

In our setup, each robot or object has an identifier denoted as $\rho$. The identifier of a robot is unique and the identifier of an object is non-unique. Each robot has access to its own identifier $\rho$ and all ground robots and objects are mounted with an AprilTag fiducial marker (i.e., a 2D barcode that contains a numeric ID code and supports 3D position and pose tracking~\cite{olson2011apriltag,wang2016apriltag}) that encodes its identifier $\rho$. 
Using its four downward-facing camera modules, each quadrotor can sense the relative positions, relative orientations, and encoded identifiers $\rho$ of AprilTags present in its field of view. Each quadrotor has a lookup table of AprilTag identifiers $\rho$ that allows it to classify them as belonging either to a robot or to an object (see Fig.~\ref{fig:robots}).

When a quadrotor $\rho_i$ senses an AprilTag, it stores: its encoded identifier $\rho_j$; its relative position, denoted by the displacement vector $\boldsymbol{d}_{\rho_i \rho_j}$ in the body frame of the quadrotor $B_{\rho_i}$; and its relative orientation, denoted by the unit quaternion $\boldsymbol{q}_{\rho_i \rho_j}$, where the Euler axis portion of the quaternion is the $z$-axis of body frame $B_{\rho_i}$ and the angle portion is relative to the $x$-axis of $B_{\rho_i}$. It stores these values either in a matrix of robot information $\bm{\mathcal{P}}^{\textsc{robot}}$ or a matrix of object information $\bm{\mathcal{P}}^{\textsc{object}}$, according to its AprilTag lookup table.

\subsubsection*{Virtual sensing}

The SoNS control algorithm (see Sec.~\ref{SM:algorithm}) is meant to be general to different robot hardware platforms and assumes that all robots are capable of sensing the relative position, relative orientation, and ID of robots in their fields of view and also capable of mutual sensing (i.e., if robot A can sense robot B, then robot B can also sense robot A), regardless of the robot type. However, no robot in our implementation has a sensor that can detect the orientation and ID of a quadrotor, because the ground robot platform has no sensor pointing upwards (see Sec.~\ref{SM:ground}) and the quadrotor has no camera modules pointing laterally.
To translate the onboard capabilities of the robots into the sensing capabilities expected by the SoNS control algorithm, we implement a virtual sensing control layer that is informed strictly by local communication with nearby robots. 

The virtual sensing layer implemented on each quadrotor effectively gives two quadrotors access to each other's relative position, relative orientation, and ID, under the condition that they both have the same ground robot in their respective fields of view, in the following way.

\rhead{Environment sensing}

At each time step, for any AprilTag with the ground robot identifier $\rho_j$ in its field of view, each quadrotor $\rho_i$ sends a message containing its information $\mathcal{P}^{\textsc{robot}}_{ij} = (\rho_i, \rho_j, \boldsymbol{d}_{\rho_i \rho_j}, \boldsymbol{q}_{\rho_i \rho_j})$ to ground robot $\rho_j$. 
If a ground robot has received such a message from a quadrotor in the previous time step, it considers that quadrotor to be in its virtual sensing range. 
At each time step, each ground robot
forwards the most recent message $\mathcal{P}^{\textsc{robot}}$ it has received from any quadrotor in its virtual sensing range to all other quadrotors in its virtual sensing range. 
In other words, if two quadrotors $\rho_i$ and $\rho_k$ respectively send messages $\mathcal{P}^{\textsc{robot}}_{ij}$ and $\mathcal{P}^{\textsc{robot}}_{kj}$ to a ground robot $\rho_j$ in time step $t-1$, then ground robot $\rho_j$ forwards $\mathcal{P}^{\textsc{robot}}_{ij}$ to quadrotor $\rho_k$ and $\mathcal{P}^{\textsc{robot}}_{kj}$ to quadrotor $\rho_i$ in time step $t$.

Each quadrotor $\rho_i$ that senses a ground robot $\rho_j$ and receives a message $\mathcal{P}^{\textsc{robot}}_{kj}$ forwarded from quadrotor $\rho_k$ has the required information to calculate the relative position $\boldsymbol{d}_{\rho_i \rho_k}$ and relative orientation $\boldsymbol{q}_{\rho_i \rho_k}$ of quadrotor $\rho_k$ w.r.t. its own body frame $B_{\rho_i}$, as follows: 
\begin{equation}
\begin{aligned}
    \boldsymbol{d}_{\rho_i \rho_k} &= \boldsymbol{d}_{\rho_k \rho_j} + \textsc{RT}\left(\boldsymbol{q}_{\rho_i \rho_j}^{-1}, \textsc{RT}(\boldsymbol{q}_{\rho_i \rho_j}^{-1}, -\boldsymbol{d}_{\rho_i \rho_j})\right)\\
    \boldsymbol{q}_{\rho_i \rho_k} &= \textsc{H}(\boldsymbol{q}_{\rho_k \rho_j}, \boldsymbol{q}_{\rho_i \rho_j}^{-1})
\end{aligned}~,
\end{equation}
where $\textsc{RT}(\boldsymbol{x}, \boldsymbol{y})$ is a function to rotate vector $\boldsymbol{y}$ by unit quaternion $\boldsymbol{x}$ using the Euler–Rodrigues formula, with the Euler parameters given by the coefficients of quaternions $\boldsymbol{y}^p = (0,\boldsymbol{y})$ and $\boldsymbol{x}$, and $\textsc{H}(\boldsymbol{x}, \boldsymbol{y})$ takes the Hamilton product of two quaternions $\boldsymbol{x}$ and $\boldsymbol{y}$. 
In this way, quadrotor $\rho_i$ virtually senses quadrotor $\rho_k$, using strictly local communication.

\subsection*{Tracking objects for flight stabilization}

As can be expected, the quadrotor drifts during flight, even when the control inputs produce desired $xy$ acceleration values of $0,0$. We therefore introduce a flight stabilization control layer on each quadrotor, which can be used optionally, to adjust the target linear velocity vector $\boldsymbol{v}^*$ and target angular velocity vector $\boldsymbol{\omega}^*$ output by the SoNS control algorithm (see Sec.~\ref{SM:algorithm}) on the respective quadrotor. 
The flight stabilization control layer effectively allows a quadrotor to adjust its target velocity vectors $\boldsymbol{v}^*$ and $\boldsymbol{\omega}^*$ by using a sensed object or ground robot as a reference landmark, whenever it has one in its field of view, as follows.

For a quadrotor $r_i$ on which the flight stabilization control layer is active, when an object or ground robot $r_j$ enters the field of view of quadrotor $r_i$, quadrotor $r_i$ saves its first respective entry $\mathcal{P}^{\textsc{object}}_{ij}$ or $\mathcal{P}^{\textsc{robot}}_{ij}$ to the matrix of initial references
$\bm{\mathcal{I}}^{\textsc{robot}}$ or $\bm{\mathcal{I}}^{\textsc{object}}$.
Then, if the quadrotor $r_i$ has any objects in its field of view, it uses the objects as reference landmarks. Otherwise, if it has at least one ground robot in its field of view, it uses a ground robot as a reference landmark.

At each time step that quadrotor $r_i$ has at least one object in its field of view, it calculates the linear and angular displacements $\Delta$ of the relative position and orientation, respectively, of each object $r_j$, according to its current entry $\boldsymbol{d}_{\rho_i \rho_j}, \boldsymbol{q}_{\rho_i \rho_j} \in \mathcal{P}^{\textsc{object}}_{ij}$ and its initial entry $\boldsymbol{d}^\mathcal{I}_{\rho_i \rho_j}, \boldsymbol{q}^\mathcal{I}_{\rho_i \rho_j} \in \mathcal{I}^{\textsc{object}}_{ij}$, as follows:
\begin{equation}
\begin{aligned}
    \Delta(\boldsymbol{d}_{\rho_i \rho_j}) &= \boldsymbol{d}_{\rho_i \rho_j} + \textsc{RT}(\boldsymbol{q}_{\rho_i \rho_j}, \boldsymbol{d}^\mathcal{I}_{\rho_i \rho_j}) \\
    \Delta(\boldsymbol{q}_{\rho_i \rho_j}) &= \textsc{H}(\boldsymbol{q}_{\rho_i \rho_j}, \boldsymbol{q}^\mathcal{I}_{\rho_i \rho_j})
\end{aligned}~.    
\label{eq:change-references}
\end{equation}
Quadrotor $r_i$ then estimates its own linear displacement $\boldsymbol{d}_{\rho_i \rho_i^{'}}$ and angular displacement $\boldsymbol{q}_{\rho_i \rho_i^{'}}$ per time step, according to its current target velocity vectors $\boldsymbol{v}^*$ and $\boldsymbol{\omega}^*$. Also at each time step, quadrotor $r_i$ adjusts $\Delta(\boldsymbol{d}_{\rho_i \rho_j})$ and $\Delta(\boldsymbol{q}_{\rho_i \rho_j})$ according to its estimated $\boldsymbol{d}_{\rho_i \rho_i^{'}}$ and $\boldsymbol{q}_{\rho_i \rho_i^{'}}$, such that they are incrementally updated to reflect the motion of quadrotor $r_i$ according to $\boldsymbol{v}^*$ and $\boldsymbol{\omega}^*$.
Then, quadrotor $r_i$ takes the average of the linear and angular displacements of all sensed objects, such that
\begin{equation}
\begin{aligned}
    \Delta_{\boldsymbol{d}} &= \textsc{avg} \left( \Delta(\boldsymbol{d}_{\rho_i \rho_j}) ~\forall \rho_j \in \bm{\mathcal{P}}^{\textsc{object}} \right) \\
    \Delta_{\boldsymbol{q}} &= \textsc{avg} \left( \Delta(\boldsymbol{q}_{\rho_i \rho_j}) ~\forall \rho_j \in \bm{\mathcal{P}}^{\textsc{object}} \right)
\end{aligned}~.
\end{equation}
Quadrotor $r_i$ then adjusts its target velocity vector $\boldsymbol{v}^*$ and target angular velocity $\boldsymbol{\omega}^*$ according to the linear and angular displacements of itself and of the objects, as follows:
\begin{equation}
\begin{aligned}
    {\boldsymbol{v}^*}' &= \left( \boldsymbol{d}_{\rho_i \rho_i^{'}} + \textsc{RT}(\boldsymbol{q}_{\rho_i \rho_i^{'}}, \boldsymbol{d}^\mathcal{I}_{\rho_i \rho_j}) \right) - \Delta_{\boldsymbol{d}} \\
    {\boldsymbol{\omega}^*}' &= \textsc{H} \left( \textsc{H}(\boldsymbol{q}_{\rho_i \rho_i^{'}}, \boldsymbol{q}^\mathcal{I}_{\rho_i \rho_j}), \Delta_{\boldsymbol{q}}^{-1} \right)
\end{aligned}~.
\label{eq:new-targets}
\end{equation}

\rhead{Flight stabilization}

At each time step that quadrotor $r_i$ has at least one ground robot but no objects in its field of view, it takes the first entry in its matrix $\bm{\mathcal{P}}^{\textsc{robot}}$ and uses it as a reference landmark.
For as long as quadrotor $r_i$ chooses ground robot $r_j$ as its reference landmark, it sends an override message to ground robot $r_j$ to ignore any of its motion control inputs except for the target velocity vectors $\boldsymbol{v}^*$ and $\boldsymbol{\omega}^*$ that it receives from quadrotor $r_i$. In this way, quadrotor $r_i$ and ground robot $r_j$ move according to the same target velocities for as long as $r_j$ is the reference landmark of $r_i$.
Quadrotor $r_i$ then calculates the linear and angular displacements $\Delta(\boldsymbol{d}_{\rho_i \rho_j})$ and $\Delta(\boldsymbol{q}_{\rho_i \rho_j})$ of ground robot $r_j$ using Eq.~\ref{eq:change-references}.
Quadrotor $r_i$ then estimates its own linear displacement $\boldsymbol{d}_{\rho_i \rho_i^{'}}$ and angular displacement $\boldsymbol{q}_{\rho_i \rho_i^{'}}$ per time step, according to its current target velocity vectors $\boldsymbol{v}^*$ and $\boldsymbol{\omega}^*$.
Finally, quadrotor $r_i$ adjusts its target velocity vector $\boldsymbol{v}^*$ and target angular velocity $\boldsymbol{\omega}^*$ according to Eq.~\ref{eq:new-targets}, using $\Delta(\boldsymbol{d}_{\rho_i \rho_j})$ and $\Delta(\boldsymbol{q}_{\rho_i \rho_j})$ instead of $\Delta_{\boldsymbol{d}}$ and $\Delta_{\boldsymbol{q}}$.

\clearpage\rhead{}
\section*{Section \ref*{SM:ground}. Ground robot setup}
\lhead{Section \ref*{SM:ground}. Ground robot}
The ground robot used in the experiments (see Fig.~\ref{fig:pipuck-simu}) is the small differential-drive \textit{e-puck} robot~\cite{mondada2009puck} with the \textit{Pi-puck} extension board~\cite{millard2017pi}, which provides an interface between the e-puck robot and a Raspberry Pi single-board computer.
Each e-puck is also mounted with a unique AprilTag fiducial marker in the tag25h9 family~\cite{olson2011apriltag,wang2016apriltag}, which can be detected by the aerial robots using onboard computer vision (for details, see the technical report of our quadrotor platform~\cite{OguHeiAllZhuWahGarDor2022:techreport-010}).
Each AprilTag is $10.75$\,cm $\times$ $10.75$\,cm and its $x$-axis is directed towards its e-puck's heading.

\begin{figure}[h]
\centering
\includegraphics[width=0.8\textwidth]{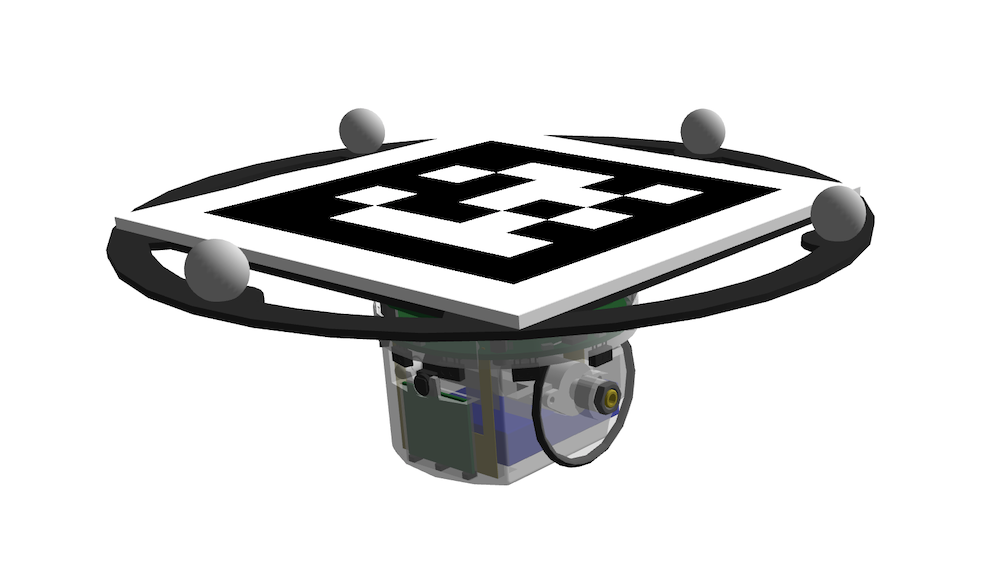}
\caption{E-puck robot with Pi-puck extension board and AprilTag fiducial marker (the spherical markers are not used by the SoNS software; they are used for data logging in the real experiments, see Sec.~S8 on the real arena). The e-puck's heading bisects the directions of the two wheels and is directed towards the left of the image.}
\label{fig:pipuck-simu}
\end{figure}

A Linux operating system compiled by Yocto is installed on the Raspberry-pi of the Pi-puck~\cite{salvador2014embedded, All2022:techreport-001}.
The control software running in Linux is ARGoS~\cite{pinciroli2012argos}. The SoNS software is comprised of Lua scripts loaded and executed by ARGoS.

\subsection*{Modeling and motion control}

Because the e-puck is a differential-drive robot, its motion control inputs (both on the real robots and in simulation) are the target velocities of the left and right wheels.
The SoNS control algorithm (see Sec.~\ref{SM:algorithm}) is meant to be general and usable with different robot hardware platforms and therefore provides omnidirectional motion control outputs regardless of the robot type.
To translate the omnidirectional outputs of the SoNS algorithm to the left and right wheel inputs for the differential-drive e-pucks, 
we introduce a reference frame for intermediary motion management.

The e-puck is modeled using the body frame (denoted $B_2$) and the intermediary motion frame (denoted $F$).
The body frame $B_2$ is a relative coordinate system that represents the body of the e-puck. The origin of the frame is the center of the e-puck fiducial marker, the $x$-axis of the frame is the longitudinal axis (i.e., directed to the front), the $y$-axis is directed to the left, and the $z$-axis is directed upwards.
The intermediary motion frame $F$ has the same origin as $B_2$ but its 3D rotation is fixed and is equivalent to the original rotation of the body frame $B_2$ at initialization.
The rotation of $F$ (i.e., the difference between the rotation of the initial $B_2$ and the current $B_2$) is defined with respect to the current $B_2$ by quaternion $\boldsymbol{q}_{B_2}$.

\begin{remark}
In this study, all positioning is relative.
Although the rotation of the intermediary motion frame $F$ is fixed, it is defined and maintained locally using onboard measurements and calculations.
The e-puck also only has access to its own calculations in its fixed-rotation intermediary motion frame, not the measurements and calculations of other robots. Hence, there is no absolute reference synchronized between robots or otherwise used to coordinate navigation.
Navigation is strictly self-organized, using exclusively local communication and relative positioning.
\end{remark}

The SoNS control algorithm outputs a target linear velocity vector $\boldsymbol{v}^*$ and target angular velocity $\boldsymbol{\omega}^*$. 
The vector $\boldsymbol{v}^*$ is transformed from the intermediary motion frame $F$ to the body frame $B_2$, using $\boldsymbol{q}_{B_2}^{-1}$, as follows:
\begin{equation}
    \boldsymbol{v}_{B_2} = \textsc{RT}(\boldsymbol{q}_{B_2}^{-1}, \boldsymbol{v}^*)
\label{eq:rotate-inputs-IMF}
\end{equation}
where $\textsc{RT}(\boldsymbol{x}, \boldsymbol{y})$ is a function to rotate vector $\boldsymbol{y}$ by unit quaternion $\boldsymbol{x}$ using the Euler–Rodrigues formula, with the Euler parameters given by the coefficients of quaternions $\boldsymbol{x}$ and $\boldsymbol{y}^p = (0,\boldsymbol{y})$.
Vector $\boldsymbol{\omega}^*$ is likewise transformed from the intermediary motion frame $F$ to the body frame $B_2$ using Eq.~\ref{eq:rotate-inputs-IMF}, which provides $\boldsymbol{\omega}_{B_2}$.

Then, $\boldsymbol{v}_{B_2} = (v^{B_2}_x, v^{B_2}_y, v^{B_2}_z)$ is translated into left $v_{\texttt{left}}$ and right $v_{\texttt{right}}$ wheel inputs as follows:
\begin{equation}
\begin{aligned}
    v_{\texttt{left}} &= v^{B_2}_x - v^{B_2}_y \sin{\theta} \\
    v_{\texttt{right}} &= v^{B_2}_x + v^{B_2}_y \sin{\theta}
\end{aligned}~,
\label{equ:differential_driving_system}
\end{equation}
where $\theta = \arctan{\frac{v^{B_2}_y}{v^{B_2}_x}}$. Note that $v_z$ is not used because the e-puck altitude does not change.
The angle portion of $\boldsymbol{q}_{B_2}$ is meanwhile updated according to $-\theta$, because $F$ remains fixed, and the intermediary motion frame $F$ is rotated according to $\boldsymbol{\omega}_{B_2}$, which will effect the left and right wheel inputs calculated in the next time step.

\rhead{Modeling and motion control}

When an e-puck sends information about relative positions or relative orientations to other robots, it always sends vectors in the body frame $B_2$, as the origin and rotation of the e-puck's body frame relative to another robot can be sensed directly by that robot via the e-puck's fiducial marker (if within the field of view).

\subsection*{Virtual sensing of robots and objects in the environment}

The SoNS control algorithm (see Sec.~\ref{SM:algorithm}) is meant to be general to different robot hardware platforms and assumes that all robots are capable of sensing the relative position, relative orientation, and ID of robots in their fields of view and that they are also capable of mutual sensing (i.e., if robot A can sense robot B, then robot B can also sense robot A), regardless of the robot type. However, the e-puck robot does not have any sensor with a field of view pointed upwards, towards the aerial robots, and also does not have any sensor capable of sensing another robot's orientation. To translate the onboard capabilities of the e-puck robot into the sensing capabilities expected by the SoNS control algorithm, we implement a virtual sensing control layer that is informed strictly by local communication with nearby robots. 

The virtual sensing layer of the e-puck robots is based on the assumption that each aerial robot can sense the relative positions, relative orientations, and encoded identifiers of the AprilTag fiducial markers in its field of view (see Sec.~\ref{SM:aerial}). The virtual sensing layer implemented on each e-puck robot effectively gives it access to the relative position, relative orientation, and ID of any aerial robot it is sensed by; as well as the relative position, relative orientation, and ID of any AprilTag fiducial markers in that aerial robot's field of view.

At each time step, each aerial robot $\rho_i$ that senses an e-puck robot $\rho_j$ sends it a message $\mathcal{P}^{\textsc{robot}}_{ij} = (\rho_i, \rho_j, \boldsymbol{d}_{\rho_i \rho_j},$ $\boldsymbol{q}_{\rho_i \rho_j})$, which includes the relative position and orientation of $\rho_j$ w.r.t. $\rho_i$, as well as both of their identifiers (see Sec.~\ref{SM:aerial}).
From this message, e-puck robot $\rho_j$ has the required information to calculate the relative position $\boldsymbol{d}_{\rho_j \rho_i}$ and relative orientation $\boldsymbol{q}_{\rho_j \rho_i}$ of aerial robot $\rho_i$ w.r.t. its own body frame $B_2^{\rho_j}$, as follows:
\begin{equation}
\begin{aligned}
    \boldsymbol{d}_{\rho_j \rho_i} &= \textsc{RT}(\boldsymbol{q}_{\rho_i \rho_j}^{-1}, -\boldsymbol{d}_{\rho_i \rho_j})\\
    \boldsymbol{q}_{\rho_j \rho_i} &= \boldsymbol{q}_{\rho_i \rho_j}^{-1}
\end{aligned}~.
\end{equation}
In this way, e-puck robot $\rho_j$ virtually senses aerial robot $\rho_k$, using strictly local communication.

At each time step, each aerial robot $\rho_i$ that senses an e-puck robot $\rho_j$ also sends $\rho_j$ a $\mathcal{P}^{\textsc{robot}}_{ik}$ or $\mathcal{P}^{\textsc{object}}_{ik}$ message for each e-puck robot or object $\rho_k$ in the field of view of $\rho_i$.
From these messages, e-puck robot $\rho_j$ has the required information to calculate the relative position $\boldsymbol{d}_{\rho_j \rho_k}$ and relative orientation $\boldsymbol{q}_{\rho_j \rho_k}$ of the robot/object $\rho_k$ w.r.t. its own body frame $B_2^{\rho_j}$, as follows:
\begin{equation}
\begin{aligned}
    \boldsymbol{d}_{\rho_j \rho_k} &= \boldsymbol{d}_{\rho_j \rho_i} + \textsc{RT}(\boldsymbol{q}_{\rho_j \rho_i}, \boldsymbol{d}_{\rho_i \rho_k})\\
    \boldsymbol{q}_{\rho_j \rho_k} &= \textsc{H}( \boldsymbol{q}_{\rho_j \rho_i}, \boldsymbol{q}_{\rho_i \rho_k})
\end{aligned}~.
\end{equation}
In this way, e-puck robot $\rho_j$ virtually senses the e-puck robot or object $\rho_k$, using strictly local communication.

\rhead{Virtual sensing}

\clearpage\rhead{}
\section*{Section \ref*{SM:real-arena}. Real indoor arena setup}
\lhead{Section \ref*{SM:real-arena}. Real arena}

We conduct our experiments with real robots in an indoor arena custom-built for this study inside a multi-use room (see Fig.~\ref{fig:arena}). The arena consists of an unobstructed flight zone and floor area surrounded by a safety buffer and removable security netting, with mounting positions above the center of the flight zone and around the perimeter. The arena is mounted with visual cameras and an optical motion capture system used for recording experiment data. Note that information from the motion capture system is not accessible by the robots and is only used for data logging.

\begin{figure}[thbp]
\includegraphics[width=1.0\textwidth]{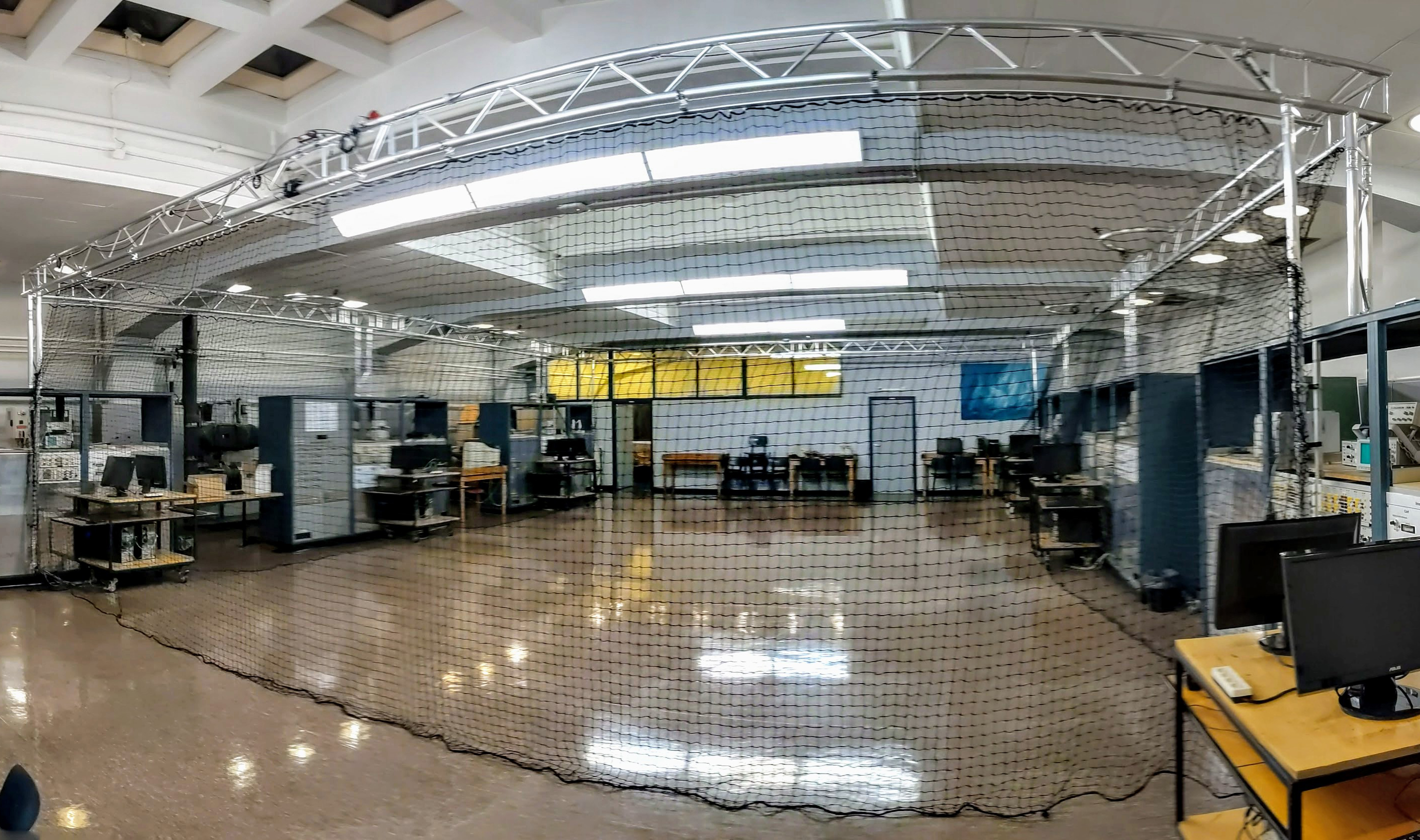}
\caption{Wide-angle photograph of the indoor arena.}
\label{fig:arena}
\end{figure}

The arena perimeter consists of a truss system built on top of permanent metal fixtures (see Fig.~\ref{fig:arena-plan}) from modular 2-point trussing segments (Global Truss F32 Ladder Truss family). This results in a rigid truss perimeter around the arena that allows the visual and motion capture cameras to consistently record from the same fixed physical locations and orientations over experimentation periods of several months.

\begin{figure}[thbp]
\centering
\includegraphics[width=1.0\textwidth]{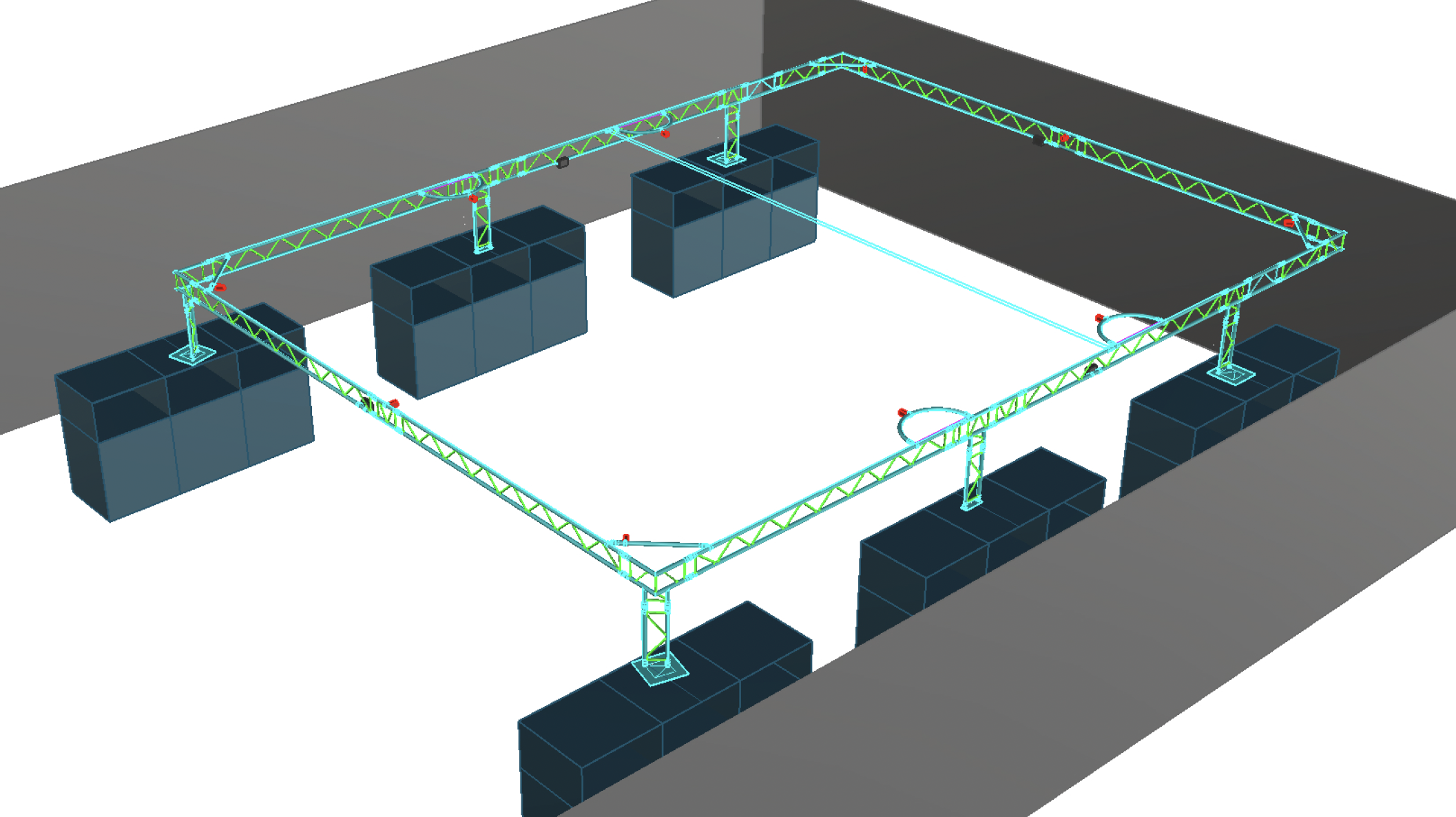}
\caption{Perspective view of the truss system of the arena. There are 10 motion capture cameras (red) mounted on the truss around the perimeter of the arena.}
\label{fig:arena-plan}
\end{figure}

\begin{figure}[thbp]
\centering
\includegraphics[height=0.35\textwidth]{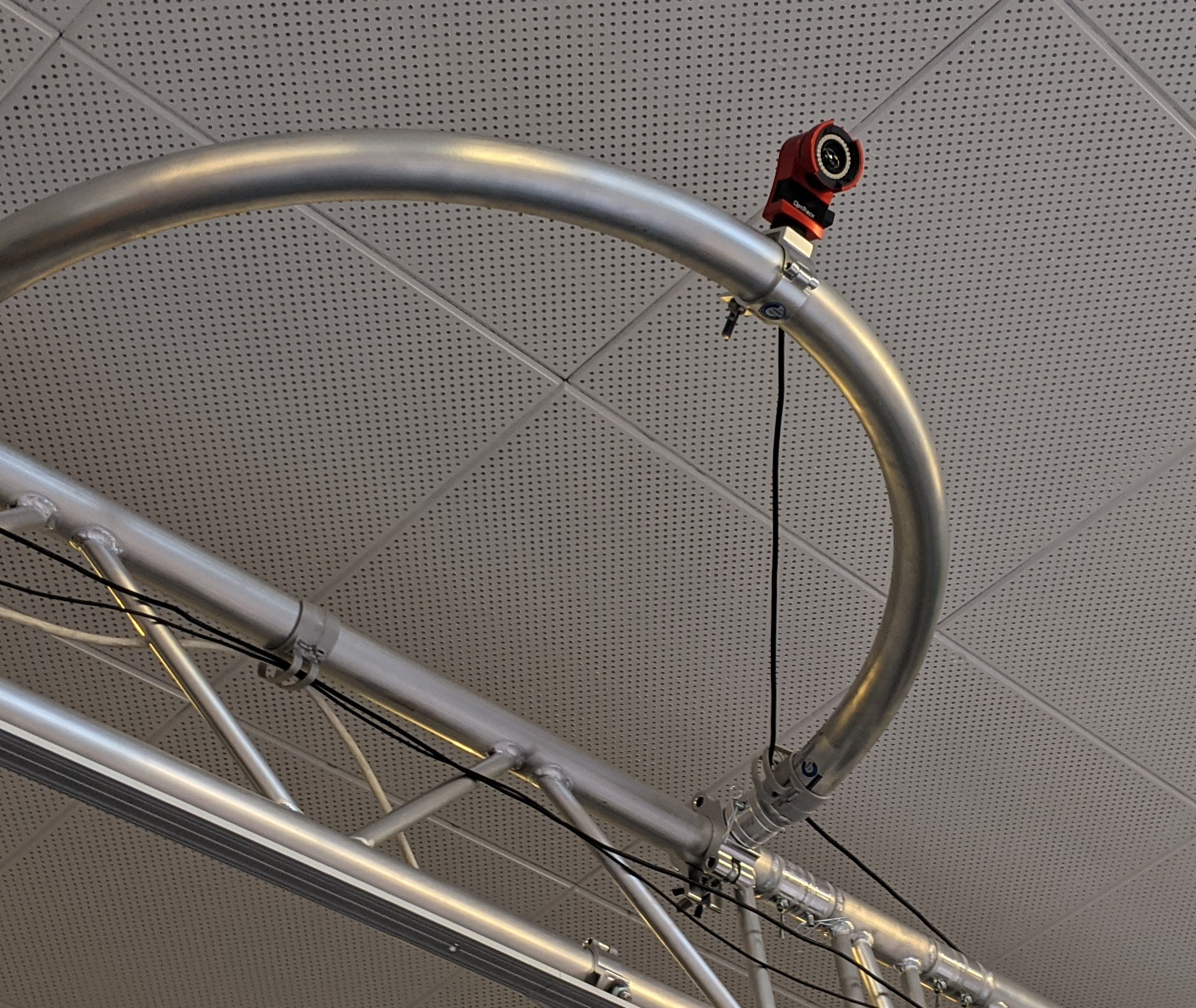}
\includegraphics[height=0.35\textwidth]{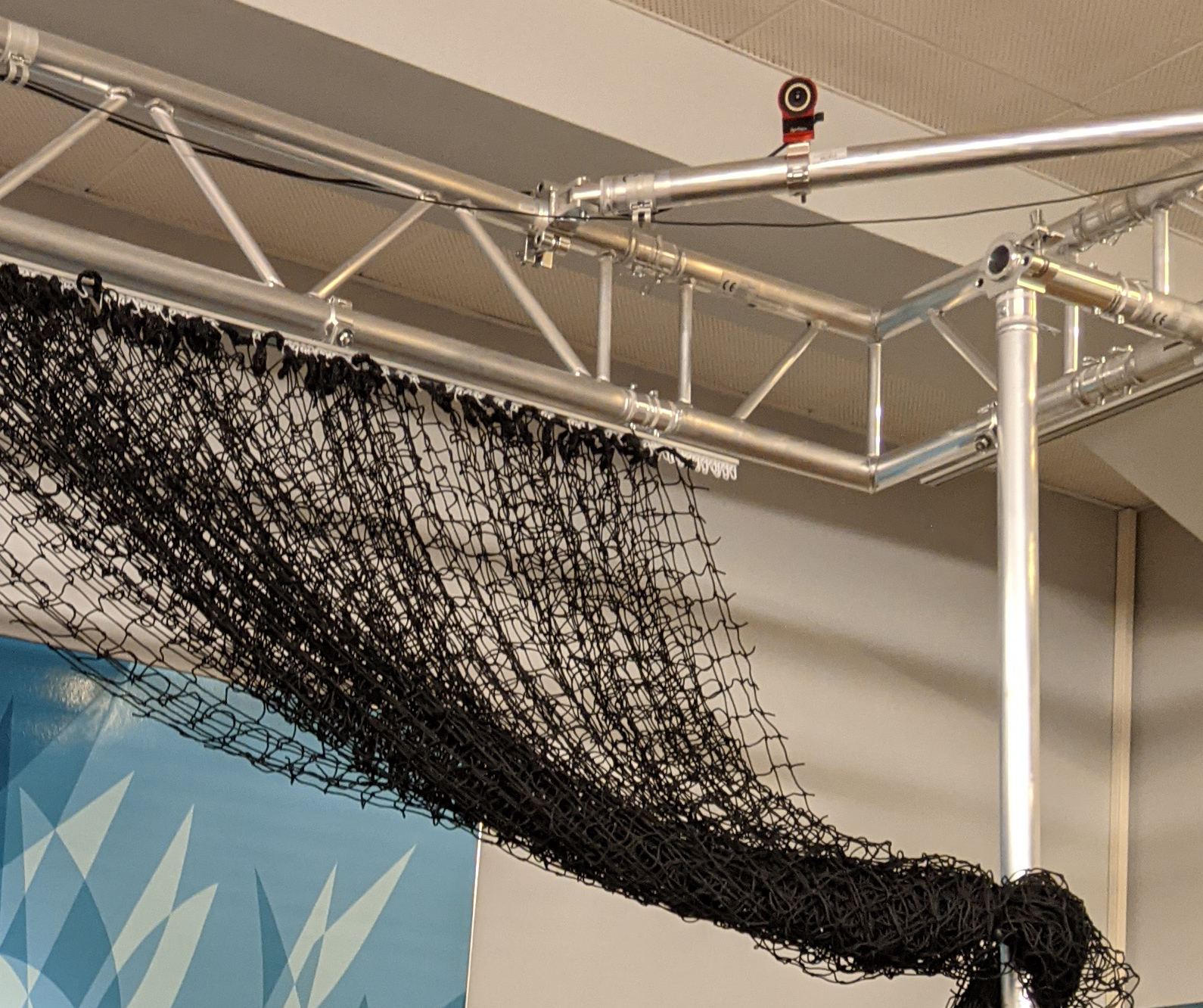}
\caption{Motion capture cameras (red) at a side location (left) and corner location (right) in the arena.}
\label{fig:cameras}
\end{figure}

We use an infrared motion capture system with 10 cameras that each have a 56$^\circ$ field of view (OptiTrack Flex 13 family, 56$^\circ$ FOV, 5.5\,mm, 800\,nm long-pass IR, see Fig.~\ref{fig:cameras}). 

The flight zone is 6\,m x 10\,m with a height of 3\,m above the floor. The motion capture cameras are mounted in an 8\,m x 12\,m rectangle centered around the flight zone, at a height of 3.4\,m above the floor. The cameras are mounted 4\,m apart from each other (one in each corner, one in the center point of each of the shorter perimeter sides, and two at the one-third midway points of each of the longer perimeter sides, see cameras marked in red in Fig~\ref{fig:arena-plan}). Note that the truss system is larger than the rectangle of the camera mounting positions. With the cameras mounted in these positions, motion capture of the ground robots at the floor level and of the aerial robots at the altitude of 1.5\,m above the floor is reliable everywhere in the 6\,m x 10\,m flight zone, except in the corners. Therefore, the corners of the arena are cut at a 45$^\circ$ angle, 0.5\,m away from the original corner point, to form the final octagonal arena shape used in the simulated and real experiments (see Fig.~\ref{fig:arenas}).

The 10 cameras are connected via cable to four USB hubs (OptiHub 2) that manage camera syncing. The hubs are each connected via cable to one PC station with motion capture software (OptiTrack Motive) that outputs the position and orientation data of the tracked rigid bodies.

In the experiments, we track each aerial or ground robot as a rigid body, using four passive markers (OptiTrack precision spheres with 3M 7610 reflective tape) affixed to the robot (see real robots in Fig.~\ref{fig:robots}). In order for the robot positions and orientations to be tracked correctly, the four markers of one robot need to be mounted in a 2D planar configuration that is asymmetrical along both axes formed between a marker and its opposite (i.e., between two opposite vertices of the quadrilateral formed by the four markers). The marker configurations of the robots also need to be geometrically unique, including when they are rotated in 2D or 3D (i.e., each marker configuration must be unique in any rotated position). The unique marker configurations are provided by custom mounting plates made for this study (see example mounting plate designs in Fig.~\ref{fig:markers}) and are associated to the correct robot IDs using the motion capture software (OptiTrack Motive) at the beginning of each experiment session.

\begin{figure}[thbp]
\includegraphics[width=1.0\textwidth]{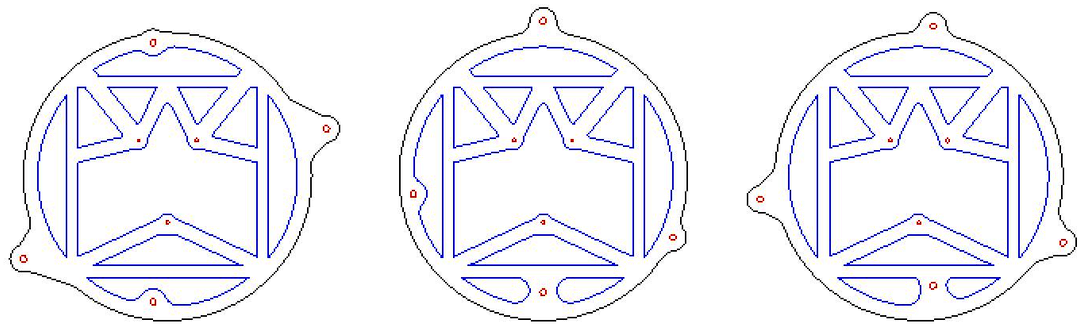}
\caption{2D mounting plate designs for unique and asymmetrical passive marker configurations for each robot, for use by the motion capture system. Three example designs for ground robots are shown. Large red circles near the outer perimeter of each plate indicate the mounting positions of the four passive markers.}
\label{fig:markers}
\end{figure}

Besides recording the motion capture data, we also record experiment videos. We use two visual cameras: an HD action camera (GoPro) mounted to the ceiling, above the center of the arena, and a DSLR camera (Canon EOS 5D Mark IV) mounted to the truss system along the perimeter.

To assist with experiment management, we use an ``experiment supervisor" software custom-developed for this study~\cite{All2022:techreport-004}\footnote{\url{https://github.com/iridia-ulb/supervisor}}. The experiment supervisor software is used: 1) before the experiment, to confirm all robots are active, correctly configured, seen by the motion capture system, and in the correct state; 2) to send a signal to all robots to start the experiment at the same time; 3) to record experiment data during the experiment; and 4) to send a signal to all robots to end the experiment at the same time. During each experiment, the experiment supervisor software records the positions and orientations output from the motion capture system as well as all SoNS information output from the robots.

During the experiments (i.e., between the signals sent to start the experiment and end the experiment), the robots do not receive any centralized commands or global information. The robots are allowed to communicate only with each other, and only under certain conditions. Communication in the SoNS occurs over a wireless network, and two robots are only allowed to communicate with each other if they are connected in the SoNS or if one of them is in the other's field of view.
In our indoor arena, the robots communicate using a wireless LAN. Messages between robots are routed by the experiment supervisor software, which has access to the MAC address associated to each robot ID.
Note that the SoNS layer is constructed in such a way that messages between robots could instead be passed using a wireless ad-hoc network, with no changes to the SoNS software.

During an experiment, the experiment supervisor software records the messages passed between robots, for the purpose of data logging. Note that, during an experiment, the robots do not receive any commands or information originating from the experiment supervisor software; robots only communicate with each other.

\clearpage\rhead{}
\section*{Supplementary Movies S1--S12}
\lhead{Supplementary Movies S1--S12}
\vspace{3mm}

\begin{center}
\begin{minipage}{.43\textwidth}
\embedvideo*{\includegraphics[page=1, width=\textwidth]{Movies/poster-play.png}}{Movies/Movie_S1.mp4}
\customlabel{MovieS1}{S1}
\begin{center}
    \vspace{-6mm}
    {\small {\bf Movie \ref*{MovieS1} :} Establishing self-organized hierarchy with real robots.}\\ \vspace{2mm}
\end{center}
\end{minipage}%
\hspace{10mm}
\begin{minipage}{.43\textwidth}
\embedvideo*{\includegraphics[page=1, width=\textwidth]{Movies/poster-play.png}}{Movies/Movie_S2.mp4}
\customlabel{MovieS2}{S2}
\begin{center}
    \vspace{-6mm}
    {\small {\bf Movie \ref*{MovieS2} :} Balancing global and local goals with real robots.}\\ \vspace{2mm}
\end{center}
\end{minipage}%
\end{center}

\begin{center}
\begin{minipage}{.43\textwidth}
\embedvideo*{\includegraphics[page=1, width=\textwidth]{Movies/poster-play.png}}{Movies/Movie_S3.mp4}
\customlabel{MovieS3}{S3}
\begin{center}
    \vspace{-6mm}
    {\small {\bf Movie \ref*{MovieS3} :} Collective sensing and actuation with real robots.}\\ \vspace{2mm}
\end{center}
\end{minipage}%
\hspace{10mm}
\begin{minipage}{.43\textwidth}
\embedvideo*{\includegraphics[page=1, width=\textwidth]{Movies/poster-play.png}}{Movies/Movie_S4.mp4}
\customlabel{MovieS4}{S4}
\begin{center}
    \vspace{-6mm}
    {\small {\bf Movie \ref*{MovieS4} :} Binary decision making with real robots.}\\ \vspace{2mm}
\end{center}
\end{minipage}%
\end{center}

\begin{center}
\begin{minipage}{.43\textwidth}
\embedvideo*{\includegraphics[page=1, width=\textwidth]{Movies/poster-play.png}}{Movies/Movie_S5.mp4}
\customlabel{MovieS5}{S5}
\begin{center}
    \vspace{-6mm}
    {\small {\bf Movie \ref*{MovieS5} :} Splitting and merging systems with real robots.}\\ \vspace{2mm}
\end{center}
\end{minipage}%
\hspace{10mm}
\begin{minipage}{.43\textwidth}
\vspace{4mm}
\embedvideo*{\includegraphics[page=1, width=\textwidth]{Movies/poster-play.png}}{Movies/Movie_S6.mp4}
\customlabel{MovieS6}{S6}
\begin{center}
    \vspace{-6mm}
    {\small {\bf Movie \ref*{MovieS6} :} Scalability in the binary decision-making mission, 125 robots in simulation.}\\ \vspace{2mm}
\end{center}
\end{minipage}%
\end{center}

\begin{center}
    {\small \it (Note: To play embedded videos, open the PDF in Adobe Acrobat.)}
\end{center}
\clearpage

\begin{center}
\begin{minipage}{.43\textwidth}
\embedvideo*{\includegraphics[page=1, width=\textwidth]{Movies/poster-play.png}}{Movies/Movie_S7.mp4}
\customlabel{MovieS7}{S7}
\begin{center}
    \vspace{-6mm}
    {\small {\bf Movie \ref*{MovieS7} :} Scalability in the establishing self-organized hierarchy mission, several example system sizes in simulation.}\\ \vspace{2mm}
\end{center}
\end{minipage}%
\hspace{10mm}
\begin{minipage}{.43\textwidth}
\embedvideo*{\includegraphics[page=1, width=\textwidth]{Movies/poster-play.png}}{Movies/Movie_S8.mp4}
\customlabel{MovieS8}{S8}
\begin{center}
    \vspace{-6mm}
    {\small {\bf Movie \ref*{MovieS8} :} Fault tolerance demonstration showing interchangeability of a failed brain robot.}\\ \vspace{2mm}
\end{center}
\end{minipage}%
\end{center}

\begin{center}
\begin{minipage}{.43\textwidth}
\vspace{-5mm}
\embedvideo*{\includegraphics[page=1, width=\textwidth]{Movies/poster-play.png}}{Movies/Movie_S9.mp4}
\customlabel{MovieS9}{S9}
\begin{center}
    \vspace{-6mm}
    {\small {\bf Movie \ref*{MovieS9} :} Fault tolerance under multiple permanent failures, with real robots.}\\ \vspace{2mm}
\end{center}
\end{minipage}%
\hspace{10mm}
\begin{minipage}{.43\textwidth}
\embedvideo*{\includegraphics[page=1, width=\textwidth]{Movies/poster-play.png}}{Movies/Movie_S10.mp4}
\customlabel{MovieS10}{S10}
\begin{center}
    \vspace{-6mm}
    {\small {\bf Movie \ref*{MovieS10} :} Fault tolerance under high-loss conditions in simulation, $66.\overline{6}$\% probability to fail.}\\ \vspace{2mm}
\end{center}
\end{minipage}%
\end{center}

\begin{center}
\begin{minipage}{.43\textwidth}
\vspace{-4mm}
\embedvideo*{\includegraphics[page=1, width=\textwidth]{Movies/poster-play.png}}{Movies/Movie_S11.mp4}
\customlabel{MovieS11}{S11}
\begin{center}
    \vspace{-6mm}
    {\small {\bf Movie \ref*{MovieS11} :} Fault tolerance under 30\,s system-wide vision failure in simulation.}\\ \vspace{2mm}
\end{center}
\end{minipage}%
\hspace{10mm}
\begin{minipage}{.43\textwidth}
\embedvideo*{\includegraphics[page=1, width=\textwidth]{Movies/poster-play.png}}{Movies/Movie_S12.mp4}
\customlabel{MovieS12}{S12}
\begin{center}
    \vspace{-6mm}
    {\small {\bf Movie \ref*{MovieS12} :} Fault tolerance under 30\,s system-wide communication failure in simulation.}\\ \vspace{2mm}
\end{center}
\end{minipage}%
\end{center}

\begin{center}
    {\small \it (Note: To play embedded videos, open the PDF in Adobe Acrobat.)}
\end{center}

\clearpage\rhead{}\lhead{}

\bibliography{main,suppl}
\bibliographystyle{IEEEtran}

\end{document}